%% file: neurips_2021.tex
\documentclass{article}
\PassOptionsToPackage{square,numbers}{natbib}

\usepackage[final]{neurips_2021}
\usepackage[utf8]{inputenc} 
\usepackage[T1]{fontenc}    
\usepackage{hyperref}       
\usepackage{url}            
\usepackage{booktabs}       
\usepackage{amsfonts}       
\usepackage{nicefrac}       
\usepackage{microtype}      
\usepackage{xcolor}         
\usepackage{chngcntr}
\usepackage{amsthm}
\usepackage{wrapfig}
\usepackage{graphicx}
\usepackage{subcaption}

\include{include/00-settings}
\include{include/00-commands}

\title{\ours{}: Bayesian Predictive Uncertainty for Node Classification}

\author{
  Maximilian Stadler\thanks{equal contribution}, Bertrand Charpentier\footnotemark[1], Simon Geisler, Daniel Zügner,\\\textbf{Stephan Günnemann}\\
  Department of Informatics\\
  Technical University of Munich, Germany\\
  \texttt{\{stadlmax, charpent, geisler, zuegnerd, guennemann\}@in.tum.de}\\
}

\begin{document}

\maketitle

\begin{acronym}
    \acro{PostNet}{Posterior Network}
    \acro{NatPN}{Natural Posterior Network}
    \acro{BGCN}{Bayesian Graph Convolutional Network}
    \acro{RGCN}{Robust Graph Convolutional Network}
    \acro{LP}{Label Propagation}
    \acro{GNN}{Graph Neural Network}
\end{acronym}

\begin{abstract}
The interdependence between nodes in graphs is key to improve class predictions on nodes and utilized in approaches like \ac{LP} or in \acp{GNN}. Nonetheless, uncertainty estimation for non-independent node-level predictions is under-explored. In this work, we explore uncertainty quantification for node classification in three ways: \textbf{(1)} We derive three axioms explicitly characterizing the expected predictive uncertainty behavior in homophilic attributed graphs. \textbf{(2)} We propose a new model \ours{} (\oursacro{}) which explicitly performs Bayesian posterior updates for predictions on \emph{interdependent} nodes. \oursacro{} provably obeys the proposed axioms. \textbf{(3)} We extensively evaluate \oursacro{} and a strong set of baselines on semi-supervised node classification including detection of anomalous features, and detection of left-out classes. \oursacro{} outperforms existing approaches for uncertainty estimation in the experiments.
\end{abstract}

\frenchspacing

\input{include/01-introduction}
\input{include/02-related_work}
\input{include/03-approach}
\input{include/05-experiment}
\input{include/06-conclusion}

\begin{ack}
This research was supported by the BMW AG, by the Helmholtz Association under the joint research school “Munich School for Data Science - MUDS“, and by a grant from Software Campus through the German Federal Ministry of Education and Research. 
\end{ack}

\newpage

\bibliography{neurips_2021}


\newpage
\appendix
\input{include/xx-appendix}

\end{document}

%% file: include/00-settings.tex
\usepackage{lipsum}

\usepackage{tikz}
\usetikzlibrary{calc}
\usetikzlibrary{patterns,positioning,decorations.pathmorphing,arrows}
\usepackage{caption}
\usepackage{subcaption}
\usepackage{wrapfig}

\newtheorem{axiom}{Axiom}[section]

\usepackage{mathtools, amsmath, bm}

\usepackage[nolist,nohyperlinks]{acronym}
\usepackage{multirow}
\usepackage{bbm}

%% file: include/00-commands.tex
\newcommand{\vect}[1]{\bm{#1}}
\newcommand{\mat}[1]{\bm{#1}}
\newcommand{\set}[1]{\mathcal{#1}}

\def\eps{{\epsilon}}
\def\veczero{\vect{0}}
\def\vecone{\vect{1}}

\def\vp{\vect{p}}

\def\vx{\vect{x}}
\def\vy{\vect{y}}
\def\vz{\vect{z}}
\def\valpha{\vect{\alpha}}
\def\vbeta{\vect{\beta}}

\def\vphi{\vect{\phi}}

\def\mA{\mat{A}}

\def\mX{\mat{X}}

\def\sG{\set{G}}
\def\sN{\set{N}}

\def\sU{\set{U}}
\def\sV{\set{V}}

\def\sL{\set{L}}
\def\sT{\set{T}}

\DeclareMathOperator{\DBer}{Ber}
\DeclareMathOperator{\DCat}{Cat}
\DeclareMathOperator{\DDir}{Dir}

\DeclareMathOperator{\DNormal}{\mathcal{N}}

\DeclareMathOperator{\real}{\mathbb{R}}

\DeclareMathOperator{\prob}{\mathbb{P}}
\DeclareMathOperator{\prior}{\mathbb{Q}}
\DeclareMathOperator{\entropy}{\mathbb{H}}
\DeclareMathOperator{\expectation}{\mathbb{E}}

\DeclareMathOperator{\variance}{Var}

\DeclareMathOperator{\loss}{\mathcal{L}}

\newcommand{\condition}{\ensuremath{\,|\,}}

\newcommand{\ExpecationArgs}[2]{\expectation_{#1}\left[#2\right]}
\newcommand{\Entropy}[1]{\entropy\left[#1\right]}

\newcommand\ours{Graph Posterior Network}
\newcommand\oursacro{GPN}

\newcommand\x{\vx}
\newcommand\z{\vz}
\newcommand\y{y}
\newcommand\p{\vp}

\newcommand{\inputdim}{D}

\newcommand\nnodes{N}
\newcommand\nodeu{u}
\newcommand\nodev{v}
\newcommand\nodew{w}

\newcommand\nodeidxv{^{(v)}}

\newcommand\namednodeidxv[1]{^{#1, (v)}}
\newcommand\iclass{c}
\newcommand\nclass{C}

\newcommand\graph{\sG}
\newcommand\vertices{\sV}

\newcommand\neighbors{\sN}
\newcommand\nodeslabeled{\sT}
\newcommand\nodesunlabeled{\sU}
\newcommand\adj{\mA}
\newcommand\features{\mX}

\newcommand\pprelem{\Pi^{ppr}}

\newtheorem{theorem}{Theorem}
\counterwithin*{theorem}{section}
\newtheorem{lemma}{Lemma}

%% file: include/01-introduction.tex
\section{Introduction}
\label{sec:introduction}

Accurate and rigorous uncertainty estimation is key for reliable machine learning models in safety-critical domains \cite{interpretable-ml}. It quantifies the confidence of machine learning models, thus allowing them to validate knowledgeable predictions or flag predictions on unknown input domains. Uncertainty is commonly divided in \emph{aleatoric} and \emph{epistemic} uncertainty \cite{Gal2016a}. The aleatoric uncertainty accounts for irreducible uncertainty (e.g., due to inherent sensor noise). The \emph{epistemic} uncertainty accounts for a lack of information for accurate prediction (e.g., test data significantly different from training data).

Traditionally, machine learning models assume i.i.d.\ inputs, thus performing predictions based on input features only. For uncertainty estimation on i.i.d.\ inputs, a large class of definitions, models and evaluation methods have been introduced \citep{Gal2016a, Malinin2017, Abdar2020, Ovadia2019, robustness-uncertainty-dirichlet}. Further, uncertainty estimation has been successfully applied to different tasks e.g. out-of-distribution (OOD) or shift detection \citep{Ovadia2019}, active learning \cite{uncertainty-meta-learning, bayesian-meta-learning}, continual learning \citep{uncertainty-continual-learning} or reinforcement learning \citep{uncertainty-rl}. 

In contrast, uncertainty estimation on interdependent nodes is more complex than on i.i.d.\ inputs and under-explored \citep{Abdar2020}. A node in an attributed graph is characterized by two types of information: its features and its neighborhood. While the feature information indicates the node position in the feature space -- similarly to i.i.d. inputs --, the neighborhood information indicates the additional node position in the network space. To leverage the neighborhood information, recent graph neural networks (GNNs) successfully proposed to enrich and correct the possibly noisy information of the features of a single node by aggregating them with the features of its neighborhood \cite{Kipf2016, Velickovic2017, Klicpera2018}. It naturally leads to the distinction between predictions \emph{without network effects} based exclusively on their own node feature representation, and predictions \emph{with network effects} based on neighborhood aggregation. The aggregation step commonly assumes \emph{network homophily} which states that nodes with similar properties tend to connect to each other more densely, thus violating the i.i.d. assumption between node features given their neighborhood. 

\looseness=-1
The core motivation of our work is to transfer some of the existing uncertainty estimation definitions, models and evaluations from i.i.d. inputs to interdependent node inputs by leveraging both the feature and the neighborhood information. In particular, we aim at an accurate quantification of the aleatoric and epistemic uncertainty without and with network effect under network homophily (see Fig.~\ref{fig:uncertainty_types_small}).

\looseness=-1
\textbf{Our contribution.} In this work, we consider uncertainty estimation on semi-supervised node classification. First, we derive three axioms which materialize reasonable uncertainty for non-independent inputs. These axioms cover the traditional notions of aleatoric and epistemic uncertainty and distinguish between the uncertainty with and without network effects. Second, we propose \ours{} (\oursacro{})\footnote{Project page including code at \url{https://www.daml.in.tum.de/graph-postnet}} for uncertainty estimation for node classification and prove formally that it follows the axiom requirements contrary to popular GNNs. Third, we build an extensive evaluation setup for uncertainty estimation which relies on the assessment of uncertainty estimation quality of OOD detection and robustness against shifts of the attributed graph properties. Both OOD data and attributed graph shifts distinguish between attribute and structure anomalies. The theoretical properties of \oursacro{} manifest in these experiments where it outperforms all other baselines on uncertainty evaluation.

\begin{figure}
\centering
	\begin{subfigure}[t]{0.245\textwidth}
	    \centering
		\includegraphics[width=\textwidth]{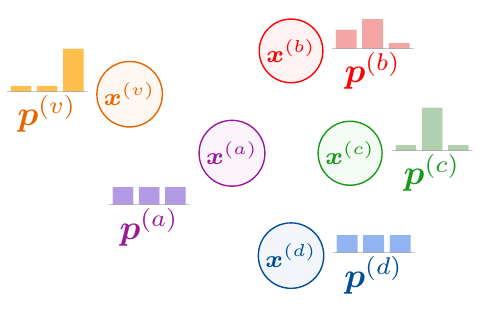}
		\caption{AU w/o network}
		\label{subfig:au_without_network}
	\end{subfigure}
	\begin{subfigure}[t]{0.245\textwidth}
	    \centering
		\includegraphics[width=\textwidth]{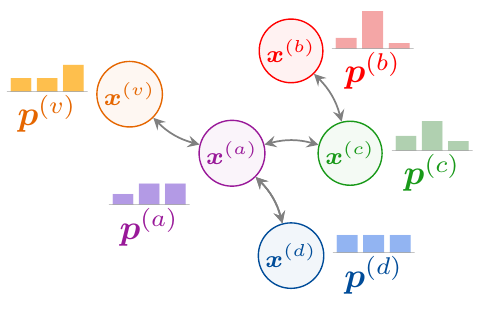}
		\caption{AU w/ network}
		\label{subfig:au_with_network}
	\end{subfigure}
	\begin{subfigure}[t]{0.245\textwidth}
	    \centering
		\includegraphics[width=\textwidth]{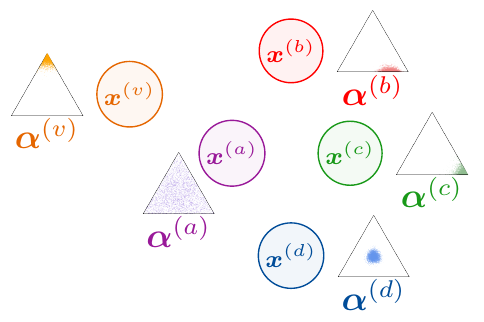}
		\caption{EU w/o network}
		\label{subfig:eu_without_network}
	\end{subfigure}
	\begin{subfigure}[t]{0.245\textwidth}
	    \centering
		\includegraphics[width=\textwidth]{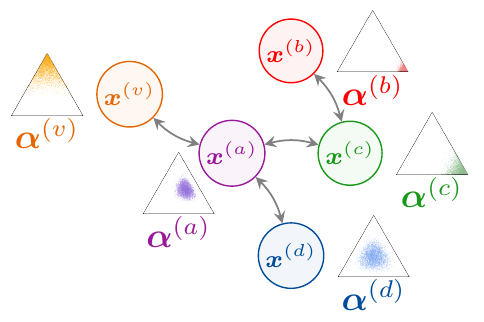}
		\caption{EU w/ network}
		\label{subfig:eu_with_network}
	\end{subfigure}
	\caption{Illustration of aleatoric uncertainty (AU) and epistemic uncertainty (EU) without and with network effects (i.e. i.i.d.\ inputs vs interdependent inputs). Nodes have the same features in all cases. Network effects are visualized through edges between nodes which change the predicted distributions. The aleatoric uncertainty is high if the categorical distribution $\hat{\y} \nodeidxv \sim \DCat(\p\nodeidxv)$ is flat. The epistemic uncertainty is high if the Dirichlet distribution $\p\nodeidxv \sim \DDir(\valpha\nodeidxv)$ is spread out. We refer the reader to Section~\ref{subsec:ours} for formal definitions of those distributions.}
    \label{fig:uncertainty_types_small}
\end{figure}

%% file: include/02-related_work.tex
\section{Related Work} \label{sec:related_work}

In this section, we cover the related work for predictive uncertainty estimation for i.i.d. inputs and for graphs. To this end, we review the commonly accepted \emph{axioms} defining the desired uncertainty estimation under different circumstances, the \emph{methods} capable of consistent uncertainty quantification and the \emph{evaluation} validating the quality of the uncertainty estimates in practice.

\textbf{Uncertainty for i.i.d.\ inputs --} The related work for uncertainty quantification on i.i.d. inputs is rich as for example shown in a recent survey \citep{Abdar2020}. \emph{\underline{Axioms:}} Far from ID data, the predicted uncertainty is expected to be high \citep{provable-uncertainty, NatPN2021, bayesian-a-bit, sufficient-conditions-no-adversarial}. Close to ID data, the desired uncertainty is more complicated. Indeed, while some works expected models to be robust to small dataset shifts \citep{Ovadia2019, confidence-calibrated-adversarial}, other works expected to detect near OOD classes based on uncertainty \citep{contrastive-ood, robustness-uncertainty-dirichlet, attack-detection}. \emph{\underline{Methods:}} Many methods already exist for uncertainty quantification for i.i.d. inputs like images or tabular data. A first family of models quantifies uncertainty by aggregating statistics (e.g. mean, variance or entropy) from sub-networks with different weights. Important examples are ensemble \citep{Lakshminarayanan2017, batch-ensembles, hyper-ensembles, mimo-independent-subnetworks}, dropout \citep{Srivastava2014} or Bayesian Neural Networks (BNN) \citep{Blundell2015, Depeweg2018, simple-baseline-uncertainty, liberty-depth-bnn, rank-1-bnn}. Most of these approaches require multiple forward-passes for uncertainty quantification. Further, dropout and BNN may have other pitfalls regarding their limited applicability to more complex tasks \citep{Osband2016, Hron2018, Graves2011, Foong2019}. A second family quantifies uncertainty by using the logit information. Important examples are temperature scaling which rescale the logits after training \citep{Guo2017, Liang2017} and energy-based models which interpret the logits as energy scores \citep{Liu2020a, Grathwohl2019}. A third family of model quantifies uncertainty based on deep Gaussian Processes (GP). Important examples use GP at activation-level \cite{gp-uncertainty-activation} or at (last) layer-level \citep{uncertainty-distance-awareness, bayesian-a-bit, duq, Bilovs2019}. Finally, a last family of models quantifies uncertainty by directly parameterizing a conjugate prior distribution over the target variable. Important examples explicitly parameterize prior distributions \citep{Sensoy2018, Malinin2019a, Malinin2018, Malinin2019b, evidential-regression} or posterior distributions \citep{Charpentier2020, NatPN2021}. Methods based on GP and conjugate prior usually have the advantage of deterministic and fast inference. \emph{\underline{Evaluation:}} Previous works have already proposed empirical evaluation of uncertainty estimation by looking at accuracy, calibration or OOD detection metrics under dataset shifts or adversarial perturbations for i.i.d.\ inputs \citep{Ovadia2019, robustness-uncertainty-dirichlet}. In contrast with all these approaches, this work studies uncertainty quantification for classification of \emph{interdependent nodes}.

\textbf{Uncertainty for graphs -- } Notably, the recent survey \citep{Abdar2020} points out that there is only a limited number of studies on uncertainty quantification on GNN and semi-supervised learning. Moreover, they recommend proposing new methods. \emph{\underline{Axioms:}} To the best of our knowledge, only \citep{Eswaran2017} proposed explicit axioms for node classification for non-attributed graphs. They expect disconnected nodes to recover prior predictions and nodes with higher beliefs to be more convincing. In this work, we clarify the desired uncertainty estimation for node classification on attributed graphs based on \emph{motivated and explicit axioms}. \emph{\underline{Methods:}} The largest family of models for uncertainty for graphs are dropout- or Bayesian-based methods. Important examples propose to drop or assign probabilities to edges \citep{Rong2019, Chen2018, Hasanzadeh2020, Dallachiesa2014, Hu2017}. Further works proposed to combine the uncertainty on the graph structure with uncertainty on the transformation weights similarly to BNN \citep{Elinas2019, Zhang2019b, Pal2019a, Pal2019b}. Importantly, these models do not directly quantify uncertainty on the prediction. Similarly to the i.i.d.\ case, a second family of models focuses on deterministic uncertainty quantification. Important examples mostly use Graph Gaussian Processes, which do not easily scale to large graphs \citep{Ng2018, Zhi2020, Liu2020c, Borovitskiy2020}. Only \citep{Zhao2020} explicitly parameterized a Dirichlet conjugate prior. They combined it with multiple components (Graph-Based Kernel, dropout, Teacher Network, loss regularizations) which cannot easily distinguish between uncertainty without and with network effects. In contrast, \oursacro{} is a simple approach based on conjugate prior parametrization and disentangles uncertainty with and without network effects. \emph{\underline{Evaluation:}} The evaluation of most of those methods was not focused on the quality of the uncertainty estimates but on the target task metrics (e.g. accuracy for classification, distance to ground truth for regression). Other methods only relied on uncertainty quantification to build more robust models \citep{Zhu2019, Feng2020}. For node classification, only few works evaluated uncertainty by using Left-Out classes or detection of missclassified samples \citep{Zhao2020}, active learning \cite{Ng2018} or visualization \citep{Borovitskiy2020}. Note that proposed uncertainty evaluations on molecules at graph level \citep{Zhang2019, Ryu2019, Akita2018, uncertainty-nn-molecules, uncertainty-material-prediction} is an orthogonal problem. In this work, we propose a \emph{sound and extensive evaluation} for uncertainty in node classification. It distinguishes between OOD nodes w.r.t.\ features and structure, and graph dataset shifts w.r.t.\ the percentage of perturbed node features and the percentage of perturbed edges.

%% file: include/03-approach.tex
\section{Uncertainty Quantification for Node Classification}

\looseness=-1
We consider the task of (semi-supervised) node classification on an attributed graph \smash{$\graph = \left(\adj, \features\right)$}  with adjacency matrix \smash{$\adj \in \left\{0, 1\right\}^{\nnodes \times \nnodes}$} and node attribute matrix $\features \in \real^{\nnodes \times \inputdim}$. We aim at inferring the labels \smash{$\y\nodeidxv \in \{1, ..., \nclass\}$} plus the the aleatoric uncertainty \smash{$u_\text{alea}\nodeidxv$} and the epistemic uncertainty \smash{$u_\text{epist}\nodeidxv$} of unlabeled nodes \smash{$\nodev \in \nodeslabeled$} given a set of labelled nodes \smash{$\nodeu \in \nodesunlabeled$} in the graph where \smash{$\vertices  = \nodeslabeled \cup \nodesunlabeled$} denotes the set of vertices.

\input{include/03a-axioms}
\input{include/03b-model}
\input{include/03c-guarantees}
\input{include/03d-limitations}

%% file: include/03a-axioms.tex
\subsection{Axioms}\label{sec:axioms}

Uncertainty estimation in the setting of interdependent inputs is not well-studied. It often leaves the expected behavior and interpretations for uncertainty estimation unclear. Thus, we need well-grounded axioms to derive meaningful models. In this section, we aim at specifying the desired uncertainty predictions under various circumstances in homophilic attributed graphs. To this end, we propose three axioms which are based on the two following distinctions. The first distinction differentiates between aleatoric and epistemic uncertainty which are commonly used concepts under the i.i.d. assumptions \cite{Gal2016a, Malinin2017}. The second distinction differentiates between uncertainty without and with network effects which are motivated by the concepts of attribute and structure anomalies used in the attributed graph setting \cite{Bojchevski2018a}. These new axioms cover all possible combinations encountered by these distinctions and extend the axioms proposed by \citep{Eswaran2017} for non-attributed graphs. We designed the axioms to be informal and generic so that they are application independent, model-agnostic and do not require complex mathematical notations similarly to \citep{Eswaran2017, graph-transduction-confidence}. In practice, formal definitions need to instantiate general concepts like aleatoric/epistemic uncertainty and with/without network effects noting that some definitions might be more convenient depending on the task. 
The first axiom deals with (epistemic and aleatoric) uncertainty estimation without network effects (see Fig.~\ref{subfig:au_without_network}, \ref{subfig:eu_without_network}). :
\begin{axiom}
    \label{ax:certainty_features}
    \looseness=-1
    A node's prediction in the absence of network effects should only depend on its own features. A node with features more different from training features should be assigned higher uncertainty.
\end{axiom}
Axiom \ref{ax:certainty_features} states that if a node \smash{$\nodev$} has no neighbors, then the final prediction \smash{$\p\nodeidxv$} should only depend on its own node features \smash{$\x\nodeidxv$}. Further, for anomalous features the model should fall back to safe prior predictions, indicating high aleatoric and epistemic uncertainty. This aligns with \cite{Eswaran2017} which expects to recover prior predictions for non-attributed nodes without network effect, and \cite{provable-uncertainty, NatPN2021} which expect to recover prior predictions far from training data for i.i.d.\ inputs. The second axiom deals with epistemic uncertainty estimation with network effects (see Fig.~\ref{subfig:eu_without_network}, \ref{subfig:eu_with_network}):
\begin{axiom}
    \label{ax:certainty_network_epistemic}
    \looseness=-1
    All else being equal, if a node's prediction in the absence of network effects is more epistemically certain, then its neighbors' predictions in the presence of network effects should become more epistemically certain.
\end{axiom}
\looseness=-1
Axiom \ref{ax:certainty_network_epistemic} states that a node \smash{$\nodev$} with confident feature predictions \smash{$\x\nodeidxv$} is expected to be convincing and make its neighbors \smash{$\nodeu \in \neighbors(\nodev)$} more confident. Conversely, a node with anomalous features is expected to make its neighborhood less confident. This axiom materializes the network homophily assumption at the epistemic level i.e. connected nodes have similar epistemic uncertainty estimates. For non-attributed graphs, \cite{Eswaran2017} similarly expects a more confident node to have more influence on a direct neighbor. The third axiom deals with aleatoric uncertainty estimation with network effects (see Fig.~\ref{subfig:au_without_network}, \ref{subfig:au_with_network}):
\begin{axiom}
    \label{ax:certainty_network_aleatoric}
    All else being equal, a node's prediction in the presence of network effects should have higher aleatoric uncertainty if its neighbors' predictions in the absence of network effects have high aleatoric uncertainty. Further, a node prediction in the presence network effects should have higher aleatoric uncertainty if its neighbors' predictions in the absence network effects are more conflicting.
\end{axiom}
Axiom \ref{ax:certainty_network_aleatoric} states that no clear classification decision should be made for a node $\nodev$ if no clear classification decisions can be made for its neighbors. Further, the classification decision becomes less certain if a neighbor has a conflicting classification decision. Note that this axiom is more subtle than the direct application of network homophily at the aleatoric level. Indeed a node can have a high aleatoric uncertainty contrary to its neighbors which predict different classes with low aleatoric uncertainty. This aligns with the intuition that conflicting information from the neighborhood provides an irreducible uncertainty to the considered node.

%% file: include/03b-model.tex
\subsection{\ours{}} \label{subsec:ours}

The Bayesian update rule is a key component of \oursacro{} to model uncertainty on the predicted categorical distribution. For a single categorical distribution \smash{$\y \sim \DCat(\p)$}, the \emph{standard} Bayesian update is straightforward. A natural choice for a prior distribution over the parameters $\p$ is its conjugate prior i.e. the Dirichlet distribution \smash{$\prob(\p) = \DDir(\valpha^\text{prior})$} with \smash{$\alpha^\text{prior}_\iclass \in \real_{+}^{\nclass}$}. Given the observations \smash{$\y^{(1)}, ..., \y^{(N)}$}, the Bayesian update then consists in applying the Bayes' theorem 
\begin{align}
    \prob\left( \p \condition \{ \y
^{(j)}\}_{j=1}^N \right) \propto \prob\left(\{ \y^{(j)}  \}_{j=1}^N \condition \p \right) \times \prob(\p)
\end{align} producing the posterior distribution \smash{$\prob(\p \condition \{ y^{(j)} \}_{j=1}^N) = \DDir(\valpha^\text{post})$} where $\valpha^\text{post} = \valpha^\text{prior} + \vbeta$ are the parameters of the posterior and \smash{$\beta_\iclass = \sum_j = \mathbbm{1}_{\y^{(j)} = c}$} are the class counts. This framework naturally disentangles the aleatoric and epistemic uncertainty by defining the Dirichlet mean \smash{$\bar{\vp} = \frac{\valpha}{\alpha_0}$} and the total evidence count \smash{$\alpha_0 = \sum_\iclass \alpha_c$}. Indeed, the aleatoric uncertainty is commonly measured by the entropy of the categorical distribution i.e. \smash{$u_\text{alea} = \Entropy{\DCat(\bar{\p})}$} \cite{Malinin2017, Charpentier2020, NatPN2021} and the epistemic uncertainty can be measured by the total evidence count $\alpha_0$ of observations i.e. \smash{ $u_\text{epist} = - \alpha_0$} \cite{Charpentier2020, NatPN2021}. Alternatively, the epistemic uncertainty can also be measured with the Dirichlet differential entropy \cite{Malinin2017}. Note that the reparameterization using \smash{$\bar{\p}$} and \smash{$\alpha_0$} can apply to any class counts including the prior counts \smash{$\valpha^\text{prior}$}, the class counts \smash{$\vbeta$} and the posterior counts \smash{$\valpha^{\text{post}}$}.

\begin{figure}
    \centering
	\includegraphics[width=.75\textwidth]{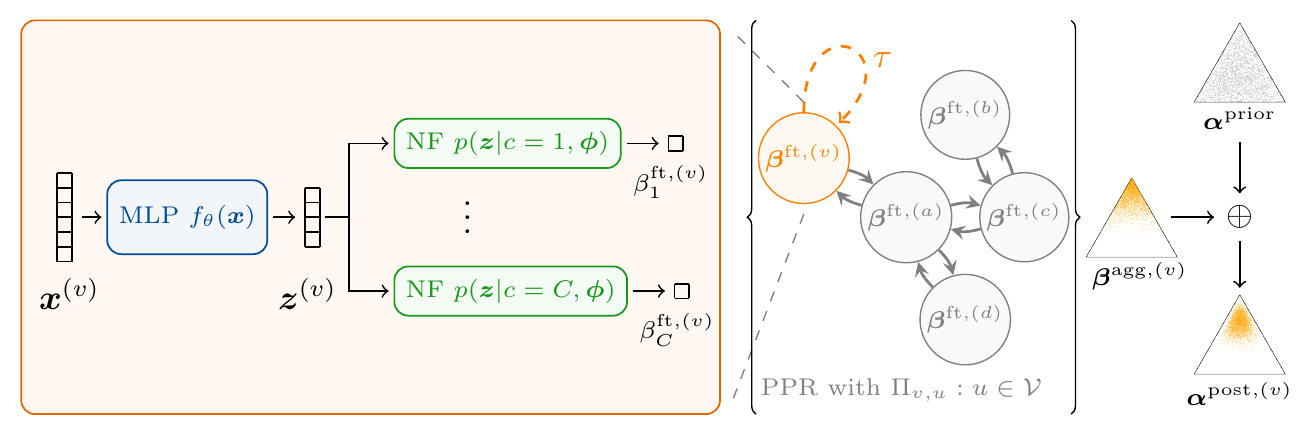}
	\caption{Overview of \ours{}: (1) node-level pseudo-counts computed by the feature encoder in the orange box, (2) PPR-based message passing visualized between the curly braces, and (3) input-dependent Bayesian update illustrated with the Dirichlet triangles on the right.}
    \label{fig:model_vis}
    \vspace{-3mm}
\end{figure}

\looseness=-1
For classification, the predicted categorical distribution \smash{$\hat{\y} \nodeidxv \sim \DCat(\p\nodeidxv)$} additionally depends on the specific input $\nodev$. Hence, the \emph{input-dependent} Bayesian rule \citep{Charpentier2020, NatPN2021} extends the Bayesian treatment of a single categorical distribution to classification by predicting an individual posterior update for any possible input. Specifically, it first introduces a fixed Dirichlet prior over the categorical distribution \smash{$\p\nodeidxv \sim \DDir(\valpha^\text{prior})$} where \smash{$\valpha^\text{prior} \in \real_+^\nclass$} is usually set to $1$, and second predicts the input-dependent update \smash{$\vbeta\nodeidxv$} which forms the posterior distribution \smash{$\p\nodeidxv \sim \DDir(\valpha^{\text{post}, (\nodev)})$} where the posterior parameters are equal to
\begin{align}\label{eq:input-posterior-update}
    \valpha^{\text{post}, (\nodev)} = \valpha^\text{prior} + \vbeta\nodeidxv.
\end{align}
The variable $\vbeta\nodeidxv$ can be interpreted as learned class pseudo-counts and its parametrization is crucial. For i.i.d. inputs, PostNet \citep{Charpentier2020} models the pseudo-counts $\vbeta\nodeidxv$ in two main steps. \textbf{(1)} it maps the inputs features $\x\nodeidxv$ onto a low-dimensional latent vector \smash{$\z\nodeidxv= f_\theta(\x\nodeidxv) \in \real^H$}. \textbf{(2)}, it fits one conditional probability density \smash{$\prob(\z\nodeidxv|\iclass; \vphi)$} per class on this latent space with normalizing flows. The final pseudo count for class $c$ is set proportional to its respective conditional density i.e. \smash{$\beta_\iclass\nodeidxv = N \prob(\z\nodeidxv|\iclass; \vphi) \prob(\iclass)$} where $N$ is a total certainty budget and $\prob(\iclass)= \frac{1}{\nclass}$ for balanced classes. Note that this implies \smash{$\alpha_0\nodeidxv = N \prob(\z\nodeidxv|\vphi)$}. This architecture has the advantage of decreasing the evidence outside the known distribution when increasing the evidence inside the known distribution, thus leading to consistent uncertainty estimation far from training data.

\textbf{Bayesian Update for Interdependent Inputs.} We propose a simple yet efficient modification for parameterizing $\beta_\iclass\nodeidxv$ to extend the input-dependent Bayesian update for interdependent attributed nodes. The core idea is to first predict the feature class pseudo-counts $\vbeta^{\text{ft}, (\nodev)}$ based on independent node features only, and then diffuse them to form the aggregated class pseudo-counts $\vbeta^{\text{agg}, (\nodev)}$ based on neighborhood features. Hence, the feature class pseudo-counts $\vbeta^{\text{ft}, (\nodev)}$ intuitively act as uncertainty estimates without network effects while the aggregated class pseudo-counts $\vbeta^{\text{agg}, (\nodev)}$ intuitively act as uncertainty estimates with network effects. 

\looseness=-1
To this end, \oursacro{} performs three main steps (see Fig.~\ref{fig:model_vis}). \textbf{(1)} A (feature) encoder maps the features of $\nodev$ onto a low-dimensional latent representation $\z$ i.e. \smash{$\z^{(\nodev)} = f_\theta(\x\nodeidxv) \in \real^H$}. In practice, we use a simple MLP encoder in our experiments similarly to APPNP \citep{Klicpera2018}. \textbf{(2)} One conditional probability density per class \smash{$\prob(\z^{(\nodev)} \condition \iclass; \vphi)$} is used to compute \smash{$\beta_\iclass^{\text{ft}, (\nodev)}$} i.e \smash{$\beta_\iclass^{\text{ft}, (\nodev)} \propto \prob(\z^{(\nodev)} \condition \iclass; \vphi)$}. Note that the the total feature evidence \smash{$\alpha_0^{\text{ft}, (\nodev)}= \sum_\iclass \beta_\iclass^{\text{ft}, (\nodev)}$} and the parameter \smash{$\bar{\vp}^{\text{ft}, (\nodev)} = \nicefrac{\vbeta^{\text{ft}, (\nodev)}}{\alpha_0^{\text{ft}, (\nodev)}}$} are only based on node features and can be seen as epistemic and aleatoric uncertainty measures \emph{without network effects}. In practice, we used radial normalizing flows for density estimation similarly to \citep{Charpentier2020} and scaled the certainty $N$ budget w.r.t. the latent dimension $H$ similarly to \citep{NatPN2021}. \textbf{(3)} A Personalized Page Rank (PPR) message passing scheme is used to diffuse the feature class pseudo-counts \smash{$\beta_\iclass^{\text{ft}, (\nodev)}$} and form the aggregated class pseudo-counts \smash{$\beta_\iclass^{\text{agg}, (\nodev)}$} i.e.
\begin{equation}\label{eq:agg-evidence}
    \beta_\iclass^{\text{agg}, (\nodev)} = \sum_{\nodeu \in \vertices} \pprelem_{v, u} \beta_\iclass^{\text{ft}, (\nodeu)}
\end{equation}
\looseness=-1
where $\pprelem_{v, u}$ are the dense PPR scores implicitly reflecting the importance of node $\nodeu$ on $\nodev$. We approximate the dense PPR scores using power iteration similarly to \citep{Klicpera2018}. The aggregated pseudo-count \smash{$\beta_\iclass^{\text{agg}, (\nodev)}$} is then used in the input-dependent Bayesian update (see Eq.~\ref{eq:input-posterior-update}). Remark that the scores \smash{$\pprelem_{v, u}$} define a valid conditional distribution over all nodes associated to the PPR random walk (i.e. \smash{$\sum_\nodeu \pprelem_{v, u} = 1$}). It can be viewed as a soft neighborhood for $\nodev$ accounting for all neighborhood hops through infinitely many message passing steps \citep{Klicpera2018}. Hence, on one hand, the PPR scores define a probability distribution over nodes using the node edges only. On the other hand, the quantity \smash{$\prob(\z^{(\nodeu)} \condition \iclass; \vphi)$} defines a probability distribution over nodes using the node features only. Therefore, we can equivalently rewrite this step using probabilistic notations \smash{$\prob(\nodev \condition \nodeu) = \pprelem_{v, u}$} and \smash{$\prob(\nodeu \condition \iclass) = \prob(\z^{(\nodeu)} \condition \iclass; \vphi)$}:
\begin{equation}\label{eq:agg-evidence-density}
    \beta_\iclass^{\text{agg}, (\nodev)} \propto \bar{\prob}(\nodev \condition \iclass) = \sum_{\nodeu \in \vertices} \prob(\nodev \condition \nodeu) \prob(\nodeu \condition \iclass)
\end{equation}
\looseness=-1
Interestingly, the quantity $\bar{\prob}(\nodev \condition \iclass)$ defines a valid distribution which normalizes over all node features and accounts for the soft neighborhood (i.e. \smash{$\int...\int\bar{\prob}(\nodev \condition \iclass) d\z^{(u_1)}...d\z^{(u_{|\vertices|})} = 1$}). Hence, the message passing step is a simple but efficient method to transform the feature distributions of a single node into a joint distributions over the soft neighborhood features. Finally, the evidence \smash{$\alpha_0^{\text{agg}, (\nodev)} = \sum_\iclass \beta_\iclass^{\text{agg}, (\nodev)}$} and the parameter \smash{$\vp^{\text{agg}, (\nodev)} = \nicefrac{\vbeta^{\text{agg}, (\nodev)}}{\alpha_0^{\text{agg}, (\nodev)}}$} are based on neighborhood features and can be seen as epistemic and aleatoric uncertainty measures \emph{with network effects}. Remark that, the sequential processing of the features (i.e. steps (1)+(2)) and network information (i.e. step (3)) in \oursacro{} is a key element to differentiate between the uncertainty without and with network effects and is a building block to provably obey the axioms.

\oursacro{} extends both APPNP \cite{Klicpera2018} and PostNet \cite{Charpentier2020} approaches. The key difference to APPNP is the density estimation modeling the epistemic uncertainty (i.e. steps (1)+(2)) and the input-dependent Bayesian update allowing to recover the prior prediction (i.e. Eq.~\ref{eq:input-posterior-update}). The key difference to PostNet is the PPR diffusion which accounts for dependence between nodes (step (3)).

\textbf{Optimization.} We follow \cite{Charpentier2020} and train \oursacro{} by minimizing the following Bayesian loss with two terms i.e.:
\begin{equation}
    \loss\nodeidxv = -\ExpecationArgs{\p\nodeidxv\sim \prior\namednodeidxv{post}}{\log \prob(\y\nodeidxv \condition \p\nodeidxv)} - \lambda \Entropy{\prior\namednodeidxv{post}}
\end{equation}
where $\lambda$ is a regularization factor. It can be computed quickly in closed-form and provides theoretical guarantees for optimal solutions \cite{Charpentier2020}. All parameters of \oursacro{} are trained jointly. Similarly to \cite{NatPN2021}, we also observed that "warm-up" training for the normalizing flows is helpful. 

%% file: include/03c-guarantees.tex
\subsection{Uncertainty Estimation Guarantees} \label{sec:guarantees}

In this section, we provide theoretical guarantees showing that \oursacro{} fulfills the three axioms under mild assumptions given the specific definitions of concepts of aleatoric/epistemic uncertainty and with/without network effects presented in Sec.~\ref{subsec:ours}. Throughout this section, we consider a \oursacro{} model parameterized with a (feature) encoder $f_{\phi}$ with piecewise ReLU activations, a PPR diffusion, and a density estimator \smash{$\prob(\z^{\text{ft}, (\nodev)} \condition \omega)$} with bounded derivatives. We present detailed proofs in appendix. 

The first theorem shows that \oursacro{} follows Ax.~\ref{ax:certainty_features} and guarantees that \oursacro{} achieves reasonable uncertainty estimation on extreme node features without network effects:
\begin{theorem}
\label{thm:axiom-feature}
Lets consider a \oursacro{} model. Let \smash{$f_{\phi}(\x\nodeidxv)= V^{(l)}\x\nodeidxv + a^{(l)}$} be the piecewise affine representation of the ReLU network \smash{$f_{\phi}$} on the finite number of affine regions \smash{$Q^{(l)}$} \cite{understanding-nn-relu}. Suppose that \smash{$V^{(l)}$} have independent rows, then for any node $\nodev$ and almost any \smash{$\x\nodeidxv$} we have \smash{$\prob(f_{\phi}(\delta \cdot \x\nodeidxv) \condition \iclass; \vphi) \underset{\delta \rightarrow \infty}{\rightarrow} 0$}. Without network effects, it implies that \smash{$\beta_\iclass^{\text{ft}, (\nodev)} = \beta_\iclass^{\text{agg}, (\nodev)} \underset{\delta \rightarrow \infty}{\rightarrow} 0$}.
\end{theorem}
The proof relies on two main points: the equivalence of the \oursacro{} and PostNet architectures without network effects, and the uncertainty guarantees of PostNet far from training data similarly to \cite{NatPN2021}. It intuitively states that, without network effects, \oursacro{} predict small evidence (i.e. \smash{$\vbeta^{\text{agg}, (\nodev)} \approx \bm{0}$}) far from training features (i.e. \smash{$||\delta \cdot \x\nodeidxv|| \rightarrow \infty$}) and thus recover the prior prediction (i.e. \smash{$\valpha^{\text{post}, (\nodev)} \approx \valpha^\text{prior}$}). 
Note that contrary to \oursacro{}, methods which do not account for node features (e.g. Label Propagation) or methods which only use ReLU activations \cite{overconfident-relu} cannot validate Ax.~\ref{ax:certainty_features}. Further, methods which perform aggregation steps in early layers (e.g. GCN \citep{Kipf2016}) do not separate the processing of the feature and network information making unclear if they fulfill the Ax.~\ref{ax:certainty_features} requirements. 

The second theorem shows that \oursacro{} follows Ax.~\ref{ax:certainty_network_epistemic} and guarantees that a node $\nodev$ becomes more epistemically certain if its neighbors are more epistemically certain:
\begin{theorem}
\label{thm:axiom-network-epistemic}
Lets consider a \oursacro{} model. Then, given a node $\nodev$, the aggregated feature evidence \smash{$\alpha_0^{\text{agg}, (\nodev)}$} is increasing if the feature evidence \smash{$\alpha_0^{\text{ft}, (\nodeu)}$} of one of its neighbors \smash{$\nodeu \in \neighbors(\nodev)$} is increasing.
\end{theorem}
The proof directly relies on Eq.~\ref{eq:agg-evidence}. Intuitively, this theorem states that the epistemic uncertainty \smash{$u_\text{epist}\nodeidxv = -\alpha_0^\text{agg, (\nodev)}$} of a node $\nodev$ with network effects decreases if the epistemic uncertainty of the neighboring nodes without network effects decreases. Note that contrary to \oursacro{}, methods which do not model the epistemic uncertainty explicitly (e.g. GCN \cite{Kipf2016}, GAT \citep{Velickovic2017} or APPNP \citep{Klicpera2018}) are not guaranteed to fulfil Ax.~\ref{ax:certainty_network_epistemic}. 

The third theorem shows that \oursacro{} follows Ax.~\ref{ax:certainty_network_aleatoric}. It guarantees that a node $\nodev$ becomes more aleatorically uncertain if its neighbors are more aleatorically uncertain, or if a neighbor prediction disagrees more with the current node prediction:
\begin{theorem}
\label{thm:axiom-network-aleatoric}
Lets consider a \oursacro{} model. Lets denote \smash{$\bar{\p}^\text{agg, (\nodev)} = \nicefrac{\vbeta^{\text{agg}, (\nodev)}}{\alpha_0^{\text{agg}, (\nodev)}}$} the diffused categorical prediction for node $\nodev$ where \smash{$\iclass^*$} is its winning class. Further, lets denote \smash{$\bar{\p}^\text{ft, (\nodeu)} = \nicefrac{\vbeta^{\text{ft}, (\nodev)}}{\alpha_0^{\text{ft}, (\nodev)}}$} the non-diffused categorical prediction for a node $\nodeu \in \vertices$. First, there exists normalized weights \smash{$\Pi_{v, u}^{'}$} such that \smash{$\sum_{\nodeu \in \vertices} \Pi_{v, u}^{'} \Entropy{\DCat(\bar{\p}^\text{ft, (\nodeu)})} \leq \Entropy{\DCat(\bar{\p}^\text{agg, (\nodev)})}$}. Second, if for any node \smash{$\nodeu \in \vertices$} the probability of $\bar{\p}_{\iclass^*}^\text{ft, (\nodeu)}$ decreases, then \smash{$\Entropy{\DCat(\bar{\p}^\text{agg, (\nodev)})}$} increases.
\end{theorem}
The proof of the first part of the theorem is based on the entropy convexity. Intuitively, it states that the aleatoric uncertainty \smash{$u_\text{alea}\nodeidxv = \Entropy{\DCat(\bar{\p}^\text{agg, (\nodev)})}$} of a node $\nodev$ with network effects is lower bounded by a weighted average of the aleatoric uncertainty without network effects of its soft neighborhood. The second part of the theorem intuitively states that if the prediction of a neighboring node $\nodeu$ without neighbor effects disagrees more with the current class prediction $\iclass^*$ of the node $\nodev$, then the aleatoric uncertainty \smash{$u_\text{alea}\nodeidxv = \Entropy{\DCat(\bar{\p}^\text{agg, (\nodev)})}$} with network effects becomes higher. Note that contrary to \oursacro{}, methods which do not use edges (e.g. PostNet \cite{Charpentier2020}) cannot validate Ax.~\ref{ax:certainty_network_aleatoric} and Ax.~\ref{ax:certainty_network_epistemic}.

%% file: include/03d-limitations.tex
\subsection{Limitations \& Impact} \label{sec:limitations}

\textbf{OOD data close to ID data.} While \oursacro{} is guaranteed to provide consistent uncertainty estimates for nodes with extreme OOD features, it does not guarantee any specific uncertainty estimation behavior for OOD data close to ID data. Note that there exist two possible desired behaviors for OOD close to ID data: being robust to small dataset shifts \citep{Ovadia2019, confidence-calibrated-adversarial} or detect near OOD data \citep{contrastive-ood, robustness-uncertainty-dirichlet, attack-detection}. The duality of these two views makes unclear what would be the desired behavior even for i.i.d. data.

\looseness=-1
\textbf{Non-homophilic uncertainty.} Our approach assumes that connected nodes are likely to have similar uncertainty estimates as defined in Ax.~\ref{ax:certainty_network_epistemic} and Ax.~\ref{ax:certainty_network_aleatoric}. Contrary to \cite{heterophily-gnn}, we do not tackle the problem of heterophilic graphs where two neighboring nodes might reasonably have different uncertainty estimates. 

\textbf{Task-specific OOD.} Density estimation is shown to be inappropriate for OOD detection when acting directly on raw images \cite{typicality_OOD_generative, anomaly-detection, deep-generative} or on arbitrarily transformed space \cite{perfect-density-no-ood-guarantee}. One of the reasons is that normalizing flows learn pixel correlations in images. This phenomena  does not happen for tabular data with more semantic features \citep{Kirichenko2020}. First note that, similarly to tabular data, semantic node features are less likely to suffer from the same flaws. Second, following previous works \citep{Charpentier2020, NatPN2021, Kirichenko2020, density-states-ood, contrastive-ood}, \oursacro{} mitigates this issue by using density estimation on a latent space which is low-dimensional and task-specific. Nonetheless, we emphasize that \oursacro{} provides predictive uncertainty estimates which depends on the considered task i.e. OOD data w.r.t. features which are not useful for the specific task are likely not to be encoded in the latent space, and thus not to be detected.

\textbf{Broader Impact.}\label{sec:broader-impact}
The Assessment List for Trustworthy AI (ALTAI) \cite{trustworthy-ai} includes robustness, safety, and accountability. Uncertainty estimation is a key element to make AI systems follow these values. For example. an automated decision maker should know when it does not know. In this regard, \oursacro{} significantly improves the reliability of predictions on interdependent data under perturbations even though a user should not blindly rely on it. Further, ALTAI also mentions privacy and fairness. Therein, we raise awareness on the risk of using interconnected information which can amplify privacy or fairness violation in the presence of personal data.

%% file: include/05-experiment.tex
\section{Experiments} \label{sec:experiments}

In this section, we provide an extensive evaluation set-up for uncertainty quantification for node classification. It compares \textbf{\oursacro{}} to 13 baselines on 8 datasets and consists in two task types. First, we evaluate the detection of OOD nodes with features perturbations and Left-Out classes. Second, we evaluate the robustness of accuracy, calibration and uncertainty metrics w.r.t. feature and edge shifts.

\subsection{Set-up}

\textbf{Ablation.} In the experiments, \oursacro{} uses a MLP as feature encoder, radial normalizing flows \citep{Rezende2015} for the density estimation and a certainty budget $N$ which scales with respect to the latent dimension \citep{NatPN2021}. We provide an ablation study covering aleatoric uncertainty through APPNP, feature-level estimates through PostNet, diffusing resulting pseudo-counts after training, and \oursacro{} with diffusion of \smash{$\log(\beta_\iclass^{\text{ft}, (\nodev)})$} instead of \smash{$\beta_\iclass^{\text{ft}, (\nodev)}$} (see App.~\ref{sec:ablation-study}). The complete \oursacro{} model outperforms the ablated models for uncertainty estimation. Further, we provide a hyper-parameter study covering for example different number of flow layers, latent dimensions, PPR teleport probabilities (see App.~\ref{sec:hyperparameter-study})). 

\textbf{Baselines.} We used 13 baselines covering a wide variety of models for semi-supervised node classification and uncertainty estimation. We show the results of 5 baselines in the main paper and the full results in appendix. It contains two standard GNNs (i.e. Vanilla GCN \textbf{VGCN} \citep{Kipf2016, Shchur2018} and \textbf{APPNP} \citep{Klicpera2018}), one robust GNN (i.e. \textbf{RGCN} \citep{Zhu2019}), one dropout-based method for GNN (i.e. \textbf{DropEdge} \citep{Rong2019}), two Graph Gaussian Processes methods (i.e. \textbf{GGP} \citep{Ng2018} and \textbf{Matern-GGP} \citep{Borovitskiy2020}), the Graph-based Kernel Dirichlet GCN method (i.e. \textbf{GKDE-GCN} \citep{Zhao2020}) and two parameter-less methods (i.e. \textbf{GKDE} \citep{Zhao2020} and Label Propagation \textbf{LP} see App.). Further, we also compared to direct adaptation of dropout (i.e. \textbf{VGCN-Dropout}\citep{Gal2016}), ensemble (i.e. \textbf{VGCN-Ensemble} \citep{Lakshminarayanan2017}), BNN (i.e. \textbf{VGCN-BNN} \citep{Blundell2015}) and energy-based models (i.e. \textbf{VGCN-Energy} \citep{Liu2020a}) to vanilla GCNs. All models are trained using the same number of layers and similar number of hidden dimensions. We used early stopping and report the used hyperparameters in appendix. The results are averaged over 10 initialization seeds per split. Further model details are given in appendix.

\textbf{Datasets.} We used 8 datasets with different properties summarized in appendix. We show the results of 3 datasets in the main paper and the full results in appendix. It contains common citation network datasets (i.e. \textbf{CoraML} \citep{Mccallum2000, Giles1998, Getoor2005, Sen2008a}, \textbf{CiteSeer} \citep{Giles1998, Getoor2005, Sen2008a}, \textbf{PubMed} \citep{Namata2012}, \textbf{CoauthorPhysics} \citep{Shchur2018} \textbf{CoauthorCS} \citep{Shchur2018})
 and co-purchase datasets (i.e. \textbf{AmazonPhotos} \citep{Mcauley2015, Shchur2018}, \textbf{AmazonComputers} \citep{Mcauley2015, Shchur2018}). The results are averaged over 10 initialization splits with a train/val/test split of $5\%/15\%/80\%$ using stratified sampling. Further, we evaluate on the large \textbf{OGBN Arxiv} dataset with $169,343$ nodes and $2,315,598$ edges \citep{ogb-dataset, microsoft-academic-graph}. Further dataset details are given in the appendix.
 
\subsection{Results}
 
\looseness=-1
\textbf{OOD Detection.} In this section, we evaluate uncertainty estimation for OOD detection. To this end, we use the Area Under Receiving Operator Characteristics Curve (AUC-ROC) with aleatoric scores \smash{$u_\text{alea}\nodeidxv$} \textbf{(Alea)} and epistemic scores \smash{$u_\text{epist}\nodeidxv$} \textbf{(Epist)} similarly to \citep{Charpentier2020, Zhao2020, Malinin2018, Malinin2019a, Malinin2019b, Liu2020a}. For \oursacro{}, we differentiate between epistemic uncertainty scores without network effects \textbf{(w/o Net.)} and with network effects \textbf{(w/ Net.)}. Further, we report results with the Area Under the Precision-Recall Curve (AUC-PR) in appendix. The definition of OOD for nodes in the presence of feature and network information is more complex than for i.i.d. input features. Hence, we propose two types of OOD nodes: nodes with OOD feature perturbations and nodes from Left-Out classes. For feature perturbations, we compute the accuracy on the perturbed nodes \textbf{(OOD-Acc)} to evaluate if the model can correct anomalous features. For Left-Out classes, we compute the accuracy on the observed classes \textbf{(ID-Acc)}. We report the short results in Tab.~\ref{tab:ood_short}. We set a threshold of 64 GiB and 12 hours per training run. We also exclude methods which do not use attributes for detection of OOD feature perturbations.

\textit{\underline{Feature perturbations:}} These perturbations aim at isolating the contribution of the node feature information on the model predictions. To this end, we randomly select a subset of the nodes. For each single node $\nodev$, we perturb individually its features using a Bernoulli or a Normal distribution (i.e. $\x\nodeidxv \sim \DBer(0.5)$ and $\x\nodeidxv \sim \DNormal(\veczero, \vecone)$) keeping all other node features fixed. We then compare the uncertainty prediction on the perturbed and unperturbed node. On one hand, Bernoulli noise corresponds to small perturbations in the domain of discrete bag-of-words features. On the other hand, Normal noise corresponds to extreme perturbations out of the domain of discrete bag-of-words features. In practice, we expect out-of-domain perturbations to be easily detected \citep{Charpentier2020}. First, we remark that uncertainty estimates of \oursacro{} based on features achieves an absolute improvement of at least $+15\%$ and $+29\%$ for Bernoulli and Normal perturbations over all baselines using network effects. This shows that \oursacro{} disentangles well the uncertainty without and with network effects. Second, we remark that all uncertainty estimates with network effects achieve poor results. This is expected if models can recover the correct prediction after aggregation steps. Specifically, we observe that \oursacro{} achieves an accuracy improvement between $+16\%$ and $+64\%$ for Normal perturbations on perturbed nodes compared to baselines. It stresses that \oursacro{} performs a consistent evidence aggregation from neighborhood to recover from anomalous features. Further, note that \oursacro{} is still capable to detect those perturbed nodes almost perfectly using feature uncertainty. These remarks aligns with Ax.~\ref{ax:certainty_features}.

\textit{\underline{Left-Out classes:}} Detection of Left-Out classes involves both feature and neighborhood information. In this case, we remove the Left-Out classes from the training set but keep them in the graph similarly to \citep{Zhao2020}. We observe that the uncertainty estimates with network effects of \oursacro{} achieves an absolute improvement between $+12\%$ and $+16\%$ compared to its uncertainty estimates without network effects. It highlights the benefit of incorporating network information for uncertainty predictions when OOD samples (i.e. samples from the Left-Out classes) are likely to be connected to each other. This remark aligns with Ax.~\ref{ax:certainty_network_epistemic}. Further, \oursacro{} outperforms other baselines by $+2\%$ to $+22\%$ for LOC detection while maintaining a competitive accuracy on other classes.

\textit{\underline{Misclassified samples:}} In addition to the OOD scores, we also report the results for the detection of misclassified samples with aleatoric and epistemic uncertainty on several datasets and models in App.~\ref{sec:add-exp-misclassification} for the sake of completeness. GPN performs competitively with the baselines. Moreover, we observe that epistemic uncertainty is better for OOD detection and aleatoric uncertainty is better for misclassification detection as already observed  e.g. in \citep{Zhao2020}.


\input{tables/main_results}

\looseness=-1
\textbf{Attributed Graph Shifts.} In this section, we focus on evaluating the robustness of the accuracy, calibration and the evolution of the uncertainty estimation under node feature shifts and edges shifts. This aligns with \citep{Ovadia2019} which aims at evaluating the reliability of uncertainty estimates under dataset shifts for i.i.d. inputs. Specifically, we evaluates the evolution of the accuracy, the ECE \citep{Naeini2015} calibration score, the epistemic and the aleatoric uncertainty measures.

\textit{\underline{Feature shifts:}} We perturbed the features of a fraction of the nodes using unit Gaussian perturbations. We report the short results in Fig.~\ref{fig:shifts-normal-cora} and the full results in appendix. On one hand, we observe that \oursacro{} is significantly more robust to feature perturbations than all baselines. Indeed, the accuracy of \oursacro{} decreases by less than $5\%$ even when $80\%$ of the nodes are perturbed while the accuracy of other baselines decreases by more than $50\%$ when only $20\%$ of the nodes are perturbed. Similarly, we observed that \oursacro{} remains calibrated even when a high fraction of nodes are perturbed contrary to baselines. Hence, \oursacro{} intuitively discards uncertain features from perturbed nodes and only accounts for certain features from other nodes for more accurate predictions. On the other hand, we observe that, as desired, the average epistemic uncertainty of \oursacro{} consistently decreases when more nodes are perturbed. This remark aligns with Ax.~\ref{ax:certainty_network_epistemic}. In contrast, baselines dangerously become more certain while achieving a poorer accuracy similarly to ReLU networks \citep{overconfident-relu}. Hence \oursacro{} predictions are significantly more reliable than baselines under feature shifts.

\textit{\underline{Edge shifts:}} For edge shifts, we perturbed a fraction of edges at random. We report the results in appendix. As desired, we observe that the aleatoric uncertainty increases for all models including \oursacro{}. This aligns with Ax.~\ref{ax:certainty_network_aleatoric} and the expectations that conflicting neighborhood should lead to more aleatorically uncertain predictions. Furthermore, the average epistemic uncertainty of \oursacro{} remains constant which is reasonable since the average evidence of a node's neighborhood remains constant.

\begin{figure}[!h]
    \centering
    \includegraphics[width=\textwidth]{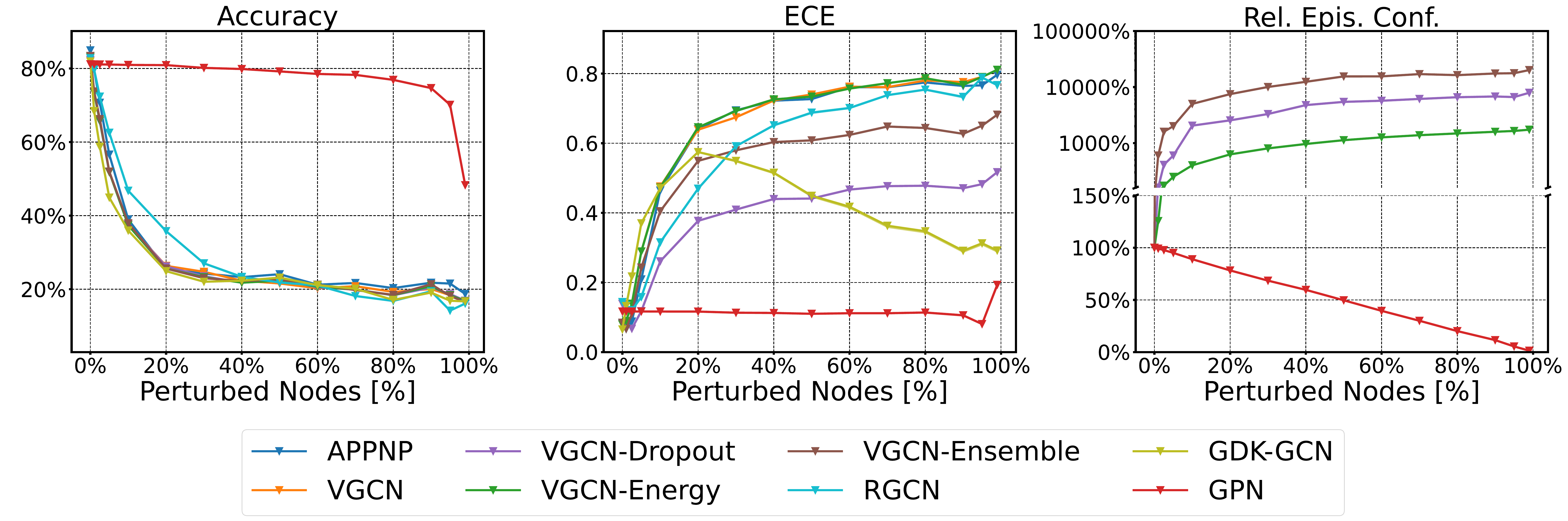}
    \caption{Accuracy, ECE, and average epistemic confidence under feature shifts for CoraML. We perturb features of different percentage of nodes using a Unit Gaussian noise.}
    \label{fig:shifts-normal-cora}
\end{figure}

\textbf{Qualitative Evaluation.} We show the abstracts of the CoraML papers achieving the highest and the lowest epistemic uncertainty without network effects in Tab.~\ref{tab:odd_abstracts} and in the appendix. Interestingly, we observed that most uncertain papers corresponds to short and unconventional abstracts which can be seen as anomalous features. Furthermore, we also ranked the nodes w.r.t. to their epistemic uncertainty with network effects. In this case, we observed that $78/100$ nodes with the highest uncertainty do not belong to the largest connected component of the CoraML dataset. We propose additional uncertainty visualizations for \oursacro{} in App.~\ref{sec:add-exp-qualitative}. 

\textbf{Inference \& training time.} We provide a comparison of inference and training times for most of the datasets and models under consideration in in App.~\ref{sec:add-exp-time}. GPN needs a single pass for uncertainty estimation but requires the additional evaluation of one normalizing flow per class compared to APPNP. Hence, GPN brings a small computational overhead for uncertainty estimation at inference time. Furthermore, GPN is usually converging relatively fast during training and does not require pre-computing kernel values. In contrast, GKDE-GCN \citep{Zhao2020} requires the computation of the underlying Graph Kernel with a complexity of $\mathcal{O}\left(\nnodes^2\right)$ where $\nnodes$ is the number of nodes in the graph. Finally, GPN is significantly more efficient than dropout or ensemble approaches as it does not require training or evaluating multiple models.

\input{tables/lowest-evidence-abstracts-short}

%% file: tables/main_results.tex
\begin{table*}[!h]
    \centering
    \resizebox{\textwidth}{!}{
    \begin{tabular}{ll|cl|cl|cl}
        \toprule
        & \textbf{Model} & \textbf{ID-ACC} & \textbf{OOD-AUC-ROC} & \textbf{OOD-ACC} & \textbf{OOD-AUC-ROC} & \textbf{OOD-ACC} & \textbf{OOD-AUC-ROC} \\
        & & \multicolumn{2}{c|}{Leave-Out Classes} & \multicolumn{2}{c|}{{$\x\nodeidxv \sim \DBer(0.5)$}} & \multicolumn{2}{c}{$\x\nodeidxv \sim \DNormal(0,1)$} \\
        \midrule
        

        \multirow{6}{*}{CoraML} 
        & Matern-GGP & ${87.03}$ & ${83.13}$  /  ${82.98}$  /  $n.a.$ & $n.a.$ & $\hphantom{00.00}$ $n.a.$ $\hphantom{00.00}$ & $n.a.$ & $\hphantom{00.00}$ $n.a.$ $\hphantom{00.00}$ \\
        & VGCN-Dropout & ${89.08}$ & ${81.27}$ / ${71.65}$ / $n.a.$ & ${77.76}$ & ${62.06}$ / ${50.38}$ / $n.a.$ & ${18.28}$ & ${40.53}$ / ${{71.06}}$ / $n.a.$\\
        & VGCN-Energy & ${89.66}$ & ${81.70}$ / ${83.15}$ / $n.a.$ & ${78.90}$ & ${{63.68}}$ / ${{66.26}}$ / $n.a.$ & ${18.37}$ & ${\hphantom{0}9.34}$ / ${\hphantom{0}0.32}$ / $n.a.$\\
        & VGCN-Ensemble & ${\mathbf{89.87}}$ & ${81.85}$ / ${74.24}$ / $n.a.$ & ${78.00}$ & ${63.58}$ / ${56.81}$ / $n.a.$ & ${21.00}$ & ${33.72}$ / ${64.92}$ / $n.a.$\\
        & GKDE-GCN & ${89.33}$ & ${82.23}$ / ${82.09}$ / $n.a.$ & ${76.40}$ & ${61.74}$ / ${63.15}$ / $n.a.$ & ${16.86}$ & ${40.03}$ / ${\hphantom{0}1.42}$ / $n.a.$\\
        & GPN & ${88.51}$ & ${{83.25}}$ / ${\mathbf{86.28}}$ / ${{80.95}}$ & ${\mathbf{80.98}}$ & ${57.99}$ / ${55.23}$ / ${\mathbf{89.47}}$ & ${\mathbf{81.53}}$ & ${55.96}$ / ${56.51}$ / ${\mathbf{100.00}}$\\

        \midrule

        \multirow{6}{*}{\shortstack{Amazon \\ Photos}}
        & Matern-GGP & ${88.65}$ & ${{87.26}}$ / ${86.75}$ / $n.a.$ & $n.a.$ & $\hphantom{00.00}$ $n.a.$ $\hphantom{00.00}$ & $n.a.$ & $\hphantom{00.00}$ $n.a.$ $\hphantom{00.00}$\\
        & VGCN-Dropout & ${94.04}$ & ${80.90}$ / ${70.11}$ / $n.a.$ & ${83.86}$ & ${56.85}$ / ${55.04}$ / $n.a.$ & ${22.29}$ & ${49.11}$ / ${66.74}$ / $n.a.$\\
        & VGCN-Energy & ${94.24}$ & ${82.44}$ / ${79.64}$ / $n.a.$ & ${83.91}$ & ${57.91}$ / ${59.07}$ / $n.a.$ & ${21.40}$ & ${31.07}$ / ${\hphantom{0}6.42}$ / $n.a.$\\
        & VGCN-Ensemble & ${\mathbf{94.28}}$ & ${82.72}$ / ${88.53}$ / $n.a.$ & ${84.40}$ & ${57.86}$ / ${56.01}$ / $n.a.$ & ${20.30}$ & ${44.14}$ / ${69.01}$ / $n.a.$\\
        & GKDE-GCN & ${89.84}$ & ${73.65}$ / ${69.09}$ / $n.a.$ & ${73.17}$ & ${57.01}$ / ${58.00}$ / $n.a.$ & ${24.04}$ & ${24.45}$ / ${\hphantom{0}9.82}$ / $n.a.$\\
        & GPN & ${94.01}$ & ${82.72}$ / ${\mathbf{91.98}}$ / ${76.57}$ & ${\mathbf{87.47}}$ & ${56.25}$ / ${60.52}$ / $\mathbf{75.24}$ & ${\mathbf{88.29}}$ & ${51.89}$ / ${61.89}$ / $\mathbf{100.00}$\\
        
        \midrule
        
        \multirow{6}{*}{\shortstack{OGBN \\ Arxiv}}
        & Matern-GGP & $n.f.$ & $\hphantom{00.00}$ $n.f.$ $\hphantom{00.00}$ & $n.f.$ &$ \hphantom{00.00}$ $n.f.$ $\hphantom{00.00}$ & $n.f.$ & $\hphantom{00.00}$ $n.f.$ $\hphantom{00.00}$\\
        & VGCN-Dropout & ${75.47}$ & ${65.35}$ / ${64.24}$ / $n.a.$ & ${65.30}$ & ${48.11}$ / ${50.64}$ / $n.a.$ & ${49.90}$ & ${{60.10}}$ / ${62.87}$ / $n.a.$\\
        & VGCN-Energy & ${75.61}$ & ${64.91}$ / ${64.50}$ / $n.a.$ & ${65.70}$ & ${46.16}$ / ${48.54}$ / $n.a.$ & ${{51.30}}$ & ${53.83}$ / ${48.53}$ / $n.a.$\\
        & VGCN-Ensemble & ${\mathbf{76.12}}$ & ${65.93}$ / ${70.77}$ / $n.a.$ & ${\mathbf{67.00}}$ & ${45.99}$ / ${47.41}$ / $n.a.$ & ${49.00}$ & ${59.94}$ / ${{66.44}}$ / $n.a.$\\
        & GKDE-GCN & ${73.89}$ & ${{68.84}}$ / ${72.44}$ / $n.a.$ & ${65.20}$ & ${{50.98}}$ / ${{51.31}}$ / $n.a.$ & ${45.40}$ & ${53.94}$ / ${55.28}$ / $n.a.$\\
        & GPN & ${73.84}$ & ${66.33}$ / ${\mathbf{74.82}}$ / ${{62.17}}$ & ${65.50}$ & ${{51.49}}$ / ${{55.82}}$ / ${\mathbf{93.05}}$ & ${\mathbf{65.50}}$ & ${51.43}$ / ${55.85}$ / ${\mathbf{95.54}}$\\
        \bottomrule
    \end{tabular}
    }
    \caption{LOC and Feature Perturbations: Accuracy is reported on ID nodes for LOC experiments and on OOD nodes for feature perturbation experiments. OOD-AUC-ROC scores are given as \emph{[Alea w/ Net] / [Epist w/ Net] / [Epist w/o Net]}. $n.a.$ means either model or metric not applicable and $n.f.$ means not finished within our constraints.}
    \label{tab:ood_short}
    \vspace{-3mm}
\end{table*}

%% file: tables/lowest-evidence-abstracts-short.tex
\begin{table}[!h]
    \centering
    \tiny
    \begin{tabular}{p{6.5cm} p{6.5cm}}
        \toprule
        {\tt IlliGAL Report No. 95006 July 1995} & 
        {\tt Report of the 1996 Workshop on Reinforcement} \\
        \midrule
        {\tt Reihe FABEL-Report Status: extern Dokumentbezeichner: Org/Reports/nr-35 Erstellt am: 21.06.94 Korrigiert am: 28.05.95 ISSN 0942-413X} & 
        {\tt We tend to think of what we really know as what we can talk about, and disparage knowledge that we can't verbalize. [Dowling 1989, p. 252]} \\
        \midrule
        {\tt Keith Mathias and Darrell Whitley Technical Report CS-94-101 January 7, 1994} & 
        {\tt Multigrid Q-Learning Charles W. Anderson and Stewart G. Crawford-Hines Technical Report CS-94-121 October 11, 1994} \\
        \midrule
        {\tt Internal Report 97-01} & {\tt A Learning Result for Abstract} \\
        \bottomrule
        \vspace{0.1cm}
    \end{tabular}
    \caption{A selection of abstracts from CoraML which are assigned low feature evidences by GPN.}
    \label{tab:odd_abstracts}
    \vspace{-3mm}
\end{table}

%% file: include/06-conclusion.tex
\section{Conclusion} \label{sec:conclusion}

We introduce a well-grounded framework for uncertainty estimation on interdependent nodes. First, we propose explicit and motivated axioms describing desired properties for aleatoric and epistemic uncertainty in the absence or in the presence of network effects. Second, we propose \oursacro{}, a GNN for uncertainty estimation which provably follows our axioms. \oursacro{} performs a Bayesian update over the class predictions based on density estimation and diffusion. Third, we conduct extensive experiments to evaluate the uncertainty performances of a broad range of baselines for OOD detection and robustness against node feature or edge shifts. \oursacro{} outperforms all baselines in these experiments.

%% file: include/xx-appendix.tex
\section{Proofs} \label{sec:proofs}

\begin{lemma}
\label{lem:relu-encoder-density}
\cite{NatPN2021} Let a model be parameterized with an encoder $f_{\phi}$ with piecewise ReLU activations, a decoder $g_{\psi}$ and the density estimator $\prob(\z \condition \bm{\omega})$. Let $f_{\phi}(\x)= V^{(l)}\x + a^{(l)}$ be the piecewise affine representation of the ReLU network $f_{\phi}$ on the finite number of affine regions $Q^{(l)}$ \cite{understanding-nn-relu}. Suppose that $V^{(l)}$ have independent rows and the density function $\prob(\z \condition \bm{\omega})$ has bounded derivatives, then for almost any $\x$ we have \smash{$\prob(f_{\phi}(\delta \cdot \x) \condition \bm{\omega}) \underset{\delta \rightarrow \infty}{\rightarrow} 0$}. i.e the evidence becomes small far from training data.
\end{lemma}

\begin{theorem}
\label{thm:axiom-feature-proof}
Lets consider a \oursacro{} model. Let \smash{$f_{\phi}(\x\nodeidxv)= V^{(l)}\x\nodeidxv + a^{(l)}$} be the piecewise affine representation of the ReLU network \smash{$f_{\phi}$} on the finite number of affine regions \smash{$Q^{(l)}$} \cite{understanding-nn-relu}. Suppose that \smash{$V^{(l)}$} have independent rows, then for any node $\nodev$ and almost any $\x\nodeidxv$ we have \smash{$\prob(f_{\phi}(\delta \cdot \x\nodeidxv) \condition \iclass; \vphi) \underset{\delta \rightarrow \infty}{\rightarrow} 0$}. Without network effects, it implies that \smash{$\beta_\iclass^{\text{ft}, (\nodev)} = \beta_\iclass^{\text{agg}, (\nodev)} \underset{\delta \rightarrow \infty}{\rightarrow} 0$}.
\end{theorem}

\begin{proof}
First, remark that each normalizing flow density in \oursacro{} fulfills the conditions of Lem.~\ref{lem:relu-encoder-density}. This means that \smash{$\prob(f_{\phi}(\delta \cdot \x\nodeidxv) \condition \iclass; \vphi) \underset{\delta \rightarrow \infty}{\rightarrow} 0$} which implies $\beta_\iclass^{\text{ft}, (\nodev)} \underset{\delta \rightarrow \infty}{\rightarrow} 0$. Further note that in the absence of network effects, the PPR diffusion has no effect on the pseudo-counts i.e. $\beta_\iclass^{\text{ft}, (\nodev)} = \beta_\iclass^{\text{agg}, (\nodev)}$.
\end{proof}

\begin{theorem}
\label{thm:axiom-network-epistemic-proof}
Lets consider a \oursacro{} model. Then, given a node $\nodev$, the aggregated feature evidence $\alpha_0^{\text{agg}, (\nodev)}$ is increasing if the feature evidence $\alpha_0^{\text{ft}, (\nodeu)}$ of one of its neighbor $\nodeu \in \neighbors(\nodev)$ is increasing.
\end{theorem}

\begin{proof}
We recall first the definition of the aggregated and feature total evidence pseudo-count and the PPR diffusion step of \oursacro{}:
\begin{align*}
    \alpha_0^{\text{ft}, (\nodev)} &= \sum_\iclass \beta_\iclass^{\text{ft}, (\nodev)}\\
    \alpha_0^{\text{agg}, (\nodev)} &= \sum_\iclass \beta_\iclass^{\text{agg}, (\nodev)}\\
    \beta_\iclass^{\text{agg}, (\nodev)} &= \sum_{\nodeu \in \vertices} \pprelem_{v, u} \beta_\iclass^{\text{ft}, (\nodeu)}
\end{align*}
We combine these three equations to show a closed-form relation between the aggregated and the feature evidence pseudo-count:
\begin{align*}
    \alpha_0^{\text{agg}, (\nodev)} &= \sum_\iclass \beta_\iclass^{\text{agg}, (\nodev)}\\
    &= \sum_\iclass \sum_{\nodeu \in \vertices} \pprelem_{v, u} \beta_\iclass^{\text{ft}, (\nodeu)}\\
    &= \sum_{\nodeu \in \vertices} \pprelem_{v, u} \sum_\iclass \beta_\iclass^{\text{ft}, (\nodeu)}\\
    &= \sum_{\nodeu \in \vertices} \pprelem_{v, u} \alpha_0^{\text{ft}, (\nodeu)}
\end{align*}
This shows that the aggregated evidence is the diffused feature evidence with PPR. Hence, the aggregated evidence $\alpha_0^{\text{agg}, (\nodev)}$ is a strictly increasing function w.r.t. to the feature evidence of each individual neighbor $\alpha_0^{\text{ft}, (\nodeu)}$ when $\nodeu \in \neighbors(\nodev)$.
\end{proof}

\begin{theorem}
\label{thm:axiom-network-aleatoric-proof}
Lets consider a \oursacro{} model. Lets denote $\bar{\p}^\text{agg, (\nodev)} = \nicefrac{\vbeta^{\text{agg}, (\nodev)}}{\alpha_0^{\text{agg}, (\nodev)}}$ the diffused categorical prediction for node $\nodev$ where $\iclass^*$ is its winning class. Further, lets denote \smash{$\bar{\p}^\text{ft, (\nodeu)} = \nicefrac{\vbeta^{\text{ft}, (\nodev)}}{\alpha_0^{\text{ft}, (\nodev)}}$} the non-diffused categorical prediction for a node $\nodeu \in \vertices$. First, there exists normalized weights $\Pi_{v, u}^{'}$ such that \smash{$\sum_{\nodeu \in \vertices} \Pi_{v, u}^{'} \Entropy{\DCat(\bar{\p}^\text{ft, (\nodeu)})} \leq \Entropy{\DCat(\bar{\p}^\text{agg, (\nodev)})}$}. Second, if for any node \smash{$\nodeu \in \vertices$} the probability of $\bar{\p}_{\iclass^*}^\text{ft, (\nodeu)}$ decreases, then \smash{$\Entropy{\DCat(\bar{\p}^\text{agg, (\nodev)})}$} increases.
\end{theorem}

\begin{proof}
We first show the first part of the theorem. To this end, we use the relation between the aggregated and the feature evidence to derive a closed form relation between $\bar{\p}^\text{agg, (\nodev)}$ and $\bar{\p}^\text{ft, (\nodeu)}$.
\begin{align*}
    \bar{\p}^\text{agg, (\nodev)} & = \frac{\vbeta^{\text{agg}, (\nodev)}}{\alpha_0^{\text{agg}, (\nodev)}} \\
    & = \frac{\sum_{\nodeu \in \vertices} \pprelem_{v, u} \vbeta^{\text{ft}, (\nodeu)}}{\sum_{\nodeu' \in \vertices} \pprelem_{v, u'} \alpha_0^{\text{ft}, (\nodeu')}}\\ 
    & = \frac{\sum_{\nodeu \in \vertices} \pprelem_{v, u} \alpha_0^{\text{ft}, (\nodeu)} \bar{\p}^\text{ft, (\nodeu)}}{\sum_{\nodeu' \in \vertices} \pprelem_{v, u'} \alpha_0^{\text{ft}, (\nodeu')}}\\
    & = \sum_{\nodeu \in \vertices} \Pi_{v, u}^{'} \bar{\p}^\text{ft, (\nodeu)}
\end{align*}
where $\Pi_{v, u}^{'} =  \frac{\pprelem_{v, u} \alpha_0^{\text{ft}, (\nodeu)} }{\sum_{\nodeu' \in \vertices} \pprelem_{v, u'} \alpha_0^{\text{ft}, (\nodeu')}}$. Hence, the probability vector $\bar{\p}^\text{agg, (\nodev)}$ is a convex combination of the probability vectors $\bar{\p}^\text{agg, (\nodeu)}$ of other nodes. Further, using the concavity of the entropy function, we obtain the results:
\begin{align*}
    \sum_{\nodeu \in \vertices} \Pi_{v, u}^{'} \Entropy{\DCat(\bar{\p}^\text{ft, (\nodeu)})} \leq \Entropy{\DCat(\bar{\p}^\text{agg, (\nodev)})}
\end{align*}
Second, we show the second part of the theorem. To this end, we suppose that, for a neighboring node $\nodeu \in \neighbors(\nodev)$, the probability of the winning class $\iclass^*$ decreases and the probability of another class $\iclass^{'}$ increases i.e. $\bar{p}_{\iclass^*}^\text{ft, (\nodeu)} - \eps$ and $\bar{p}_{\iclass^{'}}^\text{ft, (\nodeu)} + \eps$. We define the following univariate function:
\begin{align*}
    f(\eps) = \Entropy{\DCat(\bar{\p}_{\eps}^\text{agg, (\nodev)})}
\end{align*}
where $\bar{\p}_{\eps}^\text{agg, (\nodev)}$ is the new aggregated probability vector of the node $\nodev$ after the epsilon change of the probability vector for node $\nodeu$. Note that $\bar{p}_{\eps, \iclass^*}^\text{agg, (\nodev)} = \sum_{\nodew \in \vertices} \pprelem_{v, w} \bar{p}_{\iclass^*}^\text{ft, (\nodew)} - \pprelem_{v, u}\eps$ and $\bar{p}_{\eps, \iclass^{'}}^\text{agg, (\nodev)} = \sum_{\nodew \in \vertices} \pprelem_{v, w} \bar{p}_{\iclass^{'}}^\text{ft, (\nodew)} + \pprelem_{v, u}\eps$. We compute the derivative of $f(\eps)$:
\begin{align*}
    \frac{\partial f(\eps)}{\partial \eps} &= \frac{\partial (\bar{p}_{\eps, \iclass^*}^\text{agg, (\nodev)} \log \bar{p}_{\eps, \iclass^*}^\text{agg, (\nodev)} + \bar{p}_{\eps, \iclass^{'}}^\text{agg, (\nodev)} \log \bar{p}_{\eps, \iclass^{'}}^\text{agg, (\nodev)})}{\partial \eps}\\
    & = \log\frac{\sum_{\nodew \in \vertices} \pprelem_{v, w} \bar{p}_{\iclass^*}^\text{ft, (\nodew)} - \pprelem_{v, u}\eps}{\sum_{\nodew \in \vertices} \pprelem_{v, w} \bar{p}_{\iclass^{'}}^\text{ft, (\nodew)} + \pprelem_{v, u}\eps}
\end{align*}
Hence, we note that as long as the class $\iclass^*$ is winning (i.e $\bar{p}_{\eps, \iclass^*}^\text{agg, (\nodev)} = \sum_{\nodew \in \vertices} \pprelem_{v, w} \bar{p}_{\iclass^*}^\text{ft, (\nodew)} - \pprelem_{v, u}\eps \geq \sum_{\nodew \in \vertices} \pprelem_{v, w} \bar{p}_{\iclass^{'}}^\text{ft, (\nodew)} + \pprelem_{v, u}\eps = \bar{p}_{\eps, \iclass^{'}}^\text{agg, (\nodev)}$), the function $f(\eps)$ is increasing. It means that the entropy \smash{$\Entropy{\DCat(\bar{\p}^\text{agg, (\nodev)})}$} increases when a neighboring node disagree more with the winning class $\iclass^*$. In particular, note that the conclusion holds if the epsilon decrease of the winning class is compensated by the probability increase of $K$ different classes. It would correspond to composing $K$ decreasing functions $f_k(\eps_k)$ where $k \in \{0,...,K\}$ and $\sum_k \eps_k = \eps$.
\end{proof}

\clearpage

\section{Dataset Details} \label{sec:app_dataset_details}
We consider the citation network datasets \emph{CoraML} \citep{Mccallum2000, Giles1998, Getoor2005, Sen2008a, Bojchevski2017}, \emph{CiteSeer} \citep{Giles1998, Getoor2005, Sen2008a}, \emph{PubMed} \citep{Namata2012}, \emph{CoauthorPhysics} and \emph{CoauthorCS} (based on the \emph{Microsoft Academic Graph} from the \emph{KDD Cup 2016} challenge) \citep{Shchur2018}
as well as two co-purchase datasets, \emph{AmazonPhotos} and \emph{AmazonComputers} \citep{Mcauley2015, Shchur2018}. Details are presented in the Table \ref{tab:dataset-summary}. For all those datasets, we consider a train/val/test split of $5/15/80$ using stratified sampling. For \emph{CoraML}, this corresponds to the default split of $20$ training samples per class with the difference of representing larger classes with more and smaller ones with less. We also note that this is significantly closer to the default split than approaches like \citep{Wang2020, Huang2020} introducing a $60/20/20$ split. For all those datasets, we average the results over individual predictions from 10 random splits together with 10 random model initializations per split, i.e. 100 runs for each dataset and model. Those dataset are part of \emph{PyTorch-Geometric} and under a MIT license. Besides those datasets, we also report results for the large dataset \emph{OGBN Arxiv} dataset \citep{ogb-dataset} which is based on the \emph{Microsoft Academic Graph} \citep{microsoft-academic-graph} with the public split based on publication years. Since this makes random splits unnecessary, we report results as averages over 10 runs. This dataset is also available under a MIT license.

\input{tables/dataset-summary}

\section{Metrics} \label{sec:app_metrics}
\textbf{ACC - } Accuracy is simply the fraction of predictions $\hat{y}^{(u)}$ that correspond to the ground-truth targets $y^{(u)}$ out of a set of all predictions of size $N$, i.e. $acc = \frac{1}{N}\sum_{u\in \vertices} 1_{y^{(u)} = \hat{y}^{(u)}}$.

\textbf{Brier - } The Brier Score is computed as $\nicefrac{1}{C}\sum_\nodev^N ||\p\nodeidxv - \vy\nodeidxv||_2$ where $\vy\nodeidxv$ represent the one-hot encoded ground-truth label for a node $\nodev$ while $\p\nodeidxv$ is the predicted probability score. 

\textbf{ECE -}  The expected calibration error on the other hand requires binning the predicted probability scores into $M$ equally spaced bins with $conf(B_m)$ being the average probability score in the bin $m$. With $acc(B_m)$ being the average accuracy of predictions in bin $m$, we obtain the final metric as 
\begin{equation}
    ECE = \sum_m^M \frac{|B_m|}{n} |acc(B_m) - conf(B_m)|
\end{equation}

\textbf{OOD -} Furthermore, we use \textbf{AUC-ROC} and \textbf{AUC-PR} scores for OOD-detection experiments. This problem is considered as a binary classification problem with the positive targets being the OOD-nodes and the negative targets being ID-data points. We use the same aleatoric and epistemic uncertainty measures used in \citep{Charpentier2020}. For aleatoric uncertainty measures, we use $u_\text{alea}\nodeidxv = - \max_c \bar{\mathbf{p}}_c^{(u)}$. For epistemic uncertainty, we use $u_\text{epist}\nodeidxv = - \alpha_0\nodeidxv$ for Dirichlet-based methods and $u_\text{epist}\nodeidxv = \frac{1}{\sum_\iclass \variance \bar{\mathbf{p}}_{c}^{(u)}}$ for other methods. For Dirichlet-based methods and GPs, the corresponding quantities are predicted directly. For ensemble and dropout baselines, these quantities are computed based on the empirical mean and empirical variance.

Note that the \emph{vacuity} uncertainty measure proposed in \citep{Zhao2020} and motivated from work on \emph{subjective logic} is just the inverse transformation of $\alpha_0$ given by $u_\text{vacuity} = \frac{C}{\alpha_0}$. Hence, the AUC-ROC and AUC-PR scores which evaluate the ranking of the examples lead to the \emph{exact} same final scores using $u_\text{vacuity}$ or $u_\text{epist}$.

\section{Model Details} \label{sec:app_model_details}

We follow \citep{Shchur2018, Klicpera2019, Zhao2020} and baselines from the OGBN leaderboard\footnote{\url{https://ogb.stanford.edu/docs/leader_nodeprop/\#ogbn-arxiv}} for the choice of the architectures. By default, we use a hidden dimension of $h = 64$ and $l=2$ layers for parametric models on all datasets except for \emph{OGBN Arxiv}. In this case, we use early stopping with a patience of $50$ and a maximum of $100,000$ epochs. For \emph{OGBN Arxiv}, we use a hidden dimension of $h = 256$ with $l=3$ layers and use batch-norm. In this case, we use early stopping with a patience of $200$, a maximum number of $1,000$ epochs and no weight-decay for all models. For Gaussian Processes, we implement our experiment pipelines and models in \emph{PyTorch} \citep{pytorch} and rely on \emph{PyTorch-Geometric} \citep{pytorch-geometric}. For all models, we use the Adam optimizer \citep{Kingma2014} with its default parameters and a learning rate of $0.01$. For further details, we provide the code in the supplementary material.

\textbf{GPN - } We use a similar backbone architecture as for APPNP \citep{Klicpera2018}.\footnote{Code available at \url{https://www.daml.in.tum.de/graph-postnet}} We report all the used hyper-parameters in Tab.~\ref{tab:gpn-hyper-parameters}. Similarly to \citep{NatPN2021}, we use a certainty budget $N$ which scales exponentially w.r.t. the latent dimension (i.e. $N_H=\sqrt{4 \pi}^H$) and $5$ warm-up epochs maximizing the log likelihood of the normalizing flows. Furthermore, we use $K=10$ power-iterations steps to approximate PPR scores. We do not use weight decay for the Normalizing Flows. Those parameters have been obtained after conducting an ablation and a hyper-parameter study on the \emph{CoraML} and \emph{OGBN Arxiv} datasets (see Sec.\ref{sec:ablation-study} and Sec.\ref{sec:hyperparameter-study}). Finally, we recall for completeness the closed-form of the Bayesian loss introduced in \citep{Charpentier2020} when $\prior\namednodeidxv{post}=\DDir(\bm{\alpha}\nodeidxv)$ and $\prob(\y\nodeidxv \condition \p\nodeidxv) = \DCat(\y\nodeidxv \condition \p\nodeidxv)$:
\begin{equation}
    \loss\nodeidxv = -\ExpecationArgs{\p\nodeidxv\sim \prior\namednodeidxv{post}}{\log \prob(\y\nodeidxv \condition \p\nodeidxv)} - \lambda \Entropy{\prior\namednodeidxv{post}}
\end{equation}
The expected likelihood term is equal to:
\begin{align}\label{eq:expected-likelihood-cat-dir}
    \expectation_{\bm{p} \sim \DDir(\bm{\alpha})}[\log \DCat(\y \condition \bm{p})] = \psi(\alpha_{\y}) - \psi(\alpha_0)
\end{align}
The entropy term is equal to:
\begin{align}\label{eq:density-entropy-dirichlet}
        \entropy[\DDir(\bm{\alpha})] = \log B(\bm{\alpha}) + (\alpha_0 - \nclass) \psi(\alpha_0) - \sum_\iclass (\alpha_\iclass - 1) \psi(\alpha_\iclass)
\end{align}

\input{tables/gpn-hyper-parameters}

\textbf{GKDE - } We adopt the Graph Kernel Dirichlet Estimate from \citep{Zhao2020} as a standalone and parameterless baseline. With $d_{v,u}$ being the shortest path between nodes $v$ and $u$ and the Gaussian 
transformation $g(d_{v,u}) = \nicefrac{1}{\sigma \sqrt{2\pi}} \exp \left(\nicefrac{-d_{v,u}^2}{2\sigma^2}\right)$, a Dirichlet estimate is obtained in the following way
\begin{equation}
    \valpha^{(v)} = 1 + \vect{e}^{(v)} \quad \text{ with } \quad \vect{e}^{(v)} = \sum_{u \in \nodeslabeled} \vect{h}(y^{(u)}, d_{v,u}) \quad \text{ and } \quad h_c(y^{(u)}, d_{v,u}) = \begin{cases}0 & y^{(u)} \neq c\\ g(d_{v,u}) &  y^{(u)} = c \end{cases}
\end{equation}
Similar to \citep{Zhan2020}, we use $\sigma=1$. We would also like to point out that the computation of this kernel requires extracting the shorted distance of each node to each labeled node $u \in \nodeslabeled$. Larger datasets like \emph{OGBN-Arxiv} come with larger sets of labeled data with the size $|\nodeslabeled|$ having a same magnitude as the number of nodes in the graph, i.e. $|\vertices$|. This approach therefore scales quadratically with the number of nodes in the graph and therefore does not generalize well to larger datasets. 

\textbf{LP - } Following the idea of the GKDE baseline, we propose similar Dirichlet estimates by relying on Label Propagation which achieve strong results in Left-Out classes experiments. GKDE extracts Dirichlet evidence scores by relying on node distances. We propose taking the density of labeled nodes in neighborhoods instead. To this end, we define one initial conditional density per class $\rho_0(u \condition c)$ and diffuse them with Personalized Page Rank i.e.
\begin{equation}
\rho_0(u \condition c) = 
\begin{cases}
0 & u \in \sU \\
\frac{1}{|\sL_c|} \cdot \delta_{y^{(u)}, c} & u \in \sL
\end{cases} \rightarrow \rho(v \condition c) = \sum_u \pprelem_{v, u}\cdot \rho_0(u \condition c)
\end{equation}
where $\sL_c$ is the set of labeled nodes for class $c$. The diffused density $\rho(v \condition c)$ is still a valid density i.e. $\sum_c \sum_u \rho(u\mid c) = \sum_c \sum_u \rho_0(u\mid c) = 1$. Finally, we use this diffused conditional densities to obtain Dirichlet evidence scores in a similar fashion to the GKDE kernel \citep{Zhao2020}, i.e. $\alpha_c^{(v)} = 1.0 + \rho(v\mid c)$. The diffusion is performed with power-iteration similar to APPNP \citep{Klicpera2018}. We use a teleport probability of $\tau=0.1$ and $K=10$ power iteration steps.

\textbf{Gaussian Processes - } We use the official implementations for \textbf{MaternGGP} \citep{Borovitskiy2020} and the re-implementation\footnote{https://github.com/FelixOpolka/GGP-TF2} for \textbf{GGP} \citep{Ng2018}. The re-implementation transfers the official implementation to Tensorflow 2.0 \citep{tensorflow} which we wrapped in our Pytorch pipeline. Since those approaches do not scale well to large real-world datasets, we restrict to a single random initialization. GGP only finished the experiments on CoraML and CiteSeer. Similarly, MaternGGP did not finished the experiments on OGBN Arxiv. For recall, we set the memory an time limits to 64 GiB and 12 hours per run. For comparison, note that all GNN-based models require significantly less memory and finished \emph{all} runs in a couple of hours.

\textbf{GKDE-SGCN - } We use the hyper-parameters suggested in the original paper \citep{Zhao2020}. We set the regularization factor $\lambda$ to $0.001$. This factor determines the weight of the Graph-Kernel-Dirichlet-Estimate which is key for OOD detection in graphs. Note that we did not use different factors for OOD experiments and classification experiments contrary to \citep{Zhao2020} since it leads to the leakage of task information. Indeed, the clean accuracy is significantly higher for $\lambda=0.001$ compared to $\lambda=0.1$. For \emph{OGBN Arxiv}, we did not use teacher training as it harmed the performance. In this case, we used a dropout probability of $p=0.5$ and $\lambda = 0.0001$ after a small grid-search with the overall architecture following the initial remarks above.

\textbf{APPNP - } We follow \citep{Klicpera2019} and use an architecture that is comparable to other GNN approaches. We use ReLU activations, dropout with $p=0.5$, no dropout on the adjacency matrix, a teleport probability of $\tau = 0.1$ and $K=10$ power iteration steps. We also use a weight decay of $\lambda = 0.0001$.

\textbf{VGCN - } We use ReLU activations, dropout with $p=0.8$, and a weight decay of $\lambda = 0.0001$. For the larger dataset \emph{OGBN Arxiv}, we use a dropout of $p=0.5$. \textbf{DropEdge} is similar to the Vanilla GCN model with an additional dropout on the edges with a dropout probability of $p=0.5$ on both features and edges. For evaluation of dropout models, i.e. \textbf{DropEdge} and \textbf{VGCN-Dropout}, we use $S = 10$ Monte-Carlo samples having shown a reasonable estimate with more samples not leading to a visible improvement. For ensembles, i.e. \textbf{VGCN-Ensemble}, we use an ensemble of models of $10$ different random initializations. For \textbf{VGCN-Energy}, we follow \citep{Liu2020a} and use a temperature of $T = 1.0$.

\textbf{VGCN-BNN - } We follow \emph{Bayes by Backprop} \citep{Blundell2015} and adopt a Bayesian GNN with uncertain weights. We use a hidden dimension of $h=32$. This is equivalent to $h=64$ for other models as each weight is represented by one mean parameter $\mu$ and one variance parameter $\Sigma$. We use $10$ bayesian samples in our experiments. We follow the grid search suggested in the original paper \citep{Blundell2015}. We finally adopt $\pi = 0.75$, $\sigma_1 = 1.0$, and $\sigma_2 = 1.0e-6$. Since this models assumes uncertain weights, we do not apply any weight decay during training. Note that we do not report results for the larger dataset \emph{OGBN Arxiv} for this baseline. 

\textbf{RGCN - } We follow \citep{Zhu2019} for the hyper-parameter selection. We use a hidden size of $h=32$. Again this is equivalent to $h=64$ since RGCN models a mean parameter $\mu$ and a variance parameter $\Sigma$ per layer. We further use dropout $p=0.5$ on the features, $\gamma=1$, $\beta_{KL}=5.04-4$ and $\beta_{reg} = 5.0e-4$. As the latter is already a weight regularization in the loss, we do not apply weight decay.

\section{Additional Experiments}
\subsection{Additional Experiments - Ablation Study} \label{sec:ablation-study}

In this section, we evaluate the contribution of each component of \oursacro{}. To this end, we use \textbf{PostNet} which first trains the feature encoder and the normalizing flows without diffusion and \textbf{PostNet-Diff} which diffuses the ablation-counts only at test time. Further, we also compare to \textbf{APPNP} \citep{Klicpera2018} which does not model the epistemic uncertainty with density estimation and \textbf{GPN-LOG} which diffuse the parameters $\log(\beta_\iclass^{\text{ft}, (\nodev)})$ instead of $\beta_\iclass^{\text{ft}, (\nodev)}$. We observed that training with diffusion is beneficial for all metrics. Further, we noted that diffusing $\log(\beta_\iclass^{\text{ft}, (\nodev)})$ improves accuracy and calibration to the cost of a lower Left-Out classes detection scores. Similarly, APPNP also showed better accuracy when diffusing logits instead of softmax outputs in the original paper \citep{Klicpera2018}. Finally, APPNP achieves significantly worse results for all OOD detection tasks showing the benefit of modelling the epistemic uncertainty.
\input{tables/ablation-study}
\subsection{Additional Experiments - Hyper-parameter study} \label{sec:hyperparameter-study}
Besides the previously mentioned ablation study, we also performed a study on the influences of hyperparameters. We show findings for the \emph{CoraML} dataset averaging runs with $3$ different random dataset splits and $2$ different random initializations. We analyzed the influence of the latent dimension, the number of radial layers, the teleport probability, the certainty budget, weight decay, and entropy regularization. 
\input{tables/grid-search}
\begin{figure}[h]
    \centering
    \includegraphics[width=\textwidth]{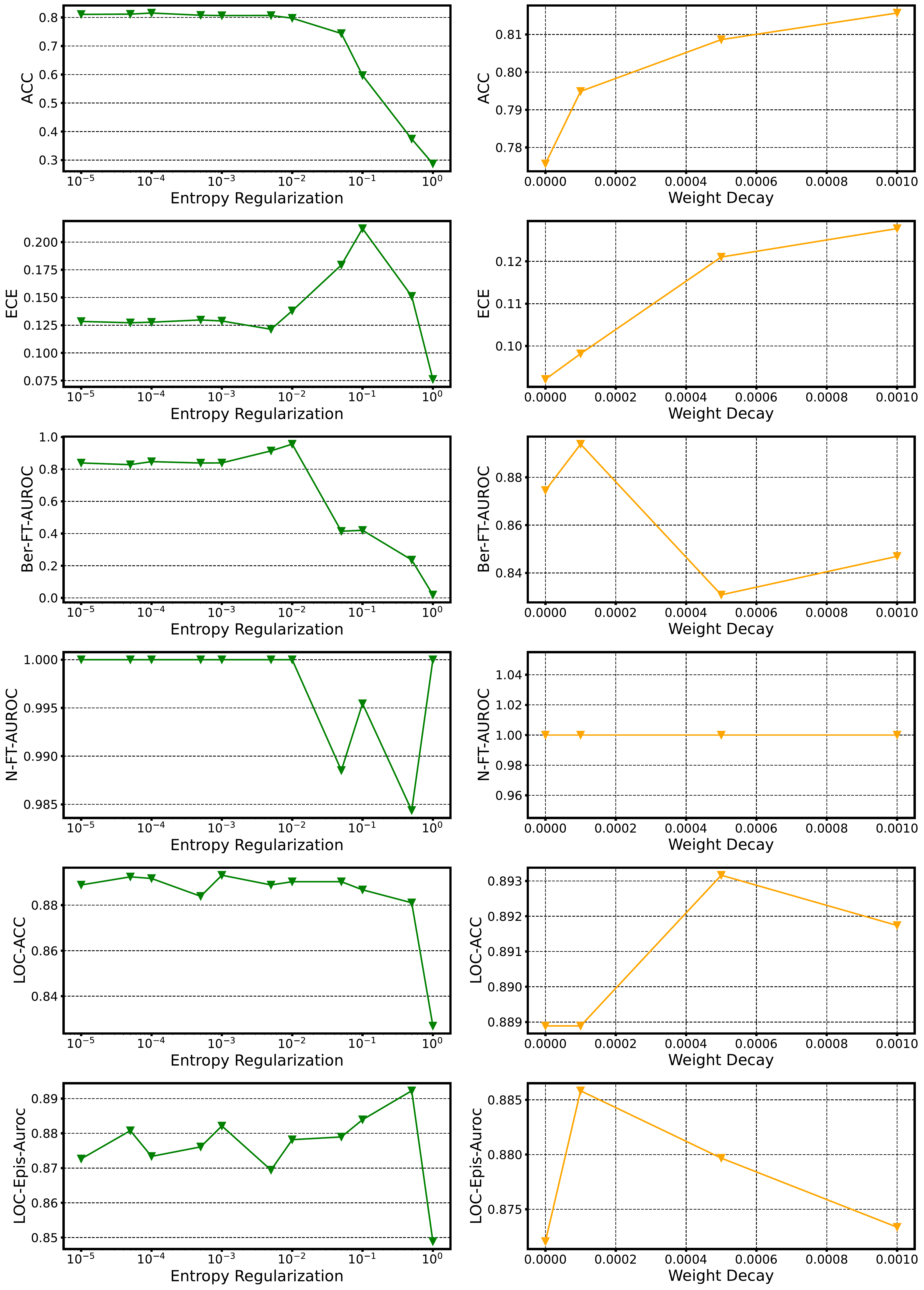}
    \caption{Accuracy, Calibration, and OOD-detection results of \oursacro{} on CoraML for the entropy regularization factor and the weight decay.}
    \label{fig:gs-regularization}
\end{figure}

\subsection{Additional Experiments - Misclassification}
\label{sec:add-exp-misclassification}

Work like \citep{Zhao2020} found that measures of aleatoric uncertainty are well suited for detecting inputs which are not classified correctly. Epistemic measures of uncertainty are found to be better suited for detecting OOD samples. Orthogonal work like \citep{Guo2017, Ovadia2019} uses calibration to determine how reliable predictions are. We adopt this point of view while reporting results for the task of OOD detection using epistemic measures. We also present aleatoric measures as a reference because simple baselines without any intrinsic uncertainty estimates solely can quantify uncertainty through simple aleatoric uncertainty measures. To facilitate an easy comparison to work like \citep{Zhao2020}, we also present results for misclassification experiments in Tab.~\ref{tab:misclassification-one} and in Tab.~\ref{tab:misclassification-two}. As in \citep{Zhan2020}, we can observe that aleatoric uncertainty mostly is better for detecting misclassified samples for all models than epistemic uncertainty. \oursacro{} achieves a similar performance in this task while also showing that for some datasets measures not accounting for network effects can detect misclassified samples better than measures which account for network effects. 

\input{tables/misclassification}

\subsection{Additional Experiments - OOD Detection}

In this section, we provide additional results for the OOD detection experiments for the 8 datasets and the 13 baselines. We show the results for feature perturbations experiments using AUC-ROC and AUC-APR scores in Tab.~\ref{tab:isolated_auroc} and Tab.~\ref{tab:isolated_apr}. We show results for clean accuracy and calibration on the unperturbed graphs, and Left-Out classes using AUC-ROC and AUC-APR scores in Tab.~\ref{tab:clean_loc_auroc_one} and Tab.~\ref{tab:clean_loc_auroc_two}. Similarly to Tab.~\ref{tab:ood_short}, we observe that \oursacro{} achieves the best results for the detection of feature perturbations with uncertainty without network effects by a significant margin while still maintaining a high accuracy. Further, \oursacro{} also performs favourably on Left-Out classes experiments using the uncertainty measures with network effects. All these observations show that \oursacro{} disentangles well between uncertainty without and with network effects while being robust against feature perturbations.

\input{tables/isolated}
\input{tables/loc}

\subsection{Additional Experiments - Attributed Graph Shifts}

In this section, we provide additional results for the attributed graph shifts experiments. We show the results of the feature shifts and the results of the edge shifts in Figures \ref{fig:shift-cora} to \ref{fig:shift-ogbn-arxiv}. The feature shifts include Bernoulli and Normal shifts (i.e. $\x\nodeidxv \sim \DBer(0.5)$ and $\x\nodeidxv \sim \DNormal(\veczero, \vecone)$) with up to $99\%$ of the nodes being perturbed. The edge shifts include randomly moved edges and DICE attacks \citep{Waniek2018} where we perturb up to $99\%$ of the edges in the graph. The results are consistent with the observations made in Sec.~\ref{sec:experiments}. For feature shifts, we observe that \oursacro{} is more robust to feature perturbations than competitors for accuracy and calibration similarly to results Tab.~\ref{fig:shifts-normal-cora}. \oursacro{} is particularly robust against unit Gaussians perturbations. Further, as desired, \oursacro{} is more epistemically uncertain when features of a larger fraction of nodes are perturbed. For edge shifts, all models including \oursacro{} become more aleatorically uncertain. This aligns with Ax.~\ref{ax:certainty_network_aleatoric}. Furthermore, the average epistemic uncertainty of \oursacro{} remains constant which is reasonable since the average evidence of a node's neighborhood remains constant. We observed that the exponential activation in the last layer of GKDE-GCN leads to numerical instabilities under perturbations.
\input{tables/shifts-cora}
\input{tables/shifts-citeseer}
\input{tables/shifts-pubmed}
\input{tables/shifts-amazon-computers}
\input{tables/shifts-amazon-photos}
\input{tables/shifts-coauthor-cs}
\input{tables/shifts-coauthor-physics}
\input{tables/shifts-ogbn-arxiv}

\subsection{Additional Experiments - Qualitative Evaluation}
\label{sec:add-exp-qualitative}

In this section, we provide additional qualitative evaluations. Therefore, we present the abstracts of the most epistemically uncertain papers and the most epistemically certain papers in CoraML in Tab.~\ref{tab:lowest_evidence_abstracts} and Tab.~\ref{tab:highest_evidence_abstracts}. Most epistemically certain papers shows significantly more reasonable abstracts compared to most epistemically uncertain papers.

Additionally, we provide visualizations of the latent space of \oursacro{} on the clean CoraML graph in Fig.~\ref{fig:latent-space-clean}, and on the CoraML graph where $10\%$ of the nodes are perturbed with  the unit Gaussian perturbation in Fig.~\ref{fig:latent-space-clean}, and on Left-Out classes experiments in Fig.~\ref{fig:loc-space-gaussian}. We used T-SNE projections for 2D visualizations. We observed that the latent representations correlate with the class assignment. Further, \oursacro{} is capable to separate nodes with perturbed features in the latent space. The nodes with perturbed features are assigned high uncertainty without network effects but low uncertainty with network effects. This stresses the capacity of \oursacro{} to recover from feature perturbations. 
\input{tables/lowest-evidence-abstracts}
\input{tables/highest-evidence-abstracts}
\input{tables/latent-space}

\subsection{Additional Experiments - Inference \& Training Time}
\label{sec:add-exp-time}

We provide a comparison of inference times for most of the datasets and models under consideration in Tab.~\ref{tab:inference-times} and a comparison of training times in Tab.~\ref{tab:training-times}. \oursacro{} needs a single pass for uncertainty estimation but requires the additional evaluation of one normalizing flow per class compared to APPNP. Hence, \oursacro{} brings a small computational overhead for uncertainty estimation during inference but is significantly faster than ensemble or dropout approaches. Furthermore, \oursacro{} is usually converging relatively fast during training and does not require a pre-computing kernel values. In contrast, GKDE-GCN requires the computation of the underlying Graph Kernel with a complexity of $\mathcal{O}(\nnodes^2)$ where $\nnodes$ is the number of nodes in the graph (see Sec.~\ref{sec:app_model_details}). Finally, \oursacro{} is significantly more efficient than dropout or ensemble approaches as it does not require training or evaluation of multiple models.

\input{tables/inference_times}
\input{tables/training_times}

\section{Axioms Diagram}

We provide a larger version of Figure \ref{fig:uncertainty_types_small} to visualize the distinction between aleatoric and epistemic uncertainty and the distinction between uncertainty without and with network effects in Fig.~\ref{fig:uncertainty_types_large}. These two distinctions are used in the axioms in Sec.~\ref{sec:axioms}.

\begin{figure}[!h]
\centering
	\begin{subfigure}[t]{0.495\textwidth}
	    \centering
		\includegraphics[width=\textwidth]{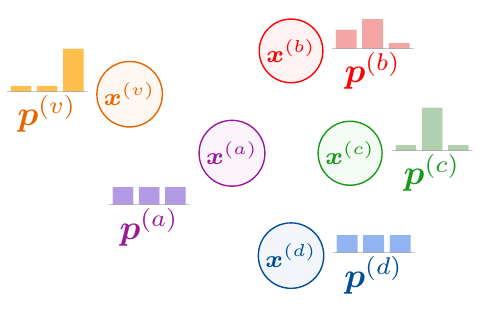}
		\caption{AU without network effects} 
		\label{subfig:au_without_network_large}
	\end{subfigure}
	\begin{subfigure}[t]{0.495\textwidth}
	    \centering
		\includegraphics[width=\textwidth]{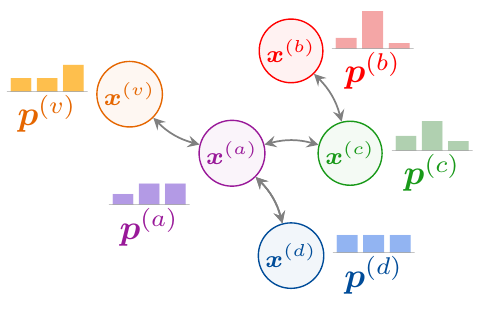}
		\caption{AU with network effects} 
		\label{subfig:au_with_network_large}
	\end{subfigure}
	\begin{subfigure}[t]{0.495\textwidth}
	    \centering
		\includegraphics[width=\textwidth]{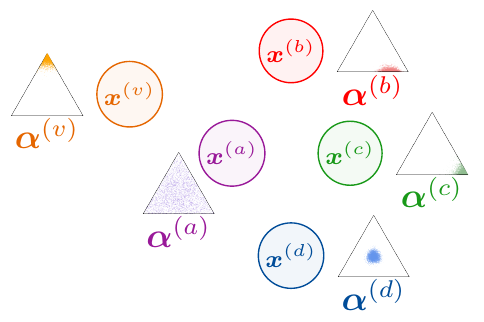}
		\caption{EU without network effects}
		\label{subfig:eu_without_network_large}
	\end{subfigure}
	\begin{subfigure}[t]{0.495\textwidth}
	    \centering
		\includegraphics[width=\textwidth]{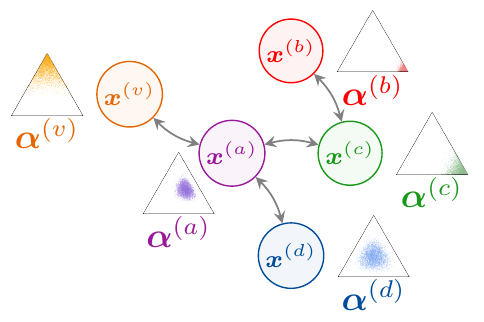}
		\caption{EU with network effects}
		\label{subfig:eu_with_network_large}
	\end{subfigure}
	\caption{Illustration of aleatoric uncertainty (AU) and epistemic uncertainty (EU) without and with network effects (i.e. i.i.d.\ inputs vs interdependent inputs). Each node of one color has the same features in all four cases. Network effects are visualized through edges between nodes which change the predicted distributions. The aleatoric uncertainty is high if the categorical distribution $\hat{\y} \nodeidxv \sim \DCat(\p\nodeidxv)$ is flat. The epistemic uncertainty is high if the Dirichlet distribution $\p\nodeidxv \sim \DDir(\valpha\nodeidxv)$ is spread out. Note that node with high epistemic certainty in the absence of network effects (e.g. orange) get less certain with neighbors being epistemically uncertain (purple). Epistemically uncertain nodes (purple) get more certain with certain neighbors on the other hand. Similar effects are shown for aleatoric uncertainty. For more details behind that reasoning, see our axiomatic approach in Sec.~\ref{sec:axioms}.}
    \label{fig:uncertainty_types_large}
\end{figure}

\clearpage

%% file: tables/dataset-summary.tex
\begin{table*}[!h]
    \centering
    \resizebox{\textwidth}{!}{
    \begin{tabular}{lccccccccc}
        \toprule
        & \textbf{CoraML} & \textbf{CiteSeer} & \textbf{PubMed} & \textbf{Amazon Computers} & \textbf{Amazon Photos} & \textbf{Coauthor CS} & \textbf{Coauthor Physics} & \textbf{OGBN - Arxiv} \\
        \midrule
        vertices &              2,995       & 3,327     & 19,717    & 13,752    & 7,650     & 18,333 &  34,493 & 169,343 \\
        edges                   & 16,316    & 9,104     & 88,648    & 491,722   & 238,162   &163,788 & 495,924 & 2,315,598\\
        homophily               & 78.86\%   & 73.55\%   & 80.24\%   & 77.22\%   & 82.72\% & 80.81\% & 93.15\% & 65.42\% \\ 
        feature dimension       & 2,879     & 3,703     & 500       & 767       & 745       & 6,805  & 8,415 & 128\\
        max words               & 176       & 54        & 122       & 767       & 745       & 666    & 335 & $N/A$\\
        mean words              & 50.47     & 31.61     & 50.11     & 267.24    & 258.81    & 59.57  & 32.97 & $N/A$\\
        median words            & 49        & 32        & 50        & 204       & 193       & 45     & 27 & $N/A$ \\
        classes                 & 7         & 6         & 3         & 10        & 8         & 15     & 5 & 40 \\
        left-out classes        & 3         & 2         & 1         & 5         & 3         & 4      & 2 & 15 \\
        fraction left-out       & 44.91\%     & 33.18\%     & 39.94\%     & 29.52\%     & 40.46\%     & 40.72\%  & 18.18\% & 39.11\% \\
        \bottomrule
    \end{tabular}}
    \caption{Dataset details summarizing the graph size, the homophily ratio (fraction of intra-class edges), the number of classes, and statistics on the Bag-Of-Word features when available. In particular, \emph{OGBN-Arxiv} uses averaged skip-gram embeddings for the nodes' features and thus does not require bag-of-word-features. Further, we provide the number of left-out classes and the fraction of left-out nodes for the LOC experiments.}
    \label{tab:dataset-summary}
\end{table*}

%% file: tables/gpn-hyper-parameters.tex
\begin{table}[!h]
    \centering
    \resizebox{\textwidth}{!}{
        \begin{tabular}{l|cccccccccc}
            \toprule
             & $H$  & $L$   &   $n_{layers}$    & $n_{radial}$  &   $ACT$   &   $p_{drop}$ &    $\tau$ & budget & $\lambda$   &   weight decay \\
             \midrule
             \textbf{(1)} & $64$ & $10$ & $2$ & $10$ & ReLU & $0.5$ & $0.2$ & $N_H \cdot C$ & $1.0e-05$ & $0.0005$ \\
             \textbf{(2)} & $64$ & $16$ & $2$ & $10$ & ReLU & $0.5$ & $0.1$ & $N_H$ & $1.0e-3$ & $0.001$ \\
             \textbf{(3)} & $256$ & $16$ & $3$ & $10$ & ReLU + BN & $0.25$ & $0.2$ & $N_H$ & $1.0e-3$ & $0.0$\\
             \bottomrule
        \end{tabular}
    }
        \vspace{2mm}
        \caption{\oursacro{} hyperparameters used in our experiments. (1) is used for \emph{Amazon Photos} and \emph{Amazon Computers} datasets, (3) is used for \emph{OGBN-Arxiv} dataset and (2) is used for all other datasets. For all those settings, we furthermore use $K=10$ power-iterations steps, 5 warmup epochs for the Normalizing Flows, and no weight decay for the normalizing flows.}
        \label{tab:gpn-hyper-parameters}
\end{table}

%% file: tables/ablation-study.tex
\begin{table}[!ht]
    \centering
    \resizebox{\textwidth}{!}{
        \begin{tabular}{lcccccc}
            \toprule
            \textbf{Model} &  \textbf{ACC} &  \textbf{ECE} &  \textbf{$\DBer$-FT-AUC-ROC} &  \textbf{$\DNormal$-FT-AUC-ROC} &  \textbf{LOC-ID-ACC} & \textbf{LOC-EPIS-AUC-ROC} \\
            \midrule
            APPNP & $82.18\pm 0.08$ & $\hphantom{0}8.46\pm 0.09$ & $60.51\pm0.06$ & $\hphantom{0}16.32\pm0.31$ & $88.71\pm0.03$ & $84.75\pm0.06$ \\
            PostNet & $52.24 \pm 0.90$  &  $\hphantom{0}8.22 \pm 1.39$  &  $83.09 \pm 4.41$  &  $100.00 \pm 0.00$  &  $69.37 \pm 0.50$  &  $70.14 \pm 0.93$ \\ 
            PostNet-Diff & $77.10 \pm 0.58$  &  $30.35 \pm 0.86$  &  $83.09 \pm 4.41$  &  $100.00 \pm 0.00$  &  $86.18 \pm 0.65$  &  $78.45 \pm 0.80$ \\ 
            GPN & $81.57 \pm 0.18$  &  $12.77 \pm 1.11$  &  $84.69 \pm 4.60$  &  $100.00 \pm 0.00$  &  $89.17 \pm 0.39$  &  $87.34 \pm 0.76$ \\ 
            GPN-LOG-$\beta$ & $82.90 \pm 0.26$  &  $\hphantom{0}8.77 \pm 0.33$  &  $91.24 \pm 2.67$  &  $100.00 \pm 0.00$  &  $91.17 \pm 0.59$  &  $70.88 \pm 1.74$ \\ 
            \bottomrule
        \end{tabular}
    }
    \vspace{1em}
    \caption{Accuracy, Calibration, and OOD-detection results for the ablation study on CoraML for the validation split. All models use epistemic uncertainty measures except APPNP which uses an aleatoric measure.}
    \label{tab:ablation}
\end{table}

%% file: tables/grid-search.tex
\begin{table}
    \centering
    \resizebox{\textwidth}{!}{
        \begin{tabular}{lcccccc}
            \toprule
            \textbf{Latent Dim} &  \textbf{ACC} &  \textbf{ECE} &  \textbf{$\DBer$-FT-AUC-ROC} &  \textbf{$\DNormal$-FT-AUC-ROC} &  \textbf{LOC-ID-ACC} & \textbf{LOC-Epis-AUC-ROC} \\
            \midrule
            6 & $79.18 \pm 0.16$  &  $10.50 \pm 0.56$  &  $81.75 \pm 4.74$  &  $100.00 \pm 0.00$  &  $90.17 \pm 0.41$  &  $88.65 \pm 0.58$ \\ 
            10 & $79.14 \pm 0.44$  &  $10.96 \pm 0.23$  &  $77.92 \pm 4.44$  &  $100.00 \pm 0.00$  &  $89.39 \pm 0.58$  &  $87.39 \pm 0.43$ \\ 
            16 & $81.57 \pm 0.18$  &  $12.77 \pm 1.11$  &  $84.69 \pm 4.60$  &  $100.00 \pm 0.00$  &  $89.17 \pm 0.39$  &  $87.34 \pm 0.76$ \\ 
        \bottomrule
        \end{tabular}
    }
    \vspace{5mm}
    \caption{Accuracy, Calibration, and OOD-detection results of \oursacro{} on the CoraML dataset with latent dimensions in $[6, 10, 16]$.}
    \label{tab:gs-dim-latent}
\end{table}
\begin{table}
    \centering
    \resizebox{\textwidth}{!}{
        \begin{tabular}{lcccccc}
            \toprule
            \textbf{Radial Layers} &  \textbf{ACC} &  \textbf{ECE} &  \textbf{$\DBer$-FT-AUC-ROC} &  \textbf{$\DNormal$-FT-AUC-ROC} &  \textbf{LOC-ID-ACC} & \textbf{LOC-Epis-AUC-ROC} \\
            \midrule 
            6 & $80.04 \pm 0.85$  &  $11.26 \pm 0.48$  &  $84.26 \pm 2.38$  &  $100.00 \pm 0.00$  &  $90.74 \pm 0.32$  &  $86.62 \pm 0.25$ \\ 
            10 & $81.57 \pm 0.18$  &  $12.77 \pm 1.11$  &  $84.69 \pm 4.60$  &  $100.00 \pm 0.00$  &  $89.17 \pm 0.39$  &  $87.34 \pm 0.76$ \\ 
            16 & $75.88 \pm 0.95$  &  $\hphantom{0}8.77 \pm 0.32$  &  $94.06 \pm 1.37$  &  $100.00 \pm 0.00$  &  $89.39 \pm 0.20$  &  $84.14 \pm 0.12$ \\ 
            \bottomrule
        \end{tabular}
    }
    \vspace{5mm}
    \caption{Accuracy, Calibration, and OOD-detection results of \oursacro{} on the CoraML dataset with the number of radial layers in $[6, 10, 16]$.}
    \label{tab:gs-n-radial}
\end{table}
\begin{table}
    \centering
    \resizebox{\textwidth}{!}{
        \begin{tabular}{lcccccc}
            \toprule
            $\mathbf{\tau_\text{teleport}}$ & \textbf{ACC} &  \textbf{ECE} &  \textbf{$\DBer$-FT-AUC-ROC} &  \textbf{$\DNormal$-FT-AUC-ROC} &  \textbf{LOC-ID-ACC} & \textbf{LOC-Epis-AUC-ROC} \\
            \midrule
            0.05 & $80.39 \pm 0.33$  &  $11.98 \pm 0.53$  &  $85.31 \pm 2.79$  &  $100.00 \pm 0.00$  &  $89.32 \pm 0.41$  &  $86.63 \pm 0.65$ \\ 
            0.1 & $81.57 \pm 0.18$  &  $12.77 \pm 1.11$  &  $84.69 \pm 4.60$  &  $100.00 \pm 0.00$  &  $89.17 \pm 0.39$  &  $87.34 \pm 0.76$ \\ 
            0.2 & $80.00 \pm 0.30$  &  $\hphantom{0}8.27 \pm 0.59$  &  $86.65 \pm 1.46$  &  $100.00 \pm 0.00$  &  $88.89 \pm 0.31$  &  $87.30 \pm 0.30$ \\ 
            \bottomrule
        \end{tabular}
    }
    \vspace{5mm}
    \caption{Accuracy, Calibration, and OOD-detection results of \oursacro{} on the CoraML dataset  with a teleport probability $\tau$ in $[0.05, 0.1, 0.2]$.}
    \label{tab:gs-alpha-teleport}
\end{table}
\begin{table}
    \centering
    \resizebox{\textwidth}{!}{
        \begin{tabular}{lcccccc}
            \toprule
            \textbf{Budget} &  \textbf{ACC} &  \textbf{ECE} &  \textbf{$\DBer$-FT-AUC-ROC} &  \textbf{$\DNormal$-FT-AUC-ROC} &  \textbf{LOC-ID-ACC} & \textbf{LOC-Epis-AUC-ROC} \\
            \midrule
            $N_H$ & $81.57 \pm 0.18$  &  $12.77 \pm 1.11$  &  $84.69 \pm 4.60$  &  $100.00 \pm 0.00$  &  $89.17 \pm 0.39$  &  $87.34 \pm 0.76$ \\ 
            $N_H \cdot C$ & $79.76 \pm 0.93$  &  $11.24 \pm 0.89$  &  $82.82 \pm 5.68$  &  $100.00 \pm 0.00$  &  $88.68 \pm 0.75$  &  $86.39 \pm 1.27$ \\ 
            \bottomrule
        \end{tabular}
    }
    \vspace{5mm}
    \caption{Accuracy, Calibration, and OOD-detection results of \oursacro{} on the CoraML dataset with a budget that scales exponentially w.r.t. the latent dimension and a budget that scales exponentially w.r.t. the latent dimension and the number of classes.}
    \label{tab:gs-budget}
\end{table}

%% file: tables/misclassification.tex
\begin{table}[ht]
    \centering
    
    \resizebox{\textwidth}{!}{
    \begin{tabular}{ll|lll|lll}
    \toprule
             \textbf{Dataset} &          \textbf{Model} & \textbf{Alea w/ Net} & \textbf{Epis w/ Net} & \textbf{Epis w/o Net} & \textbf{Alea w/ Net} & \textbf{Epis w/ Net} & \textbf{Epis w/o Net} \\
            {} & {} & \multicolumn{3}{c|}{misclassification AUC-ROC} & \multicolumn{3}{c}{misclassification AUC-PR} \\
            \midrule
                
                \multirow{8}{*}{CoraML} 
               &          APPNP &    $\mathbf{83.64 \pm 0.08}$ &               $n.a.$ &                $n.a.$ &   $48.39 \pm 0.19$ &              $n.a.$ &               $n.a.$ \\
               &           VGCN &    $81.02 \pm 0.07$ &               $n.a.$ &                $n.a.$ &   $48.30 \pm 0.23$ &              $n.a.$ &               $n.a.$ \\
               &   VGCN-Dropout &    $81.42 \pm 0.06$ &    $65.52 \pm 0.28$ &                $n.a.$ &   $49.34 \pm 0.17$ &   $26.11 \pm 0.21$ &               $n.a.$ \\
               &    VGCN-Energy &    $81.02 \pm 0.07$ &               $n.a.$ &                $n.a.$ &   $48.30 \pm 0.23$ &              $n.a.$ &               $n.a.$ \\
               &  VGCN-Ensemble &    $81.12 \pm 0.17$ &    $72.62 \pm 0.19$ &                $n.a.$ &   $49.16 \pm 0.59$ &   $31.88 \pm 0.29$ &               $n.a.$ \\
               &       VGCN-BNN &    $80.64 \pm 0.10$ &    $65.40 \pm 0.48$ &                $n.a.$ &   $47.49 \pm 0.26$ &   $26.70 \pm 0.33$ &               $n.a.$ \\
               &        GKDE-GCN &    $80.80 \pm 0.14$ &    $76.83 \pm 0.17$ &                $n.a.$ &   $\mathbf{49.61 \pm 0.47}$ &   $45.87 \pm 0.48$ &               $n.a.$ \\
               &            GPN &    $81.19 \pm 0.13$ &    $78.10 \pm 0.26$ &     $77.46 \pm 0.24$ &   $49.51 \pm 0.26$ &   $44.42 \pm 0.32$ &    $43.31 \pm 0.41$ \\
            
            \midrule
            
            \multirow{8}{*}{CiteSeer} 
             &          APPNP &    $73.55 \pm 0.08$ &               $n.a.$ &                $n.a.$ &   $51.70 \pm 0.15$ &              $n.a.$ &               $n.a.$ \\
             &           VGCN &    $74.64 \pm 0.09$ &               $n.a.$ &                $n.a.$ &   $54.32 \pm 0.18$ &              $n.a.$ &               $n.a.$ \\
             &   VGCN-Dropout &    $74.81 \pm 0.11$ &    $64.09 \pm 0.16$ &                $n.a.$ &   $55.12 \pm 0.23$ &   $39.41 \pm 0.19$ &               $n.a.$ \\
             &    VGCN-Energy &    $74.64 \pm 0.09$ &               $n.a.$ &                $n.a.$ &   $54.32 \pm 0.18$ &              $n.a.$ &               $n.a.$ \\
             &  VGCN-Ensemble &    $74.42 \pm 0.26$ &    $68.15 \pm 0.23$ &                $n.a.$ &   $53.28 \pm 0.33$ &   $43.30 \pm 0.30$ &               $n.a.$ \\
             &       VGCN-BNN &    $73.28 \pm 0.11$ &    $61.68 \pm 0.29$ &                $n.a.$ &   $54.62 \pm 0.21$ &   $37.99 \pm 0.24$ &               $n.a.$ \\
             &        GKDE-GCN &    $75.45 \pm 0.11$ &    $73.83 \pm 0.12$ &                $n.a.$ &   $54.78 \pm 0.19$ &   $53.57 \pm 0.20$ &               $n.a.$ \\
             &            GPN &    $\mathbf{75.89 \pm 0.15}$ &    $74.16 \pm 0.17$ &     $72.50 \pm 0.12$ &   $\mathbf{60.78 \pm 0.32}$ &   $59.32 \pm 0.40$ &    $52.10 \pm 0.18$ \\
              
              \midrule
              
              \multirow{8}{*}{PubMed} 
               &          APPNP &    $80.98 \pm 0.02$ &               $n.a.$ &                $n.a.$ &   $37.79 \pm 0.08$ &              $n.a.$ &               $n.a.$ \\
               &           VGCN &    $81.16 \pm 0.02$ &               $n.a.$ &                $n.a.$ &   $38.24 \pm 0.08$ &              $n.a.$ &               $n.a.$ \\
               &   VGCN-Dropout &    $80.46 \pm 0.04$ &    $72.69 \pm 0.09$ &                $n.a.$ &   $38.63 \pm 0.11$ &   $25.90 \pm 0.09$ &               $n.a.$ \\
               &    VGCN-Energy &    $81.16 \pm 0.02$ &               $n.a.$ &                $n.a.$ &   $38.24 \pm 0.08$ &              $n.a.$ &               $n.a.$ \\
               &  VGCN-Ensemble &    $\mathbf{81.31 \pm 0.06}$ &    $79.30 \pm 0.08$ &                $n.a.$ &   $38.06 \pm 0.33$ &   $31.73 \pm 0.19$ &               $n.a.$ \\
               &       VGCN-BNN &    $79.96 \pm 0.07$ &    $72.63 \pm 0.45$ &                $n.a.$ &   $39.31 \pm 0.08$ &   $27.59 \pm 0.39$ &               $n.a.$ \\
               &        GKDE-GCN &    $80.95 \pm 0.09$ &    $73.99 \pm 0.27$ &                $n.a.$ &   $39.64 \pm 0.10$ &   $33.19 \pm 0.14$ &               $n.a.$ \\
               &            GPN &    $80.46 \pm 0.13$ &    $75.38 \pm 0.25$ &     $80.48 \pm 0.13$ &   $40.74 \pm 0.19$ &   $35.11 \pm 0.11$ &    $\mathbf{51.12 \pm 0.52}$ \\
     
    \bottomrule
    \end{tabular}}
    \vspace{1em}
    \caption{Misclassification Scores on the clean graphs given as AUC-ROC and AUC-PR scores. AUC-ROC and AUC-APR scores are given as \emph{[Alea w/ Net] / [Epist w/ Net] / [Epist w/o Net]}. $n.a.$ means either model or metric not applicable. Bold numbers indicate the best model across all the uncertainty metrics for each dataset.}
    \label{tab:misclassification-one}
\end{table}

\begin{table}[ht]
    \centering
    
    \resizebox{\textwidth}{!}{
    \begin{tabular}{ll|lll|lll}
    \toprule
             \textbf{Dataset} &          \textbf{Model} & \textbf{Alea w/ Net} & \textbf{Epis w/ Net} & \textbf{Epis w/o Net} & \textbf{Alea w/ Net} & \textbf{Epis w/ Net} & \textbf{Epis w/o Net} \\
             {} & {} & \multicolumn{3}{c|}{misclassification AUC-ROC} & \multicolumn{3}{c}{misclassification AUC-PR} \\
    
            \midrule

        \multirow{8}{*}{\shortstack[l]{Amazon\\Computers}}
      &          APPNP &    $79.75 \pm 0.03$ &               $n.a.$ &                $n.a.$ &   $45.10 \pm 0.11$ &              $n.a.$ &               $n.a.$ \\
      &           VGCN &    $82.08 \pm 0.03$ &               $n.a.$ &                $n.a.$ &   $45.52 \pm 0.12$ &              $n.a.$ &               $n.a.$ \\
      &   VGCN-Dropout &    $\mathbf{82.70 \pm 0.04}$ &    $72.02 \pm 0.11$ &                $n.a.$ &   $47.19 \pm 0.17$ &   $31.45 \pm 0.14$ &               $n.a.$ \\
      &    VGCN-Energy &    $82.08 \pm 0.03$ &               $n.a.$ &                $n.a.$ &   $45.52 \pm 0.12$ &              $n.a.$ &               $n.a.$ \\
      &  VGCN-Ensemble &    $82.05 \pm 0.06$ &    $67.62 \pm 0.44$ &                $n.a.$ &   $45.40 \pm 0.24$ &   $26.33 \pm 0.37$ &               $n.a.$ \\
      &       VGCN-BNN &    $82.15 \pm 0.17$ &    $48.65 \pm 0.82$ &                $n.a.$ &   $\mathbf{69.61 \pm 0.42}$ &   $30.87 \pm 0.43$ &               $n.a.$ \\
      &        GKDE-GCN &    $79.66 \pm 0.19$ &    $73.66 \pm 0.15$ &                $n.a.$ &   $63.26 \pm 0.57$ &   $56.93 \pm 0.54$ &               $n.a.$ \\
      &            GPN &    $82.20 \pm 0.10$ &    $77.58 \pm 0.16$ &     $70.06 \pm 0.19$ &   $47.93 \pm 0.42$ &   $41.80 \pm 0.44$ &    $28.70 \pm 0.51$ \\
        
        \midrule
        
        \multirow{8}{*}{\shortstack[l]{Amazon\\Photos}}
         &          APPNP &    $85.74 \pm 0.06$ &               $n.a.$ &                $n.a.$ &   $37.00 \pm 0.20$ &              $n.a.$ &               $n.a.$ \\
         &           VGCN &    $87.94 \pm 0.05$ &               $n.a.$ &                $n.a.$ &   $48.35 \pm 0.19$ &              $n.a.$ &               $n.a.$ \\
         &   VGCN-Dropout &    $\mathbf{89.52 \pm 0.05}$ &    $78.46 \pm 0.10$ &                $n.a.$ &   $50.27 \pm 0.19$ &   $23.08 \pm 0.12$ &               $n.a.$ \\
         &    VGCN-Energy &    $87.94 \pm 0.05$ &               $n.a.$ &                $n.a.$ &   $48.35 \pm 0.19$ &              $n.a.$ &               $n.a.$ \\
         &  VGCN-Ensemble &    $88.08 \pm 0.16$ &    $76.05 \pm 0.54$ &                $n.a.$ &   $49.40 \pm 0.67$ &   $24.13 \pm 0.67$ &               $n.a.$ \\
         &       VGCN-BNN &    $84.17 \pm 0.19$ &    $51.84 \pm 0.66$ &                $n.a.$ &   $51.05 \pm 0.64$ &   $18.69 \pm 0.55$ &               $n.a.$ \\
         &        GKDE-GCN &    $84.11 \pm 0.29$ &    $75.07 \pm 0.50$ &                $n.a.$ &   $\mathbf{54.35 \pm 0.58}$ &   $45.43 \pm 0.68$ &               $n.a.$ \\
         &            GPN &    $87.21 \pm 0.10$ &    $83.38 \pm 0.30$ &     $79.93 \pm 0.23$ &   $46.32 \pm 0.39$ &   $37.07 \pm 0.60$ &    $29.90 \pm 0.71$ \\
          
          \midrule
          
          \multirow{8}{*}{\shortstack[l]{Coauthor\\CS}}
           & APPNP &    ${89.92 \pm 0.03}$ &               $n.a.$ & $n.a.$ &   $37.98 \pm 0.12$ &              $n.a.$ &               $n.a.$ \\
           & VGCN &    $89.46 \pm 0.03$ &               $n.a.$ &                $n.a.$ &   $38.86 \pm 0.10$ &              $n.a.$ &               $n.a.$ \\
           & VGCN-Dropout &    $88.46 \pm 0.04$ &    $79.03 \pm 0.08$ &                $n.a.$ &   $38.06 \pm 0.12$ &   $17.98 \pm 0.09$ &               $n.a.$ \\
           & VGCN-Energy &    $89.46 \pm 0.03$ &               $n.a.$ &                $n.a.$ &   $38.86 \pm 0.10$ &              $n.a.$ &               $n.a.$ \\
           & VGCN-Ensemble &    $89.51 \pm 0.08$ &    $86.61 \pm 0.09$ &                $n.a.$ &   $38.74 \pm 0.23$ &   $30.60 \pm 0.38$ &               $n.a.$ \\
           & VGCN-BNN &    $89.01 \pm 0.06$ &    $78.40 \pm 0.23$ &                $n.a.$ &   $38.17 \pm 0.17$ &   $19.06 \pm 0.20$ &               $n.a.$ \\
           & GKDE-GCN &    $89.24 \pm 0.05$ &    $80.98 \pm 0.13$ &                $n.a.$ &   $39.30 \pm 0.27$ &   $30.52 \pm 0.25$ &               $n.a.$ \\
           & GPN &    $85.72 \pm 0.15$ &    $81.56 \pm 0.29$ &     $\mathbf{94.41 \pm 0.11}$ &   $46.12 \pm 0.32$ &   $38.98 \pm 0.28$ &    $\mathbf{77.26 \pm 0.45}$ \\
           
           \midrule
     
        \multirow{8}{*}{\shortstack[l]{Coauthor\\Physics}}
      &          APPNP &    $93.27 \pm 0.02$ &               $n.a.$ &                $n.a.$ &   $38.14 \pm 0.09$ &              $n.a.$ &               $n.a.$ \\
      &           VGCN &    $92.86 \pm 0.02$ &               $n.a.$ &                $n.a.$ &   $37.19 \pm 0.10$ &              $n.a.$ &               $n.a.$ \\
      &   VGCN-Dropout &    $92.28 \pm 0.03$ &    $89.85 \pm 0.04$ &                $n.a.$ &   $35.47 \pm 0.11$ &   $23.70 \pm 0.08$ &               $n.a.$ \\
      &    VGCN-Energy &    $92.86 \pm 0.02$ &               $n.a.$ &                $n.a.$ &   $37.19 \pm 0.10$ &              $n.a.$ &               $n.a.$ \\
      &  VGCN-Ensemble &    $92.95 \pm 0.07$ &    $91.92 \pm 0.07$ &                $n.a.$ &   $37.96 \pm 0.28$ &   $28.44 \pm 0.21$ &               $n.a.$ \\
      &       VGCN-BNN &    $92.44 \pm 0.09$ &    $89.03 \pm 0.27$ &                $n.a.$ &   $36.79 \pm 0.21$ &   $25.11 \pm 0.49$ &               $n.a.$ \\
      &        GKDE-GCN &    $92.77 \pm 0.02$ &    $86.12 \pm 0.06$ &                $n.a.$ &   $37.08 \pm 0.11$ &   $25.13 \pm 0.10$ &               $n.a.$ \\
      &            GPN &    $91.14 \pm 0.04$ &    $89.63 \pm 0.07$ &     $\mathbf{93.89 \pm 0.05}$ &   $41.43 \pm 0.13$ &   $35.64 \pm 0.16$ &    $\mathbf{59.03 \pm 0.28}$ \\
          
          \midrule
          
          \multirow{7}{*}{\shortstack[l]{OGBN\\Arxiv}}
           &          APPNP &    $77.55 \pm 0.05$ &               $n.a.$ &                $n.a.$ &   $54.57 \pm 0.14$ &              $n.a.$ &               $n.a.$ \\
           &           VGCN &    $77.89 \pm 0.05$ &               $n.a.$ &                $n.a.$ &   $54.87 \pm 0.12$ &              $n.a.$ &               $n.a.$ \\
           &   VGCN-Dropout &    $78.11 \pm 0.05$ &    $71.74 \pm 0.14$ &                $n.a.$ &   $55.40 \pm 0.13$ &   $43.43 \pm 0.18$ &               $n.a.$ \\
           &    VGCN-Energy &    $77.89 \pm 0.05$ &               $n.a.$ &                $n.a.$ &   $54.87 \pm 0.12$ &              $n.a.$ &               $n.a.$ \\
           &  VGCN-Ensemble &             $\mathbf{78.14}$ &             $71.48$ &                $n.a.$ &            $53.95$ &            $42.87$ &               $n.a.$ \\
           &        GKDE-GCN &    $77.47 \pm 0.33$ &    $77.55 \pm 0.33$ &                $n.a.$ &   $\mathbf{61.62 \pm 1.00}$ &   $62.33 \pm 1.00$ &               $n.a.$ \\
           &            GPN &    $75.44 \pm 0.19$ &    $72.71 \pm 0.28$ &     $61.45 \pm 0.49$ &   $55.64 \pm 0.37$ &   $52.99 \pm 0.49$ &    $39.37 \pm 0.42$ \\
    \bottomrule
    \end{tabular}}
    \vspace{1em}
    \caption{Misclassification Scores on the clean graphs given as AUC-ROC and AUC-PR scores. AUC-ROC and AUC-APR scores are given as \emph{[Alea w/ Net] / [Epist w/ Net] / [Epist w/o Net]}. $n.a.$ means either model or metric not applicable. Bold numbers indicate the best model across all the uncertainty metrics for each dataset.}
    \label{tab:misclassification-two}
\end{table}

%% file: tables/isolated.tex
\begin{table*}[!h]
    \centering
    \resizebox{\textwidth}{!}{
    \begin{tabular}{ll|cccc|cccc}
        \toprule
        & & \multicolumn{4}{c|}{{$\x\nodeidxv \sim \DBer(0.5)$}} & \multicolumn{4}{c}{$\x\nodeidxv \sim \DNormal(0,1)$} \\ \midrule
        & \textbf{Model} & \textbf{OOD-ACC} & \multicolumn{3}{c|}{\textbf{OOD-AUC-ROC}} & \textbf{OOD-ACC} & \multicolumn{3}{c}{\textbf{OOD-AUC-ROC}} \\
        & & & \emph{Alea w/ Net} & \emph{Epist w/ Net} & \emph{Epist w/o Net} & & \emph{Alea w/ Net} & \emph{Epist w/ Net} & \emph{Epist w/o Net} \\
        \midrule
        
        \multirow{10}{*}{CoraML}
        & APPNP & ${80.85\tiny \pm 0.09}$ & ${64.41\tiny \pm 0.02}$ & $n.a.$ & $n.a.$ & ${17.99\tiny \pm 0.36}$ & ${7.98\tiny \pm 0.13}$ & $n.a.$ & $n.a.$\\
        & VGCN & ${78.90\tiny \pm 0.09}$ & ${63.68\tiny \pm 0.03}$ & $n.a.$ & $n.a.$ & ${18.37\tiny \pm 0.31}$ & ${9.34\tiny \pm 0.13}$ & $n.a.$ & $n.a.$\\
        & RGCN & ${79.78\tiny \pm 0.16}$ & ${{70.30}\tiny \pm 0.05}$ & $n.a.$ & $n.a.$ & ${33.37\tiny \pm 0.35}$ & ${32.13\tiny \pm 0.28}$ & $n.a.$ & $n.a.$\\
        & VGCN-Dropout & ${77.76\tiny \pm 0.15}$ & ${62.06\tiny \pm 0.06}$ & ${50.38\tiny \pm 0.12}$ & $n.a.$ & ${18.28\tiny \pm 0.35}$ & ${40.53\tiny \pm 0.25}$ & ${{71.06}\tiny \pm 0.29}$ & $n.a.$\\
        & DropEdge & ${77.40\tiny \pm 0.14}$ & ${63.10\tiny \pm 0.04}$ & ${52.84\tiny \pm 0.10}$ & $n.a.$ & ${16.60\tiny \pm 0.26}$ & ${23.10\tiny \pm 0.29}$ & ${46.82\tiny \pm 0.41}$ & $n.a.$\\
        & VGCN-Energy & ${78.90\tiny \pm 0.09}$ & ${63.68\tiny \pm 0.03}$ & ${{66.26}\tiny \pm 0.04}$ & $n.a.$ & ${18.37\tiny \pm 0.31}$ & ${9.34\tiny \pm 0.13}$ & ${0.32\tiny \pm 0.03}$ & $n.a.$\\
        & VGCN-Ensemble & ${78.00\tiny \pm 0.00}$ & ${63.58\tiny \pm 0.00}$ & ${56.81\tiny \pm 0.03}$ & $n.a.$ & ${21.00\tiny \pm 0.00}$ & ${33.72\tiny \pm 0.02}$ & ${64.92\tiny \pm 0.08}$ & $n.a.$\\
        & VGCN-BNN & ${77.01\tiny \pm 0.16}$ & ${64.74\tiny \pm 0.07}$ & ${62.45\tiny \pm 0.37}$ & $n.a.$ & ${18.79\tiny \pm 0.31}$ & ${34.85\tiny \pm 0.50}$ & ${67.43\tiny \pm 0.71}$ & $n.a.$\\
        & GKDE-GCN & ${76.40\tiny \pm 0.33}$ & ${61.74\tiny \pm 0.05}$ & ${63.15\tiny \pm 0.10}$ & $n.a.$ & ${16.86\tiny \pm 0.35}$ & ${40.03\tiny \pm 0.46}$ & ${1.42\tiny \pm 0.15}$ & $n.a.$\\
        & GPN & ${\mathbf{80.98}\tiny \pm 0.22}$ & ${57.99\tiny \pm 0.12}$ & ${55.23\tiny \pm 0.16}$ & ${\mathbf{89.47}\tiny \pm 0.86}$ & ${\mathbf{81.53}\tiny \pm 0.23}$ & ${55.96\tiny \pm 0.08}$ & ${56.51\tiny \pm 0.21}$ & ${\mathbf{100.00}\tiny \pm 0.00}$\\

        \midrule
        
        \multirow{10}{*}{CiteSeer}
        & APPNP & ${\mathbf{73.14}\tiny \pm 0.12}$ & ${65.43\tiny \pm 0.02}$ & $n.a.$ & $n.a.$ & ${20.13\tiny \pm 0.22}$ & ${4.78\tiny \pm 0.11}$ & $n.a.$ & $n.a.$\\
        & VGCN & ${71.30\tiny \pm 0.13}$ & ${65.27\tiny \pm 0.03}$ & $n.a.$ & $n.a.$ & ${17.55\tiny \pm 0.36}$ & ${5.48\tiny \pm 0.11}$ & $n.a.$ & $n.a.$\\
        & RGCN & ${72.29\tiny \pm 0.09}$ & ${{71.99}\tiny \pm 0.04}$ & $n.a.$ & $n.a.$ & ${28.15\tiny \pm 0.40}$ & ${23.28\tiny \pm 0.41}$ & $n.a.$ & $n.a.$\\
        & VGCN-Dropout & ${69.80\tiny \pm 0.19}$ & ${63.47\tiny \pm 0.09}$ & ${51.82\tiny \pm 0.12}$ & $n.a.$ & ${19.60\tiny \pm 0.28}$ & ${31.79\tiny \pm 0.27}$ & ${{72.62}\tiny \pm 0.34}$ & $n.a.$\\
        & DropEdge & ${72.00\tiny \pm 0.23}$ & ${65.00\tiny \pm 0.05}$ & ${54.71\tiny \pm 0.16}$ & $n.a.$ & ${18.00\tiny \pm 0.47}$ & ${17.80\tiny \pm 0.25}$ & ${44.78\tiny \pm 0.52}$ & $n.a.$\\
        & VGCN-Energy & ${71.30\tiny \pm 0.13}$ & ${65.27\tiny \pm 0.03}$ & ${{68.16}\tiny \pm 0.06}$ & $n.a.$ & ${17.55\tiny \pm 0.36}$ & ${5.48\tiny \pm 0.11}$ & ${0.03\tiny \pm 0.01}$ & $n.a.$\\
        & VGCN-Ensemble & ${72.00\tiny \pm 0.00}$ & ${65.20\tiny \pm 0.00}$ & ${51.81\tiny \pm 0.01}$ & $n.a.$ & ${18.00\tiny \pm 0.00}$ & ${21.22\tiny \pm 0.01}$ & ${52.80\tiny \pm 0.02}$ & $n.a.$\\
        & VGCN-BNN & ${70.38\tiny \pm 0.15}$ & ${65.52\tiny \pm 0.09}$ & ${49.33\tiny \pm 0.70}$ & $n.a.$ & ${16.27\tiny \pm 0.33}$ & ${23.24\tiny \pm 0.86}$ & ${60.07\tiny \pm 1.69}$ & $n.a.$\\
        & GKDE-GCN & ${72.75\tiny \pm 0.18}$ & ${66.70\tiny \pm 0.09}$ & ${67.29\tiny \pm 0.06}$ & $n.a.$ & ${18.79\tiny \pm 0.33}$ & ${35.46\tiny \pm 0.71}$ & ${0.21\tiny \pm 0.04}$ & $n.a.$\\
        & GPN & ${65.00\tiny \pm 0.43}$ & ${59.47\tiny \pm 0.16}$ & ${55.95\tiny \pm 0.18}$ & ${\mathbf{80.06}\tiny \pm 1.18}$ & ${\mathbf{66.70}\tiny \pm 0.23}$ & ${51.65\tiny \pm 0.16}$ & ${65.58\tiny \pm 0.26}$ & ${\mathbf{100.00}\tiny \pm 0.00}$\\

        \midrule

        \multirow{10}{*}{PubMed}
        & APPNP & ${82.80\tiny \pm 0.10}$ & ${62.22\tiny \pm 0.02}$ & $n.a.$ & $n.a.$ & ${40.38\tiny \pm 0.22}$ & ${5.41\tiny \pm 0.22}$ & $n.a.$ & $n.a.$\\
        & VGCN & ${82.49\tiny \pm 0.10}$ & ${62.16\tiny \pm 0.06}$ & $n.a.$ & $n.a.$ & ${37.80\tiny \pm 0.40}$ & ${6.54\tiny \pm 0.23}$ & $n.a.$ & $n.a.$\\
        & RGCN & ${\mathbf{83.75}\tiny \pm 0.12}$ & ${64.87\tiny \pm 0.05}$ & $n.a.$ & $n.a.$ & ${47.82\tiny \pm 0.36}$ & ${29.60\tiny \pm 0.34}$ & $n.a.$ & $n.a.$\\
        & VGCN-Dropout & ${82.26\tiny \pm 0.06}$ & ${60.39\tiny \pm 0.10}$ & ${51.80\tiny \pm 0.14}$ & $n.a.$ & ${37.79\tiny \pm 0.45}$ & ${23.86\tiny \pm 0.35}$ & ${38.16\tiny \pm 0.53}$ & $n.a.$\\
        & DropEdge & ${82.70\tiny \pm 0.12}$ & ${62.21\tiny \pm 0.10}$ & ${55.48\tiny \pm 0.18}$ & $n.a.$ & ${36.36\tiny \pm 0.47}$ & ${13.32\tiny \pm 0.35}$ & ${21.68\tiny \pm 0.53}$ & $n.a.$\\
        & VGCN-Energy & ${82.49\tiny \pm 0.10}$ & ${62.16\tiny \pm 0.06}$ & ${65.38\tiny \pm 0.07}$ & $n.a.$ & ${37.80\tiny \pm 0.40}$ & ${6.54\tiny \pm 0.23}$ & ${2.97\tiny \pm 0.10}$ & $n.a.$\\
        & VGCN-Ensemble & ${82.00\tiny \pm 0.00}$ & ${62.42\tiny \pm 0.00}$ & ${60.34\tiny \pm 0.10}$ & $n.a.$ & ${39.10\tiny \pm 0.10}$ & ${11.74\tiny \pm 0.03}$ & ${18.79\tiny \pm 0.04}$ & $n.a.$\\
        & VGCN-BNN & ${82.30\tiny \pm 0.14}$ & ${62.36\tiny \pm 0.15}$ & ${59.35\tiny \pm 0.99}$ & $n.a.$ & ${37.56\tiny \pm 0.54}$ & ${12.74\tiny \pm 0.34}$ & ${27.56\tiny \pm 0.64}$ & $n.a.$\\
        & GKDE-GCN & ${82.54\tiny \pm 0.11}$ & ${60.62\tiny \pm 0.11}$ & ${63.00\tiny \pm 0.17}$ & $n.a.$ & ${37.77\tiny \pm 0.48}$ & ${24.07\tiny \pm 0.43}$ & ${3.43\tiny \pm 0.13}$ & $n.a.$\\
        & GPN & ${81.54\tiny \pm 0.39}$ & ${57.05\tiny \pm 0.08}$ & ${58.87\tiny \pm 0.14}$ & ${\mathbf{84.07}\tiny \pm 0.55}$ & ${\mathbf{81.73}\tiny \pm 0.34}$ & ${53.43\tiny \pm 0.04}$ & ${60.94\tiny \pm 0.13}$ & ${\mathbf{100.00}\tiny \pm 0.00}$\\
                
        \midrule
        
        \multirow{10}{*}{\shortstack[l]{Amazon\\Computers}}
        & APPNP & ${75.00\tiny \pm 0.09}$ & ${{67.83}\tiny \pm 0.02}$ & $n.a.$ & $n.a.$ & ${18.25\tiny \pm 0.46}$ & ${5.94\tiny \pm 0.11}$ & $n.a.$ & $n.a.$\\
        & VGCN & ${81.54\tiny \pm 0.08}$ & ${58.03\tiny \pm 0.02}$ & $n.a.$ & $n.a.$ & ${20.38\tiny \pm 0.29}$ & ${24.56\tiny \pm 0.33}$ & $n.a.$ & $n.a.$\\
        & RGCN & ${61.39\tiny \pm 0.39}$ & ${57.92\tiny \pm 0.04}$ & $n.a.$ & $n.a.$ & ${39.60\tiny \pm 0.45}$ & ${33.60\tiny \pm 0.35}$ & $n.a.$ & $n.a.$\\
        & VGCN-Dropout & ${\mathbf{81.79}\tiny \pm 0.10}$ & ${57.15\tiny \pm 0.04}$ & ${55.52\tiny \pm 0.08}$ & $n.a.$ & ${21.52\tiny \pm 0.36}$ & ${40.32\tiny \pm 0.29}$ & ${66.21\tiny \pm 0.33}$ & $n.a.$\\
        & DropEdge & ${81.20\tiny \pm 0.08}$ & ${57.88\tiny \pm 0.02}$ & ${55.51\tiny \pm 0.07}$ & $n.a.$ & ${21.75\tiny \pm 0.36}$ & ${33.10\tiny \pm 0.26}$ & ${57.47\tiny \pm 0.37}$ & $n.a.$\\
        & VGCN-Energy & ${81.54\tiny \pm 0.08}$ & ${58.03\tiny \pm 0.02}$ & ${58.66\tiny \pm 0.03}$ & $n.a.$ & ${20.38\tiny \pm 0.29}$ & ${24.56\tiny \pm 0.33}$ & ${4.82\tiny \pm 0.09}$ & $n.a.$\\
        & VGCN-Ensemble & ${81.00\tiny \pm 0.00}$ & ${58.22\tiny \pm 0.01}$ & ${53.24\tiny \pm 0.12}$ & $n.a.$ & ${23.60\tiny \pm 0.16}$ & ${28.97\tiny \pm 0.05}$ & ${66.42\tiny \pm 0.19}$ & $n.a.$\\
        & VGCN-BNN & ${61.25\tiny \pm 0.17}$ & ${56.01\tiny \pm 0.11}$ & ${51.16\tiny \pm 0.83}$ & $n.a.$ & ${28.54\tiny \pm 0.47}$ & ${21.72\tiny \pm 0.19}$ & ${56.41\tiny \pm 0.60}$ & $n.a.$\\
        & GKDE-GCN & ${59.83\tiny \pm 0.73}$ & ${56.38\tiny \pm 0.12}$ & ${55.91\tiny \pm 0.05}$ & $n.a.$ & ${28.46\tiny \pm 0.36}$ & ${26.48\tiny \pm 0.45}$ & ${16.10\tiny \pm 0.53}$ & $n.a.$\\
        & GPN & ${79.70\tiny \pm 0.46}$ & ${61.21\tiny \pm 0.11}$ & ${{61.07}\tiny \pm 0.11}$ & ${\mathbf{86.15}\tiny \pm 0.28}$ & ${\mathbf{79.87}\tiny \pm 0.46}$ & ${{60.42}\tiny \pm 0.12}$ & ${61.56\tiny \pm 0.12}$ & ${\mathbf{100.00}\tiny \pm 0.00}$\\

        \midrule
        
        \multirow{10}{*}{\shortstack[l]{Amazon\\Photos}}
        & APPNP & ${\mathbf{88.12}\tiny \pm 0.10}$ & ${{65.02}\tiny \pm 0.03}$ & $n.a.$ & $n.a.$ & ${19.37\tiny \pm 0.45}$ & ${8.42\tiny \pm 0.29}$ & $n.a.$ & $n.a.$\\
        & VGCN & ${83.91\tiny \pm 0.08}$ & ${57.91\tiny \pm 0.02}$ & $n.a.$ & $n.a.$ & ${21.40\tiny \pm 0.49}$ & ${31.07\tiny \pm 0.34}$ & $n.a.$ & $n.a.$\\
        & RGCN & ${79.50\tiny \pm 0.72}$ & ${57.22\tiny \pm 0.04}$ & $n.a.$ & $n.a.$ & ${42.38\tiny \pm 0.40}$ & ${32.02\tiny \pm 0.31}$ & $n.a.$ & $n.a.$\\
        & VGCN-Dropout & ${83.86\tiny \pm 0.18}$ & ${56.85\tiny \pm 0.04}$ & ${55.04\tiny \pm 0.08}$ & $n.a.$ & ${22.29\tiny \pm 0.55}$ & ${49.11\tiny \pm 0.31}$ & ${66.74\tiny \pm 0.35}$ & $n.a.$\\
        & DropEdge & ${85.69\tiny \pm 0.15}$ & ${57.32\tiny \pm 0.04}$ & ${55.31\tiny \pm 0.07}$ & $n.a.$ & ${22.90\tiny \pm 0.43}$ & ${39.14\tiny \pm 0.20}$ & ${56.18\tiny \pm 0.21}$ & $n.a.$\\
        & VGCN-Energy & ${83.91\tiny \pm 0.08}$ & ${57.91\tiny \pm 0.02}$ & ${59.07\tiny \pm 0.02}$ & $n.a.$ & ${21.40\tiny \pm 0.49}$ & ${31.07\tiny \pm 0.34}$ & ${6.42\tiny \pm 0.07}$ & $n.a.$\\
        & VGCN-Ensemble & ${84.40\tiny \pm 0.16}$ & ${57.86\tiny \pm 0.01}$ & ${56.01\tiny \pm 0.19}$ & $n.a.$ & ${20.30\tiny \pm 0.21}$ & ${44.14\tiny \pm 0.05}$ & ${69.01\tiny \pm 0.14}$ & $n.a.$\\
        & VGCN-BNN & ${82.00\tiny \pm 0.19}$ & ${56.78\tiny \pm 0.09}$ & ${49.21\tiny \pm 0.58}$ & $n.a.$ & ${25.84\tiny \pm 0.46}$ & ${23.16\tiny \pm 0.37}$ & ${59.31\tiny \pm 0.73}$ & $n.a.$\\
        & GKDE-GCN & ${73.17\tiny \pm 0.94}$ & ${57.01\tiny \pm 0.10}$ & ${58.00\tiny \pm 0.05}$ & $n.a.$ & ${24.04\tiny \pm 0.42}$ & ${24.45\tiny \pm 0.62}$ & ${9.82\tiny \pm 0.36}$ & $n.a.$\\
        & GPN & ${87.47\tiny \pm 0.20}$ & ${56.25\tiny \pm 0.16}$ & ${60.52\tiny \pm 0.18}$ & ${\mathbf{75.24}\tiny \pm 0.63}$ & ${\mathbf{88.29}\tiny \pm 0.20}$ & ${51.89\tiny \pm 0.09}$ & ${61.89\tiny \pm 0.18}$ & ${\mathbf{100.00}\tiny \pm 0.00}$\\

        \midrule
        
        \multirow{10}{*}{\shortstack[l]{Coauthor\\CS}}
        & APPNP & ${89.28\tiny \pm 0.07}$ & ${{72.01}\tiny \pm 0.02}$ & $n.a.$ & $n.a.$ & ${12.58\tiny \pm 0.33}$ & ${23.09\tiny \pm 0.31}$ & $n.a.$ & $n.a.$\\
        & VGCN & ${89.33\tiny \pm 0.05}$ & ${67.65\tiny \pm 0.02}$ & $n.a.$ & $n.a.$ & ${13.29\tiny \pm 0.22}$ & ${30.13\tiny \pm 0.32}$ & $n.a.$ & $n.a.$\\
        & RGCN & ${\mathbf{90.50}\tiny \pm 0.06}$ & ${71.13\tiny \pm 0.02}$ & $n.a.$ & $n.a.$ & ${41.67\tiny \pm 0.35}$ & ${52.81\tiny \pm 0.21}$ & $n.a.$ & $n.a.$\\
        & VGCN-Dropout & ${88.96\tiny \pm 0.09}$ & ${65.91\tiny \pm 0.05}$ & ${60.56\tiny \pm 0.07}$ & $n.a.$ & ${13.31\tiny \pm 0.32}$ & ${{67.56}\tiny \pm 0.26}$ & ${85.81\tiny \pm 0.22}$ & $n.a.$\\
        & DropEdge & ${89.44\tiny \pm 0.07}$ & ${67.94\tiny \pm 0.01}$ & ${63.68\tiny \pm 0.06}$ & $n.a.$ & ${11.65\tiny \pm 0.28}$ & ${49.77\tiny \pm 0.33}$ & ${70.31\tiny \pm 0.49}$ & $n.a.$\\
        & VGCN-Energy & ${89.33\tiny \pm 0.05}$ & ${67.65\tiny \pm 0.02}$ & ${{70.14}\tiny \pm 0.02}$ & $n.a.$ & ${13.29\tiny \pm 0.22}$ & ${30.13\tiny \pm 0.32}$ & ${0.89\tiny \pm 0.06}$ & $n.a.$\\
        & VGCN-Ensemble & ${89.00\tiny \pm 0.00}$ & ${67.64\tiny \pm 0.01}$ & ${64.41\tiny \pm 0.08}$ & $n.a.$ & ${11.00\tiny \pm 0.00}$ & ${60.89\tiny \pm 0.11}$ & ${85.09\tiny \pm 0.32}$ & $n.a.$\\
        & VGCN-BNN & ${88.16\tiny \pm 0.10}$ & ${67.09\tiny \pm 0.12}$ & ${59.69\tiny \pm 0.35}$ & $n.a.$ & ${12.48\tiny \pm 0.25}$ & ${67.00\tiny \pm 0.54}$ & ${{87.27}\tiny \pm 0.47}$ & $n.a.$\\
        & GKDE-GCN & ${88.14\tiny \pm 0.14}$ & ${67.69\tiny \pm 0.07}$ & ${70.08\tiny \pm 0.14}$ & $n.a.$ & ${9.71\tiny \pm 0.29}$ & ${32.67\tiny \pm 0.51}$ & ${0.23\tiny \pm 0.03}$ & $n.a.$\\
        & GPN & ${{83.99}\tiny \pm 0.31}$ & ${57.66\tiny \pm 0.10}$ & ${62.08\tiny \pm 0.18}$ & ${\mathbf{97.84}\tiny \pm 0.23}$ & ${\mathbf{83.96}\tiny \pm 0.31}$ & ${57.04\tiny \pm 0.09}$ & ${62.39\tiny \pm 0.18}$ & ${\mathbf{100.00}\tiny \pm 0.00}$\\
                
        \midrule

        \multirow{12}{*}{\shortstack[l]{Coauthor\\Physics}}
        & APPNP & ${96.16\tiny \pm 0.08}$ & ${67.63\tiny \pm 0.02}$ & $n.a.$ & $n.a.$ & ${28.71\tiny \pm 0.40}$ & ${24.97\tiny \pm 0.20}$ & $n.a.$ & $n.a.$\\
        & VGCN & ${96.00\tiny \pm 0.00}$ & ${60.30\tiny \pm 0.02}$ & $n.a.$ & $n.a.$ & ${33.26\tiny \pm 0.67}$ & ${40.19\tiny \pm 0.42}$ & $n.a.$ & $n.a.$\\
        & RGCN & ${94.69\tiny \pm 0.07}$ & ${65.84\tiny \pm 0.03}$ & $n.a.$ & $n.a.$ & ${58.56\tiny \pm 0.36}$ & ${52.91\tiny \pm 0.32}$ & $n.a.$ & $n.a.$\\
        & VGCN-Dropout & ${95.90\tiny \pm 0.05}$ & ${58.97\tiny \pm 0.03}$ & ${57.64\tiny \pm 0.05}$ & $n.a.$ & ${32.52\tiny \pm 0.60}$ & ${55.07\tiny \pm 0.48}$ & ${61.85\tiny \pm 0.50}$ & $n.a.$\\
        & DropEdge & ${95.90\tiny \pm 0.03}$ & ${60.40\tiny \pm 0.04}$ & ${59.09\tiny \pm 0.05}$ & $n.a.$ & ${30.53\tiny \pm 0.58}$ & ${43.30\tiny \pm 0.43}$ & ${51.07\tiny \pm 0.44}$ & $n.a.$\\
        & VGCN-Energy & ${96.00\tiny \pm 0.00}$ & ${60.30\tiny \pm 0.02}$ & ${61.59\tiny \pm 0.02}$ & $n.a.$ & ${33.26\tiny \pm 0.67}$ & ${40.19\tiny \pm 0.42}$ & ${11.45\tiny \pm 0.22}$ & $n.a.$\\
        & VGCN-Ensemble & ${96.00\tiny \pm 0.00}$ & ${60.29\tiny \pm 0.00}$ & ${59.05\tiny \pm 0.01}$ & $n.a.$ & ${31.70\tiny \pm 0.21}$ & ${52.08\tiny \pm 0.10}$ & ${68.48\tiny \pm 0.10}$ & $n.a.$\\
        & VGCN-BNN & ${95.65\tiny \pm 0.07}$ & ${60.99\tiny \pm 0.19}$ & ${56.95\tiny \pm 0.35}$ & $n.a.$ & ${32.95\tiny \pm 0.60}$ & ${62.53\tiny \pm 0.75}$ & ${{71.96}\tiny \pm 0.73}$ & $n.a.$\\
        & GKDE-GCN & ${\mathbf{96.61}\tiny \pm 0.05}$ & ${60.46\tiny \pm 0.01}$ & ${60.99\tiny \pm 0.07}$ & $n.a.$ & ${28.84\tiny \pm 0.31}$ & ${29.12\tiny \pm 0.33}$ & ${2.46\tiny \pm 0.10}$ & $n.a.$\\
        & GPN & ${92.70\tiny \pm 0.11}$ & ${59.92\tiny \pm 0.08}$ & ${58.62\tiny \pm 0.18}$ & ${\mathbf{99.15}\tiny \pm 0.05}$ & ${\mathbf{92.70}\tiny \pm 0.11}$ & ${58.66\tiny \pm 0.08}$ & ${59.00\tiny \pm 0.18}$ & ${\mathbf{100.00}\tiny \pm 0.00}$\\

        \midrule
        
        \multirow{7}{*}{\shortstack[l]{OGBN\\Arxiv}}
        & APPNP & ${63.50\tiny \pm 0.95}$ & ${{62.51}\tiny \pm 0.51}$ & $n.a.$ & $n.a.$ & ${51.10\tiny \pm 1.12}$ & ${59.92\tiny \pm 0.64}$ & $n.a.$ & $n.a.$\\
        & VGCN & ${65.70\tiny \pm 0.47}$ & ${46.16\tiny \pm 0.13}$ & $n.a.$ & $n.a.$ & ${51.30\tiny \pm 0.75}$ & ${53.83\tiny \pm 0.59}$ & $n.a.$ & $n.a.$\\
        & VGCN-Dropout & ${65.30\tiny \pm 0.70}$ & ${48.11\tiny \pm 0.23}$ & ${50.64\tiny \pm 0.20}$ & $n.a.$ & ${49.90\tiny \pm 0.77}$ & ${{60.10}\tiny \pm 0.68}$ & ${62.87\tiny \pm 0.29}$ & $n.a.$\\
        & VGCN-Energy & ${65.70\tiny \pm 0.47}$ & ${46.16\tiny \pm 0.13}$ & ${48.54\tiny \pm 0.20}$ & $n.a.$ & ${51.30\tiny \pm 0.75}$ & ${53.83\tiny \pm 0.59}$ & ${48.53\tiny \pm 0.48}$ & $n.a.$\\
        & VGCN-Ensemble & ${\mathbf{67.00}}$ & ${45.99}$ & ${47.41}$ & $n.a.$ & ${49.00}$ & ${59.94}$ & ${{66.44}}$ & $n.a.$\\
        & GKDE-GCN & ${65.20\tiny \pm 0.49}$ & ${50.98\tiny \pm 0.23}$ & ${51.31\tiny \pm 0.22}$ & $n.a.$ & ${45.40\tiny \pm 0.62}$ & ${53.94\tiny \pm 1.41}$ & ${55.28\tiny \pm 1.69}$ & $n.a.$\\
        & GPN & ${65.50\tiny \pm 0.70}$ & ${51.49\tiny \pm 0.37}$ & ${{55.82}\tiny \pm 0.30}$ & ${\mathbf{93.05}\tiny \pm 3.44}$ & ${\mathbf{65.50}\tiny \pm 0.70}$ & ${51.43\tiny \pm 0.32}$ & ${55.85\tiny \pm 0.30}$ & ${\mathbf{95.54}\tiny \pm 0.89}$\\

        \bottomrule
    \end{tabular}}
    \caption{Accuracy and OOD detection scores on Bernoulli and unit Gaussian feature perturbations using AUC-ROC. OOD-AUC-ROC scores are given as \emph{[Alea w/ Net] / [Epist w/ Net] / [Epist w/o Net]}. $n.a.$ means either model or metric not applicable. Bold numbers indicate best results for Accuracy and  OOD detection.}
    \label{tab:isolated_auroc}
\end{table*}


\begin{table*}[!h]
    \centering
    \resizebox{\textwidth}{!}{
    \begin{tabular}{ll|cccc|cccc}
        \toprule
        & & \multicolumn{4}{c|}{{$\x\nodeidxv \sim \DBer(0.5)$}} & \multicolumn{4}{c}{$\x\nodeidxv \sim \DNormal(0,1)$} \\ \midrule
        & \textbf{Model} & \textbf{OOD-ACC} & \multicolumn{3}{c|}{\textbf{OOD-AUC-PR}} & \textbf{OOD-ACC} & \multicolumn{3}{c}{\textbf{OOD-AUC-PR}} \\
        & & & \emph{Alea w/ Net} & \emph{Epist w/ Net} & \emph{Epist w/o Net} & & \emph{Alea w/ Net} & \emph{Epist w/ Net} & \emph{Epist w/o Net} \\
        \midrule
        
        \multirow{10}{*}{CoraML}
        & APPNP & ${80.85\tiny \pm 0.09}$ & ${11.86\tiny \pm 0.06}$ & $n.a.$ & $n.a.$ & ${17.99\tiny \pm 0.36}$ & ${2.55\tiny \pm 0.00}$ & $n.a.$ & $n.a.$\\
        & VGCN & ${78.90\tiny \pm 0.09}$ & ${11.92\tiny \pm 0.07}$ & $n.a.$ & $n.a.$ & ${18.37\tiny \pm 0.31}$ & ${2.57\tiny \pm 0.00}$ & $n.a.$ & $n.a.$\\
        & VGCN-Dropout & ${77.76\tiny \pm 0.15}$ & ${12.51\tiny \pm 0.21}$ & ${4.58\tiny \pm 0.01}$ & $n.a.$ & ${18.28\tiny \pm 0.35}$ & ${3.91\tiny \pm 0.02}$ & ${62.19\tiny \pm 0.39}$ & $n.a.$\\
        & DropEdge & ${77.40\tiny \pm 0.14}$ & ${10.67\tiny \pm 0.09}$ & ${4.90\tiny \pm 0.01}$ & $n.a.$ & ${16.60\tiny \pm 0.26}$ & ${2.93\tiny \pm 0.01}$ & ${42.96\tiny \pm 0.39}$ & $n.a.$\\
        & VGCN-Energy & ${78.90\tiny \pm 0.09}$ & ${11.92\tiny \pm 0.07}$ & ${{10.43}\tiny \pm 0.08}$ & $n.a.$ & ${18.37\tiny \pm 0.31}$ & ${2.57\tiny \pm 0.00}$ & ${2.50\tiny \pm 0.00}$ & $n.a.$\\
        & VGCN-Ensemble & ${78.00\tiny \pm 0.00}$ & ${11.78\tiny \pm 0.00}$ & ${5.38\tiny \pm 0.00}$ & $n.a.$ & ${21.00\tiny \pm 0.00}$ & ${3.43\tiny \pm 0.00}$ & ${60.47\tiny \pm 0.09}$ & $n.a.$\\
        & VGCN-BNN & ${77.01\tiny \pm 0.16}$ & ${10.94\tiny \pm 0.09}$ & ${7.75\tiny \pm 0.15}$ & $n.a.$ & ${18.79\tiny \pm 0.31}$ & ${3.49\tiny \pm 0.03}$ & ${{62.75}\tiny \pm 0.65}$ & $n.a.$\\
        & RGCN & ${79.78\tiny \pm 0.16}$ & ${{16.43}\tiny \pm 0.14}$ & $n.a.$ & $n.a.$ & ${33.37\tiny \pm 0.35}$ & ${4.82\tiny \pm 0.14}$ & $n.a.$ & $n.a.$\\
        & GKDE-GCN & ${76.40\tiny \pm 0.33}$ & ${9.55\tiny \pm 0.10}$ & ${9.79\tiny \pm 0.09}$ & $n.a.$ & ${16.86\tiny \pm 0.35}$ & ${{34.20}\tiny \pm 0.53}$ & ${2.64\tiny \pm 0.06}$ & $n.a.$\\
        & GPN & ${\mathbf{80.98}\tiny \pm 0.22}$ & ${6.96\tiny \pm 0.07}$ & ${5.63\tiny \pm 0.03}$ & ${\mathbf{25.80}\tiny \pm 1.43}$ & ${\mathbf{81.53}\tiny \pm 0.23}$ & ${6.63\tiny \pm 0.07}$ & ${6.21\tiny \pm 0.08}$ & ${\mathbf{100.00}\tiny \pm 0.00}$\\
                        
        \midrule
        
        \multirow{10}{*}{CiteSeer}
        & APPNP & ${\mathbf{73.14}\tiny \pm 0.12}$ & ${7.49\tiny \pm 0.01}$ & $n.a.$ & $n.a.$ & ${20.13\tiny \pm 0.22}$ & ${2.27\tiny \pm 0.00}$ & $n.a.$ & $n.a.$\\
        & VGCN & ${71.30\tiny \pm 0.13}$ & ${8.92\tiny \pm 0.03}$ & $n.a.$ & $n.a.$ & ${17.55\tiny \pm 0.36}$ & ${2.28\tiny \pm 0.00}$ & $n.a.$ & $n.a.$\\
        & RGCN & ${72.29\tiny \pm 0.09}$ & ${{14.49}\tiny \pm 0.14}$ & $n.a.$ & $n.a.$ & ${28.15\tiny \pm 0.40}$ & ${3.26\tiny \pm 0.07}$ & $n.a.$ & $n.a.$\\
        & VGCN-Dropout & ${69.80\tiny \pm 0.19}$ & ${8.07\tiny \pm 0.04}$ & ${4.26\tiny \pm 0.01}$ & $n.a.$ & ${19.60\tiny \pm 0.28}$ & ${2.97\tiny \pm 0.01}$ & ${{64.29}\tiny \pm 0.37}$ & $n.a.$\\
        & DropEdge & ${72.00\tiny \pm 0.23}$ & ${9.34\tiny \pm 0.10}$ & ${4.60\tiny \pm 0.02}$ & $n.a.$ & ${18.00\tiny \pm 0.47}$ & ${2.48\tiny \pm 0.01}$ & ${42.23\tiny \pm 0.50}$ & $n.a.$\\
        & VGCN-Energy & ${71.30\tiny \pm 0.13}$ & ${8.92\tiny \pm 0.03}$ & ${9.09\tiny \pm 0.05}$ & $n.a.$ & ${17.55\tiny \pm 0.36}$ & ${2.28\tiny \pm 0.00}$ & ${2.26\tiny \pm 0.00}$ & $n.a.$\\
        & VGCN-Ensemble & ${72.00\tiny \pm 0.00}$ & ${8.76\tiny \pm 0.00}$ & ${4.34\tiny \pm 0.00}$ & $n.a.$ & ${18.00\tiny \pm 0.00}$ & ${2.58\tiny \pm 0.00}$ & ${49.06\tiny \pm 0.01}$ & $n.a.$\\
        & VGCN-BNN & ${70.38\tiny \pm 0.15}$ & ${8.86\tiny \pm 0.07}$ & ${5.14\tiny \pm 0.16}$ & $n.a.$ & ${16.27\tiny \pm 0.33}$ & ${2.68\tiny \pm 0.02}$ & ${57.70\tiny \pm 1.51}$ & $n.a.$\\
        & GKDE-GCN & ${72.75\tiny \pm 0.18}$ & ${8.57\tiny \pm 0.06}$ & ${{9.82}\tiny \pm 0.11}$ & $n.a.$ & ${18.79\tiny \pm 0.33}$ & ${{33.75}\tiny \pm 0.72}$ & ${2.27\tiny \pm 0.01}$ & $n.a.$\\
        & GPN & ${65.00\tiny \pm 0.43}$ & ${6.27\tiny \pm 0.09}$ & ${5.37\tiny \pm 0.06}$ & ${\mathbf{14.25}\tiny \pm 1.03}$ & ${\mathbf{66.70}\tiny \pm 0.23}$ & ${4.86\tiny \pm 0.03}$ & ${29.98\tiny \pm 0.62}$ & ${\mathbf{100.00}\tiny \pm 0.00}$\\

        \midrule

        \multirow{10}{*}{PubMed}
        & APPNP & ${82.80\tiny \pm 0.10}$ & ${1.08\tiny \pm 0.00}$ & $n.a.$ & $n.a.$ & ${40.38\tiny \pm 0.22}$ & ${0.40\tiny \pm 0.00}$ & $n.a.$ & $n.a.$\\
        & VGCN & ${82.49\tiny \pm 0.10}$ & ${1.29\tiny \pm 0.01}$ & $n.a.$ & $n.a.$ & ${37.80\tiny \pm 0.40}$ & ${0.40\tiny \pm 0.00}$ & $n.a.$ & $n.a.$\\
        & RGCN & ${\mathbf{83.75}\tiny \pm 0.12}$ & ${1.42\tiny \pm 0.03}$ & $n.a.$ & $n.a.$ & ${47.82\tiny \pm 0.36}$ & ${0.64\tiny \pm 0.02}$ & $n.a.$ & $n.a.$\\
        & VGCN-Dropout & ${82.26\tiny \pm 0.06}$ & ${1.41\tiny \pm 0.03}$ & ${0.77\tiny \pm 0.00}$ & $n.a.$ & ${37.79\tiny \pm 0.45}$ & ${0.48\tiny \pm 0.00}$ & ${25.20\tiny \pm 0.50}$ & $n.a.$\\
        & DropEdge & ${82.70\tiny \pm 0.12}$ & ${1.25\tiny \pm 0.01}$ & ${0.93\tiny \pm 0.01}$ & $n.a.$ & ${36.36\tiny \pm 0.47}$ & ${0.42\tiny \pm 0.00}$ & ${15.48\tiny \pm 0.51}$ & $n.a.$\\
        & VGCN-Energy & ${82.49\tiny \pm 0.10}$ & ${1.29\tiny \pm 0.01}$ & ${1.50\tiny \pm 0.01}$ & $n.a.$ & ${37.80\tiny \pm 0.40}$ & ${0.40\tiny \pm 0.00}$ & ${0.58\tiny \pm 0.04}$ & $n.a.$\\
        & VGCN-Ensemble & ${82.00\tiny \pm 0.00}$ & ${1.46\tiny \pm 0.01}$ & ${1.05\tiny \pm 0.00}$ & $n.a.$ & ${39.10\tiny \pm 0.10}$ & ${0.42\tiny \pm 0.00}$ & ${13.71\tiny \pm 0.03}$ & $n.a.$\\
        & VGCN-BNN & ${82.30\tiny \pm 0.14}$ & ${1.28\tiny \pm 0.01}$ & ${1.77\tiny \pm 0.16}$ & $n.a.$ & ${37.56\tiny \pm 0.54}$ & ${0.42\tiny \pm 0.00}$ & ${15.60\tiny \pm 0.60}$ & $n.a.$\\
        & GKDE-GCN & ${82.54\tiny \pm 0.11}$ & ${1.31\tiny \pm 0.03}$ & ${1.31\tiny \pm 0.01}$ & $n.a.$ & ${37.77\tiny \pm 0.48}$ & ${{16.95}\tiny \pm 0.49}$ & ${0.92\tiny \pm 0.08}$ & $n.a.$\\
        & GPN & ${81.54\tiny \pm 0.39}$ & ${0.85\tiny \pm 0.00}$ & ${1.01\tiny \pm 0.01}$ & ${\mathbf{3.31}\tiny \pm 0.24}$ & ${\mathbf{81.73}\tiny \pm 0.34}$ & ${0.78\tiny \pm 0.00}$ & ${1.27\tiny \pm 0.01}$ & ${\mathbf{99.98}\tiny \pm 0.00}$\\

        \midrule
        
        \multirow{10}{*}{\shortstack[l]{Amazon\\Computers}}
        & APPNP & ${75.00\tiny \pm 0.09}$ & ${{2.18}\tiny \pm 0.01}$ & $n.a.$ & $n.a.$ & ${18.25\tiny \pm 0.46}$ & ${0.56\tiny \pm 0.00}$ & $n.a.$ & $n.a.$\\
        & VGCN & ${81.54\tiny \pm 0.08}$ & ${1.43\tiny \pm 0.00}$ & $n.a.$ & $n.a.$ & ${20.38\tiny \pm 0.29}$ & ${0.67\tiny \pm 0.00}$ & $n.a.$ & $n.a.$\\
        & RGCN & ${61.39\tiny \pm 0.39}$ & ${1.40\tiny \pm 0.00}$ & $n.a.$ & $n.a.$ & ${39.60\tiny \pm 0.45}$ & ${0.86\tiny \pm 0.02}$ & $n.a.$ & $n.a.$\\
        & VGCN-Dropout & ${\mathbf{81.79}\tiny \pm 0.10}$ & ${1.38\tiny \pm 0.00}$ & ${1.24\tiny \pm 0.00}$ & $n.a.$ & ${21.52\tiny \pm 0.36}$ & ${0.88\tiny \pm 0.01}$ & ${20.49\tiny \pm 0.27}$ & $n.a.$\\
        & DropEdge & ${81.20\tiny \pm 0.08}$ & ${1.46\tiny \pm 0.00}$ & ${1.23\tiny \pm 0.00}$ & $n.a.$ & ${21.75\tiny \pm 0.36}$ & ${0.77\tiny \pm 0.00}$ & ${16.09\tiny \pm 0.22}$ & $n.a.$\\
        & VGCN-Energy & ${81.54\tiny \pm 0.08}$ & ${1.43\tiny \pm 0.00}$ & ${1.44\tiny \pm 0.00}$ & $n.a.$ & ${20.38\tiny \pm 0.29}$ & ${0.67\tiny \pm 0.00}$ & ${0.56\tiny \pm 0.00}$ & $n.a.$\\
        & VGCN-Ensemble & ${81.00\tiny \pm 0.00}$ & ${1.44\tiny \pm 0.00}$ & ${1.13\tiny \pm 0.00}$ & $n.a.$ & ${23.60\tiny \pm 0.16}$ & ${0.71\tiny \pm 0.00}$ & ${{51.48}\tiny \pm 0.16}$ & $n.a.$\\
        & VGCN-BNN & ${61.25\tiny \pm 0.17}$ & ${1.35\tiny \pm 0.03}$ & ${1.29\tiny \pm 0.05}$ & $n.a.$ & ${28.54\tiny \pm 0.47}$ & ${0.64\tiny \pm 0.00}$ & ${21.82\tiny \pm 0.68}$ & $n.a.$\\
        & GKDE-GCN & ${59.83\tiny \pm 0.73}$ & ${1.35\tiny \pm 0.01}$ & ${1.32\tiny \pm 0.00}$ & $n.a.$ & ${28.46\tiny \pm 0.36}$ & ${2.96\tiny \pm 0.14}$ & ${1.86\tiny \pm 0.16}$ & $n.a.$\\
        & GPN & ${79.70\tiny \pm 0.46}$ & ${1.55\tiny \pm 0.01}$ & ${{1.55}\tiny \pm 0.01}$ & ${\mathbf{4.26}\tiny \pm 0.08}$ & ${\mathbf{79.87}\tiny \pm 0.46}$ & ${2.56\tiny \pm 0.01}$ & ${2.77\tiny \pm 0.03}$ & ${\mathbf{100.00}\tiny \pm 0.00}$\\

        \midrule
        
        \multirow{10}{*}{\shortstack[l]{Amazon\\Photos}}
        & APPNP & ${\mathbf{88.12}\tiny \pm 0.10}$ & ${\mathbf{5.44}\tiny \pm 0.05}$ & $n.a.$ & $n.a.$ & ${19.37\tiny \pm 0.45}$ & ${1.02\tiny \pm 0.00}$ & $n.a.$ & $n.a.$\\
        & VGCN & ${83.91\tiny \pm 0.08}$ & ${3.61\tiny \pm 0.07}$ & $n.a.$ & $n.a.$ & ${21.40\tiny \pm 0.49}$ & ${1.39\tiny \pm 0.01}$ & $n.a.$ & $n.a.$\\
        & RGCN & ${79.50\tiny \pm 0.72}$ & ${3.33\tiny \pm 0.04}$ & $n.a.$ & $n.a.$ & ${42.38\tiny \pm 0.40}$ & ${1.85\tiny \pm 0.05}$ & $n.a.$ & $n.a.$\\
        & VGCN-Dropout & ${83.86\tiny \pm 0.18}$ & ${3.12\tiny \pm 0.04}$ & ${2.19\tiny \pm 0.01}$ & $n.a.$ & ${22.29\tiny \pm 0.55}$ & ${2.09\tiny \pm 0.02}$ & ${21.31\tiny \pm 0.24}$ & $n.a.$\\
        & DropEdge & ${85.69\tiny \pm 0.15}$ & ${3.56\tiny \pm 0.05}$ & ${2.25\tiny \pm 0.01}$ & $n.a.$ & ${22.90\tiny \pm 0.43}$ & ${1.63\tiny \pm 0.01}$ & ${16.54\tiny \pm 0.20}$ & $n.a.$\\
        & VGCN-Energy & ${83.91\tiny \pm 0.08}$ & ${3.61\tiny \pm 0.07}$ & ${{5.32}\tiny \pm 0.07}$ & $n.a.$ & ${21.40\tiny \pm 0.49}$ & ${1.39\tiny \pm 0.01}$ & ${1.01\tiny \pm 0.00}$ & $n.a.$\\
        & VGCN-Ensemble & ${84.40\tiny \pm 0.16}$ & ${3.27\tiny \pm 0.02}$ & ${2.27\tiny \pm 0.02}$ & $n.a.$ & ${20.30\tiny \pm 0.21}$ & ${1.87\tiny \pm 0.00}$ & ${{59.99}\tiny \pm 0.19}$ & $n.a.$\\
        & VGCN-BNN & ${82.00\tiny \pm 0.19}$ & ${3.70\tiny \pm 0.09}$ & ${2.07\tiny \pm 0.04}$ & $n.a.$ & ${25.84\tiny \pm 0.46}$ & ${1.18\tiny \pm 0.01}$ & ${28.71\tiny \pm 0.92}$ & $n.a.$\\
        & GKDE-GCN & ${73.17\tiny \pm 0.94}$ & ${2.72\tiny \pm 0.04}$ & ${3.13\tiny \pm 0.04}$ & $n.a.$ & ${24.04\tiny \pm 0.42}$ & ${4.88\tiny \pm 0.19}$ & ${1.17\tiny \pm 0.03}$ & $n.a.$\\
        & GPN & ${87.47\tiny \pm 0.20}$ & ${2.38\tiny \pm 0.01}$ & ${2.81\tiny \pm 0.02}$ & ${4.66\tiny \pm 0.18}$ & ${\mathbf{88.29}\tiny \pm 0.20}$ & ${2.10\tiny \pm 0.01}$ & ${3.32\tiny \pm 0.04}$ & ${\mathbf{100.00}\tiny \pm 0.00}$\\

        \midrule
        
        \multirow{10}{*}{\shortstack[l]{Coauthor\\CS}}
        & APPNP & ${89.28\tiny \pm 0.07}$ & ${{2.32}\tiny \pm 0.01}$ & $n.a.$ & $n.a.$ & ${12.58\tiny \pm 0.33}$ & ${0.50\tiny \pm 0.00}$ & $n.a.$ & $n.a.$\\
        & VGCN & ${89.33\tiny \pm 0.05}$ & ${1.75\tiny \pm 0.01}$ & $n.a.$ & $n.a.$ & ${13.29\tiny \pm 0.22}$ & ${0.57\tiny \pm 0.00}$ & $n.a.$ & $n.a.$\\
        & RGCN & ${\mathbf{90.50}\tiny \pm 0.06}$ & ${1.89\tiny \pm 0.00}$ & $n.a.$ & $n.a.$ & ${41.67\tiny \pm 0.35}$ & ${1.16\tiny \pm 0.02}$ & $n.a.$ & $n.a.$\\
        & VGCN-Dropout & ${88.96\tiny \pm 0.09}$ & ${1.62\tiny \pm 0.02}$ & ${1.05\tiny \pm 0.00}$ & $n.a.$ & ${13.31\tiny \pm 0.32}$ & ${1.68\tiny \pm 0.02}$ & ${71.46\tiny \pm 0.30}$ & $n.a.$\\
        & DropEdge & ${89.44\tiny \pm 0.07}$ & ${1.97\tiny \pm 0.03}$ & ${1.18\tiny \pm 0.00}$ & $n.a.$ & ${11.65\tiny \pm 0.28}$ & ${0.90\tiny \pm 0.01}$ & ${55.73\tiny \pm 0.45}$ & $n.a.$\\
        & VGCN-Energy & ${89.33\tiny \pm 0.05}$ & ${1.75\tiny \pm 0.01}$ & ${2.38\tiny \pm 0.01}$ & $n.a.$ & ${13.29\tiny \pm 0.22}$ & ${0.57\tiny \pm 0.00}$ & ${0.42\tiny \pm 0.00}$ & $n.a.$\\
        & VGCN-Ensemble & ${89.00\tiny \pm 0.00}$ & ${1.76\tiny \pm 0.00}$ & ${1.19\tiny \pm 0.00}$ & $n.a.$ & ${11.00\tiny \pm 0.00}$ & ${1.39\tiny \pm 0.01}$ & ${{73.57}\tiny \pm 0.28}$ & $n.a.$\\
        & VGCN-BNN & ${88.16\tiny \pm 0.10}$ & ${1.91\tiny \pm 0.02}$ & ${1.01\tiny \pm 0.01}$ & $n.a.$ & ${12.48\tiny \pm 0.25}$ & ${1.53\tiny \pm 0.03}$ & ${73.43\tiny \pm 0.70}$ & $n.a.$\\
        & GKDE-GCN & ${88.14\tiny \pm 0.14}$ & ${2.00\tiny \pm 0.02}$ & ${{2.70}\tiny \pm 0.02}$ & $n.a.$ & ${9.71\tiny \pm 0.29}$ & ${{15.14}\tiny \pm 0.55}$ & ${0.42\tiny \pm 0.00}$ & $n.a.$\\
        & GPN & ${83.99\tiny \pm 0.31}$ & ${1.04\tiny \pm 0.01}$ & ${1.22\tiny \pm 0.01}$ & ${\mathbf{29.25}\tiny \pm 1.92}$ & ${\mathbf{83.96}\tiny \pm 0.31}$ & ${1.02\tiny \pm 0.01}$ & ${1.26\tiny \pm 0.01}$ & ${\mathbf{100.00}\tiny \pm 0.00}$\\

        \midrule

        \multirow{10}{*}{\shortstack[l]{Coauthor\\Physics}}
        & APPNP & ${96.16\tiny \pm 0.08}$ & ${0.83\tiny \pm 0.00}$ & $n.a.$ & $n.a.$ & ${28.71\tiny \pm 0.40}$ & ${0.30\tiny \pm 0.00}$ & $n.a.$ & $n.a.$\\
        & VGCN & ${96.00\tiny \pm 0.00}$ & ${0.62\tiny \pm 0.00}$ & $n.a.$ & $n.a.$ & ${33.26\tiny \pm 0.67}$ & ${0.44\tiny \pm 0.01}$ & $n.a.$ & $n.a.$\\
        & RGCN & ${94.69\tiny \pm 0.07}$ & ${0.80\tiny \pm 0.00}$ & $n.a.$ & $n.a.$ & ${58.56\tiny \pm 0.36}$ & ${0.75\tiny \pm 0.02}$ & $n.a.$ & $n.a.$\\
        & VGCN-Dropout & ${95.90\tiny \pm 0.05}$ & ${0.61\tiny \pm 0.00}$ & ${0.57\tiny \pm 0.00}$ & $n.a.$ & ${32.52\tiny \pm 0.60}$ & ${0.75\tiny \pm 0.01}$ & ${19.76\tiny \pm 0.41}$ & $n.a.$\\
        & DropEdge & ${95.90\tiny \pm 0.03}$ & ${0.62\tiny \pm 0.00}$ & ${0.63\tiny \pm 0.01}$ & $n.a.$ & ${30.53\tiny \pm 0.58}$ & ${0.53\tiny \pm 0.01}$ & ${12.71\tiny \pm 0.29}$ & $n.a.$\\
        & VGCN-Energy & ${96.00\tiny \pm 0.00}$ & ${0.62\tiny \pm 0.00}$ & ${0.72\tiny \pm 0.00}$ & $n.a.$ & ${33.26\tiny \pm 0.67}$ & ${0.44\tiny \pm 0.01}$ & ${0.23\tiny \pm 0.00}$ & $n.a.$\\
        & VGCN-Ensemble & ${96.00\tiny \pm 0.00}$ & ${0.62\tiny \pm 0.00}$ & ${0.57\tiny \pm 0.00}$ & $n.a.$ & ${31.70\tiny \pm 0.21}$ & ${0.71\tiny \pm 0.01}$ & ${41.30\tiny \pm 0.30}$ & $n.a.$\\
        & VGCN-BNN & ${95.65\tiny \pm 0.07}$ & ${0.69\tiny \pm 0.01}$ & ${0.57\tiny \pm 0.01}$ & $n.a.$ & ${32.95\tiny \pm 0.60}$ & ${1.17\tiny \pm 0.04}$ & ${{44.88}\tiny \pm 0.80}$ & $n.a.$\\
        & GKDE-GCN & ${\mathbf{96.61}\tiny \pm 0.05}$ & ${0.62\tiny \pm 0.00}$ & ${0.74\tiny \pm 0.00}$ & $n.a.$ & ${28.84\tiny \pm 0.31}$ & ${{6.53}\tiny \pm 0.18}$ & ${0.22\tiny \pm 0.00}$ & $n.a.$\\
        & GPN & ${92.70\tiny \pm 0.11}$ & ${0.63\tiny \pm 0.00}$ & ${0.60\tiny \pm 0.00}$ & ${\mathbf{23.91}\tiny \pm 1.06}$ & ${\mathbf{92.70}\tiny \pm 0.11}$ & ${0.62\tiny \pm 0.00}$ & ${0.62\tiny \pm 0.00}$ & ${\mathbf{100.00}\tiny \pm 0.00}$\\

        \midrule
        
        \multirow{7}{*}{\shortstack[l]{OGBN\\Arxiv}}
        & APPNP & ${63.50\tiny \pm 0.95}$ & ${{0.79}\tiny \pm 0.16}$ & $n.a.$ & $n.a.$ & ${51.10\tiny \pm 1.12}$ & ${0.42\tiny \pm 0.10}$ & $n.a.$ & $n.a.$\\
        & VGCN & ${65.70\tiny \pm 0.47}$ & ${0.20\tiny \pm 0.00}$ & $n.a.$ & $n.a.$ & ${51.30\tiny \pm 0.75}$ & ${0.30\tiny \pm 0.02}$ & $n.a.$ & $n.a.$\\
        & VGCN-Dropout & ${65.30\tiny \pm 0.70}$ & ${0.21\tiny \pm 0.00}$ & ${0.37\tiny \pm 0.03}$ & $n.a.$ & ${49.90\tiny \pm 0.77}$ & ${0.47\tiny \pm 0.05}$ & ${17.68\tiny \pm 0.48}$ & $n.a.$\\
        & VGCN-Energy & ${65.70\tiny \pm 0.47}$ & ${0.20\tiny \pm 0.00}$ & ${{2.32}\tiny \pm 0.38}$ & $n.a.$ & ${51.30\tiny \pm 0.75}$ & ${0.30\tiny \pm 0.02}$ & ${0.24\tiny \pm 0.01}$ & $n.a.$\\
        & VGCN-Ensemble & ${\mathbf{67.00}}$ & ${0.20}$ & ${0.19}$ & $n.a.$ & ${49.00}$ & ${0.46}$ & ${{18.22}}$ & $n.a.$\\
        & GKDE-GCN & ${65.20\tiny \pm 0.49}$ & ${0.76\tiny \pm 0.05}$ & ${0.76\tiny \pm 0.05}$ & $n.a.$ & ${45.40\tiny \pm 0.62}$ & ${{4.99}\tiny \pm 0.97}$ & ${5.04\tiny \pm 0.97}$ & $n.a.$\\
        & GPN & ${65.50\tiny \pm 0.70}$ & ${0.23\tiny \pm 0.01}$ & ${0.26\tiny \pm 0.01}$ & ${\mathbf{46.84}\tiny \pm 5.20}$ & ${\mathbf{65.50}\tiny \pm 0.70}$ & ${0.23\tiny \pm 0.01}$ & ${0.26\tiny \pm 0.01}$ & ${\mathbf{48.97}\tiny \pm 1.51}$\\

        \bottomrule
    \end{tabular}}
    \caption{Accuracy and OOD detection scores on Bernoulli and unit Gaussian feature perturbations using AUC-APR. OOD-AUC-APR scores are given as \emph{[Alea w/ Net] / [Epist w/ Net] / [Epist w/o Net]}. $n.a.$ means either model or metric not applicable. Bold numbers indicate best results for Accuracy and  OOD detection.}
    \label{tab:isolated_apr}
\end{table*}

%% file: tables/loc.tex
\begin{table*}[!h]
    \centering
    \resizebox{\textwidth}{!}{
    \begin{tabular}{ll|cc|cccc|ccc}
        \toprule
        {} & {} & \multicolumn{2}{c|}{Clean Graph} & \multicolumn{7}{c}{Leave-Out Classes} \\ \midrule
        {} & {\textbf{Model}} & \textbf{ID-ACC} & \textbf{ID-ECE} & \textbf{ID-ACC} & \multicolumn{3}{c}{\textbf{OOD-AUC-ROC}}  & \multicolumn{3}{c}{\textbf{OOD-AUC-PR}} \\
        {} & {} & {} & {} & {} & \emph{Alea w/ Net} & \emph{Epist w/ Net} & \emph{Epist w/o Net} & \emph{Alea w/ Net} & \emph{Epist w/ Net} & \emph{Epist w/o Net}  \\

        \midrule
        
        \multirow{14}{*}{CoraML} 
        & LP & ${78.41\tiny \pm 0.00}$ & ${64.12\tiny \pm 0.00}$ & ${86.40\tiny \pm 0.00}$ & ${{83.78}\tiny \pm 0.00}$ & ${80.86\tiny \pm 0.00}$ & $n.a.$ & ${74.80\tiny \pm 0.00}$ & ${71.15\tiny \pm 0.00}$ & $n.a.$\\
        & GKDE & ${72.88\tiny \pm 0.00}$ & ${56.46\tiny \pm 0.00}$ & ${83.02\tiny \pm 0.00}$ & ${74.46\tiny \pm 0.00}$ & ${71.86\tiny \pm 0.00}$ & $n.a.$ & ${66.19\tiny \pm 0.00}$ & ${64.05\tiny \pm 0.00}$ & $n.a.$\\
        & Matern-GGP & ${79.70\tiny \pm 0.02}$ & ${9.88\tiny \pm 0.02}$ & ${87.03\tiny \pm 0.01}$ & ${83.13\tiny \pm 0.00}$ & ${82.98\tiny \pm 0.00}$ & $n.a.$ & ${71.42\tiny \pm 0.00}$ & ${71.04\tiny \pm 0.14}$ & $n.a.$\\
        & GGP & ${79.04\tiny \pm 0.01}$ & ${21.67\tiny \pm 0.01}$ & ${88.65\tiny \pm 0.00}$ & ${81.49\tiny \pm 0.00}$ & ${82.03\tiny \pm 0.00}$ & $n.a.$ & ${74.13\tiny \pm 0.00}$ & ${74.77\tiny \pm 0.00}$ & $n.a.$\\
        & APPNP & ${\mathbf{84.94}\tiny \pm 0.02}$ & ${8.27\tiny \pm 0.10}$ & ${\mathbf{90.20}\tiny \pm 0.02}$ & ${83.71\tiny \pm 0.06}$ & $n.a.$ & $n.a.$ & ${{78.77}\tiny \pm 0.07}$ & $n.a.$ & $n.a.$\\
        & VGCN & ${83.14\tiny \pm 0.03}$ & ${7.96\tiny \pm 0.16}$ & ${89.66\tiny \pm 0.05}$ & ${81.70\tiny \pm 0.07}$ & $n.a.$ & $n.a.$ & ${75.67\tiny \pm 0.10}$ & $n.a.$ & $n.a.$\\
        & RGCN & ${82.79\tiny \pm 0.06}$ & ${14.39\tiny \pm 0.17}$ & ${88.66\tiny \pm 0.03}$ & ${80.37\tiny \pm 0.11}$ & $n.a.$ & $n.a.$ & ${76.97\tiny \pm 0.11}$ & $n.a.$ & $n.a.$\\
        & VGCN-Dropout & ${82.30\tiny \pm 0.06}$ & ${13.88\tiny \pm 0.13}$ & ${89.08\tiny \pm 0.04}$ & ${81.27\tiny \pm 0.07}$ & ${71.65\tiny \pm 0.10}$ & $n.a.$ & ${75.55\tiny \pm 0.12}$ & ${60.65\tiny \pm 0.11}$ & $n.a.$\\
        & DropEdge & ${83.07\tiny \pm 0.04}$ & ${13.93\tiny \pm 0.11}$ & ${89.03\tiny \pm 0.03}$ & ${83.55\tiny \pm 0.05}$ & ${75.48\tiny \pm 0.12}$ & $n.a.$ & ${78.48\tiny \pm 0.12}$ & ${65.22\tiny \pm 0.15}$ & $n.a.$\\
        & VGCN-Energy & ${83.14\tiny \pm 0.03}$ & ${7.96\tiny \pm 0.16}$ & ${89.66\tiny \pm 0.05}$ & ${81.70\tiny \pm 0.07}$ & ${83.15\tiny \pm 0.07}$ & $n.a.$ & ${75.67\tiny \pm 0.10}$ & ${78.44\tiny \pm 0.10}$ & $n.a.$\\
        & VGCN-Ensemble & ${83.41\tiny \pm 0.01}$ & ${8.45\tiny \pm 0.01}$ & ${89.87\tiny \pm 0.00}$ & ${81.85\tiny \pm 0.00}$ & ${74.24\tiny \pm 0.00}$ & $n.a.$ & ${75.80\tiny \pm 0.00}$ & ${64.02\tiny \pm 0.00}$ & $n.a.$\\
        & VGCN-BNN & ${82.83\tiny \pm 0.06}$ & ${15.66\tiny \pm 0.13}$ & ${88.49\tiny \pm 0.04}$ & ${82.18\tiny \pm 0.23}$ & ${73.18\tiny \pm 1.10}$ & $n.a.$ & ${76.17\tiny \pm 0.36}$ & ${63.51\tiny \pm 1.40}$ & $n.a.$\\
        & GKDE-GCN & ${81.91\tiny \pm 0.19}$ & ${\mathbf{6.53}\tiny \pm 0.14}$ & ${89.33\tiny \pm 0.04}$ & ${82.23\tiny \pm 0.08}$ & ${82.09\tiny \pm 0.18}$ & $n.a.$ & ${75.88\tiny \pm 0.12}$ & ${77.03\tiny \pm 0.39}$ & $n.a.$\\
        & GPN & ${81.16\tiny \pm 0.12}$ & ${11.68\tiny \pm 0.18}$ & ${88.51\tiny \pm 0.04}$ & ${83.25\tiny \pm 0.12}$ & ${\mathbf{86.28}\tiny \pm 0.17}$ & ${{80.95}\tiny \pm 0.24}$ & ${75.79\tiny \pm 0.28}$ & ${\mathbf{79.97}\tiny \pm 0.20}$ & ${{72.81}\tiny \pm 0.46}$\\

        \midrule
        
        \multirow{14}{*}{CiteSeer} 
        & LP & ${54.05\tiny \pm 0.00}$ & ${37.38\tiny \pm 0.00}$ & ${57.34\tiny \pm 0.00}$ & ${65.99\tiny \pm 0.00}$ & ${67.54\tiny \pm 0.00}$ & $n.a.$ & ${48.12\tiny \pm 0.00}$ & ${48.59\tiny \pm 0.00}$ & $n.a.$\\
        & GKDE & ${53.67\tiny \pm 0.00}$ & ${36.29\tiny \pm 0.00}$ & ${49.62\tiny \pm 0.00}$ & ${63.75\tiny \pm 0.00}$ & ${63.91\tiny \pm 0.00}$ & $n.a.$ & ${56.74\tiny \pm 0.00}$ & ${56.79\tiny \pm 0.00}$ & $n.a.$\\
        & Matern-GGP & ${53.25\tiny \pm 0.03}$ & ${12.76\tiny \pm 0.03}$ & ${53.83\tiny \pm 0.06}$ & ${59.57\tiny \pm 0.00}$ & ${59.56\tiny \pm 0.05}$ & $n.a.$ & ${36.05\tiny \pm 0.04}$ & ${36.24\tiny \pm 0.19}$ & $n.a.$\\
        & GGP & ${68.85\tiny \pm 0.01}$ & ${20.84\tiny \pm 0.01}$ & ${69.74\tiny \pm 0.01}$ & ${\mathbf{74.56}\tiny \pm 0.06}$ & ${74.10\tiny \pm 0.07}$ & $n.a.$ & ${50.39\tiny \pm 0.08}$ & ${48.74\tiny \pm 0.08}$ & $n.a.$\\
        & APPNP & ${70.24\tiny \pm 0.02}$ & ${4.88\tiny \pm 0.03}$ & ${{72.83}\tiny \pm 0.03}$ & ${72.91\tiny \pm 0.09}$ & $n.a.$ & $n.a.$ & ${56.31\tiny \pm 0.14}$ & $n.a.$ & $n.a.$\\
        & VGCN & ${68.98\tiny \pm 0.02}$ & ${4.33\tiny \pm 0.05}$ & ${70.79\tiny \pm 0.02}$ & ${72.16\tiny \pm 0.08}$ & $n.a.$ & $n.a.$ & ${53.71\tiny \pm 0.08}$ & $n.a.$ & $n.a.$\\
        & RGCN & ${\mathbf{70.57}\tiny \pm 0.03}$ & ${15.09\tiny \pm 0.09}$ & ${72.15\tiny \pm 0.03}$ & ${74.56\tiny \pm 0.07}$ & $n.a.$ & $n.a.$ & ${\mathbf{58.63}\tiny \pm 0.07}$ & $n.a.$ & $n.a.$\\
        & VGCN-Dropout & ${68.41\tiny \pm 0.03}$ & ${7.31\tiny \pm 0.07}$ & ${70.44\tiny \pm 0.06}$ & ${71.31\tiny \pm 0.08}$ & ${60.05\tiny \pm 0.12}$ & $n.a.$ & ${52.05\tiny \pm 0.14}$ & ${36.95\tiny \pm 0.10}$ & $n.a.$\\
        & DropEdge & ${69.33\tiny \pm 0.03}$ & ${5.64\tiny \pm 0.06}$ & ${71.02\tiny \pm 0.05}$ & ${73.42\tiny \pm 0.05}$ & ${63.23\tiny \pm 0.17}$ & $n.a.$ & ${55.70\tiny \pm 0.10}$ & ${39.38\tiny \pm 0.14}$ & $n.a.$\\
        & VGCN-Energy & ${68.98\tiny \pm 0.02}$ & ${4.33\tiny \pm 0.05}$ & ${70.79\tiny \pm 0.02}$ & ${72.16\tiny \pm 0.08}$ & ${76.08\tiny \pm 0.11}$ & $n.a.$ & ${53.71\tiny \pm 0.08}$ & ${58.35\tiny \pm 0.17}$ & $n.a.$\\
        & VGCN-Ensemble & ${69.26\tiny \pm 0.00}$ & ${4.14\tiny \pm 0.02}$ & ${70.63\tiny \pm 0.00}$ & ${72.23\tiny \pm 0.00}$ & ${58.61\tiny \pm 0.01}$ & $n.a.$ & ${54.04\tiny \pm 0.00}$ & ${38.93\tiny \pm 0.01}$ & $n.a.$\\
        & VGCN-BNN & ${68.06\tiny \pm 0.07}$ & ${8.42\tiny \pm 0.20}$ & ${69.84\tiny \pm 0.04}$ & ${71.64\tiny \pm 0.31}$ & ${64.16\tiny \pm 1.75}$ & $n.a.$ & ${52.60\tiny \pm 0.47}$ & ${46.72\tiny \pm 1.76}$ & $n.a.$\\
        & GKDE-GCN & ${69.55\tiny \pm 0.03}$ & ${\mathbf{3.88}\tiny \pm 0.06}$ & ${70.76\tiny \pm 0.04}$ & ${73.34\tiny \pm 0.15}$ & ${\mathbf{76.19}\tiny \pm 0.31}$ & $n.a.$ & ${54.25\tiny \pm 0.16}$ & ${59.07\tiny \pm 0.42}$ & $n.a.$\\
        & GPN & ${66.13\tiny \pm 0.17}$ & ${7.42\tiny \pm 0.22}$ & ${69.79\tiny \pm 0.10}$ & ${72.46\tiny \pm 0.27}$ & ${70.74\tiny \pm 0.26}$ & ${{66.65}\tiny \pm 0.29}$ & ${55.14\tiny \pm 0.46}$ & ${50.52\tiny \pm 0.34}$ & ${{44.93}\tiny \pm 0.31}$\\

        \midrule
        
        \multirow{14}{*}{PubMed} 
        & LP & ${78.40\tiny \pm 0.00}$ & ${45.07\tiny \pm 0.00}$ & ${89.18\tiny \pm 0.00}$ & ${\mathbf{80.32}\tiny \pm 0.00}$ & ${{79.64}\tiny \pm 0.00}$ & $n.a.$ & ${{71.01}\tiny \pm 0.00}$ & ${\mathbf{72.98}\tiny \pm 0.00}$ & $n.a.$\\
        & GKDE & ${77.10\tiny \pm 0.01}$ & ${40.02\tiny \pm 0.01}$ & ${88.16\tiny \pm 0.00}$ & ${69.66\tiny \pm 0.00}$ & ${68.47\tiny \pm 0.00}$ & $n.a.$ & ${55.81\tiny \pm 0.00}$ & ${54.33\tiny \pm 0.00}$ & $n.a.$\\
        & Matern-GGP & ${78.77\tiny \pm 0.00}$ & ${12.37\tiny \pm 0.00}$ & ${90.33\tiny \pm 0.01}$ & ${46.69\tiny \pm 0.00}$ & ${45.75\tiny \pm 0.00}$ & $n.a.$ & ${39.85\tiny \pm 0.00}$ & ${39.63\tiny \pm 0.00}$ & $n.a.$\\
        & GGP & $n.f.$ & $n.f.$ & $n.f.$ & $n.f.$ & $n.f.$ & $n.f.$ & $n.f.$ & $n.f.$ & $n.f.$\\
        & APPNP & ${\mathbf{86.88}\tiny \pm 0.01}$ & ${2.57\tiny \pm 0.05}$ & ${94.83\tiny \pm 0.01}$ & ${74.76\tiny \pm 0.06}$ & $n.a.$ & $n.a.$ & ${61.84\tiny \pm 0.07}$ & $n.a.$ & $n.a.$\\
        & VGCN & ${86.70\tiny \pm 0.01}$ & ${2.30\tiny \pm 0.05}$ & ${94.77\tiny \pm 0.01}$ & ${72.58\tiny \pm 0.04}$ & $n.a.$ & $n.a.$ & ${60.54\tiny \pm 0.04}$ & $n.a.$ & $n.a.$\\
        & RGCN & ${85.87\tiny \pm 0.01}$ & ${4.86\tiny \pm 0.05}$ & ${94.73\tiny \pm 0.01}$ & ${71.49\tiny \pm 0.14}$ & $n.a.$ & $n.a.$ & ${60.54\tiny \pm 0.13}$ & $n.a.$ & $n.a.$\\
        & VGCN-Dropout & ${86.49\tiny \pm 0.01}$ & ${5.53\tiny \pm 0.06}$ & ${94.72\tiny \pm 0.01}$ & ${71.10\tiny \pm 0.04}$ & ${67.27\tiny \pm 0.06}$ & $n.a.$ & ${59.47\tiny \pm 0.04}$ & ${54.24\tiny \pm 0.08}$ & $n.a.$\\
        & DropEdge & ${86.57\tiny \pm 0.01}$ & ${5.11\tiny \pm 0.04}$ & ${94.72\tiny \pm 0.01}$ & ${72.09\tiny \pm 0.02}$ & ${68.57\tiny \pm 0.04}$ & $n.a.$ & ${59.84\tiny \pm 0.03}$ & ${54.95\tiny \pm 0.06}$ & $n.a.$\\
        & VGCN-Energy & ${86.70\tiny \pm 0.01}$ & ${2.30\tiny \pm 0.05}$ & ${94.77\tiny \pm 0.01}$ & ${72.58\tiny \pm 0.04}$ & ${72.63\tiny \pm 0.06}$ & $n.a.$ & ${60.54\tiny \pm 0.04}$ & ${60.63\tiny \pm 0.10}$ & $n.a.$\\
        & VGCN-Ensemble & ${86.64\tiny \pm 0.00}$ & ${2.29\tiny \pm 0.01}$ & ${\mathbf{94.88}\tiny \pm 0.00}$ & ${72.71\tiny \pm 0.00}$ & ${70.99\tiny \pm 0.00}$ & $n.a.$ & ${60.47\tiny \pm 0.00}$ & ${59.31\tiny \pm 0.00}$ & $n.a.$\\
        & VGCN-BNN & ${85.56\tiny \pm 0.05}$ & ${14.00\tiny \pm 0.11}$ & ${94.34\tiny \pm 0.03}$ & ${65.41\tiny \pm 0.58}$ & ${63.77\tiny \pm 1.50}$ & $n.a.$ & ${53.23\tiny \pm 0.40}$ & ${54.36\tiny \pm 1.54}$ & $n.a.$\\
        & GKDE-GCN & ${86.14\tiny \pm 0.07}$ & ${\mathbf{1.36}\tiny \pm 0.09}$ & ${94.66\tiny \pm 0.00}$ & ${73.53\tiny \pm 0.06}$ & ${74.47\tiny \pm 0.11}$ & $n.a.$ & ${61.36\tiny \pm 0.04}$ & ${61.96\tiny \pm 0.27}$ & $n.a.$\\
        & GPN & ${84.10\tiny \pm 0.26}$ & ${4.31\tiny \pm 0.09}$ & ${94.08\tiny \pm 0.02}$ & ${71.84\tiny \pm 0.08}$ & ${73.91\tiny \pm 0.20}$ & ${{71.20}\tiny \pm 0.15}$ & ${57.92\tiny \pm 0.10}$ & ${67.19\tiny \pm 0.25}$ & ${{59.72}\tiny \pm 0.18}$\\

        \bottomrule
    \end{tabular}}
    \caption{Accuracy and ECE scores on the clean graphs. Accuracy and OOD detection scores on Left-Out classes using AUC-ROC and AUC-PR scores. OOD-AUC-ROC and OOD-AUC-APR scores are given as \emph{[Alea w/ Net] / [Epist w/ Net] / [Epist w/o Net]}. $n.a.$ means model or metric not applicable and $n.f.$ means not finished within our constraints. Bold numbers indicate best results for Accuracy, ECE and OOD detection.}
    \label{tab:clean_loc_auroc_one}
\end{table*}

\begin{table*}[!h]
    \centering
    \resizebox{\textwidth}{!}{
    \begin{tabular}{ll|cc|cccc|ccc}
        \toprule
        {} & {} & \multicolumn{2}{c|}{Clean Graph} & \multicolumn{7}{c}{Leave-Out Classes} \\ \midrule
        {} & {\textbf{Model}} & \textbf{ID-ACC} & \textbf{ID-ECE} & \textbf{ID-ACC} & \multicolumn{3}{c}{\textbf{OOD-AUC-ROC}}  & \multicolumn{3}{c}{\textbf{OOD-AUC-PR}} \\
        {} & {} & {} & {} & {} & \emph{Alea w/ Net} & \emph{Epist w/ Net} & \emph{Epist w/o Net} & \emph{Alea w/ Net} & \emph{Epist w/ Net} & \emph{Epist w/o Net}  \\

        \midrule
        
        \multirow{14}{*}{\shortstack[l]{Amazon\\Computers}}
        & LP & ${79.49\tiny \pm 0.00}$ & ${69.49\tiny \pm 0.00}$ & ${83.28\tiny \pm 0.00}$ & ${{86.74}\tiny \pm 0.00}$ & ${83.88\tiny \pm 0.00}$ & $n.a.$ & ${{67.10}\tiny \pm 0.00}$ & ${63.08\tiny \pm 0.00}$ & $n.a.$\\
        & GKDE & ${63.49\tiny \pm 0.00}$ & ${38.93\tiny \pm 0.00}$ & ${71.41\tiny \pm 0.00}$ & ${75.14\tiny \pm 0.00}$ & ${73.58\tiny \pm 0.00}$ & $n.a.$ & ${49.21\tiny \pm 0.00}$ & ${47.68\tiny \pm 0.00}$ & $n.a.$\\
        & Matern-GGP & ${80.23\tiny \pm 0.00}$ & ${12.13\tiny \pm 0.01}$ & ${86.94\tiny \pm 0.01}$ & ${79.00\tiny \pm 0.00}$ & ${79.10\tiny \pm 0.00}$ & $n.a.$ & ${49.57\tiny \pm 0.00}$ & ${49.91\tiny \pm 0.04}$ & $n.a.$\\
        & GGP & $n.f.$ & $n.f.$ & $n.f.$ & $n.f.$ & $n.f.$ & $n.f.$ & $n.f.$ & $n.f.$ & $n.f.$\\
        & APPNP & ${80.12\tiny \pm 0.04}$ & ${11.83\tiny \pm 0.04}$ & ${87.72\tiny \pm 0.02}$ & ${81.30\tiny \pm 0.02}$ & $n.a.$ & $n.a.$ & ${53.02\tiny \pm 0.04}$ & $n.a.$ & $n.a.$\\
        & VGCN & ${81.66\tiny \pm 0.04}$ & ${9.90\tiny \pm 0.03}$ & ${88.95\tiny \pm 0.02}$ & ${82.76\tiny \pm 0.03}$ & $n.a.$ & $n.a.$ & ${57.49\tiny \pm 0.06}$ & $n.a.$ & $n.a.$\\
        & RGCN & ${68.43\tiny \pm 0.24}$ & ${23.62\tiny \pm 0.16}$ & ${78.52\tiny \pm 0.13}$ & ${76.40\tiny \pm 0.16}$ & $n.a.$ & $n.a.$ & ${53.16\tiny \pm 0.20}$ & $n.a.$ & $n.a.$\\
        & VGCN-Dropout & ${81.29\tiny \pm 0.05}$ & ${11.52\tiny \pm 0.03}$ & ${88.54\tiny \pm 0.02}$ & ${81.99\tiny \pm 0.04}$ & ${72.90\tiny \pm 0.04}$ & $n.a.$ & ${55.66\tiny \pm 0.08}$ & ${41.38\tiny \pm 0.04}$ & $n.a.$\\
        & DropEdge & ${82.09\tiny \pm 0.03}$ & ${11.40\tiny \pm 0.03}$ & ${88.62\tiny \pm 0.02}$ & ${82.79\tiny \pm 0.02}$ & ${75.35\tiny \pm 0.06}$ & $n.a.$ & ${56.77\tiny \pm 0.02}$ & ${43.94\tiny \pm 0.07}$ & $n.a.$\\
        & VGCN-Energy & ${81.66\tiny \pm 0.04}$ & ${9.90\tiny \pm 0.03}$ & ${88.95\tiny \pm 0.02}$ & ${82.76\tiny \pm 0.03}$ & ${83.43\tiny \pm 0.04}$ & $n.a.$ & ${57.49\tiny \pm 0.06}$ & ${60.64\tiny \pm 0.10}$ & $n.a.$\\
        & VGCN-Ensemble & ${81.66\tiny \pm 0.01}$ & ${9.96\tiny \pm 0.02}$ & ${\mathbf{89.00}\tiny \pm 0.02}$ & ${82.80\tiny \pm 0.02}$ & ${83.77\tiny \pm 0.55}$ & $n.a.$ & ${57.46\tiny \pm 0.04}$ & ${57.62\tiny \pm 0.84}$ & $n.a.$\\
        & VGCN-BNN & ${68.20\tiny \pm 0.15}$ & ${20.27\tiny \pm 0.45}$ & ${79.65\tiny \pm 0.06}$ & ${82.16\tiny \pm 0.24}$ & ${69.72\tiny \pm 2.37}$ & $n.a.$ & ${58.10\tiny \pm 0.53}$ & ${52.08\tiny \pm 2.48}$ & $n.a.$\\
        & GKDE-GCN & ${65.90\tiny \pm 0.53}$ & ${\mathbf{9.04}\tiny \pm 0.28}$ & ${82.73\tiny \pm 0.37}$ & ${77.03\tiny \pm 0.34}$ & ${70.32\tiny \pm 0.66}$ & $n.a.$ & ${49.81\tiny \pm 0.44}$ & ${45.92\tiny \pm 0.90}$ & $n.a.$\\
        & GPN & ${\mathbf{82.10}\tiny \pm 0.29}$ & ${9.22\tiny \pm 0.19}$ & ${88.48\tiny \pm 0.07}$ & ${82.49\tiny \pm 0.17}$ & ${\mathbf{87.63}\tiny \pm 0.18}$ & ${{74.55}\tiny \pm 0.24}$ & ${56.78\tiny \pm 0.38}$ & ${\mathbf{67.94}\tiny \pm 0.28}$ & ${{48.03}\tiny \pm 0.34}$\\

        \midrule
        
        \multirow{14}{*}{\shortstack[l]{Amazon\\Photos}} 
        & LP & ${85.88\tiny \pm 0.00}$ & ${73.38\tiny \pm 0.00}$ & ${89.27\tiny \pm 0.00}$ & ${\mathbf{94.24}\tiny \pm 0.00}$ & ${90.26\tiny \pm 0.00}$ & $n.a.$ & ${\mathbf{90.24}\tiny \pm 0.00}$ & ${85.55\tiny \pm 0.00}$ & $n.a.$\\
        & GKDE & ${75.38\tiny \pm 0.00}$ & ${52.21\tiny \pm 0.00}$ & ${85.94\tiny \pm 0.00}$ & ${76.51\tiny \pm 0.00}$ & ${60.83\tiny \pm 0.00}$ & $n.a.$ & ${66.72\tiny \pm 0.00}$ & ${59.09\tiny \pm 0.00}$ & $n.a.$\\
        & Matern-GGP & ${86.10\tiny \pm 0.01}$ & ${\mathbf{8.54}\tiny \pm 0.02}$ & ${88.65\tiny \pm 0.00}$ & ${87.26\tiny \pm 0.00}$ & ${86.75\tiny \pm 0.00}$ & $n.a.$ & ${75.22\tiny \pm 0.00}$ & ${74.39\tiny \pm 0.03}$ & $n.a.$\\
        & GGP & $n.f.$ & $n.f.$ & $n.f.$ & $n.f.$ & $n.f.$ & $n.f.$ & $n.f.$ & $n.f.$ & $n.f.$\\
        & APPNP & ${\mathbf{92.12}\tiny \pm 0.01}$ & ${13.68\tiny \pm 0.02}$ & ${\mathbf{95.42}\tiny \pm 0.01}$ & ${77.45\tiny \pm 0.10}$ & $n.a.$ & $n.a.$ & ${67.50\tiny \pm 0.12}$ & $n.a.$ & $n.a.$\\
        & VGCN & ${90.95\tiny \pm 0.01}$ & ${10.37\tiny \pm 0.03}$ & ${94.24\tiny \pm 0.01}$ & ${82.44\tiny \pm 0.07}$ & $n.a.$ & $n.a.$ & ${72.60\tiny \pm 0.11}$ & $n.a.$ & $n.a.$\\
        & RGCN & ${81.00\tiny \pm 0.42}$ & ${40.20\tiny \pm 0.25}$ & ${87.59\tiny \pm 0.69}$ & ${75.25\tiny \pm 0.23}$ & $n.a.$ & $n.a.$ & ${67.53\tiny \pm 0.29}$ & $n.a.$ & $n.a.$\\
        & VGCN-Dropout & ${90.42\tiny \pm 0.01}$ & ${12.76\tiny \pm 0.05}$ & ${94.04\tiny \pm 0.01}$ & ${80.90\tiny \pm 0.08}$ & ${70.11\tiny \pm 0.14}$ & $n.a.$ & ${70.55\tiny \pm 0.12}$ & ${53.16\tiny \pm 0.13}$ & $n.a.$\\
        & DropEdge & ${91.03\tiny \pm 0.01}$ & ${12.25\tiny \pm 0.02}$ & ${94.22\tiny \pm 0.01}$ & ${82.48\tiny \pm 0.09}$ & ${71.67\tiny \pm 0.11}$ & $n.a.$ & ${72.75\tiny \pm 0.15}$ & ${54.98\tiny \pm 0.10}$ & $n.a.$\\
        & VGCN-Energy & ${90.95\tiny \pm 0.01}$ & ${10.37\tiny \pm 0.03}$ & ${94.24\tiny \pm 0.01}$ & ${82.44\tiny \pm 0.07}$ & ${79.64\tiny \pm 0.07}$ & $n.a.$ & ${72.60\tiny \pm 0.11}$ & ${71.71\tiny \pm 0.14}$ & $n.a.$\\
        & VGCN-Ensemble & ${90.94\tiny \pm 0.01}$ & ${10.38\tiny \pm 0.02}$ & ${94.28\tiny \pm 0.01}$ & ${82.72\tiny \pm 0.03}$ & ${88.53\tiny \pm 0.39}$ & $n.a.$ & ${72.98\tiny \pm 0.08}$ & ${83.28\tiny \pm 0.46}$ & $n.a.$\\
        & VGCN-BNN & ${84.51\tiny \pm 0.23}$ & ${29.42\tiny \pm 0.32}$ & ${91.88\tiny \pm 0.19}$ & ${72.03\tiny \pm 0.38}$ & ${64.10\tiny \pm 1.98}$ & $n.a.$ & ${62.85\tiny \pm 0.49}$ & ${54.09\tiny \pm 1.76}$ & $n.a.$\\
        & GKDE-GCN & ${79.74\tiny \pm 0.99}$ & ${12.84\tiny \pm 0.31}$ & ${89.84\tiny \pm 0.73}$ & ${73.65\tiny \pm 1.13}$ & ${69.09\tiny \pm 0.81}$ & $n.a.$ & ${62.45\tiny \pm 1.20}$ & ${59.68\tiny \pm 0.75}$ & $n.a.$\\
        & GPN & ${90.44\tiny \pm 0.07}$ & ${11.95\tiny \pm 0.12}$ & ${94.01\tiny \pm 0.07}$ & ${82.72\tiny \pm 0.24}$ & ${{91.98}\tiny \pm 0.22}$ & ${76.57\tiny \pm 0.49}$ & ${74.55\tiny \pm 0.39}$ & ${{86.29}\tiny \pm 0.35}$ & ${64.00\tiny \pm 0.68}$\\

        \midrule
        
        \multirow{14}{*}{\shortstack[l]{Coauthor\\CS}}
         & LP & ${83.34\tiny \pm 0.00}$ & ${76.67\tiny \pm 0.00}$ & ${82.89\tiny \pm 0.00}$ & ${{86.64}\tiny \pm 0.00}$ & ${86.40\tiny \pm 0.00}$ & $n.a.$ & ${79.05\tiny \pm 0.00}$ & ${79.56\tiny \pm 0.00}$ & $n.a.$\\
        & GKDE & ${79.27\tiny \pm 0.00}$ & ${70.74\tiny \pm 0.00}$ & ${78.84\tiny \pm 0.00}$ & ${79.32\tiny \pm 0.00}$ & ${77.59\tiny \pm 0.00}$ & $n.a.$ & ${66.30\tiny \pm 0.00}$ & ${64.69\tiny \pm 0.00}$ & $n.a.$\\
        & Matern-GGP & ${83.56\tiny \pm 0.01}$ & ${\mathbf{6.18}\tiny \pm 0.00}$ & ${83.21\tiny \pm 0.00}$ & ${73.57\tiny \pm 0.00}$ & ${73.75\tiny \pm 0.00}$ & $n.a.$ & ${62.05\tiny \pm 0.00}$ & ${61.74\tiny \pm 0.00}$ & $n.a.$\\
        & GGP & $n.f.$ & $n.f.$ & $n.f.$ & $n.f.$ & $n.f.$ & $n.f.$ & $n.f.$ & $n.f.$ & $n.f.$\\
        & APPNP & ${\mathbf{92.96}\tiny \pm 0.01}$ & ${8.70\tiny \pm 0.02}$ & ${\mathbf{93.51}\tiny \pm 0.01}$ & ${81.88\tiny \pm 0.03}$ & $n.a.$ & $n.a.$ & ${75.85\tiny \pm 0.06}$ & $n.a.$ & $n.a.$\\
        & VGCN & ${92.61\tiny \pm 0.01}$ & ${7.00\tiny \pm 0.03}$ & ${93.07\tiny \pm 0.01}$ & ${85.35\tiny \pm 0.03}$ & $n.a.$ & $n.a.$ & ${80.87\tiny \pm 0.06}$ & $n.a.$ & $n.a.$\\
        & RGCN & ${92.02\tiny \pm 0.02}$ & ${11.52\tiny \pm 0.06}$ & ${92.49\tiny \pm 0.01}$ & ${77.00\tiny \pm 0.08}$ & $n.a.$ & $n.a.$ & ${71.87\tiny \pm 0.11}$ & $n.a.$ & $n.a.$\\
        & VGCN-Dropout & ${92.30\tiny \pm 0.01}$ & ${12.41\tiny \pm 0.05}$ & ${92.83\tiny \pm 0.01}$ & ${84.04\tiny \pm 0.03}$ & ${73.09\tiny \pm 0.10}$ & $n.a.$ & ${78.99\tiny \pm 0.07}$ & ${55.79\tiny \pm 0.14}$ & $n.a.$\\
        & DropEdge & ${92.63\tiny \pm 0.01}$ & ${11.40\tiny \pm 0.03}$ & ${92.92\tiny \pm 0.01}$ & ${84.62\tiny \pm 0.06}$ & ${75.68\tiny \pm 0.08}$ & $n.a.$ & ${79.74\tiny \pm 0.08}$ & ${58.85\tiny \pm 0.11}$ & $n.a.$\\
        & VGCN-Energy & ${92.61\tiny \pm 0.01}$ & ${7.00\tiny \pm 0.03}$ & ${93.07\tiny \pm 0.01}$ & ${85.35\tiny \pm 0.03}$ & ${87.33\tiny \pm 0.04}$ & $n.a.$ & ${80.87\tiny \pm 0.06}$ & ${82.79\tiny \pm 0.11}$ & $n.a.$\\
        & VGCN-Ensemble & ${92.66\tiny \pm 0.01}$ & ${7.03\tiny \pm 0.02}$ & ${93.07\tiny \pm 0.00}$ & ${85.43\tiny \pm 0.00}$ & ${83.19\tiny \pm 0.07}$ & $n.a.$ & ${{80.88}\tiny \pm 0.01}$ & ${72.27\tiny \pm 0.14}$ & $n.a.$\\
        & VGCN-BNN & ${92.21\tiny \pm 0.04}$ & ${12.15\tiny \pm 0.10}$ & ${92.54\tiny \pm 0.04}$ & ${80.48\tiny \pm 0.25}$ & ${70.75\tiny \pm 0.65}$ & $n.a.$ & ${73.64\tiny \pm 0.35}$ & ${54.41\tiny \pm 0.76}$ & $n.a.$\\
        & GKDE-GCN & ${92.35\tiny \pm 0.09}$ & ${8.04\tiny \pm 0.11}$ & ${93.13\tiny \pm 0.01}$ & ${85.02\tiny \pm 0.03}$ & ${84.45\tiny \pm 0.06}$ & $n.a.$ & ${80.15\tiny \pm 0.07}$ & ${77.90\tiny \pm 0.12}$ & $n.a.$\\
        & GPN & ${86.88\tiny \pm 0.10}$ & ${18.92\tiny \pm 0.11}$ & ${88.21\tiny \pm 0.10}$ & ${69.49\tiny \pm 0.39}$ & ${\mathbf{92.09}\tiny \pm 0.20}$ & ${{88.84}\tiny \pm 0.31}$ & ${55.41\tiny \pm 0.46}$ & ${\mathbf{90.28}\tiny \pm 0.18}$ & ${{86.54}\tiny \pm 0.42}$\\

        \midrule

        \multirow{14}{*}{\shortstack[l]{Coauthor\\Physics}}
        & LP & ${90.75\tiny \pm 0.00}$ & ${70.75\tiny \pm 0.00}$ & ${95.39\tiny \pm 0.00}$ & ${{91.78}\tiny \pm 0.00}$ & ${90.03\tiny \pm 0.00}$ & $n.a.$ & ${{70.58}\tiny \pm 0.00}$ & ${69.63\tiny \pm 0.00}$ & $n.a.$\\
        & GKDE & ${87.75\tiny \pm 0.00}$ & ${56.04\tiny \pm 0.00}$ & ${93.30\tiny \pm 0.00}$ & ${87.02\tiny \pm 0.00}$ & ${84.64\tiny \pm 0.00}$ & $n.a.$ & ${57.00\tiny \pm 0.00}$ & ${52.49\tiny \pm 0.00}$ & $n.a.$\\
        & Matern-GGP & $n.f.$ & $n.f.$ & $n.f.$ & $n.f.$ & $n.f.$ & $n.f.$ & $n.f.$ & $n.f.$ & $n.f.$\\
        & GGP & $n.f.$ & $n.f.$ & $n.f.$ & $n.f.$ & $n.f.$ & $n.f.$ & $n.f.$ & $n.f.$ & $n.f.$\\
        & APPNP & ${95.56\tiny \pm 0.00}$ & ${1.57\tiny \pm 0.01}$ & ${\mathbf{97.96}\tiny \pm 0.00}$ & ${90.37\tiny \pm 0.02}$ & $n.a.$ & $n.a.$ & ${61.46\tiny \pm 0.05}$ & $n.a.$ & $n.a.$\\
        & VGCN & ${95.59\tiny \pm 0.00}$ & ${\mathbf{1.28}\tiny \pm 0.02}$ & ${\mathbf{97.96}\tiny \pm 0.00}$ & ${90.29\tiny \pm 0.02}$ & $n.a.$ & $n.a.$ & ${63.63\tiny \pm 0.09}$ & $n.a.$ & $n.a.$\\
        & RGCN & ${95.33\tiny \pm 0.00}$ & ${5.11\tiny \pm 0.07}$ & ${97.92\tiny \pm 0.00}$ & ${75.62\tiny \pm 0.13}$ & $n.a.$ & $n.a.$ & ${56.27\tiny \pm 0.13}$ & $n.a.$ & $n.a.$\\
        & VGCN-Dropout & ${95.51\tiny \pm 0.00}$ & ${3.07\tiny \pm 0.02}$ & ${97.92\tiny \pm 0.00}$ & ${89.63\tiny \pm 0.03}$ & ${87.10\tiny \pm 0.04}$ & $n.a.$ & ${62.53\tiny \pm 0.10}$ & ${51.19\tiny \pm 0.12}$ & $n.a.$\\
        & DropEdge & ${95.59\tiny \pm 0.00}$ & ${2.81\tiny \pm 0.01}$ & ${97.92\tiny \pm 0.00}$ & ${90.56\tiny \pm 0.02}$ & ${88.76\tiny \pm 0.04}$ & $n.a.$ & ${64.59\tiny \pm 0.08}$ & ${55.32\tiny \pm 0.14}$ & $n.a.$\\
        & VGCN-Energy & ${95.59\tiny \pm 0.00}$ & ${1.28\tiny \pm 0.02}$ & ${97.96\tiny \pm 0.00}$ & ${90.29\tiny \pm 0.02}$ & ${91.08\tiny \pm 0.05}$ & $n.a.$ & ${63.63\tiny \pm 0.09}$ & ${69.41\tiny \pm 0.11}$ & $n.a.$\\
        & VGCN-Ensemble & ${95.58\tiny \pm 0.00}$ & ${1.29\tiny \pm 0.01}$ & ${97.96\tiny \pm 0.00}$ & ${90.35\tiny \pm 0.00}$ & ${92.39\tiny \pm 0.00}$ & $n.a.$ & ${63.67\tiny \pm 0.00}$ & ${71.30\tiny \pm 0.02}$ & $n.a.$\\
        & VGCN-BNN & ${95.46\tiny \pm 0.02}$ & ${5.64\tiny \pm 0.09}$ & ${97.94\tiny \pm 0.01}$ & ${90.73\tiny \pm 0.21}$ & ${90.09\tiny \pm 0.50}$ & $n.a.$ & ${66.98\tiny \pm 0.32}$ & ${61.27\tiny \pm 1.47}$ & $n.a.$\\
        & GKDE-GCN & ${\mathbf{95.61}\tiny \pm 0.00}$ & ${1.51\tiny \pm 0.02}$ & ${97.95\tiny \pm 0.00}$ & ${87.38\tiny \pm 0.09}$ & ${84.62\tiny \pm 0.19}$ & $n.a.$ & ${57.97\tiny \pm 0.22}$ & ${56.30\tiny \pm 0.51}$ & $n.a.$\\
        & GPN(16) & ${94.32\tiny \pm 0.02}$ & ${10.61\tiny \pm 0.03}$ & ${97.40\tiny \pm 0.01}$ & ${85.20\tiny \pm 0.17}$ & ${\mathbf{94.51}\tiny \pm 0.15}$ & ${{89.63}\tiny \pm 0.24}$ & ${61.89\tiny \pm 0.20}$ & ${\mathbf{83.73}\tiny \pm 0.31}$ & ${{66.44}\tiny \pm 0.65}$\\

        \midrule
        
         \multirow{14}{*}{\shortstack[l]{OGBN\\Arxiv}} 
        & LP & ${64.27}$ & ${61.77}$ & ${66.84}$ & ${\mathbf {80.04}}$ & ${{75.22}}$ & $n.a.$ & ${{65.21}}$ & ${\mathbf{67.69}}$ & $n.a.$\\
        & GKDE & ${48.87}$ & ${22.59}$ & ${51.51}$ & ${68.12}$ & ${65.80}$ & $n.a.$ & ${47.22}$ & ${45.23}$ & $n.a.$\\
        & Matern-GGP & $n.f.$ & $n.f.$ & $n.f.$ & $n.f.$ & $n.f.$ & $n.f.$ & $n.f.$ & $n.f.$ & $n.f.$\\
        & GGP & $n.f.$ & $n.f.$ & $n.f.$ & $n.f.$ & $n.f.$ & $n.f.$ & $n.f.$ & $n.f.$ & $n.f.$\\
        & APPNP & ${71.46\tiny \pm 0.09}$ & ${3.96\tiny \pm 0.06}$ & ${75.47\tiny \pm 0.15}$ & ${65.21\tiny \pm 0.14}$ & $n.a.$ & $n.a.$ & ${43.23\tiny \pm 0.06}$ & $n.a.$ & $n.a.$\\
        & VGCN & ${71.89\tiny \pm 0.08}$ & ${2.18\tiny \pm 0.10}$ & ${75.61\tiny \pm 0.11}$ & ${64.91\tiny \pm 0.28}$ & $n.a.$ & $n.a.$ & ${42.72\tiny \pm 0.32}$ & $n.a.$ & $n.a.$\\
        & VGCN-Dropout & ${71.76\tiny \pm 0.07}$ & ${2.29\tiny \pm 0.08}$ & ${75.47\tiny \pm 0.12}$ & ${65.35\tiny \pm 0.27}$ & ${64.24\tiny \pm 0.26}$ & $n.a.$ & ${43.09\tiny \pm 0.30}$ & ${41.58\tiny \pm 0.23}$ & $n.a.$\\
        & VGCN-Energy & ${71.89\tiny \pm 0.08}$ & ${2.18\tiny \pm 0.10}$ & ${75.61\tiny \pm 0.11}$ & ${64.91\tiny \pm 0.28}$ & ${64.50\tiny \pm 0.38}$ & $n.a.$ & ${42.72\tiny \pm 0.32}$ & ${42.41\tiny \pm 0.39}$ & $n.a.$\\
        & VGCN-Ensemble & ${\mathbf{72.59}}$ & ${\mathbf{2.10}}$ & ${\mathbf{76.12}}$ & ${65.93}$ & ${70.77}$ & $n.a.$ & ${43.84}$ & ${50.63}$ & $n.a.$\\
        & VGCN-BNN & $n.a.$ & $n.a.$ & $n.a.$ & $n.a.$ & $n.a.$ & $n.a.$ & $n.a.$ & $n.a.$ & $n.a.$\\
        & GKDE-GCN & ${68.99\tiny \pm 0.40}$ & ${8.43\tiny \pm 0.43}$ & ${73.89\tiny \pm 0.33}$ & ${68.84\tiny \pm 0.18}$ & ${72.44\tiny \pm 0.50}$ & $n.a.$ & ${49.71\tiny \pm 0.59}$ & ${52.23\tiny \pm 0.66}$ & $n.a.$\\
        & GPN & ${69.08\tiny \pm 0.21}$ & ${6.96\tiny \pm 0.31}$ & ${73.84\tiny \pm 0.21}$ & ${66.33\tiny \pm 0.23}$ & ${74.82\tiny \pm 0.27}$ & ${{62.17}\tiny \pm 0.23}$ & ${46.35\tiny \pm 0.26}$ & ${58.71\tiny \pm 0.34}$ & ${{43.01}\tiny \pm 0.36}$\\
        \bottomrule
    \end{tabular}}
    \caption{Accuracy and ECE scores on the clean graphs. Accuracy and OOD detection scores on Left-Out classes using AUC-ROC and AUC-PR scores. OOD-AUC-ROC and OOD-AUC-APR scores are given as \emph{[Alea w/ Net] / [Epist w/ Net] / [Epist w/o Net]}. $n.a.$ means either model or metric not applicable and $n.f.$ means not finished within our constraints. Bold numbers indicate best results for Accuracy, ECE and OOD detection.}
    \label{tab:clean_loc_auroc_two}
\end{table*}

%% file: tables/shifts-cora.tex
\begin{figure}
    \centering
    \begin{subfigure}{\textwidth}
        \includegraphics[width=\textwidth]{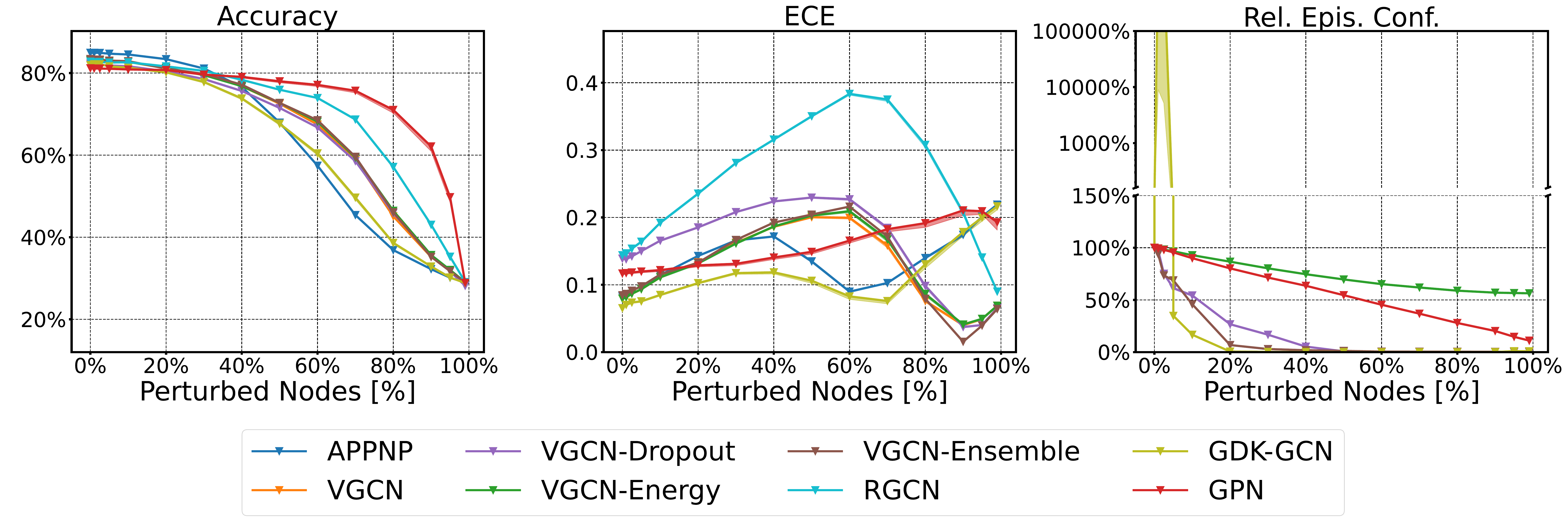}
        \caption{Bernoulli Feature Shift}
    \end{subfigure}
    \begin{subfigure}{\textwidth}
        \includegraphics[width=\textwidth]{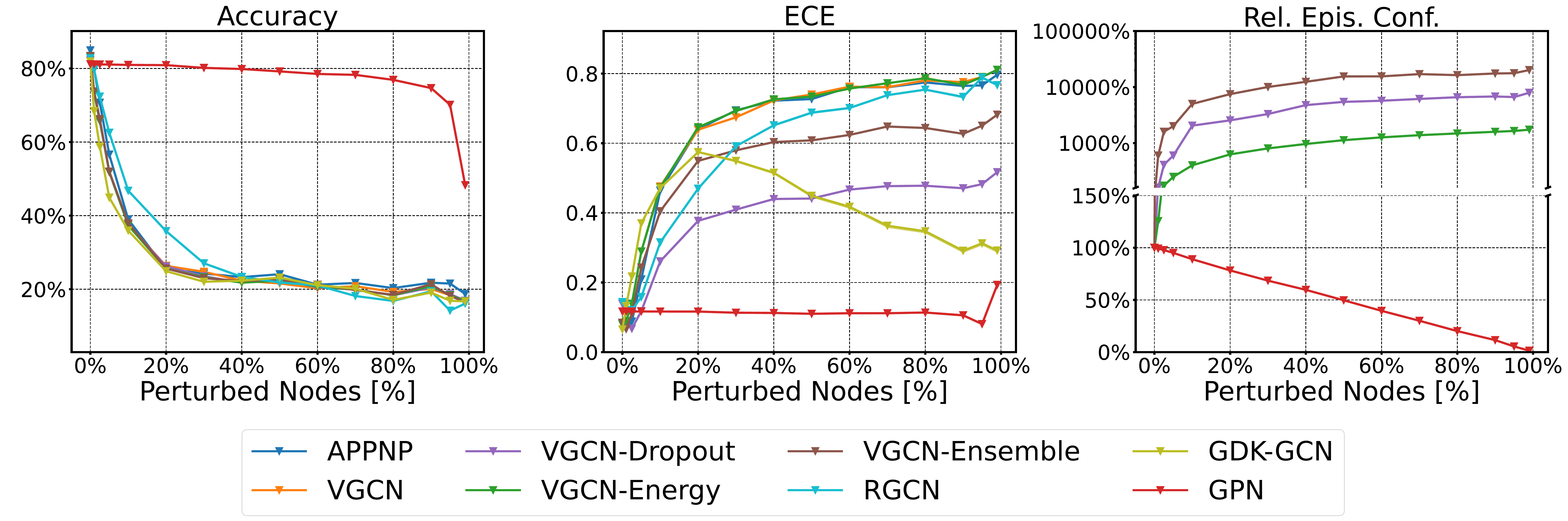}
        \caption{Unit Gaussian Feature Shift}
    \end{subfigure}
    \begin{subfigure}{\textwidth}
        \includegraphics[width=\textwidth]{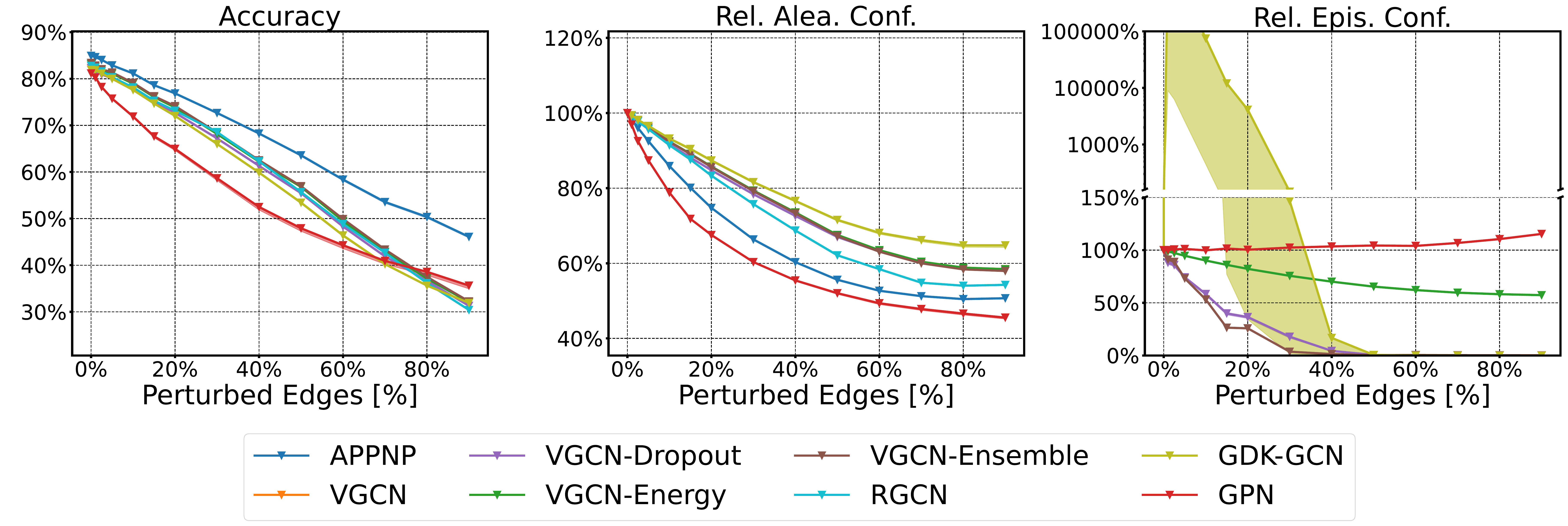}
        \caption{Random Edge Shift}
    \end{subfigure}
    \begin{subfigure}{\textwidth}
        \includegraphics[width=\textwidth]{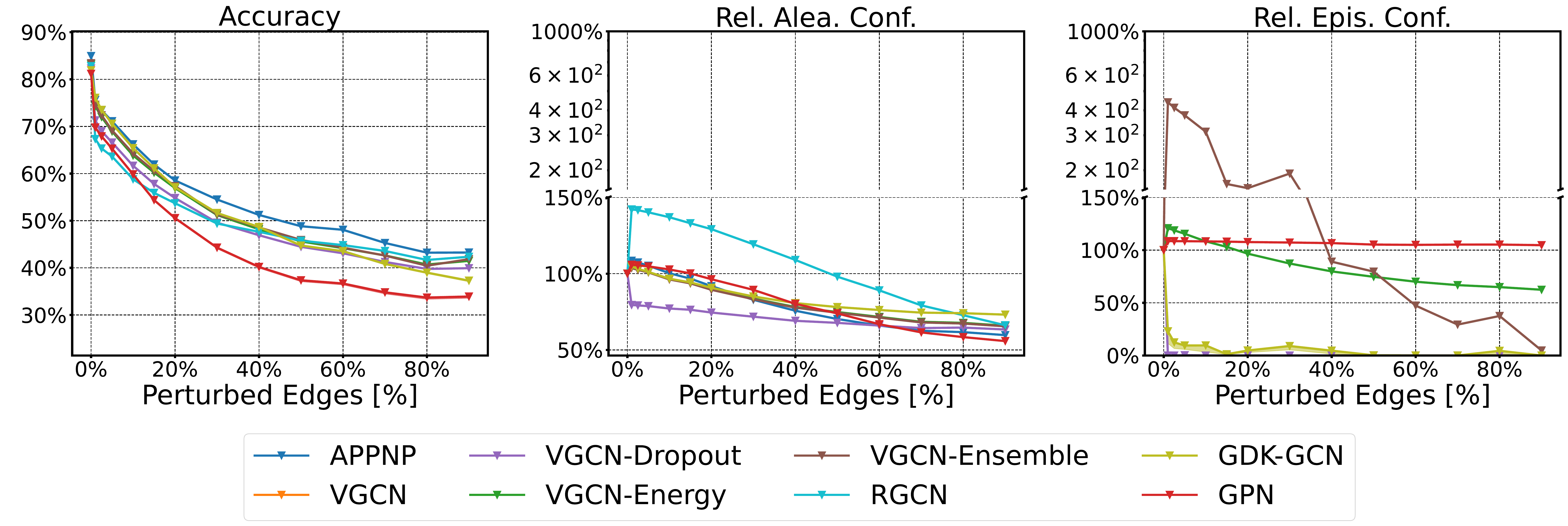}
        \caption{DICE \citep{Waniek2018} Edge Shift}
    \end{subfigure}
    \caption{Relative performance over different degrees of corruption of \emph{CoraML}. For feature shifts, we perturb different fractions of nodes (whose features are replaced with either random vectors from a Bernoulli noise or a Unit Gaussian noise) and show accuracy, ECE, and relative average epistemic confidence. For edge shifts, we perturb different fractions of edges (by replacing them at random or using the global and untargeted DICE \citep{Waniek2018} attack) and show accuracy, relative average aleatoric confidence, and relative average epistemic confidence.}
    \label{fig:shift-cora}
\end{figure}

%% file: tables/shifts-citeseer.tex
\begin{figure}
    \centering
    \begin{subfigure}{\textwidth}
        \includegraphics[width=\textwidth]{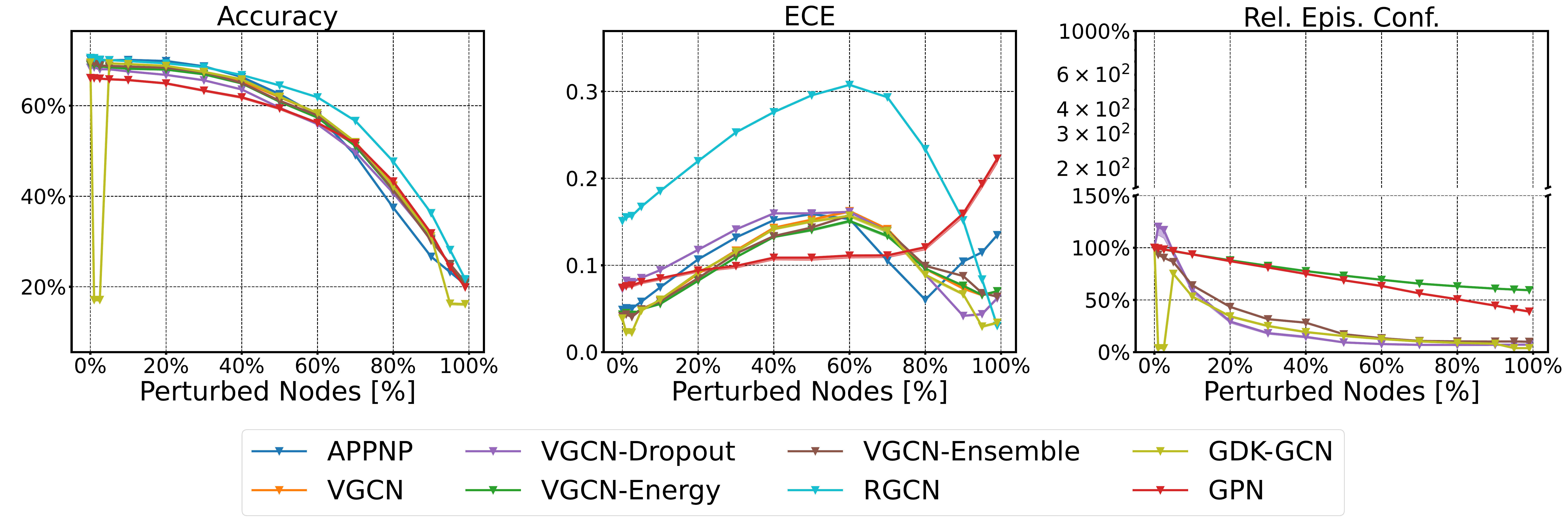}
        \caption{Bernoulli Feature Shift}
    \end{subfigure}
    \begin{subfigure}{\textwidth}
        \includegraphics[width=\textwidth]{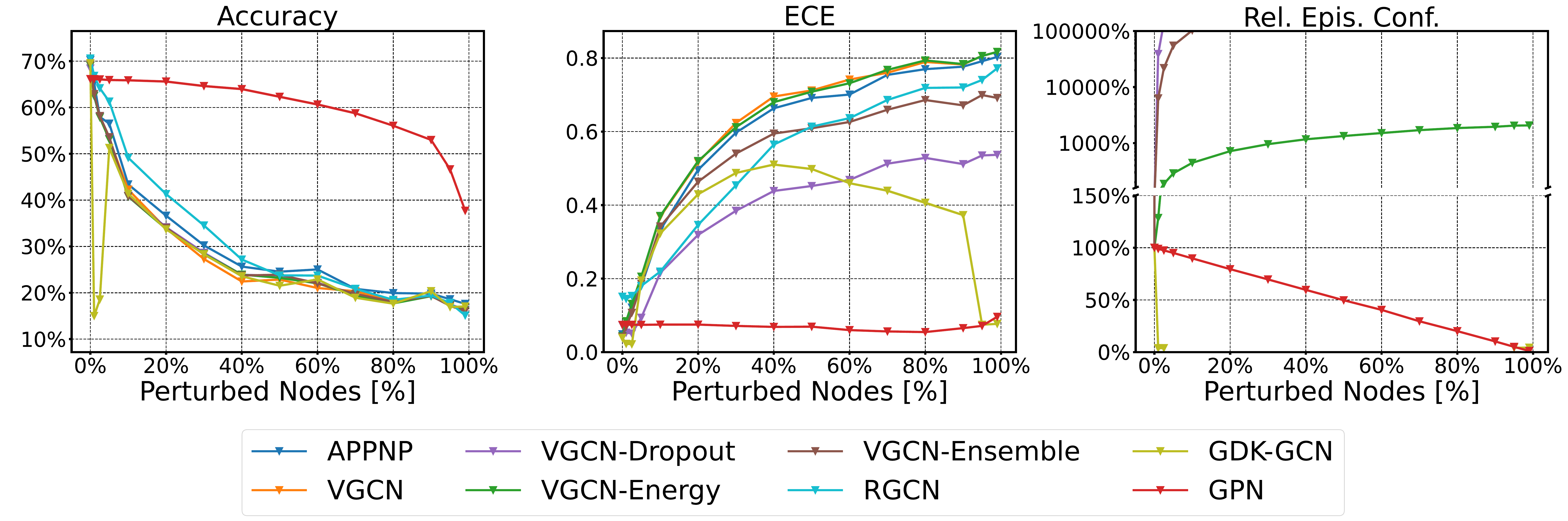}
        \caption{Unit Gaussian Feature Shift}
    \end{subfigure}
    \begin{subfigure}{\textwidth}
        \includegraphics[width=\textwidth]{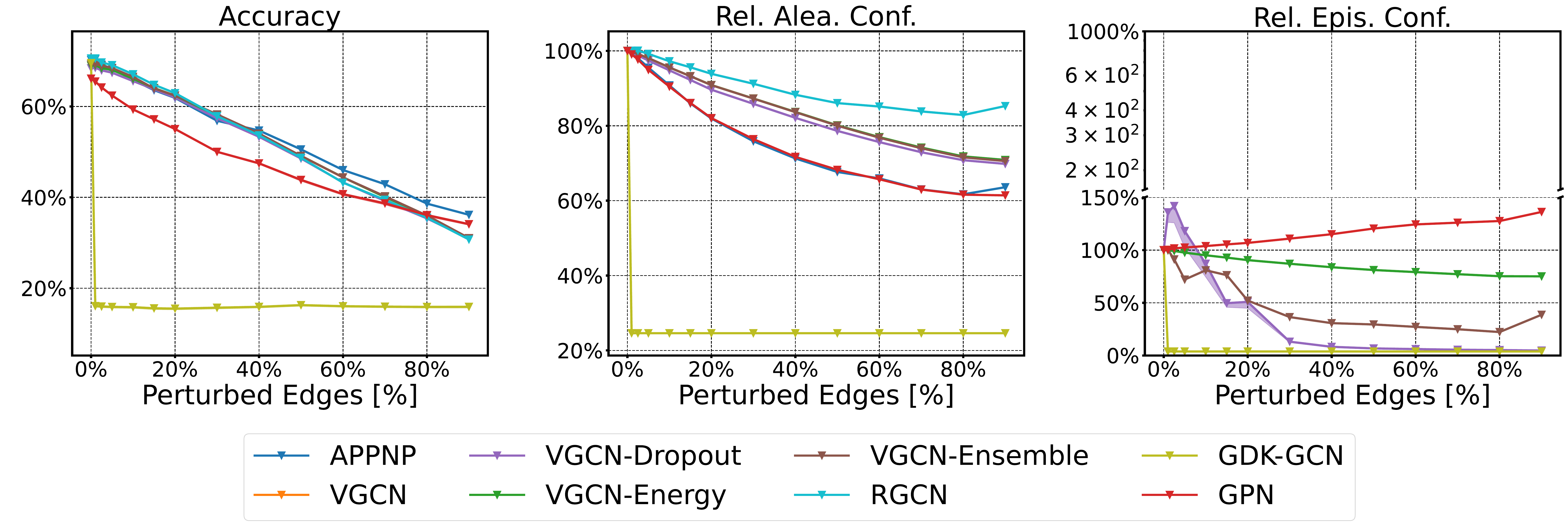}
        \caption{Random Edge Shift}
    \end{subfigure}
    \begin{subfigure}{\textwidth}
        \includegraphics[width=\textwidth]{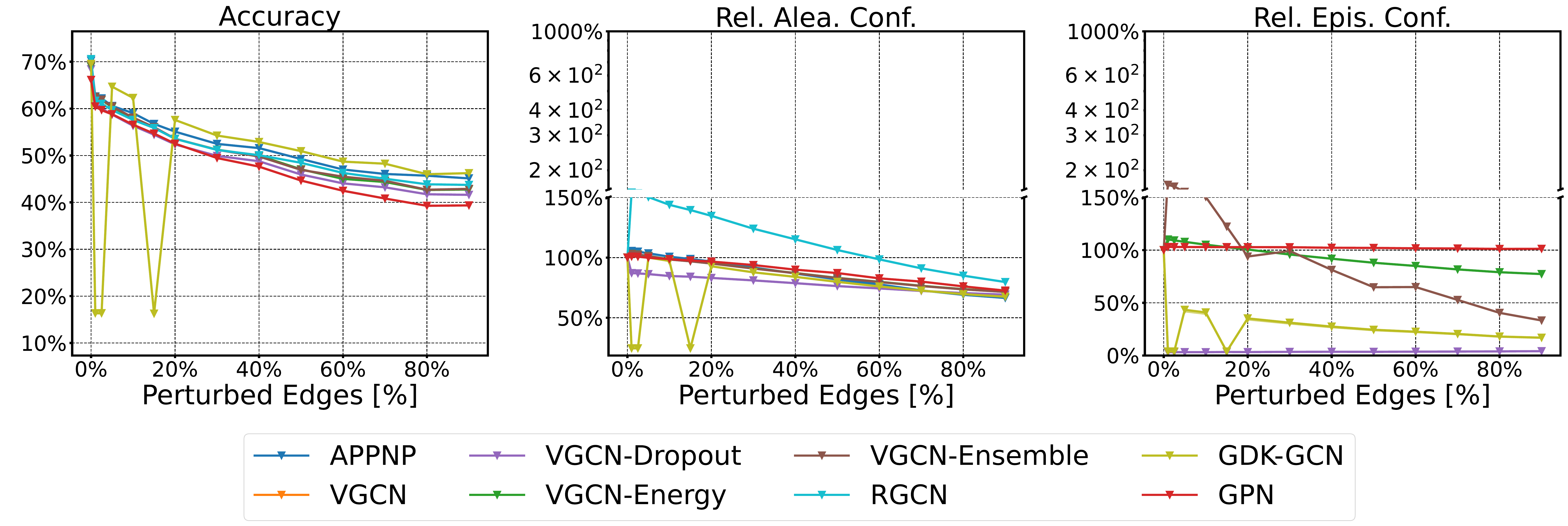}
        \caption{DICE \citep{Waniek2018} Edge Shift}
    \end{subfigure}
    \caption{Relative performance over different degrees of corruption of \emph{CiteSeer}. For feature shifts, we perturb different fractions of nodes (whose features are replaced with either random vectors from a Bernoulli noise or a Unit Gaussian noise) and show accuracy, ECE, and relative average epistemic confidence. For edge shifts, we perturb different fractions of edges (by replacing them at random or using the global and untargeted DICE \citep{Waniek2018} attack) and show accuracy, relative average aleatoric confidence, and relative average epistemic confidence.}
    \label{fig:shift-citeseer}
\end{figure}

%% file: tables/shifts-pubmed.tex
\begin{figure}
    \centering
    \begin{subfigure}{\textwidth}
        \includegraphics[width=\textwidth]{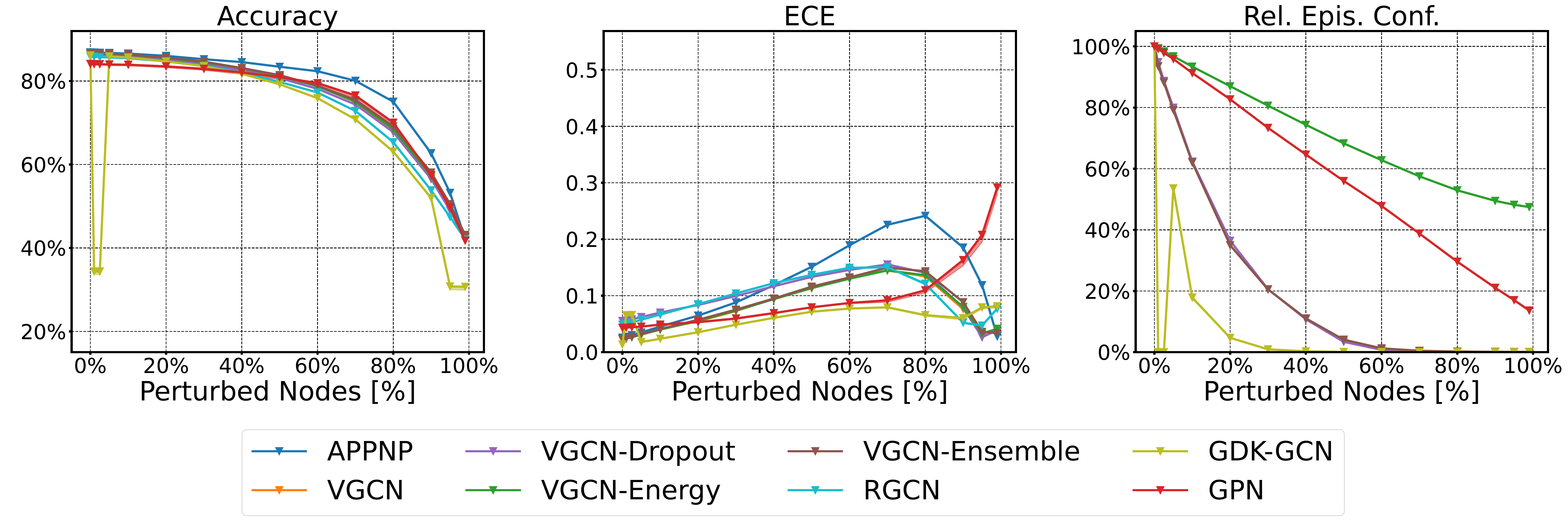}
        \caption{Bernoulli Feature Shift}
    \end{subfigure}
    \begin{subfigure}{\textwidth}
        \includegraphics[width=\textwidth]{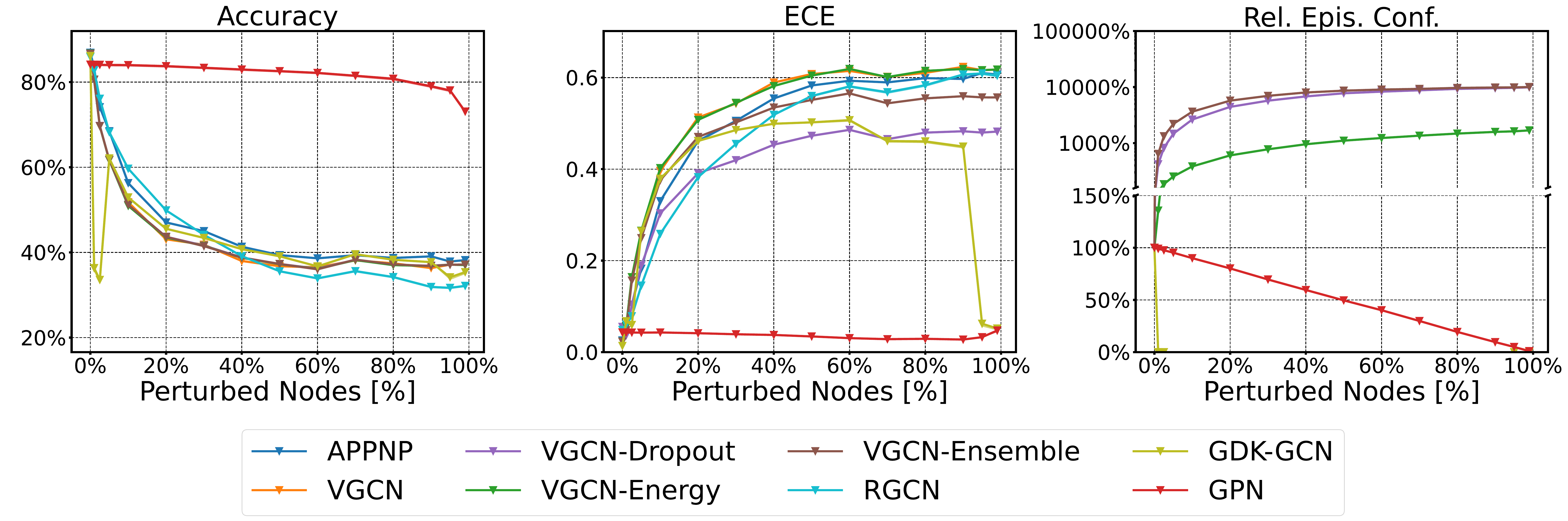}
        \caption{Unit Gaussian Feature Shift}
    \end{subfigure}
    \begin{subfigure}{\textwidth}
        \includegraphics[width=\textwidth]{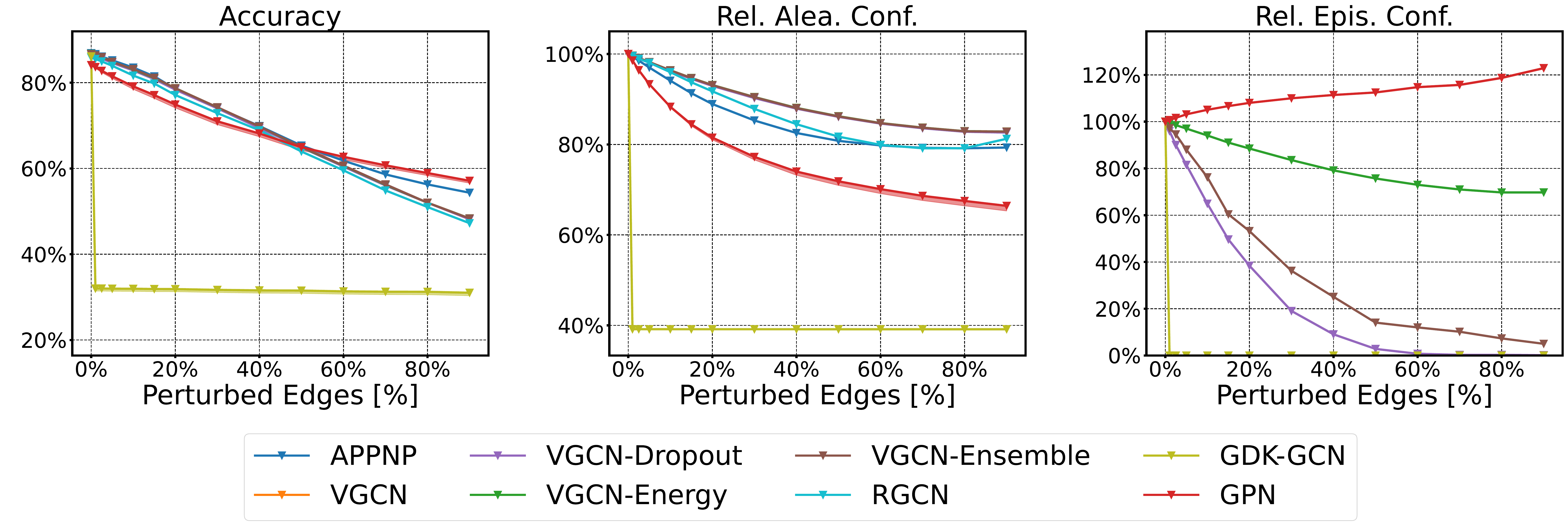}
        \caption{Random Edge Shift}
    \end{subfigure}
    \begin{subfigure}{\textwidth}
        \includegraphics[width=\textwidth]{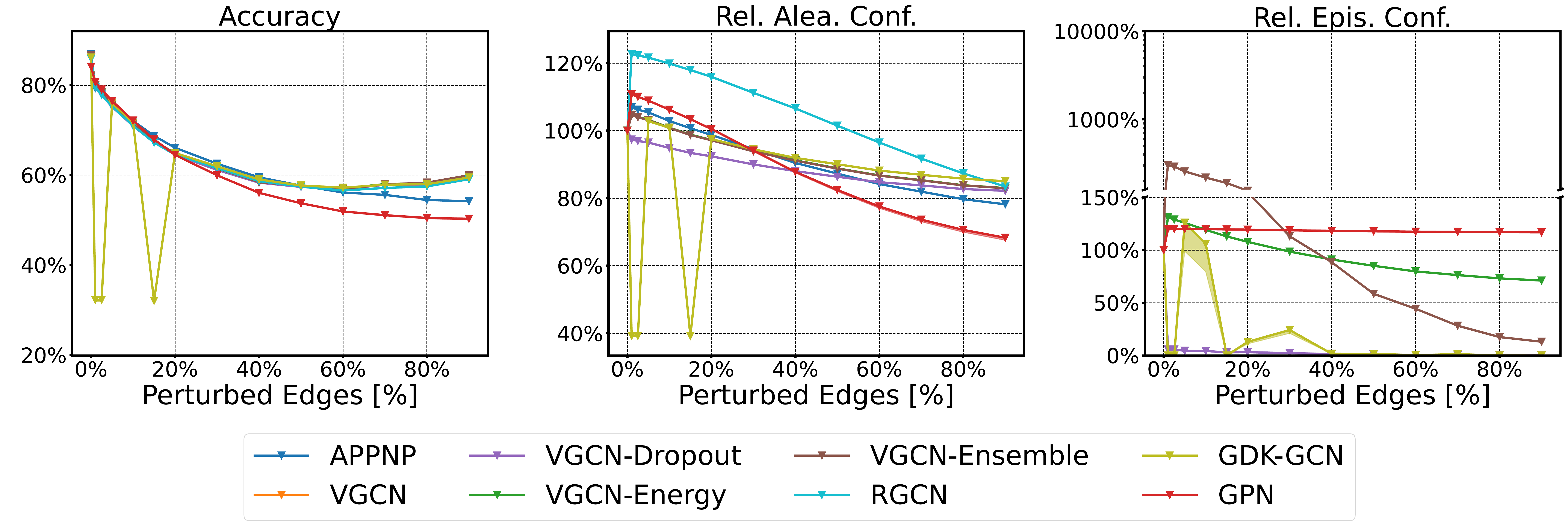}
        \caption{DICE \citep{Waniek2018} Edge Shift}
    \end{subfigure}
    \caption{Relative performance over different degrees of corruption of \emph{PubMed}. For feature shifts, we perturb different fractions of nodes (whose features are replaced with either random vectors from a Bernoulli noise or a Unit Gaussian noise) and show accuracy, ECE, and relative average epistemic confidence. For edge shifts, we perturb different fractions of edges (by replacing them at random or using the global and untargeted DICE \citep{Waniek2018} attack) and show accuracy, relative average aleatoric confidence, and relative average epistemic confidence.}
    \label{fig:shift-pubmed}
\end{figure}

%% file: tables/shifts-amazon-computers.tex
\begin{figure}
    \centering
    \begin{subfigure}{\textwidth}
        \includegraphics[width=\textwidth]{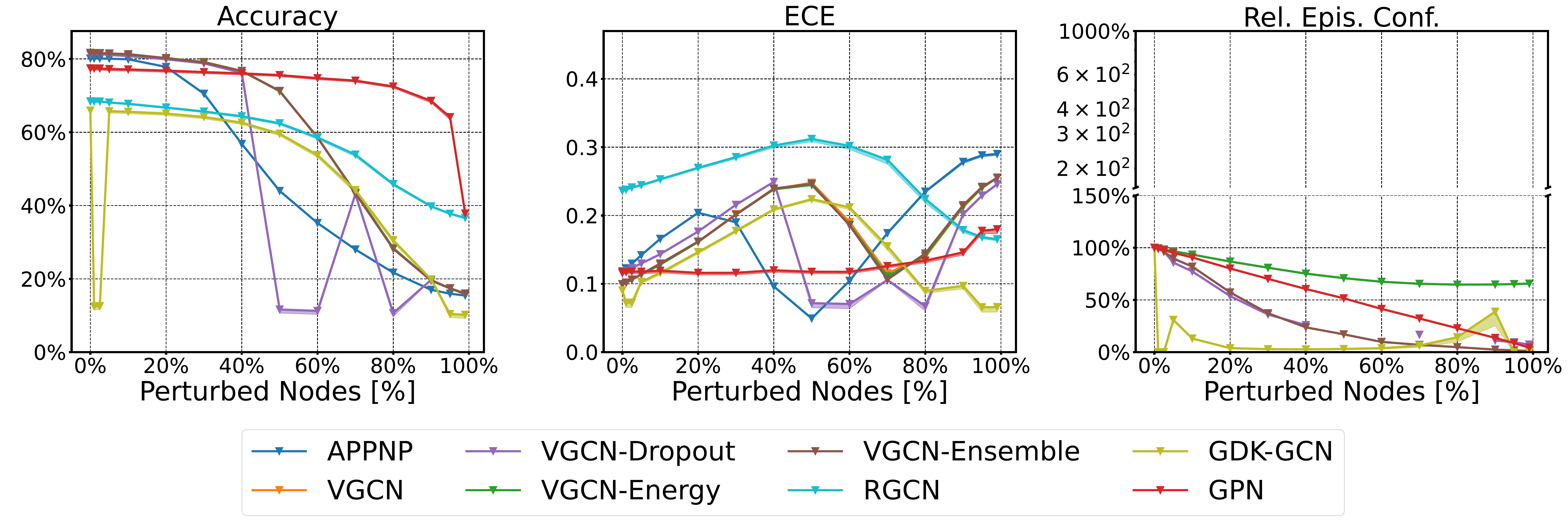}
        \caption{Bernoulli Feature Shift}
    \end{subfigure}
    \begin{subfigure}{\textwidth}
        \includegraphics[width=\textwidth]{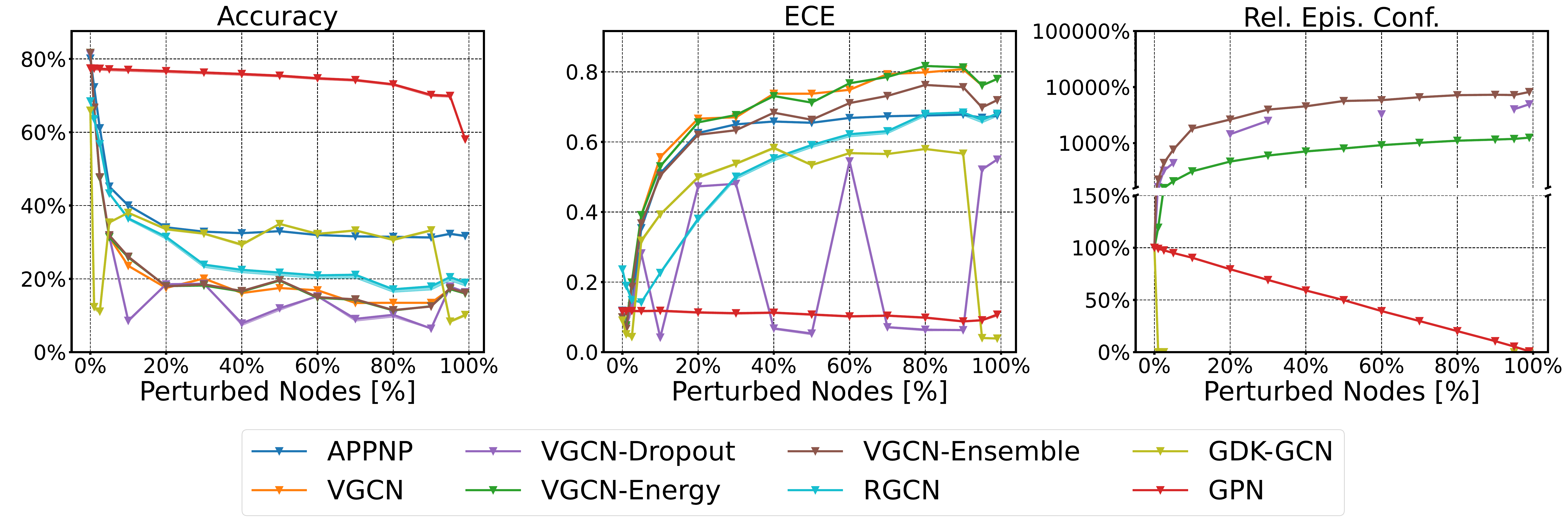}
        \caption{Unit Gaussian Feature Shift}
    \end{subfigure}
    \begin{subfigure}{\textwidth}
        \includegraphics[width=\textwidth]{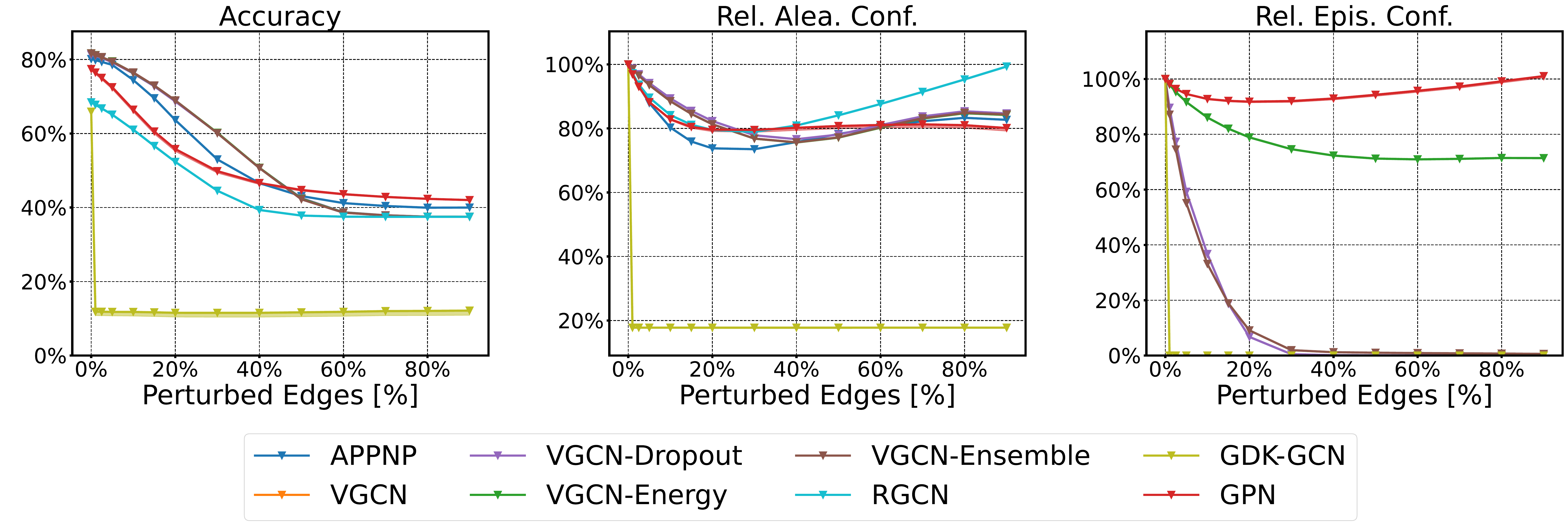}
        \caption{Random Edge Shift}
    \end{subfigure}
    \begin{subfigure}{\textwidth}
        \includegraphics[width=\textwidth]{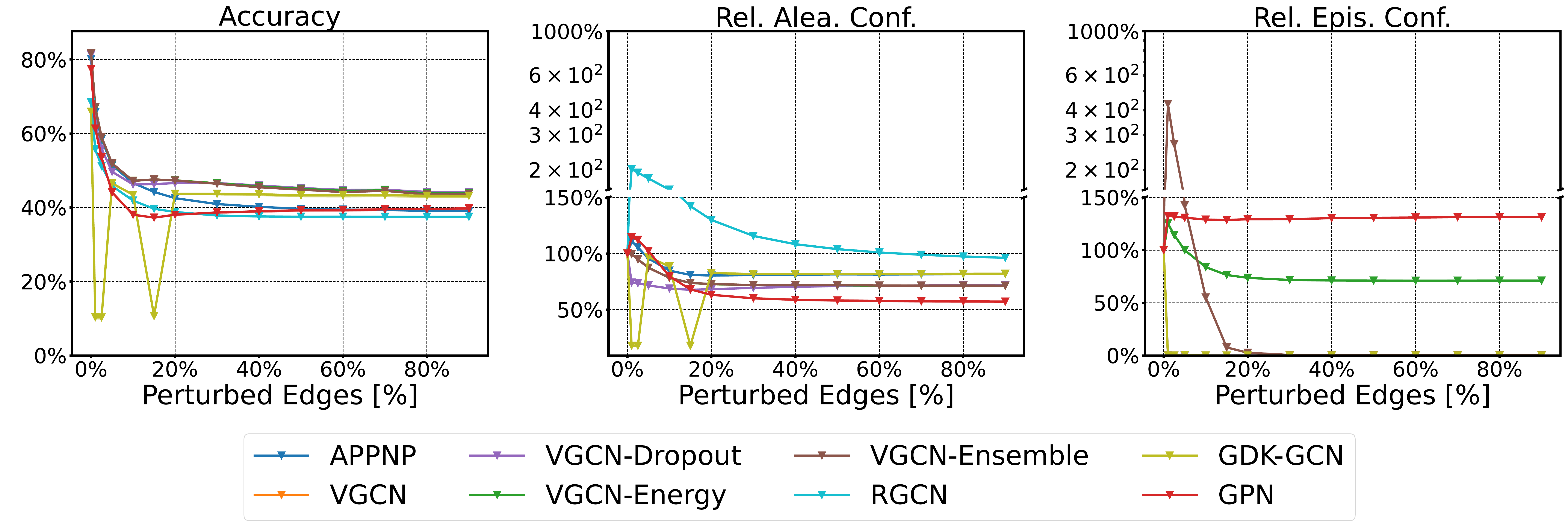}
        \caption{DICE \citep{Waniek2018} Edge Shift}
    \end{subfigure}
    \caption{Relative performance over different degrees of corruption of \emph{Amazon Computers}. For feature shifts, we perturb different fractions of nodes (whose features are replaced with either random vectors from a Bernoulli noise or a Unit Gaussian noise) and show accuracy, ECE, and relative average epistemic confidence. For edge shifts, we perturb different fractions of edges (by replacing them at random or using the global and untargeted DICE \citep{Waniek2018} attack) and show accuracy, relative average aleatoric confidence, and relative average epistemic confidence.}
    \label{fig:shift-amazon-computers}
\end{figure}

%% file: tables/shifts-amazon-photos.tex
\begin{figure}
    \centering
    \begin{subfigure}{\textwidth}
        \includegraphics[width=\textwidth]{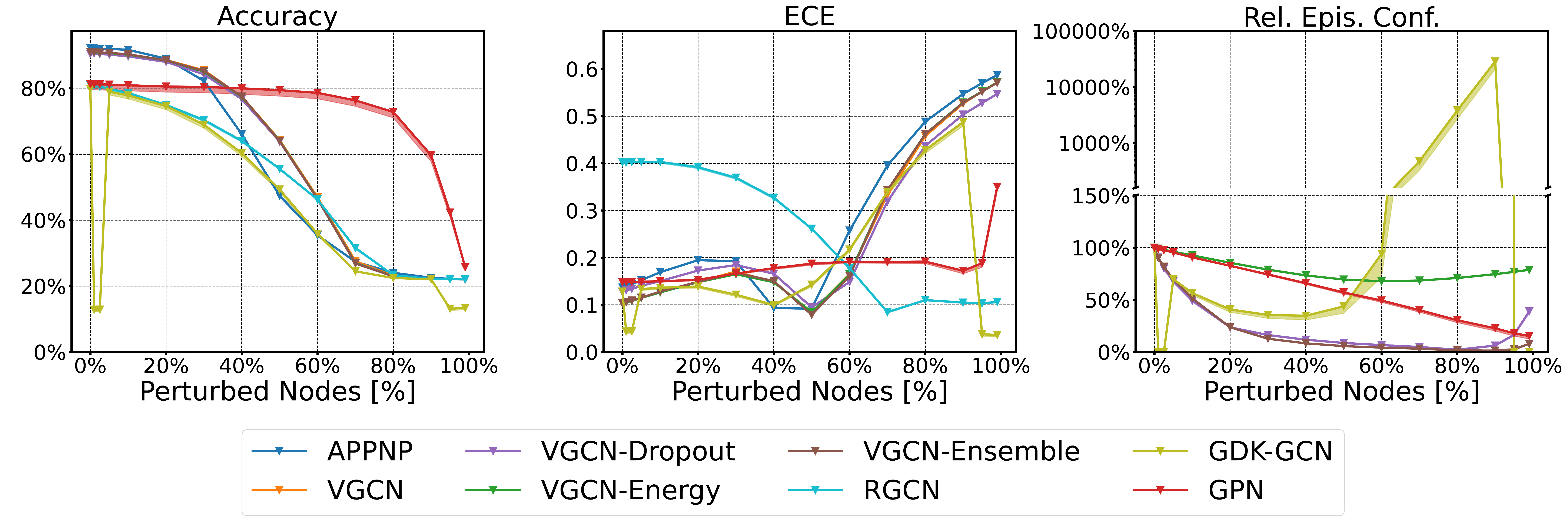}
        \caption{Bernoulli Feature Shift}
    \end{subfigure}
    \begin{subfigure}{\textwidth}
        \includegraphics[width=\textwidth]{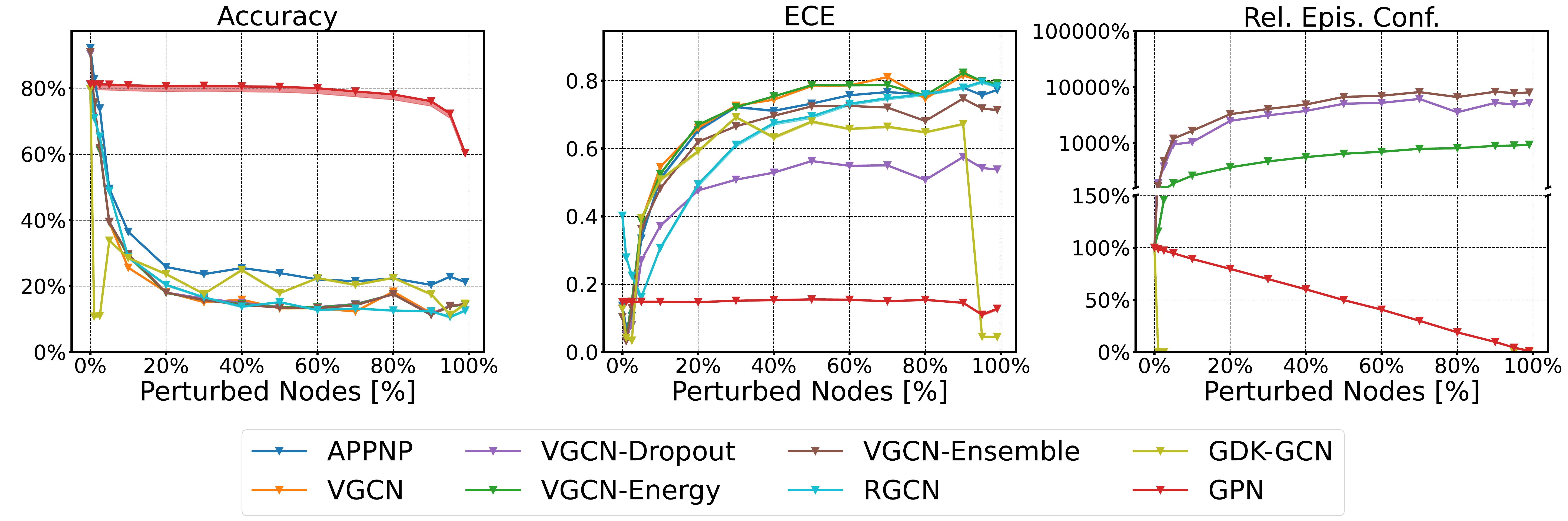}
        \caption{Unit Gaussian Feature Shift}
    \end{subfigure}
    \begin{subfigure}{\textwidth}
        \includegraphics[width=\textwidth]{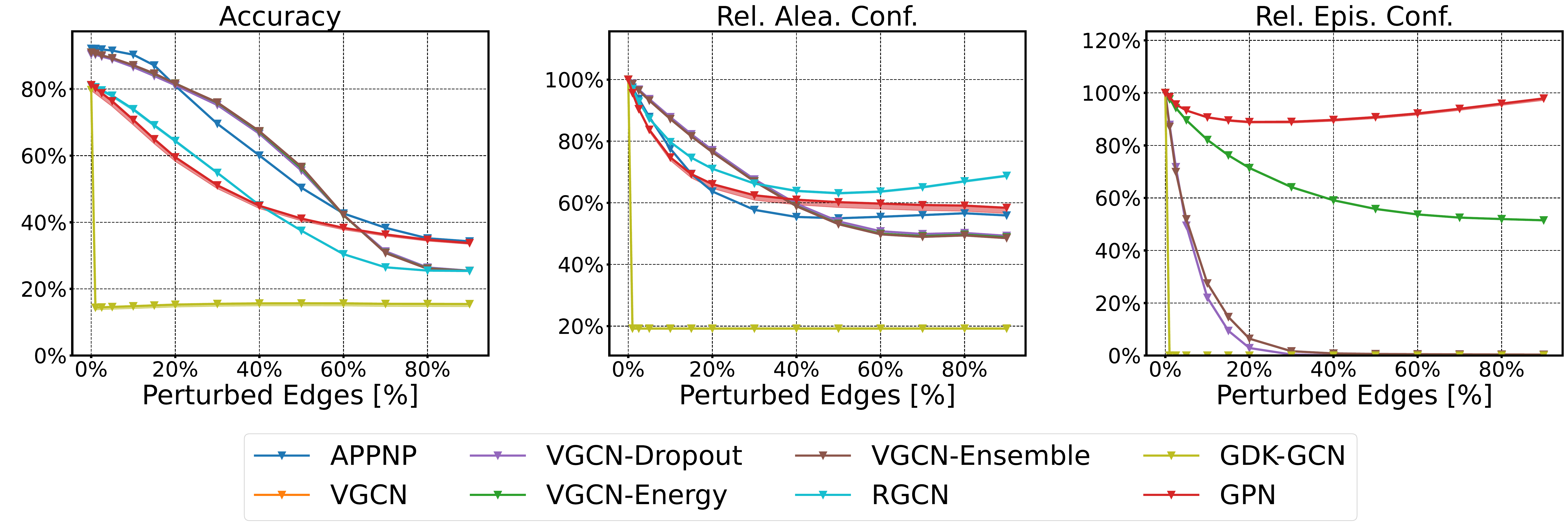}
        \caption{Random Edge Shift}
    \end{subfigure}
    \begin{subfigure}{\textwidth}
        \includegraphics[width=\textwidth]{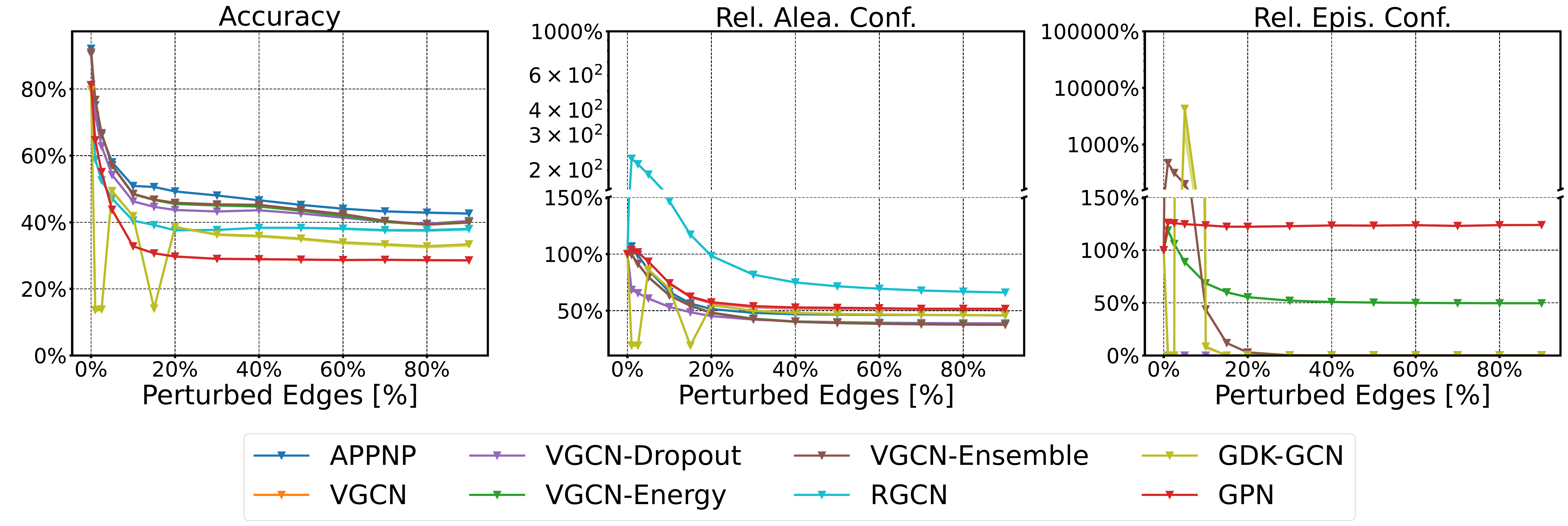}
        \caption{DICE \citep{Waniek2018} Edge Shift}
    \end{subfigure}
    \caption{Relative performance over different degrees of corruption of \emph{Amazon Photos}. For feature shifts, we perturb different fractions of nodes (whose features are replaced with either random vectors from a Bernoulli noise or a Unit Gaussian noise) and show accuracy, ECE, and relative average epistemic confidence. For edge shifts, we perturb different fractions of edges (by replacing them at random or using the global and untargeted DICE \citep{Waniek2018} attack) and show accuracy, relative average aleatoric confidence, and relative average epistemic confidence.}
    \label{fig:shift-amazon-photos}
\end{figure}

%% file: tables/shifts-coauthor-cs.tex
\begin{figure}
    \centering
    \begin{subfigure}{\textwidth}
        \includegraphics[width=\textwidth]{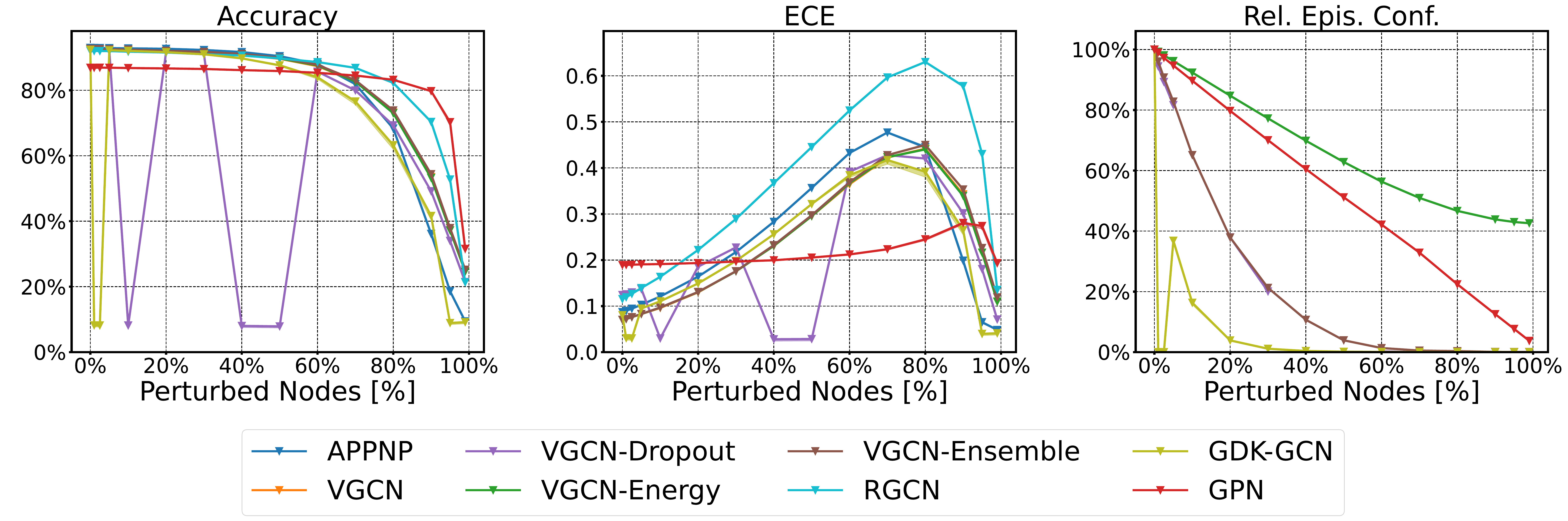}
        \caption{Bernoulli Feature Shift}
    \end{subfigure}
    \begin{subfigure}{\textwidth}
        \includegraphics[width=\textwidth]{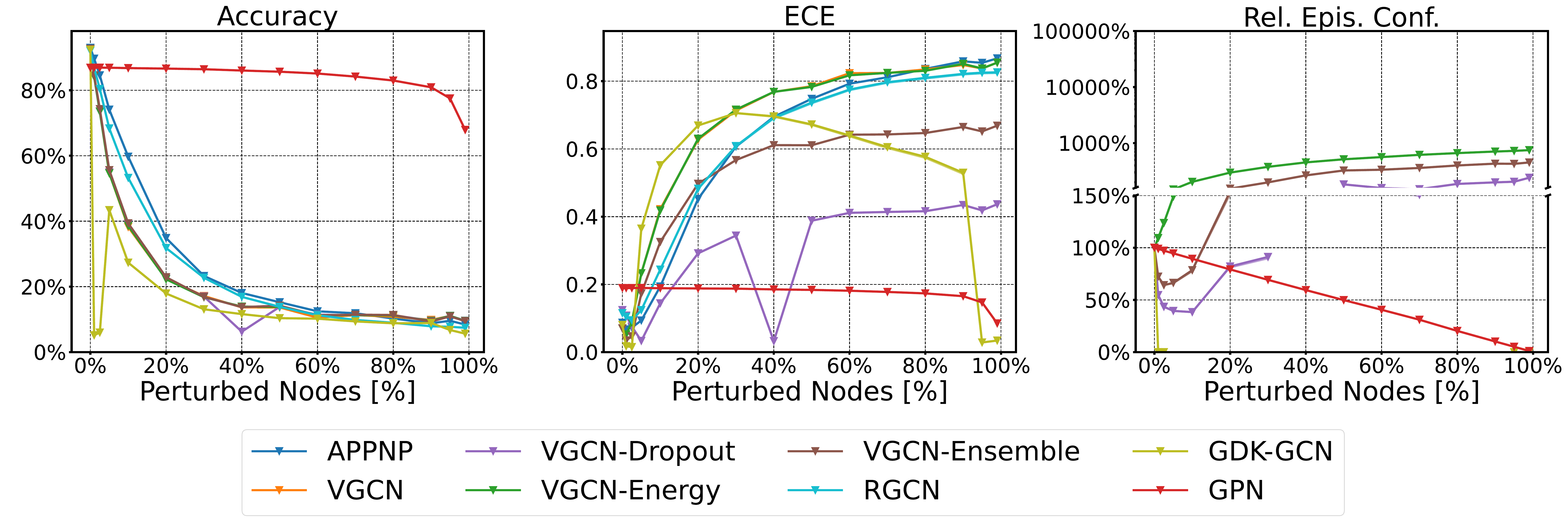}
        \caption{Unit Gaussian Feature Shift}
    \end{subfigure}
    \begin{subfigure}{\textwidth}
        \includegraphics[width=\textwidth]{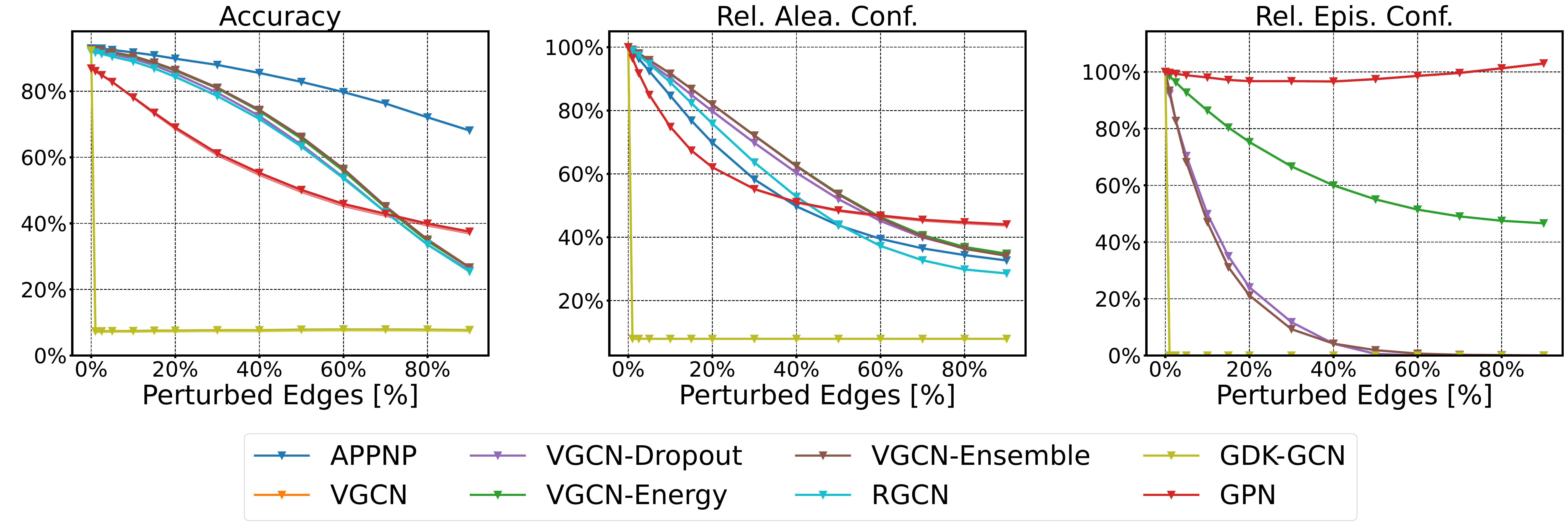}
        \caption{Random Edge Shift}
    \end{subfigure}
    \begin{subfigure}{\textwidth}
        \includegraphics[width=\textwidth]{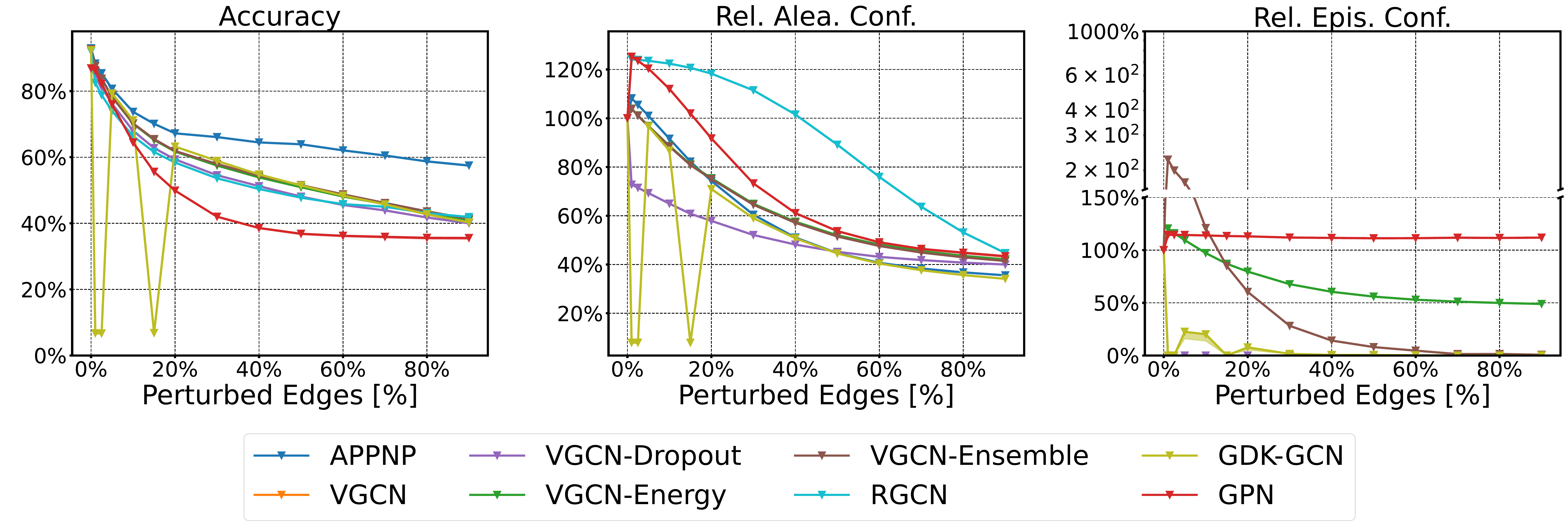}
        \caption{DICE \citep{Waniek2018} Edge Shift}
    \end{subfigure}
    \caption{Relative performance over different degrees of corruption of \emph{Coauthor CS}. For feature shifts, we perturb different fractions of nodes (whose features are replaced with either random vectors from a Bernoulli noise or a Unit Gaussian noise) and show accuracy, ECE, and relative average epistemic confidence. For edge shifts, we perturb different fractions of edges (by replacing them at random or using the global and untargeted DICE \citep{Waniek2018} attack) and show accuracy, relative average aleatoric confidence, and relative average epistemic confidence.}
    \label{fig:shift-coauthor-cs}
\end{figure}

%% file: tables/shifts-coauthor-physics.tex
\begin{figure}
    \centering
    \begin{subfigure}{\textwidth}
        \includegraphics[width=\textwidth]{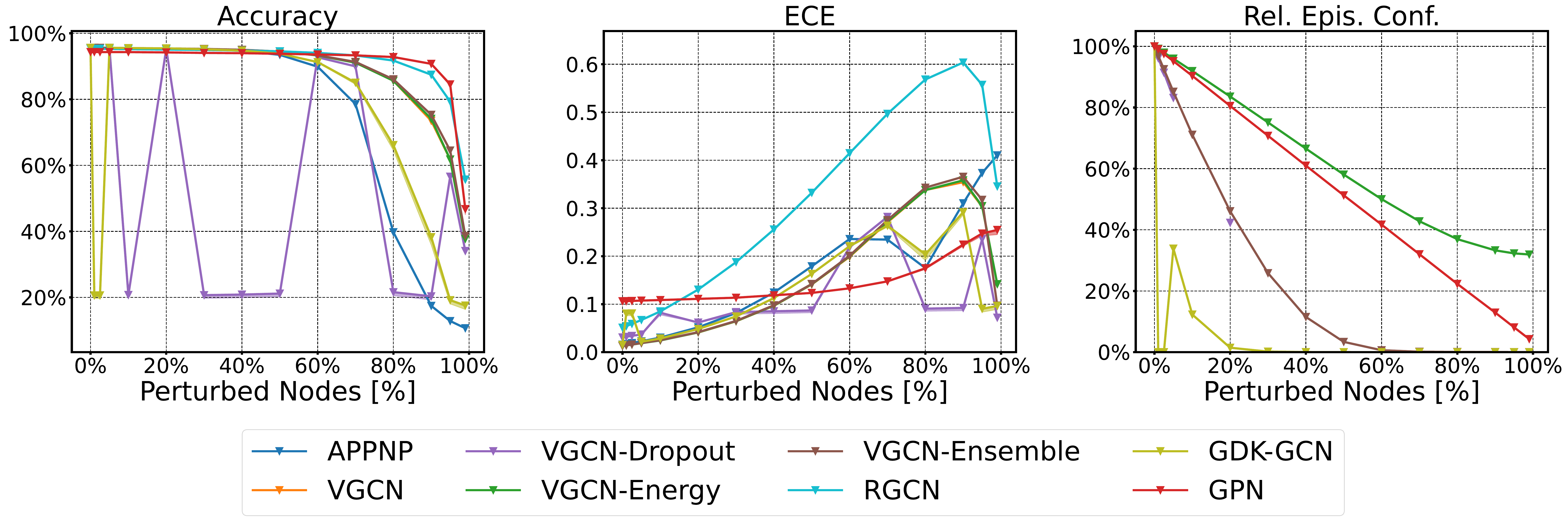}
        \caption{Bernoulli Feature Shift}
    \end{subfigure}
    \begin{subfigure}{\textwidth}
        \includegraphics[width=\textwidth]{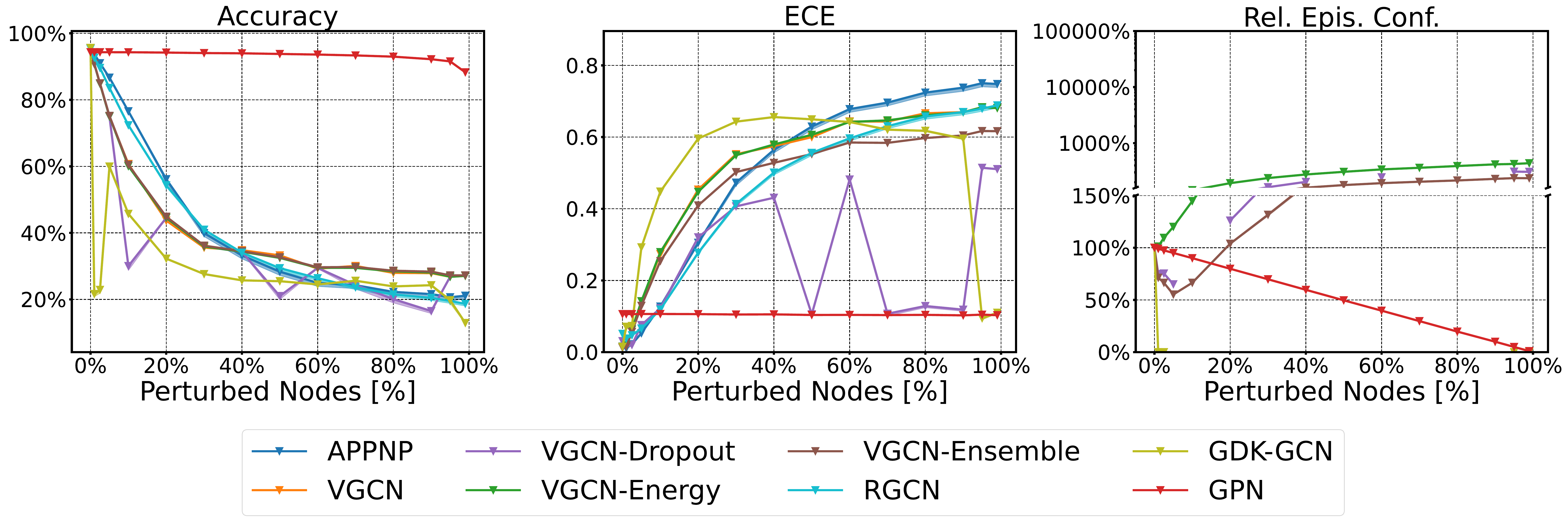}
        \caption{Unit Gaussian Feature Shift}
    \end{subfigure}
    \begin{subfigure}{\textwidth}
        \includegraphics[width=\textwidth]{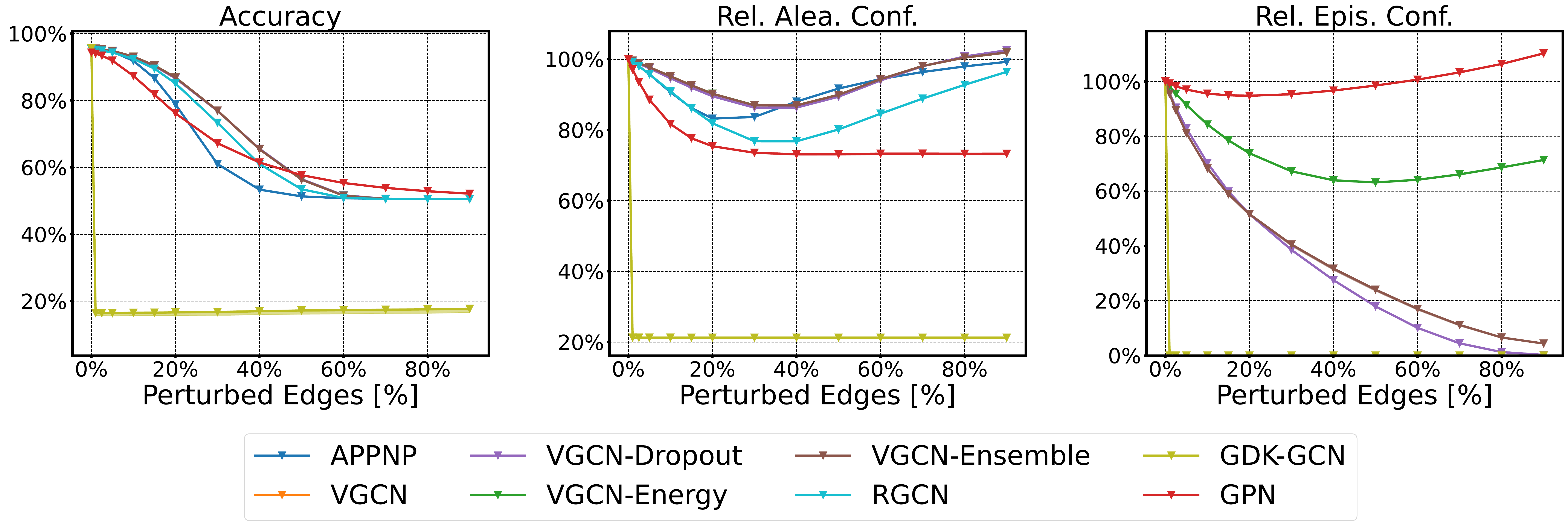}
        \caption{Random Edge Shift}
    \end{subfigure}
    \begin{subfigure}{\textwidth}
        \includegraphics[width=\textwidth]{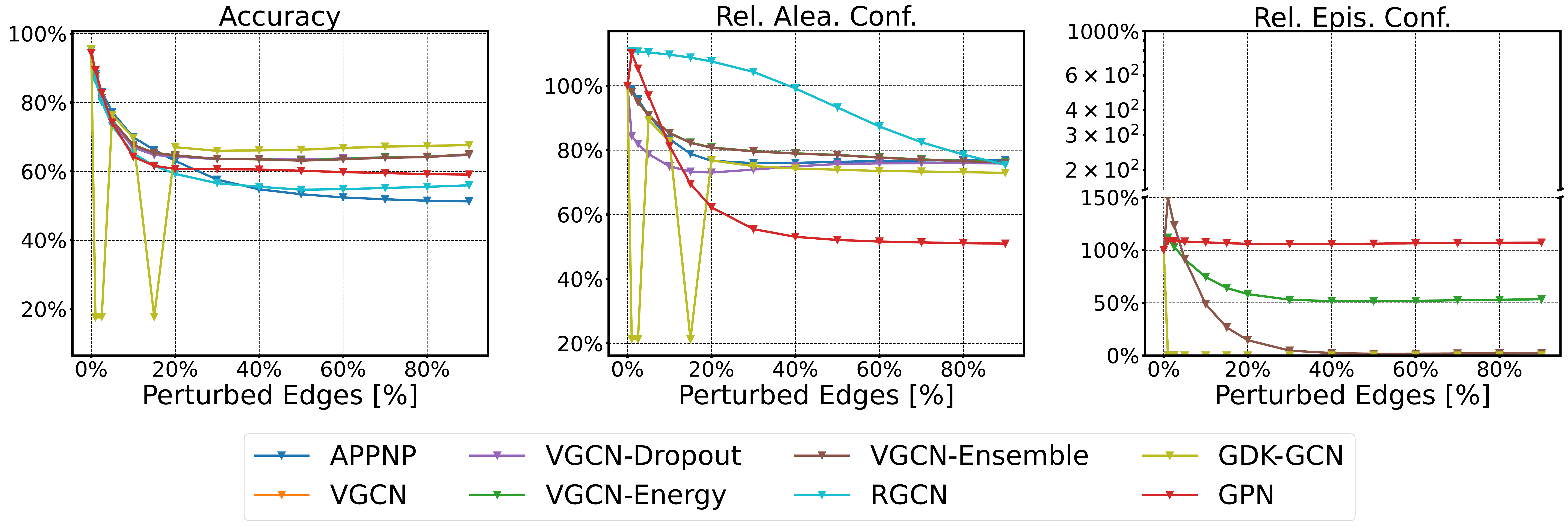}
        \caption{DICE \citep{Waniek2018} Edge Shift}
    \end{subfigure}
    \caption{Relative performance over different degrees of corruption of \emph{Coauthor Physics}. For feature shifts, we perturb different fractions of nodes (whose features are replaced with either random vectors from a Bernoulli noise or a Unit Gaussian noise) and show accuracy, ECE, and relative average epistemic confidence. For edge shifts, we perturb different fractions of edges (by replacing them at random or using the global and untargeted DICE \citep{Waniek2018} attack) and show accuracy, relative average aleatoric confidence, and relative average epistemic confidence.}
    \label{fig:shift-coauthor-physics}
\end{figure}

%% file: tables/shifts-ogbn-arxiv.tex
\begin{figure}
    \centering
    \begin{subfigure}{\textwidth}
        \includegraphics[width=\textwidth]{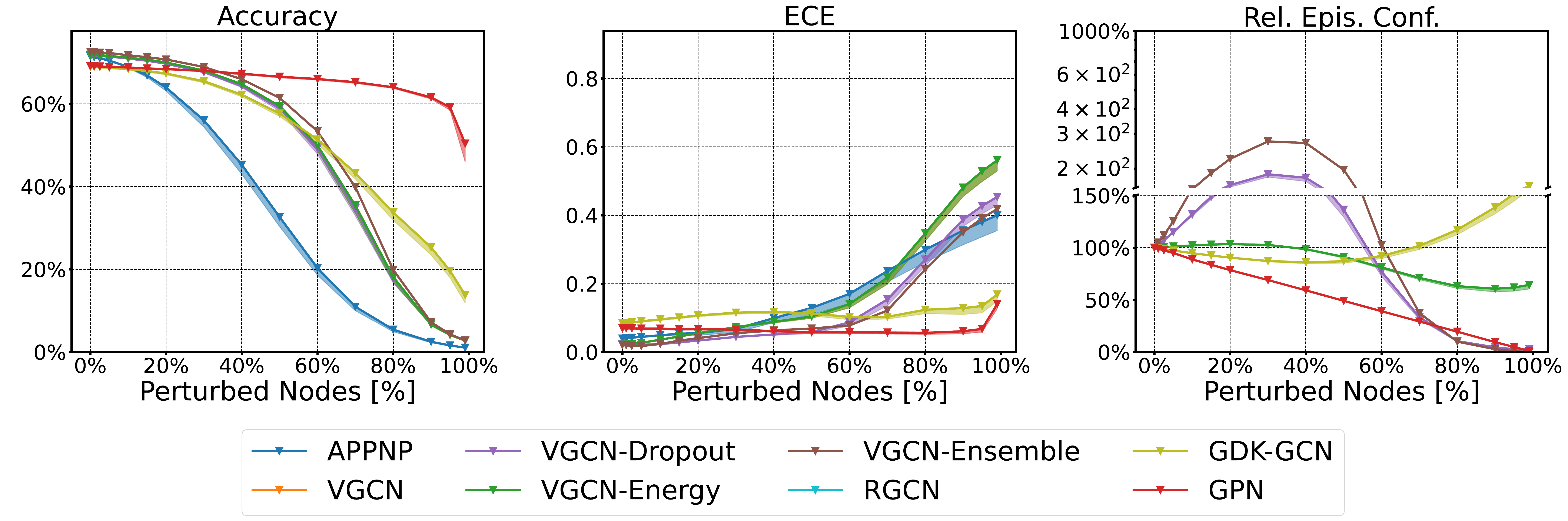}
        \caption{Bernoulli Feature Shift}
    \end{subfigure}
    \begin{subfigure}{\textwidth}
        \includegraphics[width=\textwidth]{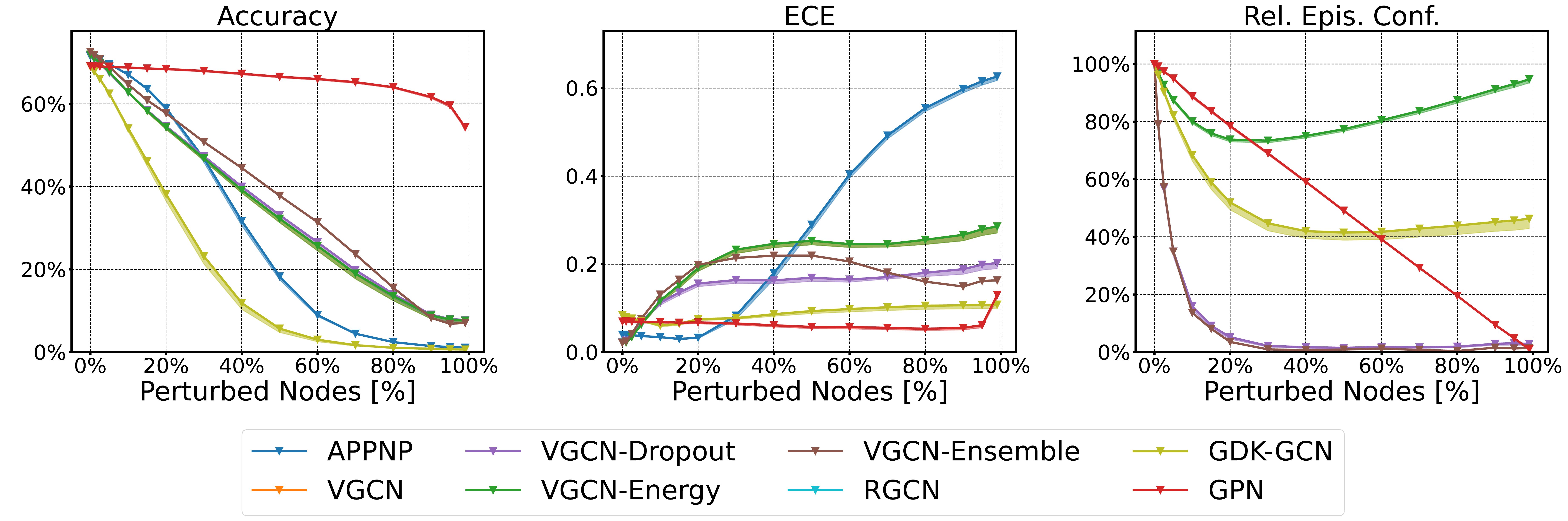}
        \caption{Unit Gaussian Feature Shift}
    \end{subfigure}
    \begin{subfigure}{\textwidth}
        \includegraphics[width=\textwidth]{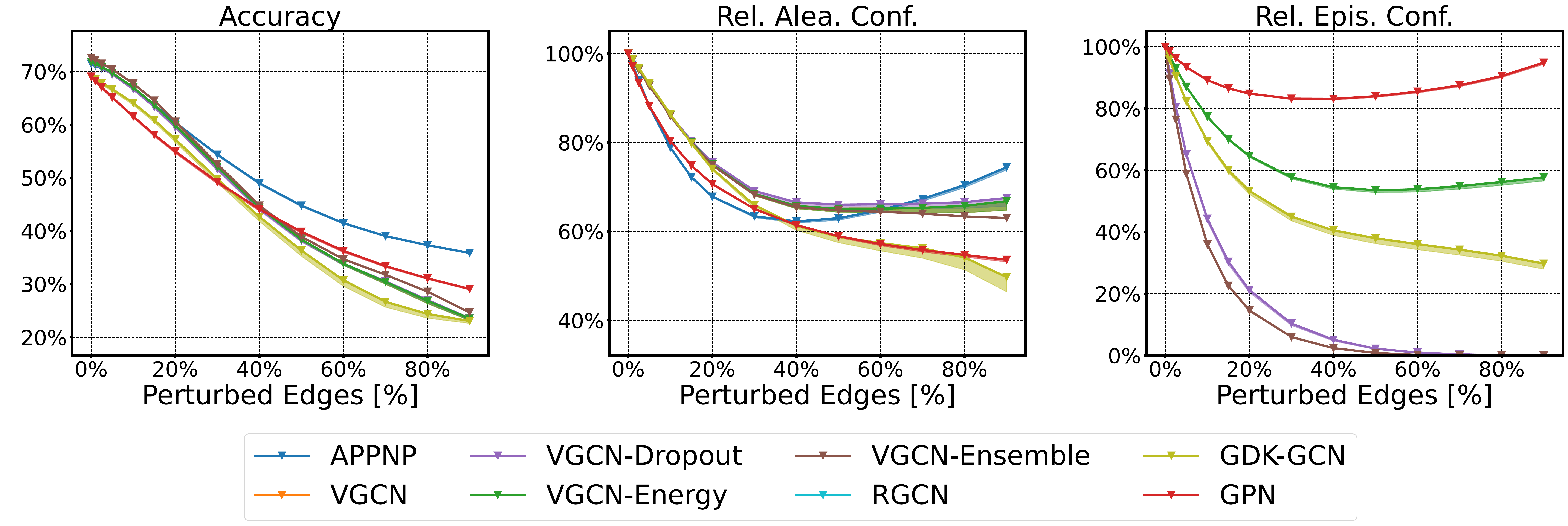}
        \caption{Random Edge Shift}
    \end{subfigure}
    \begin{subfigure}{\textwidth}
        \includegraphics[width=\textwidth]{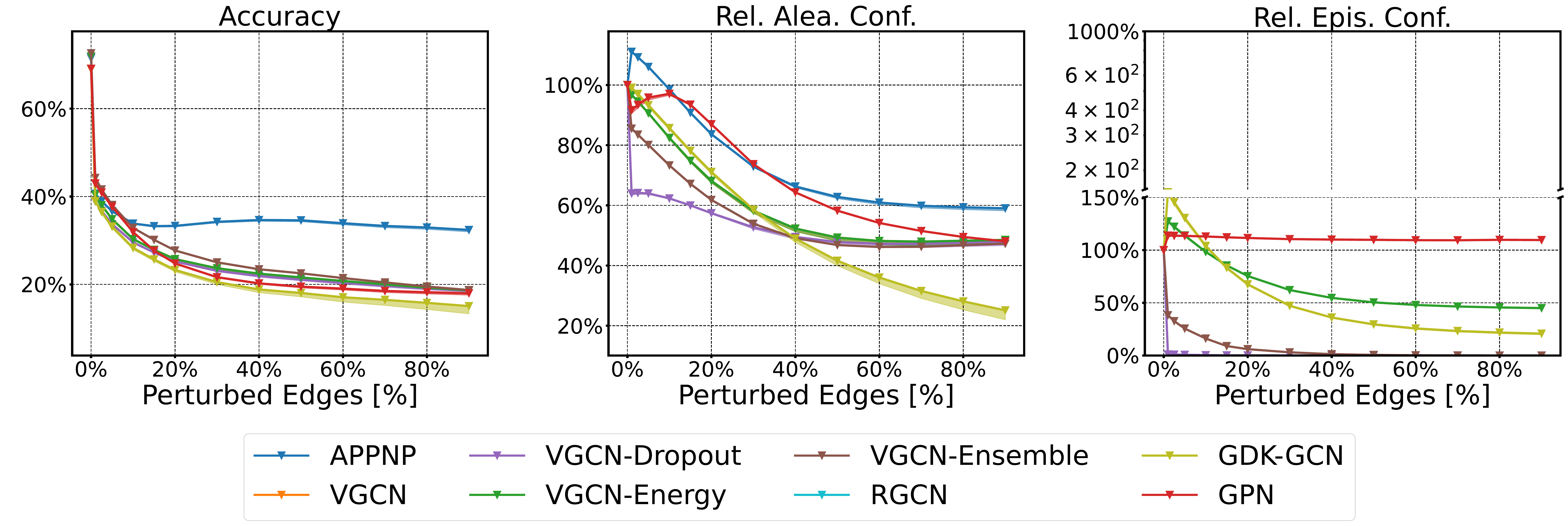}
        \caption{DICE \citep{Waniek2018} Edge Shift}
    \end{subfigure}
    \caption{Relative performance over different degrees of corruption of \emph{OGBN Arxiv}. For feature shifts, we perturb different fractions of nodes (whose features are replaced with either random vectors from a Bernoulli noise or a Unit Gaussian noise) and show accuracy, ECE, and relative average epistemic confidence. For edge shifts, we perturb different fractions of edges (by replacing them at random or using the global and untargeted DICE \citep{Waniek2018} attack) and show accuracy, relative average aleatoric confidence, and relative average epistemic confidence.}
    \label{fig:shift-ogbn-arxiv}
\end{figure}

%% file: tables/lowest-evidence-abstracts.tex
\begin{table}[!h]
    \centering
        \scriptsize
        \resizebox{\textwidth}{!}{
            \begin{tabular}{p{1cm} p{12cm}}
             \textbf{Node} & \textbf{Node Representation} \\
             \toprule
             \multirow{2}{*}{220} & \textbf{Abstract:} {\tt IlliGAL Report No. 95006 July 1995} \\
             & \textbf{Bag-of-Word:} {\tt ['1995', 'report', 'july']} \\
             
             \midrule
             
             \multirow{2}{*}{197} & \textbf{Abstract:} {\tt Report of the 1996 Workshop on Reinforcement} \\
             & \textbf{Bag-of-Word:} {\tt ['workshop', 'reinforcement', 'report', '1996']} \\
             
             \midrule
             
             \multirow{2}{*}{193}  & \textbf{Abstract:} {\tt Report of the 1996 Workshop on Reinforcement} \\
             & \textbf{Bag-of-Word:} {\tt ['workshop', 'reinforcement', 'report', '1996']} \\
             
             \midrule
             
             \multirow{2}{*}{258} & \textbf{Abstract:} {\tt TIK-Report Nr. 11, December 1995 Version 2 (2. Edition)} \\
             & \textbf{Bag-of-Word:} {\tt ['version', '1995', 'report', '11']} \\
             
             \midrule
             
             \multirow{2}{*}{2319} & \textbf{Abstract:} {\tt We tend to think of what we really know as what we can talk about, and disparage knowledge that we can't verbalize. [Dowling 1989, p. 252]} \\
             & \textbf{Bag-of-Word:} {\tt ['knowledge', 'know', '1989', 'tend']} \\
             
             \midrule
             
             \multirow{2}{*}{1945} & \textbf{Abstract:} {\tt Reihe FABEL-Report Status: extern Dokumentbezeichner: Org/Reports/nr-35 Erstellt am: 21.06.94 Korrigiert am: 28.05.95 ISSN 0942-413X} \\
             & \textbf{Bag-of-Word:} {\tt ['reports', '94', 'report', '95']} \\
             
             \midrule
             
             \multirow{2}{*}{293} & \textbf{Abstract:} {\tt Keith Mathias and Darrell Whitley Technical Report CS-94-101 January 7, 1994} \\
             & \textbf{Bag-of-Word:} {\tt ['technical', '1994', '94', 'cs', 'report', 'january']} \\
             
             \midrule
             
             \multirow{2}{*}{99} & \textbf{Abstract:} {\tt Multigrid Q-Learning Charles W. Anderson and Stewart G. Crawford-Hines Technical Report CS-94-121 October 11, 1994} \\
             & \textbf{Bag-of-Word:} {\tt ['learning', 'technical', '1994', '94', 'cs', 'report', '11']} \\
             
             \midrule
             
             \multirow{2}{*}{766} & \textbf{Abstract:} {\tt Internal Report 97-01} \\
             & \textbf{Bag-of-Word:} {\tt ['internal', 'report', '97']} \\
             
             \midrule
             
             \multirow{2}{*}{992} &  \textbf{Abstract:} {\tt A Learning Result for Abstract} \\
             & \textbf{Bag-of-Word:} {\tt ['learning', 'result', 'abstract']} \\
             
             \midrule
             
             \multirow{2}{*}{2030} &  \textbf{Abstract:} {\tt DRAFT March 16, 1998 available via anonymous ftp: site ftp.uwasa.fi directory cs/report94-1 file gaGPbib.ps.Z} \\
             & \textbf{Bag-of-Word:} {\tt ['available', '1998', 'cs', '16', 'anonymous', 'ftp', 'march', 'fi', 'site']} \\
             
             \midrule
             
             \multirow{2}{*}{991} &  \textbf{Abstract:} {\tt In IEEE Transactions on Neural Networks, 7(1):97-106, 1996 Also available as GMD report 794} \\
             & \textbf{Bag-of-Word: } {\tt ['available', 'networks', 'neural', 'report', 'ieee', '1996', '97']} \\
             
             \midrule
             
             \multirow{2}{*}{74} & \textbf{Abstract:} {\tt Empirical Comparison of Gradient Descent and Exponentiated Gradient Descent in Supervised and Reinforcement Learning Technical Report 96-70} \\
             & \textbf{Bag-of-Words:} {\tt ['learning', 'technical', 'empirical', 'gradient', 'comparison', 'reinforcement', 'supervised', 'report', 'descent', '96']} \\
             
             \midrule
             
             \multirow{2}{*}{865} & \textbf{Abstract:} {\tt fl Partially supported by the Advanced Research Projects Agency (AFOSR 90-0083). y Partially supported by the Air Force Office of Scientific Research (AFOSR F49620-92-J-0499), the Advanced Research Projects Agency (ONR N00014-92-J-4015), and the Office of Naval Research (ONR N00014-91-J-4100). z Partially funded by the Air Force Office of Scientific Research (AFOSR F49620-92-J-0334) and the Office of Naval Research (ONR N00014-91-J-4100 and ONR N00014-94-1-0597).} \\ 
             & \textbf{Bag-of-Words:} {\tt ['fl', 'research', '90', 'partially', '94', 'advanced', 'projects', 'agency', 'supported', 'n00014', 'office', 'naval', 'force', 'air', 'scientific', 'afosr', '91', '92', 'onr', 'funded']} \\
             
             \midrule
             
             \multirow{2}{*}{858} &  \textbf{Abstract:} {\tt Technical Report UMIACS-TR-97-77 and CS-TR-3843 Abstract} \\
             & \textbf{Bag-of-Words:} {\tt ['technical', 'abstract', 'cs', 'report', '97', 'tr']} \\
             
             \midrule
             
             \multirow{2}{*}{1944} &  \textbf{Abstract:} {\tt Reihe FABEL-Report Status: extern Dokumentbezeichner: Org/Reports/nr-13 Erstellt am: 22.12.93 Korrigiert am: 02.02.94 ISSN 0942-413X} \\
             & \textbf{Bag-of-Words:} {\tt ['reports', '13', '94', 'report', '22', '93', '12']} \\
             
             \midrule
             
             \multirow{2}{*}{1710} & \textbf{Abstract:} {\tt Knowledge Systems Laboratory March 1995 Report No. KSL 95-32} \\
             & \textbf{Bag-of-Words:} {\tt ['systems', 'knowledge', 'laboratory', '1995', 'report', 'march', '95']} \\
             
             \midrule
             
             \multirow{2}{*}{2491} &  \textbf{Abstract:} {\tt Edward S. Orosz and Charles W. Anderson Technical Report CS-94-111 April 27, 1994} \\
             & \textbf{Bag-of-Words:} {\tt ['technical', '1994', '94', 'cs', 'report', 'april']} \\
             
             \midrule
             
             \multirow{2}{*}{221} &  \textbf{Abstract:} {\tt V. Scott Gordon and Darrell Whitley Technical Report CS-93-114 September 16, 1993} \\
             & \textbf{Bag-of-Words:} {\tt ['technical', '1993', 'cs', 'report', '16', '93']} \\
             
             \midrule
             
             \multirow{2}{*}{1874} & \textbf{Abstract:} {\tt COINS Technical Report 94-61 August 1994} \\
             & \textbf{Bag-of-Words:} {\tt ['technical', '1994', '94', 'report', 'august']} \\
             \bottomrule
        \end{tabular}
    }
    \vspace{5mm}
    \caption{20 abstracts and their representation from the CoraML dataset obtained after selecting the abstracts \oursacro{} has assigned the lowest feature evidence. Abstracts sorted in ascending order of feature evidence.}
    \label{tab:lowest_evidence_abstracts}
\end{table}

%% file: tables/highest-evidence-abstracts.tex
\begin{table}[!h]
    \centering
        \scriptsize
        \resizebox{\textwidth}{!}{
            \begin{tabular}{p{1cm} p{12cm}}
             \textbf{Node} & \textbf{Node Representation} \\
             \toprule
             \multirow{2}{*}{1637} & \textbf{Abstract:} {\tt A pervasive, yet much ignored, factor in the analysis of processing-failures is the problem of misorganized knowledge. If a systems knowledge is not indexed or organized correctly, it may make an error, not because it does not have either the general capability or specific knowledge to solve a problem, but rather because it does not have the knowledge sufficiently organized so that the appropriate knowledge structures are brought to bear on the problem at the appropriate time. In such cases, the system can be said to have forgotten the knowledge, if only in this context. This is the problem of forgetting or retrieval failure. This research presents an analysis along with a declarative representation of a number of types of forgetting errors. Such representations can extend the capability of introspective failure-driven learning systems, allowing them to reduce the likelihood of repeating such errors. Examples are presented from the Meta-AQUA program, which learns to improve its performance on a story understanding task through an introspective meta-analysis of its knowledge, its organization of its knowledge, and its reasoning processes.} \\
             & \textbf{Bag-of-Word:} {\tt['retrieval', 'number', 'problem', 'learning', 'systems', 'representations', 'knowledge', 'representation', 'understanding', 'program', 'learns', 'time', 'make', 'examples', 'appropriate', 'general', 'error', 'analysis', 'presented', 'cases', 'performance', 'correctly', 'task', 'presents', 'improve', 'reasoning', 'likelihood', 'research', 'driven', 'processes', 'does', 'reduce', 'processing', 'organization', 'extend', 'solve', 'specific', 'meta', 'factor', 'declarative', 'failure', 'capability', 'context', 'allowing', 'structures', 'types', 'ignored', 'failures', 'story', 'errors', 'sufficiently', 'brought', 'organized', 'introspective', 'bear']} \\
             
             \midrule
             
             \multirow{2}{*}{1375} & \textbf{Abstract:} {\tt This paper describes an interactive planning system that was developed inside an Intelligent Decision Support System aimed at supporting an operator when planning the initial attack to forest fires. The planning architecture rests on the integration of case-based reasoning techniques with constraint reasoning techniques exploited, mainly, for performing temporal reasoning on temporal metric information. Temporal reasoning plays a central role in supporting interactive functions that are provided to the user when performing two basic steps of the planning process: plan adaptation and resource scheduling. A first prototype was integrated with a situation assessment and a resource allocation manager subsystem and is currently being tested.} \\
             & \textbf{Bag-of-Word:} {\tt ['intelligent', 'information', 'paper', 'planning', 'based', 'case', 'integrated', 'functions', 'describes', 'decision', 'allocation', 'architecture', 'support', 'reasoning', 'mainly', 'techniques', 'process', 'supporting', 'role', 'integration', 'exploited', 'currently', 'initial', 'user', 'tested', 'interactive', 'basic', 'developed', 'temporal', 'central', 'assessment', 'situation', 'steps', 'adaptation', 'operator', 'scheduling', 'constraint', 'performing', 'plays', 'provided', 'resource', 'prototype', 'metric', 'aimed', 'plan', 'inside']} \\
             
             \midrule
             
             \multirow{2}{*}{2672}  & \textbf{Abstract:} {\tt Inductive learning systems are designed to induce hypotheses, or general descriptions of concepts, from instances of these concepts. Among the wide variety of techniques used in inductive learning systems, algorithms derived from nearest neighbour (NN) pattern classification have been receiving attention lately, mainly due to their incremental nature. Nested Generalized Exemplar (NGE) theory is an inductive learning theory which can be viewed as descent from nearest neighbour classification. In NGE theory, the induced concepts take the form of hyperrectangles in a n-dimensional Euclidean space. The axes of the space are defined by the attributes used for describing the examples. This paper proposes a fuzzified version of the original NGE algorithm, which accepts input examples given as feature/fuzzy value pairs, and generalizes them as fuzzy hyperrectangles. It presents and discusses a metric for evaluating the fuzzy distance between examples, and between example and fuzzy hyperrectangles; criteria for establishing the reliability of fuzzy examples, by strengthening the exemplar which makes the right prediction and weakening the exemplar which makes a wrong one and criteria for producing fuzzy generalizations, based on the union of fuzzy sets. Keywords : exemplar-based learning, nested generalized exemplar, nearest neighbour, fuzzy NGE.} \\
             & \textbf{Bag-of-Word:} {\tt Bag-of-Words ['paper', 'used', 'learning', 'systems', 'feature', 'theory', 'based', 'input', 'algorithm', 'generalizations', 'algorithms', 'given', 'hypotheses', 'form', 'attributes', 'examples', 'general', 'concepts', 'instances', 'evaluating', 'prediction', 'classification', 'space', 'keywords', 'presents', 'pattern', 'version', 'receiving', 'mainly', 'techniques', 'nature', 'sets', 'wide', 'right', 'makes', 'variety', 'attention', 'example', 'original', 'describing', 'distance', 'generalized', 'inductive', 'nearest', 'dimensional', 'descriptions', 'value', 'incremental', 'designed', 'viewed', 'discusses', 'derived', 'fuzzy', 'pairs', 'generalizes', 'producing', 'defined', 'descent', 'induced', 'nn', 'wrong', 'criteria', 'reliability', 'proposes', 'metric', 'induce', 'euclidean']} \\
             \bottomrule
        \end{tabular}
    }
    \vspace{5mm}
    \caption{3 abstracts and their representation from the CoraML dataset obtained after selecting the abstracts \oursacro{} has assigned the highest feature evidence. Abstracts sorted in descending order of feature evidence.}
    \label{tab:highest_evidence_abstracts}
\end{table}

%% file: tables/latent-space.tex
\begin{figure}[!h]
    \centering
	\begin{subfigure}[t]{\textwidth}
	    \centering
		\includegraphics[height=0.25\textheight]{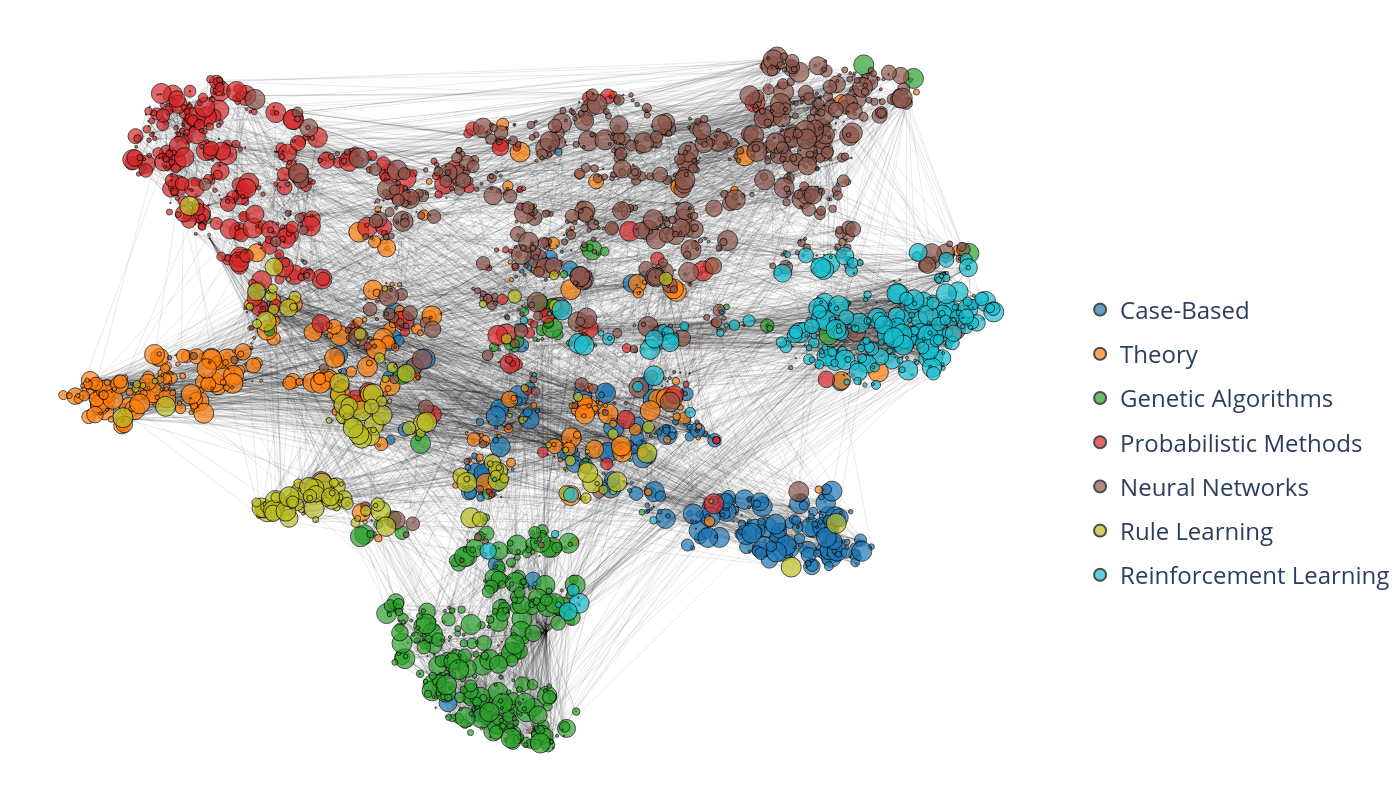}
		\caption{Ground-Truth Classes} 
		\label{subfig:latent-clean-classes}
	\end{subfigure}
	\begin{subfigure}[t]{\textwidth}
	    \centering
		\includegraphics[height=0.25\textheight]{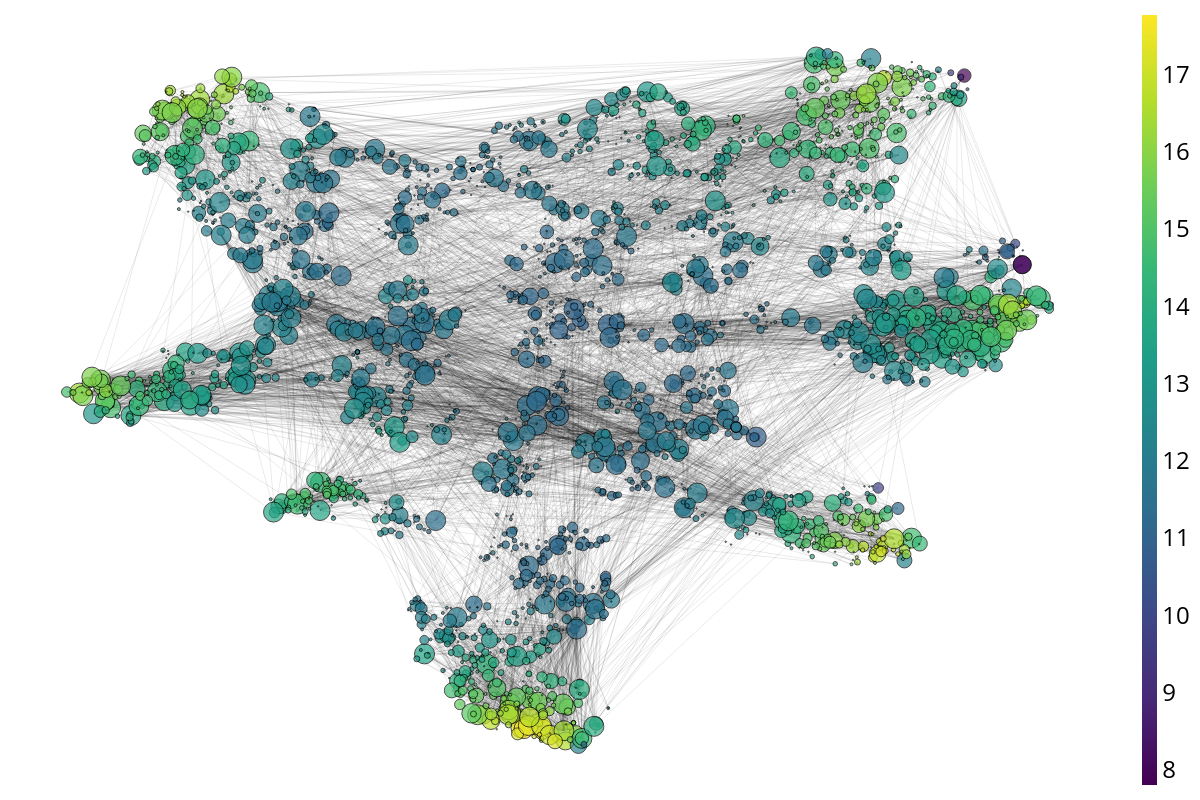}
		\caption{Feature Evidence} 
		\label{subfig:latent-clean-feature-evidence}
	\end{subfigure}
	\begin{subfigure}[t]{\textwidth}
	    \centering
		\includegraphics[height=0.25\textheight]{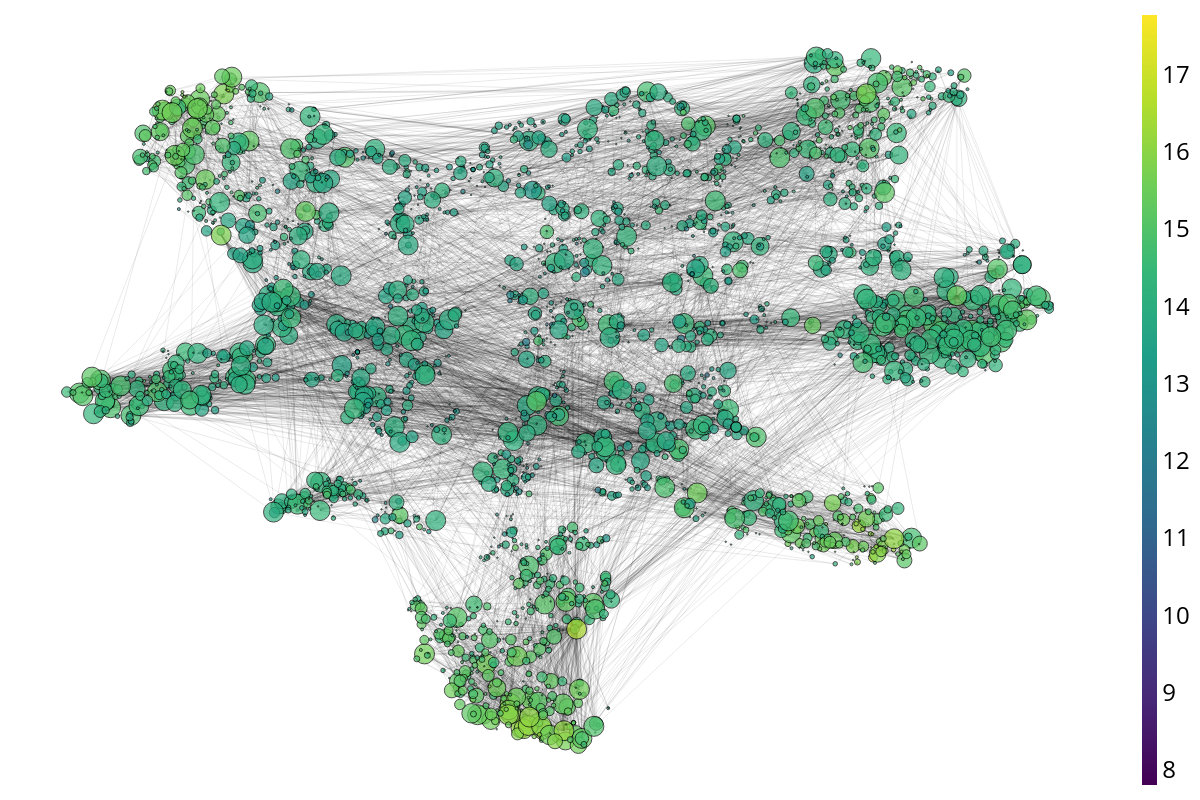}
		\caption{Aggregated Evidence}
		\label{subfig:latent-clean-evidence}
	\end{subfigure}
    \caption{Latent space visualizations on the clean CoraML graph. The ground-truth classes are shown in different colors. The feature and aggregated evidence are represented in log-scale.}
    \label{fig:latent-space-clean}
\end{figure}
\begin{figure}[!h]
    \centering
	\begin{subfigure}[t]{\textwidth}
	    \centering
		\includegraphics[height=0.25\textheight]{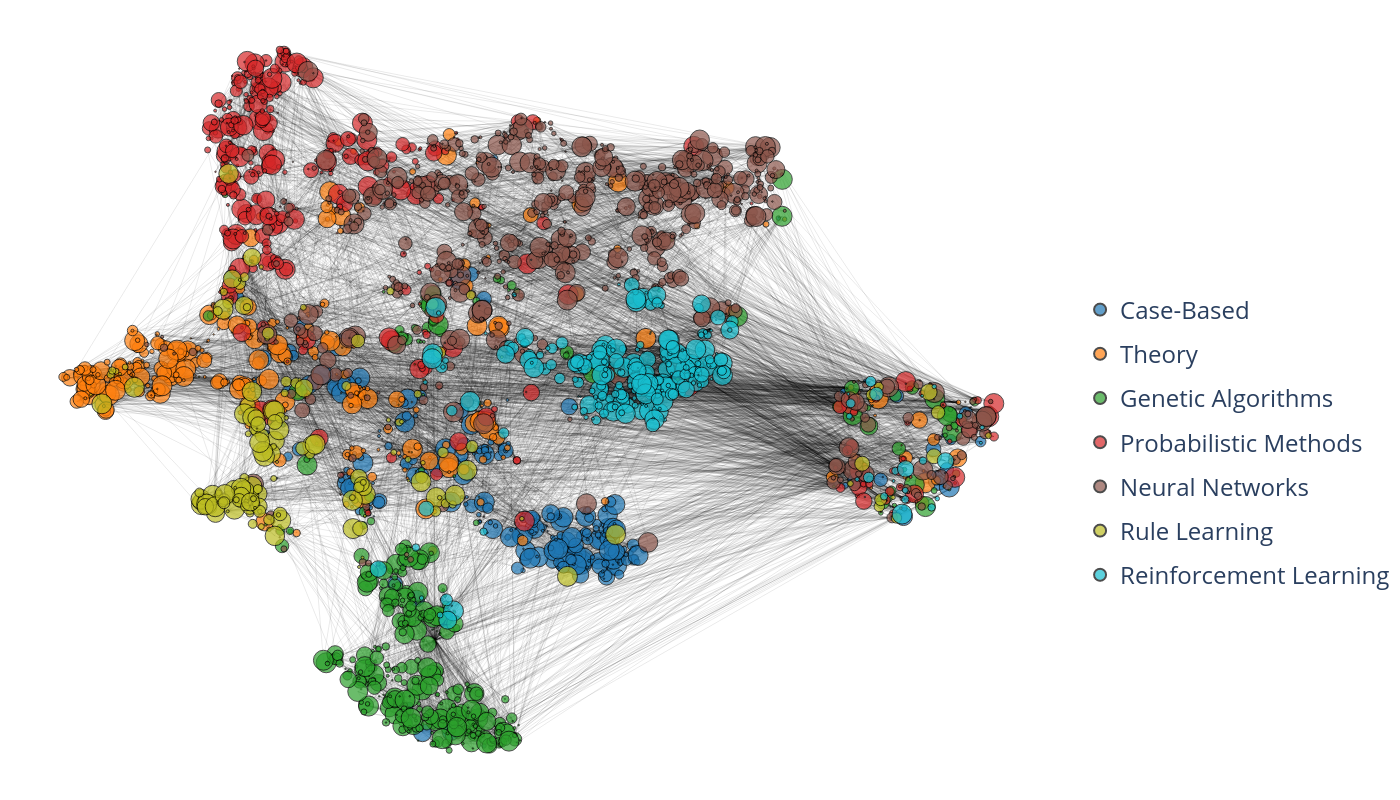}
		\caption{Ground-Truth Classes} 
		\label{subfig:gaussian-clean-classes}
	\end{subfigure}
	\begin{subfigure}[t]{\textwidth}
	    \centering
		\includegraphics[height=0.25\textheight]{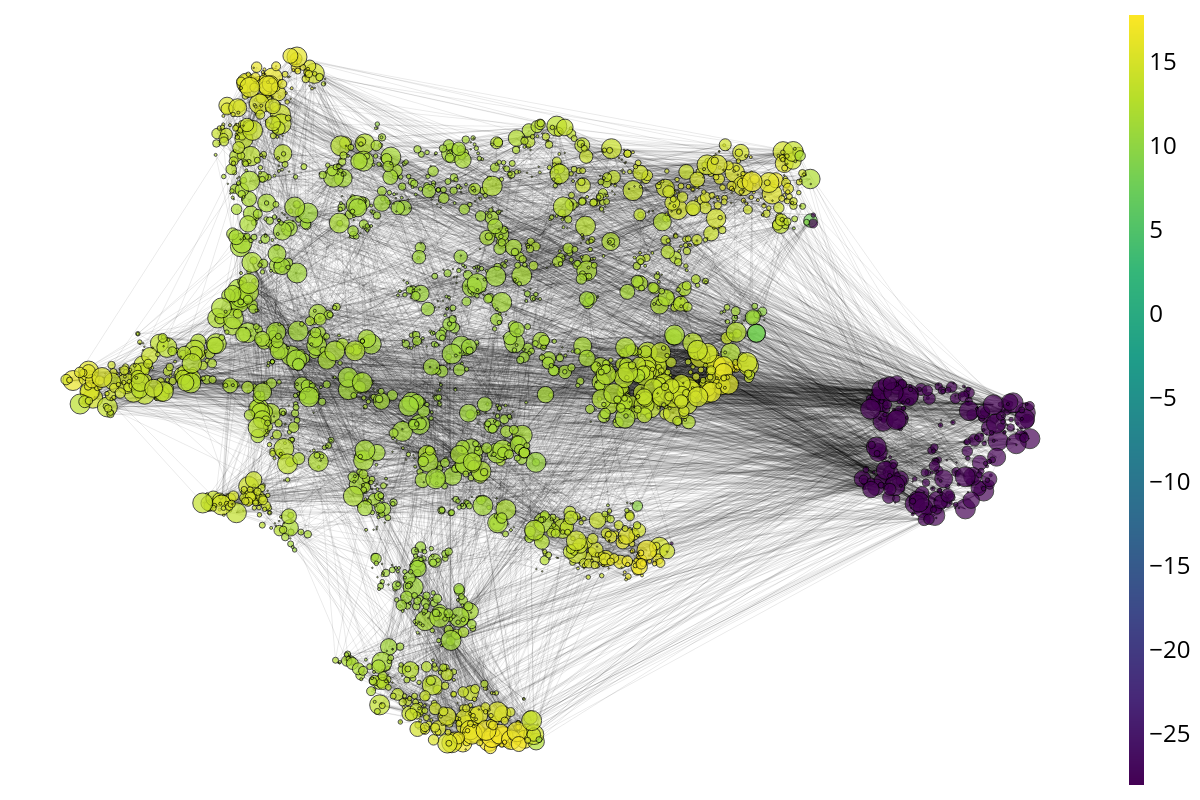}
		\caption{Feature Evidence} 
		\label{subfig:gaussian-clean-feature-evidence}
	\end{subfigure}
	\begin{subfigure}[t]{\textwidth}
	    \centering
		\includegraphics[height=0.25\textheight]{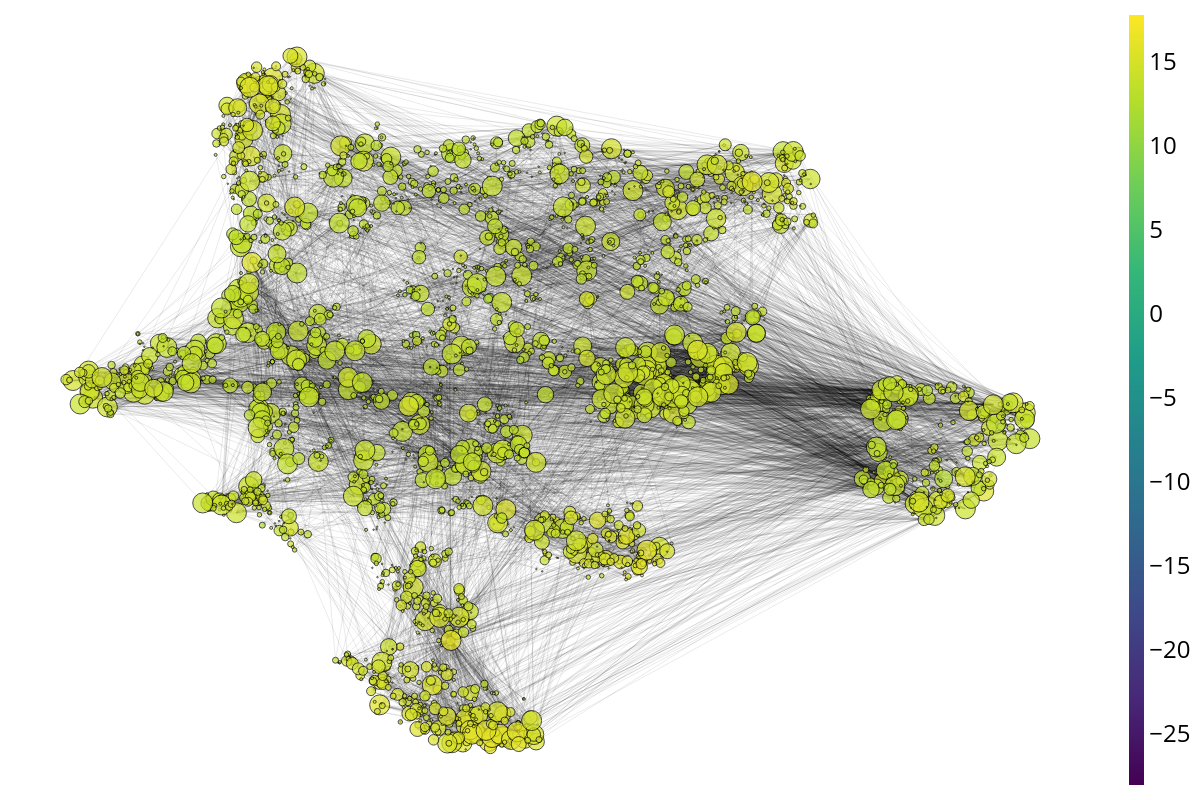}
		\caption{Aggregated Evidence}
		\label{subfig:gaussian-clean-evidence}
	\end{subfigure}
    \caption{Latent space visualizations on the CoraML where $10\%$ of the nodes are perturbed with unit Gaussians. The ground-truth classes are shown in different colors. The feature and aggregated evidence are represented in log-scale.}
    \label{fig:latent-space-gaussian}
\end{figure}
\begin{figure}[!h]
    \centering
	\begin{subfigure}[t]{\textwidth}
	    \centering
		\includegraphics[height=0.25\textheight]{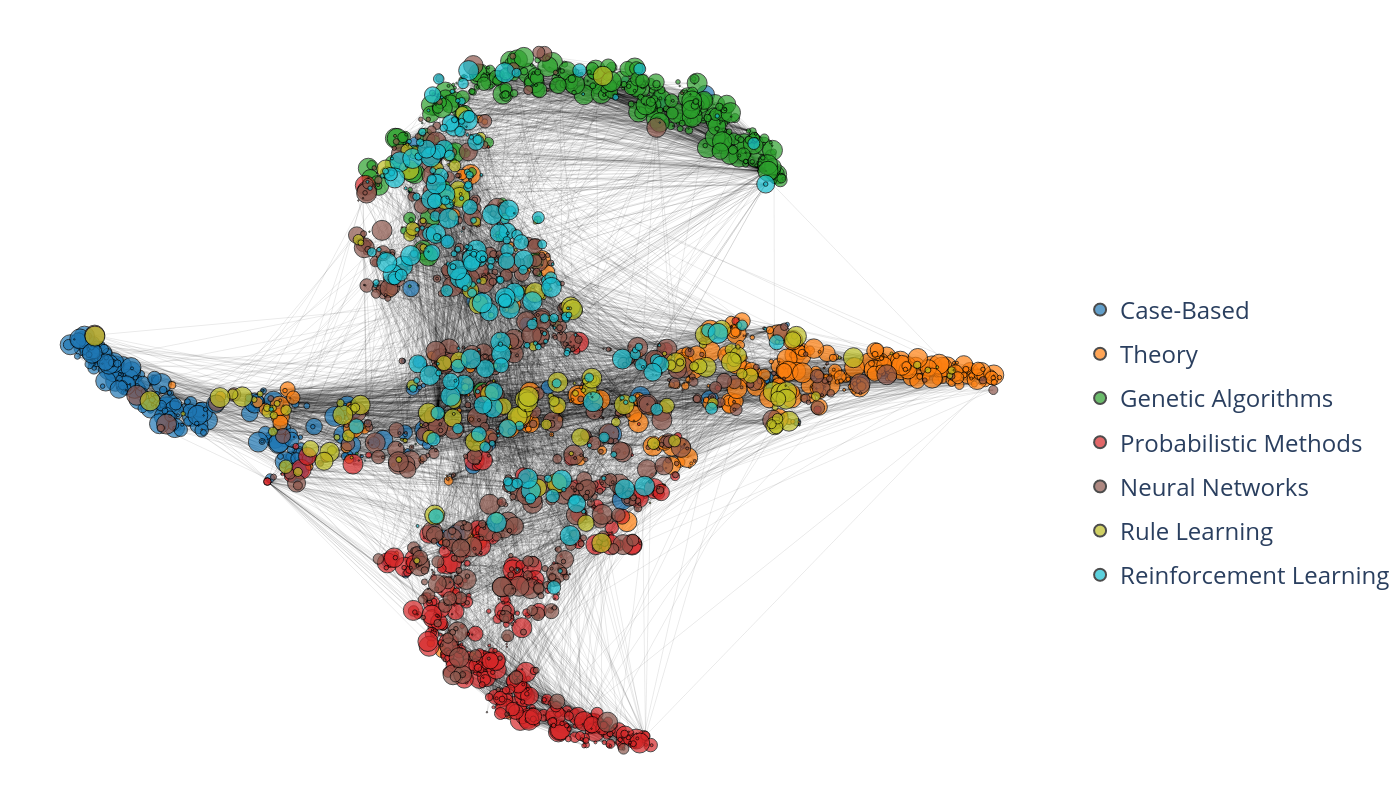}
		\caption{Ground-Truth Classes} 
		\label{subfig:loc-clean-classes}
	\end{subfigure}
	\begin{subfigure}[t]{\textwidth}
	    \centering
		\includegraphics[height=0.25\textheight]{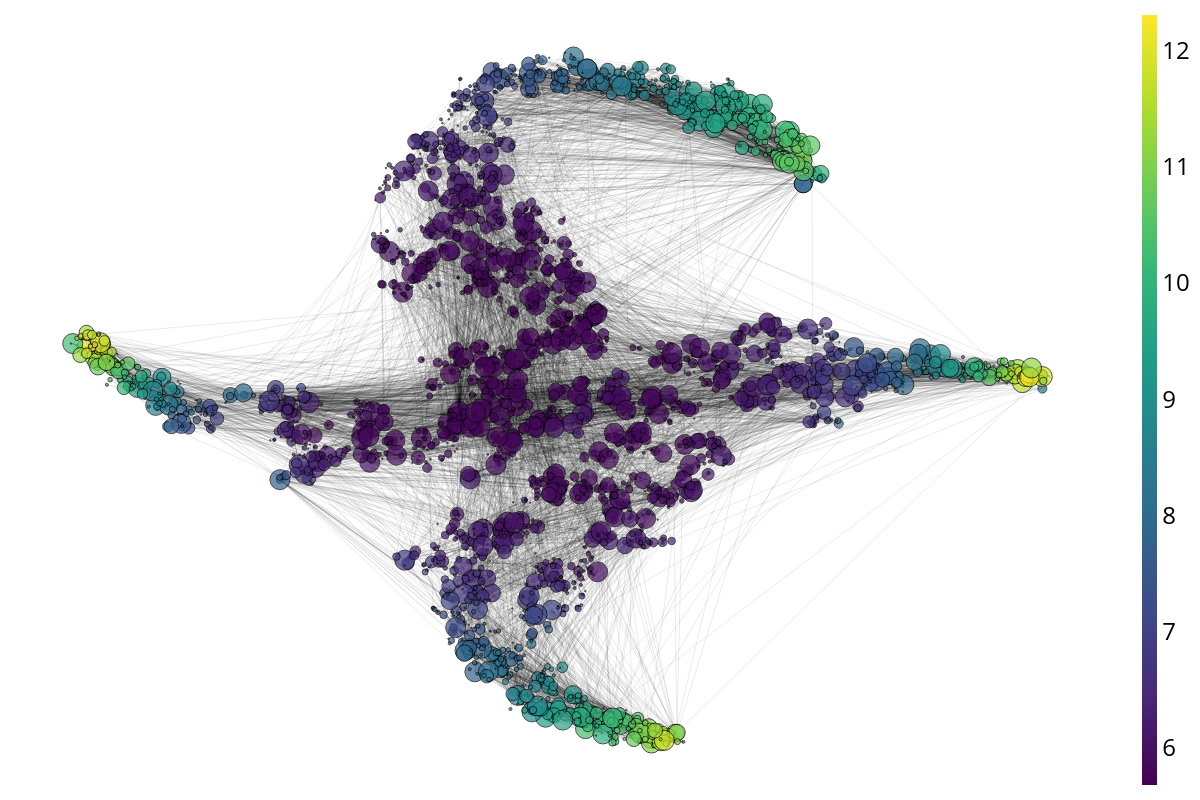}
		\caption{Feature Evidence} 
		\label{subfig:loc-clean-feature-evidence}
	\end{subfigure}
	\begin{subfigure}[t]{\textwidth}
	    \centering
		\includegraphics[height=0.25\textheight]{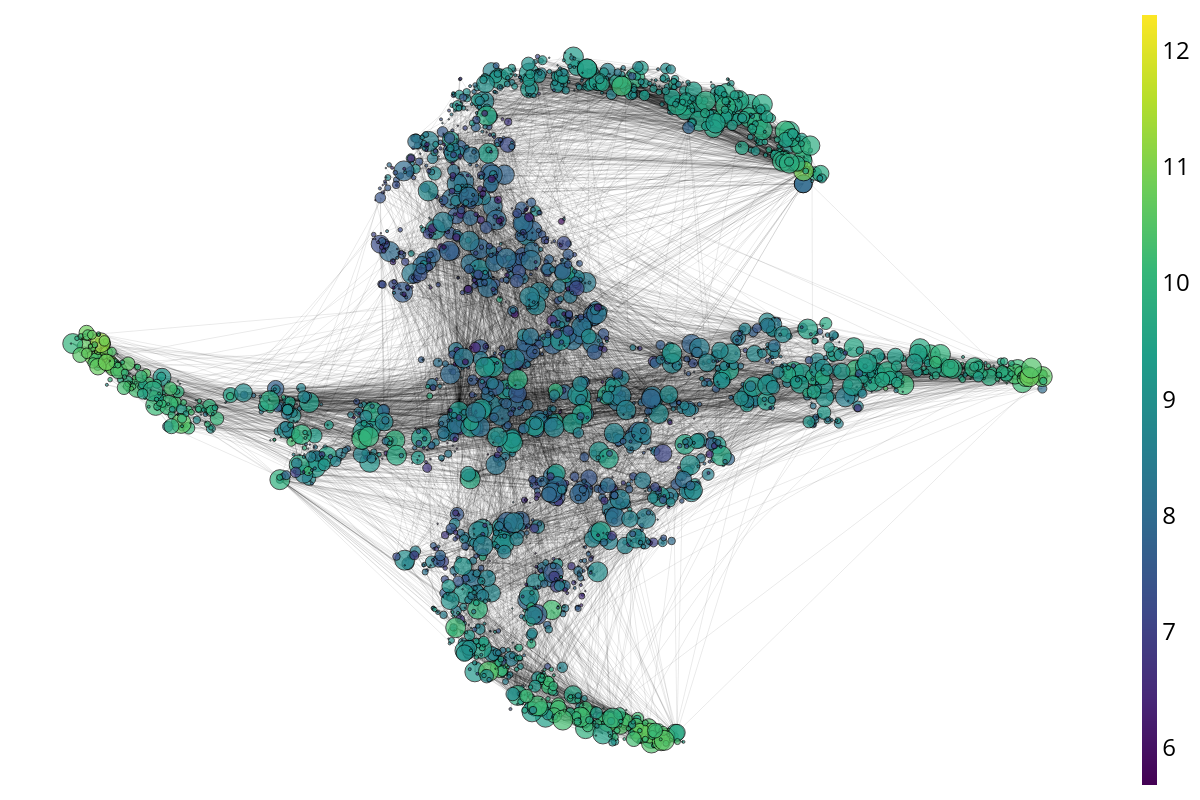}
		\caption{Aggregated Evidence}
		\label{subfig:loc-clean-evidence}
	\end{subfigure}
    \caption{Latent space visualizations on the CoraML with Left-Out classes experiments. The ground-truth classes are shown in different colors. Nodes of the classes \emph{Neural Networks}, \emph{Rule Learning}, and \emph{Reinforcement Learning} have been left out for training but are kept in the graph. The feature and aggregated evidence are represented in log-scale.}
    \label{fig:loc-space-gaussian}
\end{figure}

%% file: tables/inference_times.tex
\begin{table}[ht]
    \centering
    \resizebox{\textwidth}{!}{
        \begin{tabular}{lrrrrrrrr}
        \toprule
\textbf{Model} &            \textbf{CoraML} &          \textbf{CiteSeer} &            \textbf{PubMed} &   \textbf{Amazon C.} &      \textbf{Amazon Ph.} &        \textbf{Coauthor CS} &    \textbf{Coauthor Ph.} &         \textbf{OGBN Arxiv} \\
\midrule
APPNP &   $2.87 \pm 0.01$ &   $2.97 \pm 0.01$ &   $3.07 \pm 0.02$ &   $6.53 \pm 0.05$ &   $3.50 \pm 0.03$ &   $6.67 \pm 0.01$ &   $10.33 \pm 0.04$ &   $\mathbf{49.50 \pm 0.33}$ \\
VGCN &   $\mathbf{2.55 \pm 0.04}$ &   $2.73 \pm 0.02$ &   $2.70 \pm 0.01$ &   $6.41 \pm 0.09$ &   $3.96 \pm 0.03$ &   $6.42 \pm 0.02$ &   $12.27 \pm 0.04$ &   $60.90 \pm 0.04$ \\
VGCN-Dropout &  $24.40 \pm 0.17$ &  $26.96 \pm 0.25$ &  {\color{red}$30.43 \pm 0.20$} &  {\color{red}$58.24 \pm 0.20$} &  {\color{red}$36.56 \pm 0.13$} &  {\color{red}$62.21 \pm 0.26$} &  {\color{red}$120.68 \pm 0.21$} &  {\color{red}$637.63 \pm 0.43$} \\
VGCN-Energy &   $2.78 \pm 0.03$ &   $2.78 \pm 0.02$ &   $3.24 \pm 0.03$ &   $6.03 \pm 0.03$ &   $3.79 \pm 0.01$ &   $6.59 \pm 0.02$ &   $12.46 \pm 0.03$ &   $62.45 \pm 0.32$ \\
VGCN-Ensemble &  $23.18 \pm 0.63$ &  $23.10 \pm 0.02$ &  $27.22 \pm 0.42$ &  $52.78 \pm 0.65$ &  $32.81 \pm 0.14$ &  $54.17 \pm 0.46$ &  $108.58 \pm 0.51$ &           $548.27$ \\
GKDE-GCN &   $2.61 \pm 0.03$ &   $\mathbf{2.18 \pm 0.01}$ &   $\mathbf{2.57 \pm 0.02}$ &   $\mathbf{3.54 \pm 0.04}$ &   $\mathbf{3.00 \pm 0.10}$ &   $\mathbf{4.92 \pm 0.02}$ &    $\mathbf{8.79 \pm 0.09}$ &   $62.21 \pm 0.34$ \\
\midrule
GPN &   $9.35 \pm 0.06$ &   $9.02 \pm 0.02$ &  $14.13 \pm 0.02$ &  $27.48 \pm 0.09$ &  $14.52 \pm 0.05$ &  $48.07 \pm 0.15$ &   $35.94 \pm 0.13$ &  $275.69 \pm 0.91$ \\
        \bottomrule
        \end{tabular}
    }
    \vspace{1em}
    \caption{Inference time (in ms) across different datasets evaluated on a single NVIDIA GTX 1080 Ti. Bold numbers indicate the fastest algorithm during inference for each dataset while red numbers indicate the slowest one.}
    \label{tab:inference-times}
\end{table}

%% file: tables/training_times.tex
\begin{table}[ht]
    \centering
    \resizebox{\textwidth}{!}{
        \begin{tabular}{lrrrrrrrr}
\toprule
\textbf{Model} &    \textbf{CoraML} &  \textbf{CiteSeer} &    \textbf{PubMed} & \textbf{Amazon C.} & \textbf{Amazon Ph.} & \textbf{Coauthor CS} & \textbf{Coauthor Ph.} &  \textbf{OGBN Arxiv} \\
\midrule
APPNP &   $\hphantom{0}45.44$ &   $\hphantom{0}27.74$ &   $\mathbf{\hphantom{0}54.92}$ &        $257.55$ &     $230.75$ &   $166.84$ &        $217.73$ &    $425.75$ \\
VGCN &   $47.28$ &   $32.98$ &   $55.77$ &        $211.79$ &     $174.00$ &   $167.77$ &        $194.53$ &    $\mathbf{329.43}$ \\
VGCN-Dropout &   $47.28$ &   $32.98$ &   $55.77$ &        $211.79$ &     $174.00$ &   $167.77$ &        $194.53$ &    $\mathbf{329.43}$ \\
VGCN-Energy &   $47.28$ &   $32.98$ &   $55.77$ &        $211.79$ &     $174.00$ &   $167.77$ &        $194.53$ &    $\mathbf{329.43}$ \\
VGCN-Ensemble &  {\color{red}$472.82$} &  {\color{red}$329.84$} &  \color{red}{$557.65$} &       {\color{red}$2117.86$} &    {\color{red}$1740.00$} &  {\color{red}$1677.72$} &       {\color{red}$1945.34$} &   $3294.34$ \\
GKDE-GCN &   $46.48$ &   $\mathbf{17.77}$ &  $523.41$ &        $354.66$ &     $473.71$ &   $247.01$ &        $475.40$ &  {\color{red}$60185.40$} \\
{\color{gray}GKDE} &    {\color{gray}$5.12$} &    {\color{gray}$5.04$} &   {\color{gray}$54.72$} & {\color{gray}$61.27$} & {\color{gray}$21.24$} &   {\color{gray}$103.34$} &  {\color{gray}$320.98$} &  {\color{gray}$59775.90$} \\
\midrule
GPN &   $\mathbf{10.20}$ &   $39.40$ &   $59.15$ &         $\mathbf{81.59}$ &      $\mathbf{64.72}$ &    $\mathbf{32.80}$ &         $\mathbf{93.38}$ &   $2393.03$ \\
\bottomrule
\end{tabular}
    }
    \vspace{1em}
    \caption{Average training times (in s) for a single model and initializations across different datasets evaluated mostly on a single NVIDIA GTX 1080 Ti. Bold numbers indicate the fastest algorithm during training for each dataset while red numbers indicate the slowest one. The gray line shows the GKDE component of the GKDE-GCN approach as reference.}
    \label{tab:training-times}
\end{table}

%% file: neurips_2021.bbl
\begin{thebibliography}{104}
\providecommand{\natexlab}[1]{#1}
\providecommand{\url}[1]{\texttt{#1}}
\expandafter\ifx\csname urlstyle\endcsname\relax
  \providecommand{\doi}[1]{doi: #1}\else
  \providecommand{\doi}{doi: \begingroup \urlstyle{rm}\Url}\fi

\bibitem[tru(2020)]{trustworthy-ai}
The assessment list for trustworthy artificial intelligence (altai) for self
  assessment.
\newblock \emph{European Commission}, 2020.

\bibitem[Abadi et~al.(2015)Abadi, Agarwal, Barham, Brevdo, Chen, Citro,
  Corrado, Davis, Dean, Devin, Ghemawat, Goodfellow, Harp, Irving, Isard, Jia,
  Jozefowicz, Kaiser, Kudlur, Levenberg, Man\'{e}, Monga, Moore, Murray, Olah,
  Schuster, Shlens, Steiner, Sutskever, Talwar, Tucker, Vanhoucke, Vasudevan,
  Vi\'{e}gas, Vinyals, Warden, Wattenberg, Wicke, Yu, and Zheng]{tensorflow}
M.~Abadi, A.~Agarwal, P.~Barham, E.~Brevdo, Z.~Chen, C.~Citro, G.~S. Corrado,
  A.~Davis, J.~Dean, M.~Devin, S.~Ghemawat, I.~Goodfellow, A.~Harp, G.~Irving,
  M.~Isard, Y.~Jia, R.~Jozefowicz, L.~Kaiser, M.~Kudlur, J.~Levenberg,
  D.~Man\'{e}, R.~Monga, S.~Moore, D.~Murray, C.~Olah, M.~Schuster, J.~Shlens,
  B.~Steiner, I.~Sutskever, K.~Talwar, P.~Tucker, V.~Vanhoucke, V.~Vasudevan,
  F.~Vi\'{e}gas, O.~Vinyals, P.~Warden, M.~Wattenberg, M.~Wicke, Y.~Yu, and
  X.~Zheng.
\newblock {TensorFlow}: Large-scale machine learning on heterogeneous systems,
  2015.

\bibitem[Abdar et~al.(2020)Abdar, Pourpanah, Hussain, Rezazadegan, Liu,
  Ghavamzadeh, Fieguth, Khosravi, Acharya, Makarenkov, et~al.]{Abdar2020}
M.~Abdar, F.~Pourpanah, S.~Hussain, D.~Rezazadegan, L.~Liu, M.~Ghavamzadeh,
  P.~Fieguth, A.~Khosravi, U.~R. Acharya, V.~Makarenkov, et~al.
\newblock A review of uncertainty quantification in deep learning: Techniques,
  applications and challenges.
\newblock \emph{arXiv preprint arXiv:2011.06225}, 2020.

\bibitem[Ahn et~al.(2019)Ahn, Cha, Lee, and
  Moon]{uncertainty-continual-learning}
H.~Ahn, S.~Cha, D.~Lee, and T.~Moon.
\newblock Uncertainty-based continual learning with adaptive regularization.
\newblock In H.~Wallach, H.~Larochelle, A.~Beygelzimer, F.~d'Alch\'{e} Buc,
  E.~Fox, and R.~Garnett, editors, \emph{Advances in Neural Information
  Processing Systems}, 2019.

\bibitem[Akita et~al.(2018)Akita, Nakago, Komatsu, Sugawara, Maeda, Baba, and
  Kashima]{Akita2018}
H.~Akita, K.~Nakago, T.~Komatsu, Y.~Sugawara, S.-i. Maeda, Y.~Baba, and
  H.~Kashima.
\newblock Bayesgrad: Explaining predictions of graph convolutional networks.
\newblock In \emph{Advances in Neural Information Processing Systems}, 2018.

\bibitem[Amini et~al.(2020)Amini, Schwarting, Soleimany, and
  Rus]{evidential-regression}
A.~Amini, W.~Schwarting, A.~Soleimany, and D.~Rus.
\newblock Deep evidential regression.
\newblock In \emph{Advances in Neural Information Processing Systems}, 2020.

\bibitem[Arora et~al.(2018)Arora, Basu, Mianjy, and
  Mukherjee]{understanding-nn-relu}
R.~Arora, A.~Basu, P.~Mianjy, and A.~Mukherjee.
\newblock Understanding deep neural networks with rectified linear units.
\newblock In \emph{International Conference on Learning Representations}, 2018.

\bibitem[Bilo{\v{s}} et~al.(2019)Bilo{\v{s}}, Charpentier, and
  G{\"u}nnemann]{Bilovs2019}
M.~Bilo{\v{s}}, B.~Charpentier, and S.~G{\"u}nnemann.
\newblock Uncertainty on asynchronous time event prediction.
\newblock \emph{Advances in Neural Information Processing Systems}, 2019.

\bibitem[Blundell et~al.(2015)Blundell, Cornebise, Kavukcuoglu, and
  Wierstra]{Blundell2015}
C.~Blundell, J.~Cornebise, K.~Kavukcuoglu, and D.~Wierstra.
\newblock Weight uncertainty in neural network.
\newblock In \emph{International Conference on Machine Learning}, 2015.

\bibitem[Bojchevski and G{\"u}nnemann(2017)]{Bojchevski2017}
A.~Bojchevski and S.~G{\"u}nnemann.
\newblock Deep gaussian embedding of graphs: Unsupervised inductive learning
  via ranking.
\newblock \emph{arXiv preprint arXiv:1707.03815}, 2017.

\bibitem[Bojchevski and G{\"u}nnemann(2018)]{Bojchevski2018a}
A.~Bojchevski and S.~G{\"u}nnemann.
\newblock Bayesian robust attributed graph clustering: Joint learning of
  partial anomalies and group structure.
\newblock In \emph{AAAI Conference on Artificial Intelligence}, 2018.

\bibitem[Borovitskiy et~al.(2021)Borovitskiy, Azangulov, Terenin, Mostowsky,
  Deisenroth, and Durrande]{Borovitskiy2020}
V.~Borovitskiy, I.~Azangulov, A.~Terenin, P.~Mostowsky, M.~Deisenroth, and
  N.~Durrande.
\newblock {Mat\'ern Gaussian Processes on Graphs}.
\newblock In \emph{International Conference on Artificial Intelligence and
  Statistics}, 2021.

\bibitem[Carlini and Wagner(2017)]{attack-detection}
N.~Carlini and D.~A. Wagner.
\newblock Adversarial examples are not easily detected: Bypassing ten detection
  methods.
\newblock \emph{Computing Research Repository}, 2017.

\bibitem[Charpentier et~al.(2020)Charpentier, Z{\"{u}}gner, and
  G{\"{u}}nnemann]{Charpentier2020}
B.~Charpentier, D.~Z{\"{u}}gner, and S.~G{\"{u}}nnemann.
\newblock {Posterior network: Uncertainty estimation without ood samples via
  density-based pseudo-counts}.
\newblock \emph{Advances in Neural Information Processing Systems}, 2020.

\bibitem[Charpentier et~al.(2021)Charpentier, Borchert, Zügner, Geisler, and
  Günnemann]{NatPN2021}
B.~Charpentier, O.~Borchert, D.~Zügner, S.~Geisler, and S.~Günnemann.
\newblock Natural posterior network: Deep bayesian predictive uncertainty for
  exponential family distributions, 2021.

\bibitem[Chen et~al.(2018)Chen, Ma, and Xiao]{Chen2018}
J.~Chen, T.~Ma, and C.~Xiao.
\newblock Fastgcn: fast learning with graph convolutional networks via
  importance sampling.
\newblock \emph{arXiv preprint arXiv:1801.10247}, 2018.

\bibitem[Choi et~al.(2019)Choi, Jang, and Alemi]{anomaly-detection}
H.~Choi, E.~Jang, and A.~A. Alemi.
\newblock Generative ensembles for robust anomaly detection.
\newblock In \emph{International Conference on Learning Representations}, 2019.

\bibitem[Clements et~al.(2019)Clements, Delft, Robaglia, Slaoui, and
  Toth]{uncertainty-rl}
W.~R. Clements, B.~V. Delft, B.-M. Robaglia, R.~B. Slaoui, and S.~Toth.
\newblock Estimating risk and uncertainty in deep reinforcement learning, 2019.

\bibitem[Dallachiesa et~al.(2014)Dallachiesa, Aggarwal, and
  Palpanas]{Dallachiesa2014}
M.~Dallachiesa, C.~Aggarwal, and T.~Palpanas.
\newblock Node classification in uncertain graphs.
\newblock In \emph{Proceedings of the 26th international conference on
  scientific and statistical database management}, pages 1--4, 2014.

\bibitem[Depeweg et~al.(2018)Depeweg, Hernandez-Lobato, Doshi-Velez, and
  Udluft]{Depeweg2018}
S.~Depeweg, J.-M. Hernandez-Lobato, F.~Doshi-Velez, and S.~Udluft.
\newblock Decomposition of uncertainty in bayesian deep learning for efficient
  and risk-sensitive learning.
\newblock In \emph{International Conference on Machine Learning}, 2018.

\bibitem[Dusenberry et~al.(2020)Dusenberry, Jerfel, Wen, Ma, Snoek, Heller,
  Lakshminarayanan, and Tran]{rank-1-bnn}
M.~W. Dusenberry, G.~Jerfel, Y.~Wen, Y.-A. Ma, J.~Snoek, K.~Heller,
  B.~Lakshminarayanan, and D.~Tran.
\newblock Efficient and scalable bayesian neural nets with rank-1 factors.
\newblock In \emph{International Conference on Machine Learning}, 2020.

\bibitem[Elinas et~al.(2020)Elinas, Bonilla, and Tiao]{Elinas2019}
P.~Elinas, E.~V. Bonilla, and L.~Tiao.
\newblock Variational inference for graph convolutional networks in the absence
  of graph data and adversarial settings.
\newblock \emph{Advances in Neural Information Processing Systems}, 2020.

\bibitem[Eswaran et~al.(2017)Eswaran, G{\"u}nnemann, and
  Faloutsos]{Eswaran2017}
D.~Eswaran, S.~G{\"u}nnemann, and C.~Faloutsos.
\newblock The power of certainty: A dirichlet-multinomial model for belief
  propagation.
\newblock In \emph{Proceedings of the 2017 SIAM International Conference on
  Data Mining}, 2017.

\bibitem[Farquhar et~al.(2020)Farquhar, Smith, and Gal]{liberty-depth-bnn}
S.~Farquhar, L.~Smith, and Y.~Gal.
\newblock Liberty or depth: Deep bayesian neural nets do not need complex
  weight posterior approximations.
\newblock In \emph{Advances in Neural Information Processing Systems}, 2020.

\bibitem[Feng et~al.(2020)Feng, Wang, Wang, and Ding]{Feng2020}
B.~Feng, Y.~Wang, Z.~Wang, and Y.~Ding.
\newblock Uncertainty-aware attention graph neural network for defending
  adversarial attacks.
\newblock \emph{arXiv preprint arXiv:2009.10235}, 2020.

\bibitem[Fey and Lenssen(2019)]{pytorch-geometric}
M.~Fey and J.~E. Lenssen.
\newblock Fast graph representation learning with {PyTorch Geometric}.
\newblock In \emph{ICLR Workshop on Representation Learning on Graphs and
  Manifolds}, 2019.

\bibitem[Foong et~al.(2019)Foong, Burt, Li, and Turner]{Foong2019}
A.~Y. Foong, D.~R. Burt, Y.~Li, and R.~E. Turner.
\newblock On the expressiveness of approximate inference in bayesian neural
  networks.
\newblock \emph{arXiv preprint arXiv:1909.00719}, 2019.

\bibitem[Gal(2016)]{Gal2016a}
Y.~Gal.
\newblock {Uncertainty in Deep Learning}.
\newblock 2016.

\bibitem[Gal and Ghahramani(2016)]{Gal2016}
Y.~Gal and Z.~Ghahramani.
\newblock Dropout as a bayesian approximation: Representing model uncertainty
  in deep learning.
\newblock In \emph{International Conference on Machine Learning}, 2016.

\bibitem[Gal and Smith(2018)]{sufficient-conditions-no-adversarial}
Y.~Gal and L.~Smith.
\newblock Sufficient conditions for idealised models to have no adversarial
  examples: a theoretical and empirical study with bayesian neural networks,
  2018.

\bibitem[Getoor(2005)]{Getoor2005}
L.~Getoor.
\newblock Link-based classification.
\newblock In \emph{Advanced methods for knowledge discovery from complex data}.
  Springer, 2005.

\bibitem[Giles et~al.(1998)Giles, Bollacker, and Lawrence]{Giles1998}
C.~L. Giles, K.~D. Bollacker, and S.~Lawrence.
\newblock Citeseer: An automatic citation indexing system.
\newblock In \emph{Proceedings of the third ACM conference on Digital
  libraries}, pages 89--98, 1998.

\bibitem[Grathwohl et~al.(2019)Grathwohl, Wang, Jacobsen, Duvenaud, Norouzi,
  and Swersky]{Grathwohl2019}
W.~Grathwohl, K.-C. Wang, J.-H. Jacobsen, D.~Duvenaud, M.~Norouzi, and
  K.~Swersky.
\newblock Your classifier is secretly an energy based model and you should
  treat it like one.
\newblock \emph{arXiv preprint arXiv:1912.03263}, 2019.

\bibitem[Graves(2011)]{Graves2011}
A.~Graves.
\newblock Practical variational inference for neural networks.
\newblock In \emph{Advances in Neural Information Processing Systems}.
  Citeseer, 2011.

\bibitem[Guo et~al.(2017)Guo, Pleiss, Sun, and Weinberger]{Guo2017}
C.~Guo, G.~Pleiss, Y.~Sun, and K.~Q. Weinberger.
\newblock On calibration of modern neural networks.
\newblock In \emph{International Conference on Machine Learning}, 2017.

\bibitem[Hasanzadeh et~al.(2020)Hasanzadeh, Hajiramezanali, Boluki, Zhou,
  Duffield, Narayanan, and Qian]{Hasanzadeh2020}
A.~Hasanzadeh, E.~Hajiramezanali, S.~Boluki, M.~Zhou, N.~Duffield,
  K.~Narayanan, and X.~Qian.
\newblock Bayesian graph neural networks with adaptive connection sampling.
\newblock In \emph{International Conference on Machine Learning}, 2020.

\bibitem[Havasi et~al.(2021)Havasi, Jenatton, Fort, Liu, Snoek,
  Lakshminarayanan, Dai, and Tran]{mimo-independent-subnetworks}
M.~Havasi, R.~Jenatton, S.~Fort, J.~Z. Liu, J.~Snoek, B.~Lakshminarayanan,
  A.~M. Dai, and D.~Tran.
\newblock Training independent subnetworks for robust prediction.
\newblock In \emph{International Conference on Learning Representations}, 2021.

\bibitem[Hein et~al.(2019)Hein, Andriushchenko, and
  Bitterwolf]{overconfident-relu}
M.~Hein, M.~Andriushchenko, and J.~Bitterwolf.
\newblock Why relu networks yield high-confidence predictions far away from the
  training data and how to mitigate the problem.
\newblock \emph{Computer Vision and Pattern Recognition}, 2019.

\bibitem[Hirschfeld et~al.(2020)Hirschfeld, Swanson, Yang, Barzilay, and
  Coley]{uncertainty-nn-molecules}
L.~Hirschfeld, K.~Swanson, K.~Yang, R.~Barzilay, and C.~W. Coley.
\newblock Uncertainty quantification using neural networks for molecular
  property prediction, 2020.

\bibitem[Hron et~al.(2018)Hron, Matthews, and Ghahramani]{Hron2018}
J.~Hron, A.~Matthews, and Z.~Ghahramani.
\newblock Variational bayesian dropout: pitfalls and fixes.
\newblock In \emph{International Conference on Machine Learning}, 2018.

\bibitem[Hu et~al.(2017)Hu, Cheng, Huang, Fang, and Luo]{Hu2017}
J.~Hu, R.~Cheng, Z.~Huang, Y.~Fang, and S.~Luo.
\newblock On embedding uncertain graphs.
\newblock In \emph{Information and Knowledge Management}, 2017.

\bibitem[Hu et~al.(2020)Hu, Fey, Zitnik, Dong, Ren, Liu, Catasta, and
  Leskovec]{ogb-dataset}
W.~Hu, M.~Fey, M.~Zitnik, Y.~Dong, H.~Ren, B.~Liu, M.~Catasta, and J.~Leskovec.
\newblock Open graph benchmark: Datasets for machine learning on graphs.
\newblock \emph{arXiv preprint arXiv:2005.00687}, 2020.

\bibitem[Huang et~al.(2020)Huang, He, Singh, Lim, and Benson]{Huang2020}
Q.~Huang, H.~He, A.~Singh, S.-N. Lim, and A.~R. Benson.
\newblock Combining label propagation and simple models out-performs graph
  neural networks.
\newblock \emph{arXiv preprint arXiv:2010.13993}, 2020.

\bibitem[Kingma and Ba(2014)]{Kingma2014}
D.~P. Kingma and J.~Ba.
\newblock Adam: A method for stochastic optimization.
\newblock \emph{arXiv preprint arXiv:1412.6980}, 2014.

\bibitem[Kipf and Welling(2016)]{Kipf2016}
T.~N. Kipf and M.~Welling.
\newblock {Semi-Supervised Classification with Graph Convolutional Networks}.
\newblock \emph{International Conference on Learning Representations}, 2016.

\bibitem[Kirichenko et~al.(2020)Kirichenko, Izmailov, and
  Wilson]{Kirichenko2020}
P.~Kirichenko, P.~Izmailov, and A.~G. Wilson.
\newblock Why normalizing flows fail to detect out-of-distribution data.
\newblock \emph{arXiv preprint arXiv:2006.08545}, 2020.

\bibitem[Klicpera et~al.(2018)Klicpera, Bojchevski, and
  G{\"u}nnemann]{Klicpera2018}
J.~Klicpera, A.~Bojchevski, and S.~G{\"u}nnemann.
\newblock Predict then propagate: Graph neural networks meet personalized
  pagerank.
\newblock \emph{arXiv preprint arXiv:1810.05997}, 2018.

\bibitem[Klicpera et~al.(2019)Klicpera, Wei{\ss}enberger, and
  G{\"u}nnemann]{Klicpera2019}
J.~Klicpera, S.~Wei{\ss}enberger, and S.~G{\"u}nnemann.
\newblock Diffusion improves graph learning.
\newblock \emph{arXiv preprint arXiv:1911.05485}, 2019.

\bibitem[Kopetzki et~al.(2020)Kopetzki, Charpentier, Z{\"{u}}gner, Giri, and
  G{\"{u}}nnemann]{robustness-uncertainty-dirichlet}
A.~Kopetzki, B.~Charpentier, D.~Z{\"{u}}gner, S.~Giri, and S.~G{\"{u}}nnemann.
\newblock Evaluating robustness of predictive uncertainty estimation: Are
  dirichlet-based models reliable?
\newblock \emph{Computing Research Repository}, 2020.

\bibitem[Kristiadi et~al.(2020)Kristiadi, Hein, and Hennig]{bayesian-a-bit}
A.~Kristiadi, M.~Hein, and P.~Hennig.
\newblock Being bayesian, even just a bit, fixes overconfidence in relu
  networks, 2020.

\bibitem[Lakshminarayanan et~al.(2017)Lakshminarayanan, Pritzel, and
  Blundell]{Lakshminarayanan2017}
B.~Lakshminarayanan, A.~Pritzel, and C.~Blundell.
\newblock Simple and scalable predictive uncertainty estimation using deep
  ensembles.
\newblock In \emph{Proceedings of the 31st International Conference on Neural
  Information Processing Systems}, pages 6405--6416, 2017.

\bibitem[Lakshminarayanan et~al.(2020)Lakshminarayanan, Tran, Liu, Padhy,
  Bedrax-Weiss, and Lin]{uncertainty-distance-awareness}
B.~Lakshminarayanan, D.~Tran, J.~Liu, S.~Padhy, T.~Bedrax-Weiss, and Z.~Lin.
\newblock Simple and principled uncertainty estimation with deterministic deep
  learning via distance awareness.
\newblock In \emph{Advances in Neural Information Processing Systems}, 2020.

\bibitem[Lan and Dinh(2020)]{perfect-density-no-ood-guarantee}
C.~L. Lan and L.~Dinh.
\newblock Perfect density models cannot guarantee anomaly detection.
\newblock \emph{arXiv preprint arXiv:2012.03808}, 2020.

\bibitem[Lee et~al.(2020)Lee, Lee, Na, Kim, Park, Yang, and
  Hwang]{bayesian-meta-learning}
H.~B. Lee, H.~Lee, D.~Na, S.~Kim, M.~Park, E.~Yang, and S.~J. Hwang.
\newblock Learning to balance: Bayesian meta-learning for imbalanced and
  out-of-distribution tasks, 2020.

\bibitem[Liang et~al.(2017)Liang, Li, and Srikant]{Liang2017}
S.~Liang, Y.~Li, and R.~Srikant.
\newblock {Enhancing the reliability of out-of-distribution image detection in
  neural networks}.
\newblock \emph{arXiv preprint arXiv:1706.02690}, 2017.

\bibitem[Liu et~al.(2020{\natexlab{a}})Liu, Wang, Owens, and Li]{Liu2020a}
W.~Liu, X.~Wang, J.~Owens, and Y.~Li.
\newblock Energy-based out-of-distribution detection.
\newblock \emph{Advances in Neural Information Processing Systems},
  2020{\natexlab{a}}.

\bibitem[Liu et~al.(2020{\natexlab{b}})Liu, Li, Chen, Hu, and Huang]{Liu2020c}
Z.-Y. Liu, S.-Y. Li, S.~Chen, Y.~Hu, and S.-J. Huang.
\newblock Uncertainty aware graph gaussian process for semi-supervised
  learning.
\newblock In \emph{AAAI Conference on Artificial Intelligence}, volume~34,
  pages 4957--4964, 2020{\natexlab{b}}.

\bibitem[Maddox et~al.(2019)Maddox, Izmailov, Garipov, Vetrov, and
  Wilson]{simple-baseline-uncertainty}
W.~J. Maddox, P.~Izmailov, T.~Garipov, D.~P. Vetrov, and A.~G. Wilson.
\newblock A simple baseline for bayesian uncertainty in deep learning.
\newblock In \emph{Advances in Neural Information Processing Systems}, 2019.

\bibitem[Malinin and Gales(2018)]{Malinin2018}
A.~Malinin and M.~Gales.
\newblock Predictive uncertainty estimation via prior networks.
\newblock \emph{arXiv preprint arXiv:1802.10501}, 2018.

\bibitem[Malinin and Gales(2019)]{Malinin2019b}
A.~Malinin and M.~Gales.
\newblock {Reverse kl-divergence training of prior networks: Improved
  uncertainty and adversarial robustness}.
\newblock \emph{arXiv preprint arXiv:1905.13472}, 2019.

\bibitem[Malinin et~al.(2017)Malinin, Ragni, Knill, and Gales]{Malinin2017}
A.~Malinin, A.~Ragni, K.~Knill, and M.~Gales.
\newblock {Incorporating uncertainty into deep learning for spoken language
  assessment}.
\newblock In \emph{Annual Meeting of the Association for Computational
  Linguistics}, 2017.

\bibitem[Malinin et~al.(2019)Malinin, Mlodozeniec, and Gales]{Malinin2019a}
A.~Malinin, B.~Mlodozeniec, and M.~Gales.
\newblock Ensemble distribution distillation.
\newblock \emph{arXiv preprint arXiv:1905.00076}, 2019.

\bibitem[McAuley et~al.(2015)McAuley, Targett, Shi, and Van
  Den~Hengel]{Mcauley2015}
J.~McAuley, C.~Targett, Q.~Shi, and A.~Van Den~Hengel.
\newblock Image-based recommendations on styles and substitutes.
\newblock In \emph{Research and Development in Information Retrieval}, 2015.

\bibitem[McCallum et~al.(2000)McCallum, Nigam, Rennie, and
  Seymore]{Mccallum2000}
A.~K. McCallum, K.~Nigam, J.~Rennie, and K.~Seymore.
\newblock {Automating the construction of internet portals with machine
  learning}.
\newblock \emph{Information Retrieval}, 2000.

\bibitem[Meinke and Hein(2020)]{provable-uncertainty}
A.~Meinke and M.~Hein.
\newblock Towards neural networks that provably know when they don't know.
\newblock In \emph{International Conference on Learning Representations}, 2020.

\bibitem[Molnar(2020)]{interpretable-ml}
C.~Molnar.
\newblock Interpretable machine learning, 2020.
\newblock URL \url{https://christophm.github.io/interpretable-ml-book/}.

\bibitem[Morales-Alvarez et~al.(2021)Morales-Alvarez, Hern{\'a}ndez-Lobato,
  Molina, and Hern{\'a}ndez-Lobato]{gp-uncertainty-activation}
P.~Morales-Alvarez, D.~Hern{\'a}ndez-Lobato, R.~Molina, and J.~M.
  Hern{\'a}ndez-Lobato.
\newblock Activation-level uncertainty in deep neural networks.
\newblock In \emph{International Conference on Learning Representations}, 2021.

\bibitem[Morningstar et~al.(2020)Morningstar, Ham, Gallagher, Lakshminarayanan,
  Alemi, and Dillon]{density-states-ood}
W.~R. Morningstar, C.~Ham, A.~G. Gallagher, B.~Lakshminarayanan, A.~A. Alemi,
  and J.~V. Dillon.
\newblock Density of states estimation for out-of-distribution detection.
\newblock \emph{arXiv preprint arXiv:2006.09273}, 2020.

\bibitem[Naeini et~al.(2015)Naeini, Cooper, and Hauskrecht]{Naeini2015}
M.~P. Naeini, G.~Cooper, and M.~Hauskrecht.
\newblock Obtaining well calibrated probabilities using bayesian binning.
\newblock In \emph{AAAI Conference on Artificial Intelligence}, volume~29,
  2015.

\bibitem[Nalisnick et~al.(2019)Nalisnick, Matsukawa, Teh, Gorur, and
  Lakshminarayanan]{deep-generative}
E.~Nalisnick, A.~Matsukawa, Y.~W. Teh, D.~Gorur, and B.~Lakshminarayanan.
\newblock Do deep generative models know what they don't know?
\newblock \emph{International Conference on Learning Representations}, 2019.

\bibitem[Nalisnick et~al.(2020)Nalisnick, Matsukawa, Teh, and
  Lakshminarayanan]{typicality_OOD_generative}
E.~Nalisnick, A.~Matsukawa, Y.~W. Teh, and B.~Lakshminarayanan.
\newblock Detecting out-of-distribution inputs to deep generative models using
  typicality.
\newblock \emph{arXiv preprint arXiv:1906.02994}, 2020.

\bibitem[Namata et~al.(2012)Namata, London, Getoor, Huang, and EDU]{Namata2012}
G.~Namata, B.~London, L.~Getoor, B.~Huang, and U.~M.~D. EDU.
\newblock {Query-driven active surveying for collective classification}.
\newblock In \emph{10th International Workshop on Mining and Learning with
  Graphs}, 2012.

\bibitem[Ng et~al.(2018)Ng, Colombo, and Silva]{Ng2018}
Y.~C. Ng, N.~Colombo, and R.~Silva.
\newblock Bayesian semi-supervised learning with graph gaussian processes.
\newblock \emph{arXiv preprint arXiv:1809.04379}, 2018.

\bibitem[Nguyen et~al.(2019)Nguyen, Do, and
  Carneiro]{uncertainty-meta-learning}
C.~Nguyen, T.-T. Do, and G.~Carneiro.
\newblock Uncertainty in model-agnostic meta-learning using variational
  inference, 2019.

\bibitem[Orbach and Crammer(2012)]{graph-transduction-confidence}
M.~Orbach and K.~Crammer.
\newblock Graph-based transduction with confidence.
\newblock In P.~A. Flach, T.~De~Bie, and N.~Cristianini, editors, \emph{Machine
  Learning and Knowledge Discovery in Databases}, 2012.

\bibitem[Osband(2016)]{Osband2016}
I.~Osband.
\newblock {Risk versus uncertainty in deep learning: Bayes, bootstrap and the
  dangers of dropout}.
\newblock In \emph{NeurIPS Workshop on Bayesian Deep Learning}, 2016.

\bibitem[Ovadia et~al.(2019)Ovadia, Fertig, Ren, Nado, Sculley, Nowozin,
  Dillon, Lakshminarayanan, and Snoek]{Ovadia2019}
Y.~Ovadia, E.~Fertig, J.~Ren, Z.~Nado, D.~Sculley, S.~Nowozin, J.~V. Dillon,
  B.~Lakshminarayanan, and J.~Snoek.
\newblock Can you trust your model's uncertainty? evaluating predictive
  uncertainty under dataset shift.
\newblock \emph{arXiv preprint arXiv:1906.02530}, 2019.

\bibitem[Pal et~al.(2019{\natexlab{a}})Pal, Regol, and Coates]{Pal2019a}
S.~Pal, F.~Regol, and M.~Coates.
\newblock Bayesian graph convolutional neural networks using node copying.
\newblock \emph{arXiv preprint arXiv:1911.04965}, 2019{\natexlab{a}}.

\bibitem[Pal et~al.(2019{\natexlab{b}})Pal, Regol, and Coates]{Pal2019b}
S.~Pal, F.~Regol, and M.~Coates.
\newblock Bayesian graph convolutional neural networks using non-parametric
  graph learning.
\newblock \emph{arXiv preprint arXiv:1910.12132}, 2019{\natexlab{b}}.

\bibitem[Paszke et~al.(2019)Paszke, Gross, Massa, Lerer, Bradbury, Chanan,
  Killeen, Lin, Gimelshein, Antiga, Desmaison, Kopf, Yang, DeVito, Raison,
  Tejani, Chilamkurthy, Steiner, Fang, Bai, and Chintala]{pytorch}
A.~Paszke, S.~Gross, F.~Massa, A.~Lerer, J.~Bradbury, G.~Chanan, T.~Killeen,
  Z.~Lin, N.~Gimelshein, L.~Antiga, A.~Desmaison, A.~Kopf, E.~Yang, Z.~DeVito,
  M.~Raison, A.~Tejani, S.~Chilamkurthy, B.~Steiner, L.~Fang, J.~Bai, and
  S.~Chintala.
\newblock Pytorch: An imperative style, high-performance deep learning library.
\newblock In H.~Wallach, H.~Larochelle, A.~Beygelzimer, F.~d'~Alch\'{e}-Buc,
  E.~Fox, and R.~Garnett, editors, \emph{Advances in Neural Information
  Processing Systems}. 2019.

\bibitem[Rezende and Mohamed(2015)]{Rezende2015}
D.~Rezende and S.~Mohamed.
\newblock Variational inference with normalizing flows.
\newblock In \emph{International Conference on Machine Learning}, 2015.

\bibitem[Rong et~al.(2019)Rong, Huang, Xu, and Huang]{Rong2019}
Y.~Rong, W.~Huang, T.~Xu, and J.~Huang.
\newblock Dropedge: Towards deep graph convolutional networks on node
  classification.
\newblock \emph{arXiv preprint arXiv:1907.10903}, 2019.

\bibitem[Ryu et~al.(2019)Ryu, Kwon, and Kim]{Ryu2019}
S.~Ryu, Y.~Kwon, and W.~Y. Kim.
\newblock Uncertainty quantification of molecular property prediction with
  bayesian neural networks.
\newblock \emph{arXiv preprint arXiv:1903.08375}, 2019.

\bibitem[Sen et~al.(2008)Sen, Namata, Bilgic, Getoor, Galligher, and
  Eliassi-Rad]{Sen2008a}
P.~Sen, G.~Namata, M.~Bilgic, L.~Getoor, B.~Galligher, and T.~Eliassi-Rad.
\newblock Collective classification in network data.
\newblock \emph{AI magazine}, 29\penalty0 (3):\penalty0 93--93, 2008.

\bibitem[Sensoy et~al.(2018)Sensoy, Kaplan, and Kandemir]{Sensoy2018}
M.~Sensoy, L.~Kaplan, and M.~Kandemir.
\newblock Evidential deep learning to quantify classification uncertainty.
\newblock In \emph{Advances in Neural Information Processing Systems}, 2018.

\bibitem[Shchur et~al.(2018)Shchur, Mumme, Bojchevski, and
  G{\"u}nnemann]{Shchur2018}
O.~Shchur, M.~Mumme, A.~Bojchevski, and S.~G{\"u}nnemann.
\newblock Pitfalls of graph neural network evaluation.
\newblock \emph{arXiv preprint arXiv:1811.05868}, 2018.

\bibitem[Srivastava et~al.(2014)Srivastava, Hinton, Krizhevsky, Sutskever, and
  Salakhutdinov]{Srivastava2014}
N.~Srivastava, G.~Hinton, A.~Krizhevsky, I.~Sutskever, and R.~Salakhutdinov.
\newblock Dropout: a simple way to prevent neural networks from overfitting.
\newblock \emph{The journal of machine learning research}, 2014.

\bibitem[Stutz et~al.(2020)Stutz, Hein, and
  Schiele]{confidence-calibrated-adversarial}
D.~Stutz, M.~Hein, and B.~Schiele.
\newblock Confidence-calibrated adversarial training: Generalizing to unseen
  attacks.
\newblock In H.~D. III and A.~Singh, editors, \emph{International Conference on
  Machine Learning}, Proceedings of Machine Learning Research, 2020.

\bibitem[Tran et~al.(2020)Tran, Neiswanger, Yoon, Zhang, Xing, and
  Ulissi]{uncertainty-material-prediction}
K.~Tran, W.~Neiswanger, J.~Yoon, Q.~Zhang, E.~Xing, and Z.~W. Ulissi.
\newblock Methods for comparing uncertainty quantifications for material
  property predictions, 2020.

\bibitem[van Amersfoort et~al.(2020)van Amersfoort, Smith, Teh, and Gal]{duq}
J.~van Amersfoort, L.~Smith, Y.~W. Teh, and Y.~Gal.
\newblock Uncertainty estimation using a single deep deterministic neural
  network, 2020.

\bibitem[Veli{\v{c}}kovi{\'c} et~al.(2017)Veli{\v{c}}kovi{\'c}, Cucurull,
  Casanova, Romero, Lio, and Bengio]{Velickovic2017}
P.~Veli{\v{c}}kovi{\'c}, G.~Cucurull, A.~Casanova, A.~Romero, P.~Lio, and
  Y.~Bengio.
\newblock Graph attention networks.
\newblock \emph{arXiv preprint arXiv:1710.10903}, 2017.

\bibitem[Wang and Leskovec(2020)]{Wang2020}
H.~Wang and J.~Leskovec.
\newblock Unifying graph convolutional neural networks and label propagation.
\newblock \emph{arXiv preprint arXiv:2002.06755}, 2020.

\bibitem[Wang et~al.(2020)Wang, Shen, Huang, Wu, Dong, and
  Kanakia]{microsoft-academic-graph}
K.~Wang, Z.~Shen, C.~Huang, C.-H. Wu, Y.~Dong, and A.~Kanakia.
\newblock Microsoft academic graph: When experts are not enough.
\newblock \emph{Quantitative Science Studies}, 2020.

\bibitem[Waniek et~al.(2018)Waniek, Michalak, Wooldridge, and
  Rahwan]{Waniek2018}
M.~Waniek, T.~P. Michalak, M.~J. Wooldridge, and T.~Rahwan.
\newblock Hiding individuals and communities in a social network.
\newblock \emph{Nature Human Behaviour}, 2018.

\bibitem[Wen et~al.(2020)Wen, Tran, and Ba]{batch-ensembles}
Y.~Wen, D.~Tran, and J.~Ba.
\newblock Batchensemble: an alternative approach to efficient ensemble and
  lifelong learning.
\newblock In \emph{International Conference on Learning Representations}, 2020.

\bibitem[Wenzel et~al.(2020)Wenzel, Snoek, Tran, and Jenatton]{hyper-ensembles}
F.~Wenzel, J.~Snoek, D.~Tran, and R.~Jenatton.
\newblock Hyperparameter ensembles for robustness and uncertainty
  quantification.
\newblock \emph{Advances in Neural Information Processing Systems}, 2020.

\bibitem[Winkens et~al.(2020)Winkens, Bunel, Guha~Roy, Stanforth, Natarajan,
  Ledsam, MacWilliams, Kohli, Karthikesalingam, Kohl, Cemgil, Eslami, and
  Ronneberger]{contrastive-ood}
J.~Winkens, R.~Bunel, A.~Guha~Roy, R.~Stanforth, V.~Natarajan, J.~R. Ledsam,
  P.~MacWilliams, P.~Kohli, A.~Karthikesalingam, S.~Kohl, T.~Cemgil, S.~M.~A.
  Eslami, and O.~Ronneberger.
\newblock Contrastive training for improved out-of-distribution detection.
\newblock \emph{arXiv preprint arXiv:2007.05566}, 2020.

\bibitem[Zhan and Pei(2020)]{Zhan2020}
H.~Zhan and X.~Pei.
\newblock I-gcn: Robust graph convolutional network via influence mechanism.
\newblock \emph{arXiv preprint arXiv:2012.06110}, 2020.

\bibitem[Zhang and Others(2019)]{Zhang2019}
Y.~Zhang and Others.
\newblock {Bayesian semi-supervised learning for uncertainty-calibrated
  prediction of molecular properties and active learning}.
\newblock \emph{Chemical science}, 2019.

\bibitem[Zhang et~al.(2019)Zhang, Pal, Coates, and Ustebay]{Zhang2019b}
Y.~Zhang, S.~Pal, M.~Coates, and D.~Ustebay.
\newblock Bayesian graph convolutional neural networks for semi-supervised
  classification.
\newblock In \emph{AAAI Conference on Artificial Intelligence}, 2019.

\bibitem[Zhao et~al.(2020)Zhao, Chen, Hu, and Cho]{Zhao2020}
X.~Zhao, F.~Chen, S.~Hu, and J.-H. Cho.
\newblock Uncertainty aware semi-supervised learning on graph data.
\newblock \emph{Advances in Neural Information Processing Systems}, 2020.

\bibitem[Zhi et~al.(2020)Zhi, Ng, and Dong]{Zhi2020}
Y.-C. Zhi, Y.~C. Ng, and X.~Dong.
\newblock Gaussian processes on graphs via spectral kernel learning.
\newblock \emph{arXiv preprint arXiv:2006.07361}, 2020.

\bibitem[Zhu et~al.(2019)Zhu, Zhang, Cui, and Zhu]{Zhu2019}
D.~Zhu, Z.~Zhang, P.~Cui, and W.~Zhu.
\newblock Robust graph convolutional networks against adversarial attacks.
\newblock In \emph{SIGKDD International Conference on Knowledge Discovery \&
  Data Mining}, 2019.

\bibitem[Zhu et~al.(2020)Zhu, Yan, Zhao, Heimann, Akoglu, and
  Koutra]{heterophily-gnn}
J.~Zhu, Y.~Yan, L.~Zhao, M.~Heimann, L.~Akoglu, and D.~Koutra.
\newblock Beyond homophily in graph neural networks: Current limitations and
  effective designs.
\newblock In H.~Larochelle, M.~Ranzato, R.~Hadsell, M.~Balcan, and H.~Lin,
  editors, \emph{Advances in Neural Information Processing Systems}, 2020.

\end{thebibliography}
